\documentclass[10pt]{article} % For LaTeX2e
\usepackage[accepted]{tmlr}
\usepackage{hyperref}
\usepackage{url}

\usepackage[utf8]{inputenc} % allow utf-8 input
\usepackage[T1]{fontenc}    % use 8-bit T1 fonts
\usepackage{hyperref}       % hyperlinks
\usepackage{url}            % simple URL typesetting
\usepackage{booktabs}       % professional-quality tables
\usepackage{amsfonts}       % blackboard math symbols
\usepackage{nicefrac}       % compact symbols for 1/2, etc.
\usepackage{microtype}      % microtypography
\usepackage{amsmath}
\usepackage{amssymb}
\usepackage{mathtools}
\usepackage{amsthm}

\usepackage[capitalize,noabbrev]{cleveref}

\usepackage{enumitem}

\theoremstyle{plain}
\newtheorem{theorem}{Theorem}[section]
\newtheorem{proposition}[theorem]{Proposition}
\newtheorem{lemma}[theorem]{Lemma}
\newtheorem{corollary}[theorem]{Corollary}
\theoremstyle{definition}
\newtheorem{definition}[theorem]{Definition}
\newtheorem{assumption}[theorem]{Assumption}

\theoremstyle{remark}
\newtheorem{remark}[theorem]{Remark}
\crefname{assumption}{Assumption}{Assumptions}

\newcommand{\bfA}{\mathbf{A}}
\newcommand{\bfB}{\mathbf{B}}

\newcommand{\bfD}{\mathbf{D}}

\newcommand{\bfU}{\mathbf{U}}
\newcommand{\bfV}{\mathbf{V}}
\newcommand{\bfW}{\mathbf{W}}
\newcommand{\bfX}{\mathbf{X}}

\newcommand{\ind}{\mathbf{1}}

\newcommand{\bfu}{\mathbf{u}}

\newcommand{\bfw}{\mathbf{w}}
\newcommand{\bfx}{\mathbf{x}}
\newcommand{\bfy}{\mathbf{y}}

\DeclareMathOperator{\var}{Var}

\newcommand{\R}{\mathbb{R}}

\def\bi{\begin{itemize}}
	\def\ei{\end{itemize}}

\def\be{\begin{equation}}
	\def\ee{\end{equation}}
\def\bea{\begin{eqnarray}}
	\def\eea{\end{eqnarray}}

\def\({\left(}
\def\){\right)}

\def\1{\mathbf{1}}

\def \bx{\mathbf{x}}

\def \bX{\mathbf{X}}

\def \by{\mathbf{y}}

\DeclareMathOperator{\uniform}{Uniform}
\def \din{{d}}

\DeclareMathOperator{\cvd}{\mathtt{d}}
\let\cite\citep
\DeclareMathOperator{\cvdto}{\overset{\cvd}{\to}}

\newcommand{\nnodes}{m}

\newcommand{\relu}{\sigma}
\newcommand{\brelu}{\boldsymbol\relu}
\newcommand{\ntker}{\Theta_{\nnodes}}

\newcommand{\ntgram}{\widehat\Theta_{\nnodes}}
\newcommand{\ntgrampart}[1]{\ntgram^{(#1)}}
\newcommand{\ntkerlim}{\Theta^\ast}
\newcommand{\ntgramlim}{\widehat\Theta^\ast}

\DeclareMathOperator{\idmat}{I}

\DeclareMathOperator{\eigmin}{\operatorname{eig}_{\min}}
\DeclareMathOperator{\eigmax}{\operatorname{eig}_{\max}}

\DeclareMathOperator{\iid}{\overset{\text{iid}}{\sim}}

\newcommand{\paramall}{\bfW}
\newcommand{\param}{\bfw}
\newcommand{\paramj}{\param_{j}}
\newcommand{\paramtj}{{\param_{tj}}}
\newcommand{\paramsj}{{\param_{sj}}}
\newcommand{\lamj}{\lambda_{\nnodes,j}}
\newcommand{\lamjpart}[1]{\lamj^{(#1)}}

\newcommand{\aj}{a_j}

\DeclareMathOperator{\diag}{diag}
\newcommand{\tillam}{\widetilde\lambda}
\DeclareMathOperator{\trace}{\text{trace}}

\newcommand{\bfbeta}{\boldsymbol{\beta}}

\title{Over-parameterised Shallow Neural Networks with Asymmetrical Node Scaling: Global Convergence Guarantees and Feature Learning}

\author{
	\name Fran\c{c}ois Caron \email caron@stats.ox.ac.uk \\
	\addr Department of Statistics \\
	University of Oxford, United Kingdom
	\AND
	\name Fadhel Ayed \email fadhel.ayed@gmail.com \\
	\addr Huawei Technologies \\
	Paris, France
	\AND
	\name Paul Jung \email paul.jung@gmail.com \\
	\addr Department of Mathematics \\
	Fordham University, USA
	\AND
	\name Hoil Lee \email hoil.lee@alumni.kaist.ac.kr \\
	\addr Samsung SDS \\
	South Korea
	\AND
	\name Juho Lee \email juholee@kaist.ac.kr \\
	\addr Kim Jaechul Graduate School of AI \\
	KAIST, South Korea
	\AND
	\name Hongseok Yang \email hongseok.yang@kaist.ac.kr \\
	\addr School of Computing \\
	KAIST, South Korea
}

\def\month{02}  % Insert correct month for camera-ready version
\def\year{2025} % Insert correct year for camera-ready version
\def\openreview{\url{https://openreview.net/forum?id=Sx1khIIi95}} % Insert correct link to OpenReview for camera-ready version

\begin{document}

	\maketitle
	\begin{abstract}
		We consider gradient-based optimisation of wide, shallow neural networks, where the output of each hidden node is scaled by a positive parameter. The scaling parameters are non-identical, differing from the classical Neural Tangent Kernel (NTK) parameterisation.
		We prove that for large such neural networks, with high probability, gradient flow and gradient descent converge
		to a global minimum \emph{and} can learn features in some sense, unlike in the NTK parameterisation. We perform experiments illustrating our theoretical results and discuss the benefits of such scaling in terms of prunability and transfer learning.
	\end{abstract}

	\section{Introduction}
	\label{introduction}
	The training of neural networks typically involves the minimisation of a non-convex objective function. However, first-order optimisation methods, such as gradient descent (GD) and its variants, often find solutions with low training error.
	To gain a better understanding of this phenomenon, one fruitful direction of research has been to analyse properties of GD training
	of over-parameterised, large-width neural networks; that is, neural networks where the number $\nnodes$ of hidden nodes in a
	given layer is very large. In particular, under a ``$\sqrt{1/\nnodes}$'' scaling of the hidden nodes, \citet{Jacot2018} have
	shown that, as the number of nodes $m$ tends to infinity, the solution obtained by GD achieves zero training error,
	and coincides with that of kernel regression under a so-called limiting \textit{Neural Tangent Kernel} (NTK). Under the same node scaling, called \textit{NTK scaling}, quantitative theoretical guarantees for the global convergence and generalisation
	properties have then been obtained for large (but finite) width neural networks \cite{Du2019,Du2019a,Oymak2020,Arora2019a,Bartlett2021}.
	However, it has been noted in a number of articles \cite{Chizat2019,Yang2019,Arora2019a, Yang2021} that under NTK scaling,
	feature learning does not occur and GD training is performed in a \textit{lazy-training} regime, in contrast with the typical
	feature-learning regime exhibited in deep neural networks.

	\paragraph{Main contributions.}  We investigate global convergence properties and feature learning in gradient-type training of large-width feedforward neural networks (FFNNs) under a more general asymmetrical node scaling. In particular, each hidden node $j=1,\ldots,\nnodes$ has a fixed node-specific scaling $\sqrt{\lamj}$ with
	\begin{align}
		\lamj=
		\gamma \cdot \frac{1}{\nnodes}+(1-\gamma) \cdot \frac{\tillam_j}{\sum_{k=1}^{\nnodes} \tillam_k} \label{eq:lamj}
	\end{align}
	where $\gamma\in[0,1]$ and $1\geq \tillam_1\geq \tillam_2\geq \ldots \geq 0$ are nonnegative fixed scalars with $\sum_{j=1}^\infty \tillam_j= 1$. Note that $\gamma=1$ corresponds to the $\sqrt{1/m}$ NTK scaling. If $\gamma<1$, the node scaling is necessarily asymmetrical for large-width networks.
	%\footnote{That is, for $m$ large enough, there exist $j$ and $k$ such that $\lambda_{\nnodes,j}\neq \lambda_{\nnodes,k}$.}
	Two typical examples of the scalars $(\tillam_j)_{j \geq 1}$ are (a) $\tillam_j=6\pi^{-2}j^{-2}$
	for all $j\geq1$, and (b) $\tillam_1 = \ldots = \tillam_K =1/K$ and $\tillam_j=0$ for all $j> K$, for a fixed $K$.
	
	We consider a shallow FFNN with a smooth activation function and without bias, where the first layer weights are trained
	via gradient flow or descent and empirical risk minimisation under the $\ell_2$ loss. We show that, under
	similar assumptions as in \citet{Du2019,Du2019a} on the data, activation function, and initialisation, when
	the number of hidden nodes $\nnodes$ is sufficiently large: (i) if $\gamma>0$, the training error goes to $0$
	at a linear rate with high probability; and (ii) feature learning (in the sense of the definitions given in \Cref{sec:defFL}) occurs if and only if $\gamma<1$.
	We provide numerical experiments which illustrate the theoretical results and demonstrate empirically that such node-scaling is also useful for pruning and for transfer learning.
	
	\paragraph{Organisation of the paper.} \Cref{sec:related} discusses related work.  \Cref{sec:problemsetup}
	introduces the FFNN model with asymmetrical node scaling, gradient flow or gradient descent updates,
	and the main assumptions on the data, activation function, and initialisation. \Cref{sec:limitingNTGinit}
	discusses the properties of the NTK of such a model at initialisation, and its infinite-width limit.
	\Cref{sec:globalconvergence,sec:globalconvergenceGD} derive our main results on the convergence to a global minimum
	of gradient flow and gradient descent and sketch their proofs. \Cref{sec:featurelearning} gives the main results
	regarding feature learning. \Cref{sec:experiments} describes our experiments on simulated and real datasets,
	whose results illustrate our theoretical results and their potential applications. The Supplementary Material
	contains detailed proofs, as well as results on additional convergence of gradient flow and feature learning
	when using the ReLU activation function.

	\paragraph{Notations.} For an integer $n\geq 1$, let $[n]=\{1,\ldots,n\}$. 	For a multivariate real-valued function $f:\mathbb R^n\to\mathbb R$, the gradient $\nabla_\mathbf{v} f(\mathbf{v})$ is the $n$-dimensional column vector of partial derivatives $\nabla_\mathbf{v} f(\mathbf{v})= (\frac{\partial f}{\partial v_1}(\mathbf{v}),\ldots,\frac{\partial f}{\partial v_n}(\mathbf{v}))^\top$ where $\mathbf{v}=(v_1,\ldots,v_n)^\top$.
	For a square matrix $B$, we denote its minimum and maximum eigenvalues by $\eigmin(B)$
	and $\eigmax(B)$, respectively. For a vector $\mathbf{v} \in\mathbb R^n$, we write $B=\diag(\mathbf{v})$
	for the $n$-by-$n$ diagonal matrix with $B_{ii}=v_i$ for $i \in [n]$.

	\section{Related work}
	\label{sec:related}
	\textbf{Large-width FFNNs.} The analysis of large-width FFNNs goes back to \citet{Neal1996} who showed the connection between Gaussian processes and FFNNs in the large-width limit. Recent work has explored this connection under varying assumptions \cite{Matthews2018,Lee2018,Yang2019,Favaro2020,Bracale2021,Lee2022,Jung2021}.
	
	\textbf{Large-width FFNNs under NTK scaling.} Following the seminal work of \citet{Jacot2018}, a number of articles have investigated the benefits of over-parameterisation for gradient descent training, with the ``$1/\sqrt{m}$'' NTK scaling \cite{Arora2019,Du2019,Du2019a,Lee2019a,Zou2019,Oymak2020,Zou2020}. Crucially, when the width of the network is large enough with respect to the size of the training set, the training loss converges to a global minimum at a linear rate under gradient flow or gradient descent. However, under this symmetrical NTK scaling, the hidden-layer features do not move significantly when the width is large, and the scaling has been coined lazy-training regime for this reason \cite{Chizat2019,Woodworth2020}.
	
	\textbf{Large-width FFNNs under mean-field scaling.}  An alternative scaling is the ``$1/m$'' mean-field scaling \cite{Rotskoff2018,Mei2018,Mei2019,Chizat2019,Sirignano2020,Ghorbani2020,Chen2021,Tao21}.
	This scaling is also equivalent, up to the so-called abc-scaling symmetry \citep{Yang2021}, to the  $\mu P$ parameterisation of \citet{Yang2021} in the case of shallow networks.
	Feature learning is known to occur under this mean-field scaling.
	Also, \citet{chizat2018global} showed that under the same scaling, if the training of the model converges, it converges to a global minimum.

	\textbf{Asymmetrical scaling in FFNNs. } The idea of using asymmetrical scaling parameters in the context of GD optimisation of deep FFNNs has been previously introduced by \citet{Wolinski2020a}. The focus of \citet{Wolinski2020a} was on the (empirical) usefulness in terms of pruning. Indeed, our experiments in \cref{sec:experiments} are also in line with their findings. The work of \citet{Wolinski2020a}, however, only considered asymmetrical scaling with $\gamma=0$ (no fixed part), and did not investigate global convergence properties under such scaling. The properties of random FFNNs under \textit{random} asymmetrical node scaling in the large-width limit has also been considered by \citet{Lee2022}; but this paper did not investigate the training properties under gradient flow or gradient descent.

	\section{Problem setup}
	
	\label{sec:problemsetup}
	
	\subsection{Statistical model}
	\label{sec:statisticalmodel}

	We consider a shallow FFNN with one hidden layer of $\nnodes\geq 1$ hidden nodes and a scalar output. To simplify the analysis, we assume that there is no bias term. Let $\bx\in\mathbb R^{\din}$ be some input vector, where $\din\geq 1$ is the input dimension. The model is defined as
	\begin{equation}
		\label{eq:FFNN1}
		f_{\nnodes}(\bx;\paramall) = \sum_{j=1}^{\nnodes} \sqrt{\lamj}  \aj \relu(Z_{j}(\bx; \paramall))
		\quad\text{with}\quad
		Z_{j}(\bx;\paramall) = \frac{1}{\sqrt{\din}} \paramj^\top \bx~~\text{for}\ j \in [\nnodes]
	\end{equation}
	where $f_{\nnodes}(\bx;\paramall)$ is the scalar output of the FFNN, $Z_{j}(\bx;\paramall)$ is the pre-activation of the $j$-th hidden node, $\relu:\mathbb R\to\mathbb R$ is the activation function,
	$\paramj\in\mathbb R^{\din}$ is the column vector of weights between node $j$ of the hidden layer and the input nodes, and $\aj\in\mathbb R$ is the weight between the hidden node $j$ and the output node. The $\lamj$'s for $j \in [\nnodes]$ are nonnegative scaling parameters for the hidden nodes. The parameters to be optimised are contained in $\paramall$ which
	is an $m\din$-dimensional column vector $(\param_1^\top, \ldots,\param_{\nnodes}^\top)^\top$.  We assume that $\relu$ admits a derivative $\relu'$.
	
	For $n\geq 1$, let $\brelu:\mathbb R^n\to\mathbb  R^n$ (resp. $\brelu':\mathbb R^n\to\mathbb  R^n$) be the vector-valued multivariate function that applies the real-valued function $\relu$ (resp. $\relu'$) element-wise on each of the $n$ input variables. To simplify the analysis, we assume throughout this article that the output weights $(\aj)_{j\in[m]}$ are randomly initialised and fixed afterwards:
	%	$\aj\iid \uniform(\{-1,1\})$ for all $j\geq 1$.
	\begin{align}
		\aj\iid \uniform(\{-1,1\}),~~j\geq 1.
		\label{eq:priorU}
	\end{align}
	This simplifying assumption is often made when analysing large shallow networks (see e.g. \cite{Du2019,Bartlett2021}), and the analysis can also be extended to the case where both layers are trained.
	The scaling parameters $(\lamj)_{j\in[m]}$ are fixed and parameterised as in \cref{eq:lamj}.
	By construction, we have $\lambda_{m,1}>0$ and $\sum_{j=1}^{\nnodes} \lamj= 1$ for all $\nnodes\geq1$.  Recall that
	the case $\gamma=1$ corresponds to NTK scaling. Also, note that our model covers finite FFNNs: when
	\[
	\gamma=0
	\quad\text{and}\quad
	\tillam_j =
	\begin{cases}
		1/K & \text{if}\ j \in [K]
		\\
		0 & \text{otherwise}
	\end{cases}
	\]
	for some $K\leq m$, the model becomes a finite network of width $K$. In the experiments,
	we will consider the special case where $(\tillam_j)_{j\geq 1}$ are the probability masses of a Zipf law:
	\begin{align}
		\tillam_j = \frac{1}{\zeta(1/\alpha)}\frac{1}{j^{1/\alpha}},~~j\geq 1\label{eq:tildelambdazipf}
	\end{align}
	for some $\alpha\in(0,1)$, where $\zeta$ is the Riemann zeta function.
	The parameter $\alpha$ tunes how quickly  $\tillam_j$ decreases with $j$,
	smaller values corresponding to more rapid decrease and more asymmetry.

	\subsection{Training}

	Let $\mathcal{D}_n =\{(\bx_i,y_i)\}_{i \in [n]}$ be the training dataset, where $n\geq 1$ is the number of observations. Let $\bX$ denote the $n$-by-$\din$ matrix whose $i$th row is $\bx_i^\top$.  We want to minimise the empirical risk under $\ell_2$ loss.	
	Consider the objective function
	\begin{align}
		L_\nnodes(\paramall) =\frac{1}{2}\sum_{i=1}^n (y_i-f_{\nnodes}(\bx_i;\paramall))^2 \label{eq:objectivefunction}
	\end{align}	
	which is non-convex in general. For a given dataset $\mathcal{D}_n$, width $\nnodes\geq 1$, output weights $(\aj)_{j\in[m]}$, and scaling parameters $(\lamj)_{j \in [\nnodes]}$, we aim to estimate the trainable parameters $\paramall$ by minimising $L_\nnodes(\paramall)$ using gradient flow or gradient descent. Let $\paramall_0$ be some initialisation. In gradient flow,
	$(\paramall_t)_{t>0}$ is the solution to the following ordinary differential equation (ODE):
	\begin{equation*}
		\frac{d\paramall_t}{dt} = -\nabla_\paramall L_\nnodes(\paramall_t)
	\end{equation*}
	with $\lim_{t\to 0}\paramall_t=\paramall_0$. Let $\paramtj$ be the value of the parameter $\paramj$ at time $t$, and define $Z_{tj}(\bx) =Z_{j}(\bx;\paramall_t)$. Note that
	$\nabla_{\paramj} f_{\nnodes}(\bx;\paramall) =\sqrt{\lamj}  \aj \relu'(Z_{j}(\bx;\paramall))\cdot \frac{1}{\sqrt{\din}}\bx.$  Under gradient flow, for $j \in [\nnodes]$,
	\[
	\frac{d\paramtj}{dt}
	= \left(\sum_{i=1}^n  (y_i-f_{\nnodes}(\bx_i;\paramall_t))\nabla_{\paramj} f_{\nnodes}(\bx_i;\paramall_t)\right)
	= \left(\frac{ \sqrt{\lamj}\aj}{\sqrt{\din}} \sum_{i=1}^n  (y_i-f_{\nnodes}(\bx_i;\paramall_t))\relu'(Z_{tj}(\bx_i))\bx_i\right).
	\]
	Thus, the derivatives associated with each hidden node $j$ are scaled by $\sqrt{\lamj}$. For an input $\bx\in\mathbb R^{\din}$, the output of the FFNN therefore satisfies the ODE
	\[
	\frac{df_{\nnodes}(\bx;\paramall_t)}{dt}
	= \left(\nabla_\paramall f_{\nnodes}(\bx;\paramall_t)^\top \frac{d\paramall_t}{dt}\right)
	= \left(\sum_{i=1}^n (y_i-f_{\nnodes}(\bx_i;\paramall_t)) \ntker(\bx,\bx_i;\paramall_t)\right),
	\]
	where $\ntker:\mathbb R^{\din}\times\mathbb R^{\din}\to\mathbb R$ is the neural tangent kernel
	for the network $f_\nnodes(\bx;\paramall)$:
	\begin{align}
		\ntker(\bx,\bx';\paramall)
		=\frac{\bx^\top \bx'}{\din} \sum_{j=1}^\nnodes \lamj \relu'(Z_{j}(\bx;\paramall))\relu'(Z_{j}(\bx';\paramall)).\label{eq:NTK}
	\end{align}
	The associated  neural tangent Gram (NTG) matrix $\ntgram(\bX;\paramall)$ is the $n$-by-$n$ positive semi-definite matrix whose $(i,j)$-th entry is $\ntker(\bx_i,\bx_j;\paramall)$. It takes the form
	\begin{align}
		\ntgram(\bX;\paramall)=  \frac{1}{\din}\sum_{j=1}^\nnodes \lamj\diag\!\left(\brelu'\!\left(\frac{\bX \paramj}{\sqrt{\din}}\right)\!\right)\!\bX \bX^\top \diag\!\left(\brelu'\!\left(\frac{\bX \paramj}{\sqrt{\din}}\right)\!\right)\!.\label{eq:NTG}
	\end{align}
	
	Gradient descent is a discretisation of gradient flow. Under gradient descent, the parameters are updated by
	\begin{equation}
		\paramall_{s} = \paramall_{s-1} - \eta \nabla_\paramall L_\nnodes(\paramall_{s-1})
		\quad \text{for all $s \in \mathbb{N}$},
	\end{equation}
	where $\eta > 0$ is a learning rate. These updates give rise to the family $(\paramall_s)_{s \in \mathbb{N}\cup\{0\}}$ indexed by discrete time steps $s = 0,1,2,\ldots$, rather than continuous times $t \geq 0$.

	\subsection{Main assumptions}
	\label{sec:assumptions}

        Throughout the paper, we assume that the activation function $\sigma$
        satisfies the following standard condition: for all
        random variables $Z \sim \mathcal{N}(0, s^2)$ for
        some $s > 0$,
        \begin{equation}
                \label{eqn:assumption-finite-positive-activation}
                |\mathbb{E}_Z[\sigma(Z)]| < \infty
                \qquad\text{and}\qquad
                0 < \mathbb{E}_{Z}[\sigma(Z)^2] < \infty
        \end{equation}
        This assumption is made all the time, and so we do not mention its
        use explicitly in the paper.
	
	The results of this article on global convergence and feature learning use several further assumptions. The first set of these assumptions, which are mild and similar to other assumptions used in the literature, is on the training dataset $\mathcal{D}_n =\{(\bx_i,y_i)\}_{i \in [n]}$.
	\begin{assumption}[Dataset]\label{assump:data} (a) All inputs are non-zero and have norms at most $1$:  $0 < \|\bx_i\| \leq 1$ for all $i \geq 1$. (b) For all $i\neq i'$ and $c\in \mathbb{R}$, $\bx_i\neq c\bx_{i'}$. (c) There is $C>0$ such that $|y_i|\leq C$ for all $i\geq 1$.
	\end{assumption}
	The next assumption concerns the activation function $\relu$. Standard activation functions (softplus, tanh, sigmoid, swish) satisfy this assumption, but not the ReLU. However, some of our results, such as global
	convergence of gradient flow and feature-learning results, also hold in the ReLU case, as shown in \cref{sec:results-relu} in the Supplementary Material.

		\begin{assumption}[Activation function]
		The activation function is analytic, with $|\relu'(x)|\leq 1$ and $|\relu''(x)|\leq M$ for some $M>0$, but it is not a polynomial.
		\label{assump:activation}
	\end{assumption} The last assumption, which is standard, is on the initialisation of the weights.
	
	\begin{assumption}[Initialisation]		
		For $j \in \mathbb{N}$, $\param_{0j} \iid \mathcal{N}(  0,\idmat_{\din})$, where $\idmat_{\din}$ is the $\din$-by-$\din$ identity matrix.
		\label{assump:init_V}
	\end{assumption}

	\section{
		Neural Tangent Kernel at initialisation and its limit}
	\label{sec:limitingNTGinit}
	\paragraph{Mean  NTG at initialisation and its minimum eigenvalue.} Let $\paramall_0$ be a
	random initialisation from \cref{assump:init_V}. Consider the mean NTK at initialisation
	\begin{equation}
		\ntkerlim(\bx,\bx')
		= \mathbb E \left [ \ntker(\bx,\bx';\paramall_0 )\right ]
		= \frac{\bx^\top \bx'}{\din} \mathbb E \left [ \relu'\left(\frac{1}{\sqrt{\din}} \param_{01}^\top \bx\right)\relu'\left(\frac{1}{\sqrt{\din}} \param_{01}^\top \bx'\right)\right ].\label{eq:ntkerlim}
	\end{equation}
	The mean NTK, which is also, by the law of large numbers, the limiting NTK under $1/\sqrt{m}$ scaling \cite{Jacot2018},
	does not depend on $(\lamj)_{j\geq 1}$ nor $m$. Let $\ntgramlim(\bX)=\mathbb E[ \ntgram(\bX; \paramall_0) ]$ be the
	associated $n$-by-$n$ mean
	NTG matrix at initialisation, whose $(i,j)$-th entry is $\ntkerlim(\bx_i,\bx_j)$. Let $\kappa_n= \eigmin(\ntgramlim(\bX))$
	be the minimum eigenvalue of the mean NTG matrix at initialisation. This minimum eigenvalue plays an important role
	in the analysis of global convergence properties in the lazy-training regime.
	Many authors (see e.g. \cite{ElKaroui2010,Nguyen2021}) have shown that, under some assumptions on the data, activation function, and initialisation, $\kappa_n$ is strictly positive or bounded away from zero.
	 such a result, under the Assumptions of \cref{sec:assumptions}.
	
	\begin{proposition}[{\cite[Theorem 3.1]{Du2019} and \cite[Proposition F.1]{Du2019a}}]
		When \cref{assump:data,assump:activation,assump:init_V} hold,
		we have $\kappa_n>0$.
	\end{proposition}
	
	\begin{remark}
		\citet{Du2019,Du2019a} make the assumption that each $\bx_i$ has unit norm. But their proof holds under the less
		strict \cref{assump:data}(a). The above proposition also holds if \cref{assump:activation} is replaced by
		the assumption that $\sigma$ is the ReLU function.
	\end{remark}

	\paragraph{Limiting NTG.} To give some intuition, we now describe the limiting behaviour of the NTG, for a fixed sample size $n$, as the width $\nnodes$ goes to infinity. The proof, given in \cref{sec:prooflimitingNTG} in the Supplementary Material, follows from the triangle inequality and the law of large numbers, together with  $|\relu'(z)|\leq 1$ and $\sum_{j\geq 1}  \tillam_j =1$.
	
	\begin{proposition}\label{prop:limitingNTG}
		Consider a sequence $(\param_{0j})_{j\geq 1}$ of iid random vectors distributed as in \cref{assump:init_V}.
		Suppose \cref{assump:activation} holds. Then,
		\begin{align}
			\ntgram(\bX;\paramall_0)\to  \widehat\Theta_\infty(\bX;\paramall_0)
		\end{align}
		almost surely as $\nnodes\to\infty$, where
		$\widehat\Theta_\infty(\bX;\paramall_0)=\gamma\ntgramlim(\bX)+(1-\gamma)\widehat\Theta_\infty^{(2)}(\bX;\paramall_0)$,
		with $\widehat\Theta_\infty^{(2)}(\bX;\paramall_0)$ being the following random positive semi-definite matrix:
		\begin{equation}
			\widehat\Theta_\infty^{(2)}(\bX;\paramall_0) =
			\frac{1}{\din}\sum_{j=1}^\infty \widetilde\lambda_j\diag \left(\brelu'\left(\frac{\bX \param_{0j}}{\sqrt{\din}}\right)\right)\bX \bX^\top \diag\left(\brelu'\left(\frac{\bX \param_{0j}}{\sqrt{\din}}\right)\right).
		\end{equation}
		Also, $\mathbb E [\widehat\Theta_\infty(\bX;\paramall_0)]=\mathbb E [\widehat\Theta_\infty^{(2)}(\bX;\paramall_0)]=\ntgramlim(\bX)$, and
		\begin{align}
			\mathbb E \left [ \Vert\widehat\Theta_\infty(\bX;\paramall_0)- \ntgramlim(\bX)\Vert_F^2 \right ]
			=C_0(\bX)(1-\gamma)^2\sum_{j\geq 1}\widetilde\lambda^2_j \label{eq:totalvarianceNTGatinit}
		\end{align}
		where $\Vert\cdot\Vert_F$ denotes the Frobenius norm, and $C_0(\bX)> 0$ is a positive constant equal to
		\begin{align*}
			\sum_{1\leq i,i'\leq n} {\!\!\left(\frac{\bx_{i}^\top \bx_{i'}}{d}\right)\!}^2
			\var\!\left(\relu'\!\left(\frac{1}{\sqrt{\din}} \param_{01}^\top \bx_i\right)\!\relu'\!\left(\frac{1}{\sqrt{\din}} \param_{01}^\top \bx_{i'}\right)\!\right)\!.
		\end{align*}
		
	\end{proposition}
	
	When $\gamma=1$ (NTK scaling), the NTG converges to a constant matrix, and solutions obtained by gradient flow coincide with that of kernel regression. Whenever $\gamma<1$, \cref{prop:limitingNTG} shows that the  NTG is random at initialisation, even in the infinite-width limit, contrary to that of NTK scaling. As shown in  \cref{eq:totalvarianceNTGatinit}, the departure from the symmetric regime, as measured by the total variance of the limiting random NTG, can be quantified by the nonnegative constant $(1-\gamma)^2\sum_{j\geq 1}\widetilde\lambda_j^2 \in[0,1].$
	When this constant is close to 0, we approach the kernel regime, and increasing this value leads to a departure from the regime.
	The quantity $\sum_{j\geq 1}\widetilde\lambda_j^2\in(0,1]$ is always strictly positive. More rapid decrease of the $\widetilde\lambda_j$ as $j$ increases will lead to higher values of $\sum_{j\geq 1}\widetilde\lambda_j^2$. For example, when using the Zipf weights in \cref{eq:tildelambdazipf}, we have  $\sum_{j\geq 1}\widetilde\lambda_j^2=\frac{\zeta(2/\alpha)}{\zeta(1/\alpha)^2}$, which decreases with $\alpha$, as shown in \cref{fig:zipflawalpha} in the Supplementary Material.

	This section has described the behaviour of the NTG at initialisation in the infinite-width limit, and has provided intuition on the node-scaling parameters. The next three sections contain results on global convergence and feature learning properties of large, but finite, FFNNs under such asymmetrical scaling.

	\section{Global convergence analysis for gradient flow}
	\label{sec:globalconvergence}

Our global convergence theorem, which is given below, explains what happens during training via gradient flow.
	Recall that $\kappa_n$ is the minimum eigenvalue of the mean NTG matrix $\ntgramlim(\bX)$ at initialisation.
	Our theorem says that with high probability, (i) the loss decays exponentially fast with respect to $\kappa_n$ and
	training time $t$, and (ii) the  NTG and weights $\param_{tj}$ change, respectively, by
	\[
	\|\ntgram(\bX;\paramall_t)-\ntgram(\bX;\paramall_0)\|_2=O\left(\frac{n^3 \sum_{j=1}^{\nnodes}\lamj^2}{\kappa_n^2 d^{3}\gamma^2}
	+\frac{n^2 \sqrt{\sum_{j=1}^{\nnodes}\lamj^2}}{\kappa_n d^2\gamma}\right)
	\quad\text{and}\quad
	\|\param_{tj} - \param_{0j}\|=O\left(\frac{n\lamj^{1/2}}{\kappa_n\din^{1/2}\gamma}\right).
	\]
	Define
	\begin{equation}
		\label{eqn:definition-C1}
		C_1 = \sup_{c \in (0,1]}\mathbb{E}_{z\sim\mathcal{N}(0,1)}\left[\relu\left(\frac{cz}{\sqrt{d}}\right)^2\right].
	\end{equation}

	\begin{theorem}(Global convergence, gradient flow)
		\label{thm:upper-bound:smooth-activation}
		Let $\delta \in (0,1)$. Suppose \cref{assump:data,assump:activation,assump:init_V} hold, and that
		\[
		\gamma > 0,
		\quad\text{and}\quad
		\nnodes
		\geq
		\max \bigg(
		\frac{2^3n\log \frac{2n}{\delta}}{\kappa_n\din},\;
		\frac{2^{10} n^3M^{2}(C^2 + C_1)}{\kappa_n^3d^3\gamma^2\delta},\;
		\frac{2^{15} n^4M^2(C^2 + C_1)}{\kappa^4_n d^4\gamma^2\delta}
		\bigg),
		\]
		where $C$ is the bound on the $y_i$'s in \cref{assump:data}.
		Then, with probability at least $1-\delta$, the following properties hold for all $t \geq 0$:
		\begin{enumerate}[nosep]
			\item[(a)] $\eigmin(\ntgram(\bX;\paramall_t)) \geq \frac{\gamma \kappa_n}{4}$;
			\item[(b)] $L_\nnodes(\paramall_t) \leq e^{-(\gamma\kappa_n t)/2} L_\nnodes(\paramall_0)$;
			\item[(c)] $\|\param_{tj} - \param_{0j}\| \leq
			\frac{n\sqrt{\lamj}}{\kappa_n\din^{1/2}} \sqrt{\frac{2^7(C^2 + C_1)}{\gamma^2\delta}}$ for all $j  \in [\nnodes]$;
			\item[(d)] $\|\ntgram(\bX;\paramall_t)-\ntgram(\bX;\paramall_0)\|_2 \leq
			\Big(
			\frac{2^7 n^3M^{2}(C^2 + C_1)}{\kappa_n^2 d^{3}\gamma^2\delta} \cdot  {} \sum_{j=1}^{\nnodes}\lamj^2\Big)
			+
			\Big(\frac{2^5n^2M(C^2+C_1)^{1/2}}{\kappa_n d^2\gamma\delta^{1/2}} \cdot
			\sqrt{\sum_{j=1}^{\nnodes}\lamj^2}\Big)$.
		\end{enumerate}
	\end{theorem}
	
	The above theorem implies that, whenever $\gamma>0$, the training error converges to 0 exponentially fast. Additionally, the weight change is bounded by a factor $\sqrt{\lamj}$ and the   NTG change is bounded by a factor $\sqrt{\sum_{j=1}^{\nnodes}\lamj^{2}}$.
	We have (see \cref{sec:lambdasproperties} in the Supplementary Material) that as $m\to\infty$,
	\[
	\lamj\to(1-\gamma)\tillam_j\, \text{ for all $j\geq 1$}
	\qquad
	\text{and}
	\qquad
	\sum_{j=1}^{\nnodes}\lamj^{2}\to (1-\gamma)^{2}\sum_{j=1}^{\infty}\tillam_j^{2}.
	\]
	If $\tillam_j>0$ (note that we necessarily have $\tillam_1>0$),
	the upper bound in (c) is therefore vanishing in the infinite-width limit if and only if $\gamma=1$ (lazy-training regime). Similarly, the upper bound in (d) is vanishing if and only if $\gamma=1$. Although we were not able to obtain matching lower bounds, we show in \cref{sec:featurelearning} that feature learning arises whenever $\gamma<1$.
	
	\begin{remark}
                We make two comments on \Cref{thm:upper-bound:smooth-activation}.
                First, although $d$ represents the input dimension,
                all the occurrences of $d$ in the theorem, such as those
                in the lower bound of the width $m$,
                do not come from the complexity of the input dimension. Instead, it
                comes from the fact that our model uses the $1/\sqrt{d}$ scaling
                when computing the pre-activation values of the first layer. If
                this scalining were removed in our model, the statement of the theorem
                would not include $d$ (i.e.,
                we would have the theorem with $d$ set to $1$). Second, a result
                similar to \Cref{thm:upper-bound:smooth-activation} also holds
                for the ReLU activation function. See \cref{thm:upper-bound:relu} in
                the Supplementary Material.
	\end{remark}

	\paragraph{Sketch of the proof.}
	
	We give here a sketch of the proof of \cref{thm:upper-bound:smooth-activation}
	(and of \cref{thm:upper-bound:relu}, its ReLU counterpart, in the Supplementary Material). The detailed proofs are given
	in \cref{sec:proofboundrelu,sec:proofboundsmooth} in the Supplementary Material, with secondary
	lemmas given in \cref{sec:secondarythm:init,sec:secondarythm:dynamics} there. The structures of the proofs
	of \cref{thm:upper-bound:relu,thm:upper-bound:smooth-activation} are similar to that of \cite[Theorem 3.2]{Du2019}, which showed analogous results on the global convergence under NTK scaling. However, there are some key differences which we highlight below.
	
	Gradient flow converges to a global minimum of the objective function if  the minimum eigenvalue of the   NTG matrix $\ntgram(\bX;\paramall_t)$ is bounded away from zero, for $m$ sufficiently large, by some positive constant for all $t\geq 0$. In the NTK scaling
	case ($\gamma=1$), \citet{Du2019} showed that the following is satisfied, for $m$ sufficiently large, with high probability: (i) the NTG matrix at initialisation is close to its mean, and the minimum eigenvalue is close to that of the mean NTG, (ii) the weights $\paramtj$ are nearly constant in time, which implies that (iii) the NTG matrix is nearly constant in time, hence (iv) the minimum eigenvalue of the NTG matrix at time $t$ is close to its value at initialisation, which is bounded away from zero.
	
	However, in the case of asymmetrical node scaling ($\gamma<1$), none of the points (i-iv) holds. At initialisation, the random   NTG matrix may be significantly different from its mean. Additionally, both the weights and the   NTG matrix substantially change over time. This therefore requires a somewhat different approach that we now describe.

	Let $\lamjpart1= \gamma /\nnodes$ and $\lamjpart2=((1-\gamma)\tillam_j) / \sum_{k=1}^{\nnodes} \tillam_k$, and note that $\lamjpart1+\lamjpart2=\lamj$. For $k\in\{1,2\}$, let $\ntgrampart{k}$ be the $n$-by-$n$ symmetric positive semi-definite matrices defined by \cref{eq:NTG}, with $\lamj$ replaced by either $\lamjpart{k}$. Note that
	\begin{align}
		\ntgram(\bX;\paramall_t)=\ntgrampart1(\bX;\paramall_t)+\ntgrampart2(\bX;\paramall_t)
		\label{eq:decompositionNTG}
	\end{align}
	with $\mathbb E[\ntgrampart{1}(\bX;\paramall_0)]=\gamma \ntgramlim(\bX)$.

	The key idea of the proof is to use the above decomposition of the NTG matrix as a sum of two terms, and to show that,
	while the second term may change over time, the first term is close to its mean at initialisation, and does not change much over time.
	The important points of the proof are as follows. For large $m$, with high probability:
	(i) $\ntgrampart1(\bX;\paramall_0)$ is close to its mean $\gamma \ntgramlim(\bX)$ and its minimum eigenvalue is therefore lower bounded by $(\gamma\kappa_n) / 2$;

The important points of the proof are as follows. For large $m$, with high probability:
(i) $\ntgrampart1(\bX;\paramall_0)$ is close to its mean $\gamma \ntgramlim(\bX)$ and its minimum eigenvalue is therefore lower bounded by $(\gamma\kappa_n) / 2$; (ii) while the weights $\paramall_t$ may change significantly over time, $\ntgrampart1(\bX;\paramall_t)$ remains nearly constant over time; (iii) as a result, the minimum eigenvalue of $\ntgrampart1(\bX;\paramall_t)$ can be lower bounded by $(\gamma\kappa_n)/4$; (iv) this implies that the minimum eigenvalue of the overall   NTG matrix $\ntgram(\bX;\paramall_t)$ is lower bounded by $(\gamma\kappa_n)/4$.
	
Since showing that the first part of the NTK/NTG in \Cref{eq:decompositionNTG} does not change in the limit of $m\to\infty$ is a key component of the proof, we give here an outline of it in the simplified case of $d=1$ and $\sigma$ being smooth. Let
\[
\ntker^{(1)}(\bx,\bx';\paramall_t)=\bx \bx' \frac{\gamma}{m}\sum_{j=1}^m  \sigma'(\param_{tj} \bx)\sigma'(\param_{tj} \bx')
\]
be the first part of the NTK. We have, over gradient flow,
\begin{align*}
\left |\frac{d\ntker^{(1)}(\bx,\bx';\paramall_t)}{dt}\right|
&
=\left |\bx \bx' \frac{\gamma}{m}\sum_{j=1}^m \Big(\bx \sigma''(\param_{tj} \bx)\sigma'(\param_{tj} \bx')+\bx' \sigma'(\param_{tj} \bx)\sigma''(\param_{tj} \bx')\Big)\frac{d\param_{tj}}{dt}\right |\\
&\leq 2M \frac{\gamma}{m}\sum_{j=1}^m \left|\frac{d\param_{tj}}{dt}\right|,
\end{align*}
where the last inequality follows from the triangle inequality and \cref{assump:data,assump:activation}. Furthermore,
\begin{align*}
\left|\frac{d\paramtj}{dt}\right| &=\sqrt{\lamj}\left| \sum_{i=1}^n  (y_i-f_{\nnodes}(\bx_i;\paramall_t))\relu'(Z_{tj}(\bx_i))\bx_i\right|\\
&\leq \sqrt{\lamj} \sum_{i=1}^n  \left|y_i-f_{\nnodes}(\bx_i;\paramall_t)\right|\\
&\leq\sqrt{\lamj}2n  L_m(\paramall_t),
\end{align*}	
where the first inequality follows from the triangle inequality and \cref{assump:data,assump:activation}, and the second inequality follows from Cauchy-Schwarz. The change in the first part of the NTK is thus bounded by a quantity involving $\frac{\gamma}{m}\sum_{j=1}^m \sqrt{\lamj}$. Under the scaling in \Cref{eq:lamj},
$\frac{\gamma}{m}\sum_{j=1}^m \sqrt{\lamj}\to 0$ as $m\to\infty$ (see \cref{sec:lambdasproperties} in the Supplementary Material). Hence, the change in the first part of the NTK/NTG becomes asymptotically small as the width $m$ increases. It is worth noting that for the full NTK, i.e. the sum of both parts in \Cref{eq:decompositionNTG}, a similar derivation leads to an upper bound of the order of $(\frac{\gamma}{m}\sum_{j=1}^m \sqrt{\lamj}) + (1-\gamma)\frac{\sum_{j=1}^m \tillam_j\sqrt{\lamj}}{\sum_{k=1}^{\nnodes} \tillam_k}$. Due to the second term, this quantity does not converges to 0 as $m\to\infty$, unless $\gamma=1$ (symmetric case).

	\section{Global convergence analysis for gradient descent}
	\label{sec:globalconvergenceGD}

	Let $\mathbf{y} \in \mathbb{R}^n$ be the vector of outputs $(y_1,\ldots,y_n)^\top$  in the training dataset $\mathcal{D}_n = \{(\bx_i,y_i)\}_{i \in [n]}$. For each gradient-descent step $s$, let $\mathbf{u}_s \in \mathbb{R}^n$ be the outputs at step $s$ based on the inputs in $\mathcal{D}_n$, that is, $\mathbf{u}_s = (f_m(\bx_1;\paramall_s),\ldots,f_m(\bx_n;\paramall_s))^\top$. The following convergence theorem intuitively says that if the learning rate $\eta$, of gradient descent, is sufficiently small and the width of the network is large enough, then with high probability,
	the training error of the network decays exponentially fast to $0$. We remark that our node scaling is independent of the learning rate $\eta$. Thus, the only conditions on $\eta$ in this work are in this section where we extend our results from gradient flow to gradient descent.
	\begin{theorem}
		\label{thm:convergence-gd:smooth-case}
		Consider $\delta \in (0,1)$. Suppose \cref{assump:data,assump:activation,assump:init_V} hold, and that $\gamma > 0$, and
		\[
		0 < \eta < \min\left(\frac{2}{\gamma \kappa_n},\,\frac{\gamma \kappa_n d^2}{8 n^2},\, \frac{\gamma \kappa_n d^2 \delta^{1/2}}{2^{9/2} n^2M(C^2 + C_1)^{1/2}} \right),
		\]
		where $C$ and $C_1$ are from \cref{assump:data,eqn:definition-C1}.
		Let $\beta = (1 - \eta \gamma \kappa_n/2)^{1/2}$. If
		\[
		m
		\geq
		\max\bigg(\frac{2^3 n \log\frac{2n}{\delta}}{\kappa_n d},\;
		\frac{2^5 \eta^2 n^3 M^2  (C^2+C_1)}{   \kappa_n d^3(1 - \beta)^2\delta},\;
		\frac{2^{11}\eta^2 n^4 M^2  (C^2 + C_1)}{ \kappa_n^2 d^4(1 - \beta)^2\delta}
		\bigg),
		\]
		then with probability at least $1 - \delta$,
		\begin{equation}
			\label{eqn:convergence-gd:smooth-case:0}
			\Vert \mathbf{y} - \mathbf{u}_s \Vert^2 \leq (1 - \alpha)^s \Vert \mathbf{y} - \mathbf{u}_0 \Vert^2
			\  \text{for all $s \in \mathbb{N} \cup \{0\}$}.
		\end{equation}
	\end{theorem}
	Note that the condition on the learning rate requires $\eta = O(\gamma\kappa_n/n^2)$. Thus, the best possible convergence rate from the theorem is $(1-(\eta \gamma \kappa_n / 2)) = (1 - (C_0 \gamma^2 \kappa_n^2 / n^2))$ for some constant $C_0$.
	
	The proof is by induction on the gradient-descent step~$s$, and is described
	in detail in \cref{appendix:global-convergence-gd} in the Supplementary Material.
	It is similar to the proof of \cite[Theorem 5.1]{Du2019a}, but the two proofs differ significantly because, as in the case of gradient flow, the weights $\paramsj$ and the Gram matrix $\ntgram(\bX;\paramall_s)$ change during gradient descent in our case, while they remain nearly constant in the case of \cite{Du2019a}.

\section{Feature learning analysis}
\label{sec:featurelearning}

In this section, we present some results about feature learning. We focus on node scalings $(\lamj)_{j \in [\nnodes]}$ of the form in \cref{eq:lamj}, both asymmetrical ($\gamma<1$) and symmetrical ($\gamma=1$; that is, $\lambda_{m,j}=1/m$), but we also discuss alternative parameterisations such as mean-field and $\mu P$. For shallow neural networks trained by gradient descent, mean-field and $\mu P$ parameterisations are equivalent~\cite{Yang2021}, and correspond to node scalings $\lambda_{m,j}=1/m^2$ and learning rate $\eta=\eta_0 m$ for some $\eta_0>0$. We start with some definitions of feature learning in the context of potentially asymmetric scalings, generalising existing definitions. We then present feature learning results first under a linear activation function, and next under a general nonlinear activation function.

\subsection{Definitions}
\label{sec:defFL}

	\begin{definition}[Feature learning]
	\label{def:feature}
	Let $(\param_{0j})_{j\geq 1}$ be a sequence of random initialisations for nodes $j\geq 1$. We will say
	that {\it feature learning} occurs during training
	if\footnote{In \cite{Yang2021}, feature learning was defined in the
	large-width limits for stable and nontrivial parameterisations.
	Nontrivial means that the neural network function $f_m$ is not constant in time. Stable means that both the preactivations and activations
	have $\Theta(1)$ coordinates at initialisation and $O(1)$ coordinates throughout
	training. Both of these properties are satisfied by our model.}
	\begin{align}\label{def:fl}
		\liminf_{m\to\infty} \frac{\sum_{j=1}^\nnodes \lamj \Big( \relu(Z_{j}(\bx;\paramall_t))- \relu(Z_{j}(\bx;\paramall_0))\Big)^2}{\sum_{k=1}^\nnodes \lambda_{m,k} \Big(\relu(Z_{k}(\bx;\paramall_0))\Big)^2}>0
	\end{align}
	almost surely, or in probability, for some $t \in (0,\infty]$ and $\bx$. Here $t = \infty$ refers
	to the case that the ratio in the above inequality is the limit as $t$ tends to $\infty$.
\end{definition}

The left-hand-side quantity in \Cref{def:fl} corresponds to the relative change in
the scaled $m$-dimensional feature maps between time 0 and time $t$. \Cref{def:feature}
matches the definitions of \cite[Definitions 3.5, H.2, H.9]{Yang2021} or
\cite[Proposition 3.2]{Frei2023} in the case of symmetrical NTK or mean-field node
scalings.

\begin{remark}
As already noted by \cite[Remark H.10]{Yang2021}, \cref{def:feature} is a relatively weak notion of feature learning. It only requires a change in the feature map for some $t \in (0,\infty]$ and some $\bx$, and does not relate to the relevance of the learn features for prediction. However, we show empirically in \Cref{sec:experiments} that the feature learning property leads to better performances in terms of prunability and transfer learning.
\end{remark}

The previous definition ensures that a change occurs in the feature map. However, it may still be the case that the contributions from all the individual nodes remain asymptotically infinitesimally small, in such a way that there are no nodes representing important features in the network. This is problematic if one is interested in pruning the nodes of the network, as we show theoretically for a linear activation in \cref{sec:featurelearninglinear}, and empirically in \Cref{sec:experiments}. We introduce below the stronger definition of \textit{non-uniform feature learning}.

	\begin{definition}[Non-uniform feature learning]
	Let $(\param_{0j})_{j\geq 1}$ be a sequence of random initialisations for nodes $j\geq 1$. We will say that {\it non-uniform feature learning} occurs during training if
	\begin{align}\label{def:nofl}
		\liminf_{m\to\infty} \frac{\max_{j \in [\nnodes]} \lamj \Bigr( \relu(Z_{j}(\bx;\paramall_t))- \relu(Z_{j}(\bx;\paramall_0))\Bigr)^2}{\sum_{k=1}^\nnodes \lambda_{m,k} \Bigr(\relu(Z_{k}(\bx;\paramall_0))\Bigr)^2} >0
	\end{align}
	{almost surely} or in probability, for some $t \in (0,\infty]$ and $\bx$. \label{def:nonuniformfeature}
\end{definition}

Note that non-uniform feature learning implies feature learning, but the converse does not hold. For instance, feature learning holds under the mean-field parameterisation, but not non-uniform feature learning (see the next subsection for an illustration with a linear activation).
	
\subsection{Linear activation function}
\label{sec:featurelearninglinear}
	
	We now describe, with the following theorem, analytic results in the case of a linear activation function.
	Although for fixed second-layer weights, the NTK does not change in this linear-activation case,
	the evolution of the weights provides useful insights into the differences between the symmetrical
	and asymmetrical scalings in terms of weight change.
	\begin{theorem}
		\label{thm:featurelearninglinear}
		Assume that the activation function is $\sigma(x)=x$, i.e., the identity map. Let $(\lamj)$ be some node scalings, not necessarily of the form \eqref{eq:lamj}\footnote{Note that we assume none of \cref{assump:data,assump:activation,assump:init_V}
			here.}. Let $\bfX=\bfU\bfD\bfV^\top$ be a reduced singular value decomposition of $\bfX$, where $\bfU$ is an $n\times k$ matrix with orthonormal columns, $\bfD$ is a diagonal $k\times k$ matrix, $\bfV$ is a $d\times k$ matrix with orthonormal columns, and $k\leq \min(n,\din)$ is the rank of $\bfX$. For all $j \in [m]$, the difference between the solution of gradient descent\footnote{with a step-size less than $d(D_{\max}^2\sum_k \lambda_{m,k})^{-1}$, where $D_{\max}$ is the largest entry of $\bfD$.}
/flow $\param_{\infty j}$ and the initialisation $\param_{0 j}$ is given by
		\begin{equation}
			\label{eq:featurelearninglinear-weights}
			\param_{\infty j} - \param_{0 j} = \frac{\sqrt{\lamj}}{\sum_{k=1}^m \lambda_{m,k}}  a_j (\bfbeta_\infty - \bfV \bfV^\top \bfbeta_0)
		\end{equation}
		where $\bfbeta_0 = \sum_{j = 1}^m \sqrt{\lambda_{m,j}} a_j \param_{0j}$, and
		$\bfbeta_{\infty}=\sqrt{d}\, \bfV \bfD^{-1}\bfU^\top \bfy$ is the minimum-norm solution of
		\[
		\mathop{\mathrm{argmax}}_{\bfbeta}\; \frac{1}{2}\|\bfy - \frac{1}{\sqrt{d}} \bfX {\bfbeta}\|^2.
		\]
		The learnt function is
		\begin{equation}
			\label{eq:featurelearninglinear-function}
			f_m(\bx;\paramall_\infty)=\frac{1}{\sqrt{\din}}\bx^\top \left ( \sum_{j = 1}^m \sqrt{\lambda_{m,j}} a_j \param_{\infty j}\right )=\frac{1}{\sqrt{\din}}\bx^\top \left (\bfbeta_{\infty} +(I_{\din}- \bfV \bfV^\top) \bfbeta_0  \right).
		\end{equation}
	\end{theorem}
	The proof of this theorem is given
	in \cref{sec:app:linear} in the Supplementary Material.
	
	\Cref{thm:featurelearninglinear} says that
	along the dimensions spanned by the data, the weight vector of a node $j$ moves
	by a quantity proportional to $\sqrt{\lamj}$ towards the minimum-norm solution. The form of the learnt function
	in \cref{eq:featurelearninglinear-function} implies that the contribution of each hidden node $j$ to the function's output
	is proportional to $\sqrt{\lambda_{m,j}} a_j \param_{\infty j}$. The next theorem analyses
	the asymptotic behaviour of this contribution in the infinite-width limit and the associated feature learning properties. It shows that, under the scaling of \cref{eq:lamj}, both feature learning and non-uniform feature learning occur if and only if $\gamma<1$.

	\begin{theorem}[Feature learning - linear activation]
		\label{thm:featurelearninglinear-infinite}
		Assume the setting of \cref{thm:featurelearninglinear}, and that \cref{assump:init_V} holds.
\begin{itemize}
\item Under the node scalings \eqref{eq:lamj}, both feature learning (\Cref{def:feature}) and non-uniform feature learning (\Cref{def:nonuniformfeature}) hold if and only if $\gamma<1$.
\item Under the mean-field scaling, feature learning (\Cref{def:feature}) holds, but non-uniform feature learning (\Cref{def:nonuniformfeature}) does not.
\end{itemize}
The contribution from the $j$th node is marginally normally distributed with
		\[
		\sqrt{\lambda_{m,j}} a_j \param_{\infty j}
		\;\sim\;
		\mathcal N\left(\frac{\lamj}{\sum_k \lambda_{m,k} } \bfbeta_{\infty},\; \lamj \left(I_\din-\frac{\lamj}{\sum_k \lambda_{m,k}} \bfV\bfV^\top \right)\right)
		\quad\text{for all $\nnodes \geq 1$ and $1\leq j\leq\nnodes$}.
		\]
This implies that, under the mean-field scaling, contributions are asymptotically vanishing. Under the scaling \eqref{eq:lamj}, we have
		\begin{equation}
			\sqrt{\lambda_{m,j}} a_j \param_{\infty j}
			\;\cvdto\;
			\mathcal N\left((1-\gamma)\tillam_j \bfbeta_{\infty},\; (1-\gamma)\tillam_j \left(I_\din-(1-\gamma)\tillam_j\bfV\bfV^\top\right)\right)
			\quad\text{as $\nnodes\to\infty$}.
			\label{eq:convergencewinf}
		\end{equation}
		As a result, if $\gamma<1$ and $\tillam_j>0$, the contribution from the $j$th node to the output is non-vanishing in the infinite-width limit.
	\end{theorem}

	Before moving on to the case of the nonlinear activation function, we analyse the consequence of pruning
	nodes of a linear network based on the scaling parameters $\lamj$. The result of this analysis
	is given in the following proposition, which shows the benefit of our asymmetric scaling in pruning.
	\begin{proposition}
		\label{thm:featurelearninglinearpruning}
		Assume the setting of \cref{thm:featurelearninglinear}, and that \cref{assump:init_V} holds. Assume also the node scaling \eqref{eq:lamj}. Let $\rho \in(0,1)$.
		Consider the following pruned network that is obtained by keeping
		the $\lfloor\rho m\rfloor$ hidden nodes with largest scalings $\lamj$
		and pruning the other nodes:
		\[
		\widetilde f_{m,\rho}(\bx;\paramall_\infty)
		=
		\left(
		\sum_{j = 1}^{\lfloor \rho m \rfloor} \sqrt{\lambda_{m,j}} a_j \frac{\param_{\infty j}^\top\, \bx}{\sqrt{d}}
		\right)
		=
		\frac{1}{\sqrt{\din}}\bx^\top
		\left(\sum_{j = 1}^{\lfloor \rho m \rfloor} \sqrt{\lambda_{m,j}} a_j \param_{\infty j}\right).
		\]
		Then, for all $\varepsilon>0$, we have the following bound on the pruning error:
		\[
		\Pr\left(\left|\widetilde f_{m,\rho}(\bx;\paramall_\infty)-f_{m}(\bx;\paramall_\infty)\right|>\varepsilon\right)
		\;\leq\;
		\frac{\|\bx\|}{\varepsilon\sqrt{\din}} \left(\left(\left\| \bfbeta_\infty \right\|+\sqrt{\din}\right)\left(\sum_{j>\lfloor \rho m \rfloor} \lambda_{m,j}\right)+\sqrt{\din\sum_{j>\lfloor \rho m \rfloor} \lambda_{m,j}}
		\right).
		\]
		Since  $\sum_{j>\lfloor \rho m \rfloor} \lambda_{m,j}\to \gamma(1-\rho)$ as $\nnodes\to\infty$,
		the above bound implies the following. If $\gamma=0$ (no symmetric part), then
		$\Pr(|\widetilde f_{m,\rho}(\bx;\paramall_\infty)-f_{m}(\bx;\paramall_\infty)|>\varepsilon)\to 0$
		and so the network can be compressed to a smaller network via pruning. Otherwise, the pruning error is controlled by the proportion of the symmetric
		part $\gamma$ in the infinite-width limit:
		\begin{align*}
			\lim_{m\to\infty}\Pr\left(\left|\widetilde f_{m,\rho}(\bx;\paramall_\infty)-f_{m}(\bx;\paramall_\infty)\right|>\varepsilon\right)
			\;\leq\;
			\frac{\|\bx\|}{\varepsilon\sqrt{\din}} \left(\left(\left\| \bfbeta_\infty \right\|+\sqrt{\din}\right)\gamma(1-\rho)+\sqrt{\din \gamma(1-\rho)}
			\right).
		\end{align*}
	\end{proposition}

\subsection{Nonlinear activation function}
	\label{sec:featurelearningnonlinear}
		
We now analyse feature learning in our model when the activation function is nonlinear. Our analysis assumes
the following two changes in our setup:
\begin{assumption}[Zeroed initialisation]
	\label{assump:zeroed-initialisation}
	We assume the model has the following form:
	\[
		f_\nnodes(\bx;\paramall) =
		\left(\sum_{j = 1}^{\nnodes} \sqrt{\lamj} \aj \sigma(Z_j(\bx;\paramall))\right)
		-
		\left(\sum_{j = 1}^{\nnodes} \sqrt{\lamj} \aj \sigma(Z_j(\bx;\paramall_0))\right).
	\]
	That is, we subtract from the original model, a duplicate version whose parameters are set to $\paramall_0$
	and are unchanging throughout training. This is a commonly used simplification
	in theoretical analyses of neural networks. It ensures that at initialisation, the model satisfies
	\[
		f_\nnodes(\bx;\paramall_0) = 0 \text{ for all } \bx \in \R^d.
	\]
\end{assumption}
\begin{assumption}[Random outputs]
	\label{assump:random-outputs}
	We regard the outputs $y_1,\ldots,y_n$ as random variables, so that the probabilities in \cref{def:fl,def:nofl}
	refer to the randomness of the outputs as well. We further assume that
	$y_1,\ldots,y_n$ are independent and continuous (i.e., the distribution of $y_i$
	has a density with respect to Lebesgue measure), and that they are also independent from
	$\paramall_0$ and the $\aj$'s.
	Note that we still treat the inputs $x_1,\ldots,x_n$ as deterministic variables.
	This assumption is met if, for example, there exists a true generating function $f^*$ such that
	$y_i = f^*(x_i) + \epsilon_i(x_i)$ where the $\epsilon_i(x_i)$ are independent continuous noise variables.
\end{assumption}

Our analysis on feature learning considers gradient descent with a learning rate $\eta$ that
does not depend on~$\nnodes$ just as we did in \cref{sec:globalconvergenceGD}. This is in contrast to parameterisations such as mean-field parameterisation where
the learning rate has a scaling dependent on $\nnodes$.
However, let us note that key reasoning steps in our proofs also apply to
such $\nnodes$-dependent learning rates after minor modifications, allowing  us to recover existing feature-learning results as we will explain shortly after
\cref{thm:featurelearning-smooth}. The full proofs of all the theorems in this subsection
are given in \cref{sec:featurelearning-proof} in the Supplementary Material. Also, the theorems in this subsection have
counterparts that hold for the ReLU activation function. \cref{sec:featurelearning-relu} in the Supplementary Material
contains those feature-learning results for the ReLU case.

We show that
if the activation function is continuously differentiable and its derivative is always positive, as in the case of sigmoid,
then after the first gradient-descent step, (i) both feature learning and non-uniform feature
learning in \cref{def:feature,def:nonuniformfeature} occur (\cref{thm:featurelearning-smooth})
and (ii) the squared norm of each weight vector $\param_j$ changes almost surely by the amount $\Omega(\tillam_1)$
in the infinite-width limit (\cref{thm:weight-norm-square-change-smooth}).
\begin{theorem}
	\label{thm:featurelearning-smooth}
	Suppose that \cref{assump:data,assump:init_V,assump:zeroed-initialisation,assump:random-outputs} hold.
	Suppose also  that $\gamma < 1$ and that the activation function $\sigma$ is continuously differentiable with $\sigma'(x) > 0$ for all $x \in \R$. Let $i \in [n]$. Then, both feature learning
	and non-uniform feature learning occur after the first gradient-descent step with respect to the input $\bx_i$ in the almost-sure sense, i.e., the following inequalities hold almost surely:
	\begin{align*}
		& \liminf_{\nnodes \to \infty} \frac{\sum_{j = 1}^{\nnodes} \lamj \Big(\sigma(Z_j(\bx_i;\paramall_1)) - \sigma(Z_j(\bx_i;\paramall_0))\Big)^2}{\sum_{j = 1}^\nnodes \lamj \Big(\sigma(Z_j(\bx_i;\paramall_0))\Big)^2} > 0
		\\
		& \qquad \text{and}\qquad \liminf_{\nnodes \to \infty} \frac{\max_{j \in [\nnodes]} \lamj \Big(\sigma(Z_j(\bx_i;\paramall_1)) - \sigma(Z_j(\bx_i;\paramall_0))\Big)^2}{\sum_{j=1}^\nnodes \lamj \Big(\sigma(Z_j(\bx_i;\paramall_0))\Big)^2} > 0.
	\end{align*}
\end{theorem}

\paragraph{Sketch of the proof.}
Since non-uniform feature learning implies feature learning, we prove the former only.
The denominator $\sum_{j=1}^\nnodes \lamj (\sigma(Z_j(\bx_i;\paramall_0)))^2$ in the condition for
non-uniform feature learning converges to a
positive finite value almost surely as $m$ tends to $\infty$. Thus, it is enough to prove that
\[
	\liminf_{\nnodes \to \infty} \left(\max_{j \in [\nnodes]}\lamj\left(\sigma(Z_j(\bx_i;\paramall_1)) - \sigma(Z_j(\bx_i;\paramall_0))\right)^2\right) > 0
	\quad\text{almost surely},
\]
which is implied by
\begin{equation}
	\label{eqn:featurelearning-smooth:proof-sketch:0}
	\liminf_{\nnodes \to \infty}\left(\sigma(Z_1(\bx_i;\paramall_1)) - \sigma(Z_1(\bx_i;\paramall_0))\right)^2 > 0
	\quad\text{almost surely}.
\end{equation}
Note that
the limits from above are not redundant since $\paramall_1$ depends on $\nnodes$.
The sufficient condition in \cref{eqn:featurelearning-smooth:proof-sketch:0} can be simplified further. The assumptions of the theorem allow us
to use the inverse function theorem to deduce that the condition in \cref{eqn:featurelearning-smooth:proof-sketch:0} holds whenever
\begin{equation}
	\label{eqn:featurelearning-smooth:proof-sketch:1}
	\liminf_{\nnodes \to \infty}\left(Z_1(\bx_i;\paramall_1) - Z_1(\bx_i;\paramall_0)\right)^2 > 0
	\quad\text{almost surely.}
\end{equation}
The majority of the detailed proof concerns proving the condition in \cref{eqn:featurelearning-smooth:proof-sketch:1}. To that end, we compute the following $\nnodes$-independent
lower bound: for all $\nnodes$,
\begin{align}
	\label{eqn:featurelearning-smooth:proof-sketch:2}
	\Big(Z_1(\bx_i;\paramall_1) - Z_1(\bx_i;\paramall_0)\Big)^2
	& {} =
	\frac{\eta^2\lambda_{\nnodes,1}}{d^2}\left(
		\sum_{i' = 1}^n y_{i'}
		\left(\sigma'\left(\frac{\param_{01}^\top \bx_{i'}}{\sqrt{d}}\right)
		\bx_{i'}^\top \bx_i\right)
	\right)^2
	\\
	\label{eqn:featurelearning-smooth:proof-sketch:3}
	& {} \geq
	\frac{\eta^2(1-\gamma)\tillam_{1}}{d^2}\left(
		\sum_{i' = 1}^n y_{i'}
		\left(\sigma'\left(\frac{\param_{01}^\top \bx_{i'}}{\sqrt{d}}\right)
		\bx_{i'}^\top \bx_i\right)
	\right)^2.
\end{align}
Then, we show that the right-hand side is almost surely positive, which implies the conclusion
of the theorem. The justification of this almost-sure positivity relies on \cref{assump:random-outputs,assump:data}, the positivity of $\sigma'$, and the assumption that $\gamma < 1$. Concretely,
the assumption on $\gamma$ implies that $\eta^2(1-\gamma)\tillam_{1} / d^2 > 0$,
while \cref{assump:random-outputs,assump:data} and the positivity of $\sigma'$ imply that the squared sum on the right-hand side is almost surely a positive random variable.

We point out that the proof also reveals that if we were to allow the
learning rate $\eta$ to depend on $\nnodes$ in a particular way, then feature
learning could occur even for $\gamma = 1$ (although the condition for
non-uniform feature learning might fail). For instance, when $\eta = \sqrt{\nnodes}$ and $\gamma = 1$,
\cref{eqn:featurelearning-smooth:proof-sketch:2} still holds,
and its right-hand side is almost surely positive because
$\eta^2 \lambda_{\nnodes,1}/d^2 = 1/d^2 > 0$. This almost-sure positivity
and the assumptions on $\sigma$ then imply that
\begin{equation}
\label{eqn:featurelearning-smooth:proof-sketch:4}
\mathbb{E}[(\sigma(Z_1(\bx_i;\paramall_1)) - \sigma(Z_1(\bx_i;\paramall_0)))^2] > 0,
\end{equation}
which in turn ensures that the condition for feature learning holds, i.e., almost surely,
\begin{align*}
	\liminf_{\nnodes \to \infty}
	\frac{
		\sum_{j = 1}^{\nnodes} \lamj \big(\sigma(Z_j(\bx_i;\paramall_1)) - \sigma(Z_j(\bx_i;\paramall_0))\big)^2
	}{
		\sum_{j = 1}^\nnodes \lamj \big(\sigma(Z_j(\bx_i;\paramall_0))\big)^2
	}
	& =
	\frac{
		\lim_{\nnodes \to \infty} \frac{1}{\nnodes} \sum_{j=1}^\nnodes
	    \big(\sigma(Z_j(\bx_i;\paramall_1)) - \sigma(Z_j(\bx_i;\paramall_0))\big)^2
	}{
		\lim_{\nnodes \to \infty} \frac{1}{\nnodes} \sum_{j=1}^\nnodes
		\sigma(Z_j(\bx_i;\paramall_0))^2
	}
	\\
	&
	{} =
	\frac{
		\mathbb{E}\big[
	    \big(\sigma(Z_1(\bx_i;\paramall_1)) - \sigma(Z_1(\bx_i;\paramall_0))\big)^2\big]
	}{
		\mathbb{E}[
		\sigma(Z_1(\bx_i;\paramall_0))^2
		]
	}
	> 0.
\end{align*}
However, note that in this case,
by \cite[Theorem 3.3]{Yang2021}, the model becomes unstable and could blow up (cf. the footnote regarding \cref{def:fl}).
Also, if $\eta = \nnodes$ and we were to use $1/\nnodes^2$ for $\lamj$ (i.e., mean-field
parameterisation), the essentially same argument would apply and lead to
feature learning,
although such $\lamj$'s are not covered by our setup. Just as before,
\cref{eqn:featurelearning-smooth:proof-sketch:2} would hold
and its right-hand side would be almost surely positive in this case,
because $\eta^2 \lambda_{\nnodes,1}/d^2 = 1/d^2 > 0$. This almost-sure positivity implies
the inequality in \cref{eqn:featurelearning-smooth:proof-sketch:4},
which then gives the condition for feature learning because the
$\liminf$
formula in the definition of feature learning can be simplified to
the same ratio of expectations as the one before.

\begin{theorem}
\label{thm:weight-norm-square-change-smooth}
Suppose \cref{assump:data,assump:init_V,assump:zeroed-initialisation,assump:random-outputs} hold.
Suppose also that the activation function $\sigma$ is continuously differentiable
and satisfies $\sigma'(x) > 0$ for all $x \in \R$. Then, for all $j$, the following holds almost surely:
\begin{equation}
\label{eqn:weight-norm-square-change-smooth:0}
	\liminf_{\nnodes \to \infty}
	\left\|\left.\nabla_{\param_{tj}} L(\paramall_t)\right|_{t=0}\right\|^2
	\geq
	\frac{(1-\gamma)\tillam_j}{d}
	\left|\sum_{i = 1}^n \sum_{i'=1}^n
		y_iy_{i'}
		\left(\bx_{i}^\top \bx_{i'}
			\sigma'\left(\frac{\param_{0j}^\top \bx_i}{\sqrt{d}}\right)
			\sigma'\left(\frac{\param_{0j}^\top \bx_{i'}}{\sqrt{d}}\right)
		\right)
	\right|.
\end{equation}
In particular, if $\gamma < 1$ and $\tillam_j > 0$, then the lower bound is positive almost surely so that
we have
\[
	\liminf_{\nnodes \to \infty}
	\left\|\left.\nabla_{\param_{tj}} L(\paramall_t)\right|_{t=0}\right\|^2 > 0
\]
with probability $1$.
\end{theorem}

\paragraph{Sketch of the proof.}
The lower bound in \cref{eqn:weight-norm-square-change-smooth:0} is obtained by
a relatively straightforward calculation since the calculation of the gradient is simplified due to \cref{assump:zeroed-initialisation}. The second part of the theorem follows
from the fact that when conditioned on $\param_{0j}$, the absolute value on the right-hand side of \cref{eqn:weight-norm-square-change-smooth:0} is a continuous
random variable that depends only on the $y_i$'s, and so it is strictly positive with probability one.

\section{Experiments}
\label{sec:experiments}

We use here a (smooth) swish activation function  $\sigma(z)=z/(1+e^{-z})$. We obtained quantitatively similar results with the ReLU activation function; see \cref{sec:additionalexperiments-relu} in the Supplementary Material.

\subsection{Simulated data.}

\subsubsection{Illustration of the main results}
We first illustrate our theory on simulated data.\footnote{The code can found at \url{https://github.com/juho-lee/asymmetrical_scaling}} We generate $n=100$ observations where for $i \in [n]$, $\bx_i$ is $d=50$ dimensional and sampled uniformly on the unit sphere, and
$y_i = \frac{5}{d} \sum_{j=1}^d \sin (\pi x_{i,j}) + \varepsilon_i$ where $\varepsilon_i \iid \mathcal{N}(0, 1)$. We use the FFNN of \cref{sec:statisticalmodel}, with the swish activation function, $\nnodes=2000$ hidden nodes, and $\lamj$ as in \cref{eq:lamj,eq:tildelambdazipf}.
We consider the four values $(\gamma,\alpha)\in\{(1,-), (0.5,0.7), (0.2, 0.5), (0, 0.4)\}$. For each setting, we run GD with a learning rate of 1.0 for 50\,000 steps, which is repeated five times to get average results.
We summarise the results in \cref{fig:simulated}, which shows the training error and the evolution of the weights,   NTG, and minimum eigenvalue of the   NTG as a function of the GD iterations. We see a clear correspondence between the theory and the empirical results. For $\gamma>0$, GD achieves near-zero training error. The minimum eigenvalue and training rates increase with $\gamma$. For $\gamma=1$, we have the highest minimum eigenvalue and the fastest training rate, but there is no/very little feature learning: the weights and the   NTG do not change significantly over the GD iterations. When $\gamma<1$, there is clear evidence of feature learning: both the weights and the   NTG change significantly over time; the smaller the $\gamma$ and $\alpha$, the more feature learning arises.

\begin{figure*}
	\centering
	\includegraphics[width=0.24\linewidth]{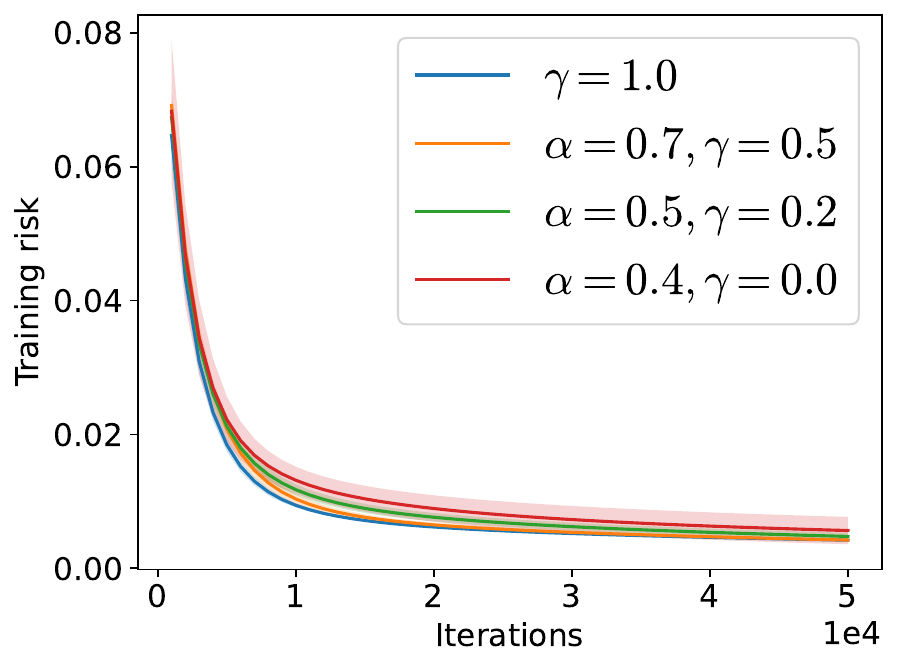}
	\includegraphics[width=0.24\linewidth]{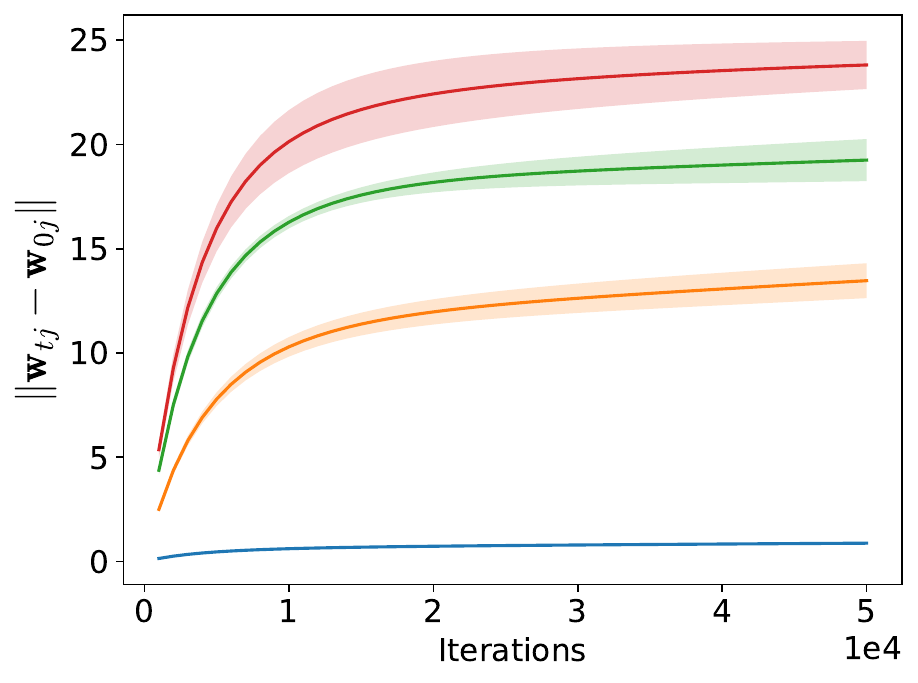}
	\includegraphics[width=0.24\linewidth]{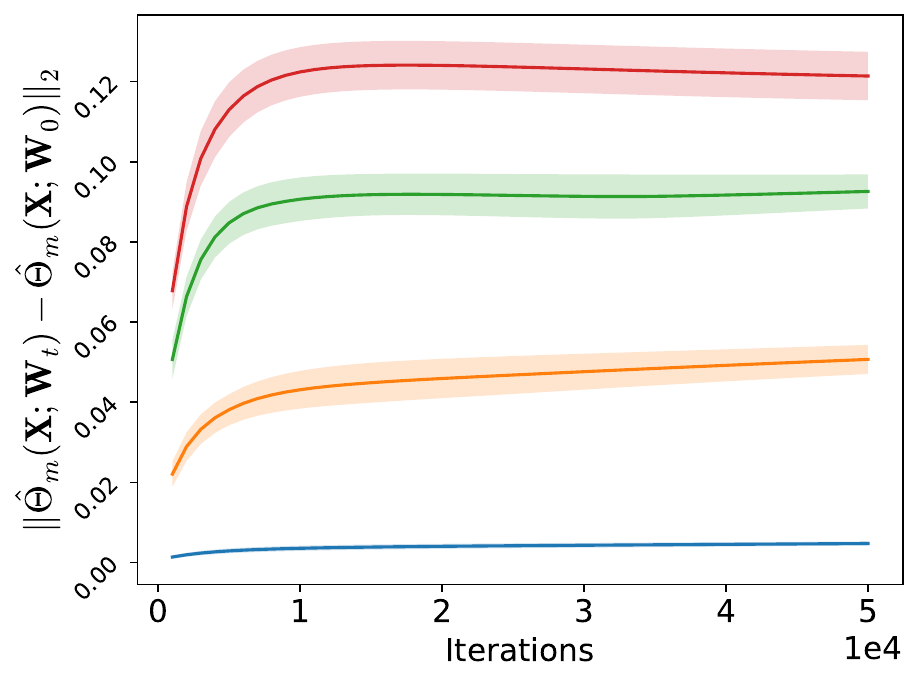}
	\includegraphics[width=0.24\linewidth]{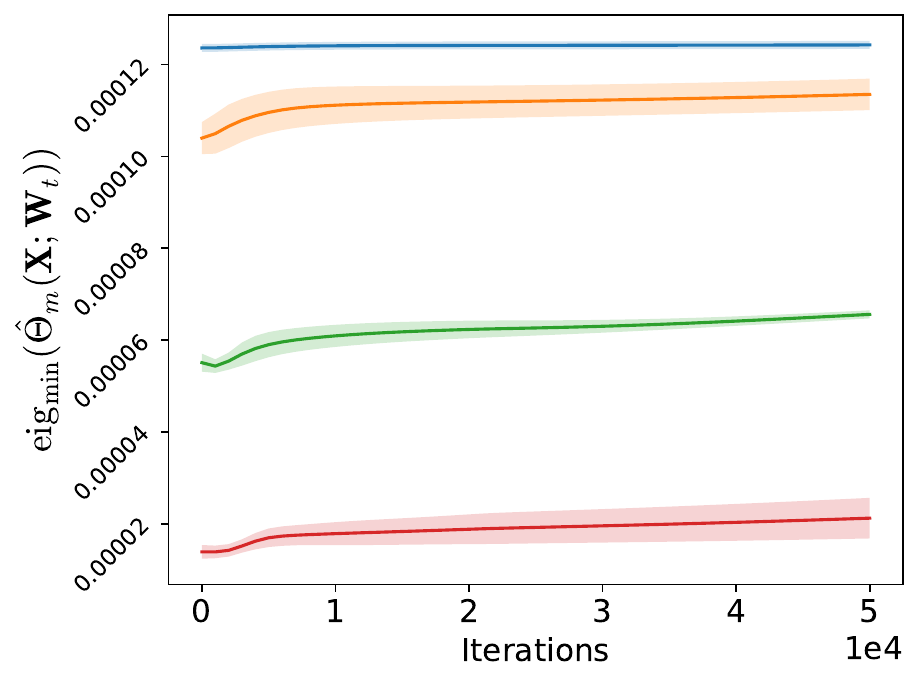}
	\caption{Results on simulated data. From left to right, 1) training risks, 2) differences in weight norms $\Vert \mathbf{w}_{tj} - \mathbf{w}_{0j}\Vert$ with the $j$'s being those neurons which have maximal differences at the end of the training, 3) differences in  NTG matrices, and 4) minimum eigenvalues of   NTG matrices.}
	\label{fig:simulated}
\end{figure*}

\subsubsection{Improved generalization: Single ReLU experiment.}
To illustrate the benefits of asymmetrical scaling, we consider here the scenario where the function to learn is a single-unit ReLU, a setting known to be challenging for the lazy-training regime \citep{malach2021quantifying}. Consider the following data-generating process:
\begin{eqnarray*}
	X &\sim& \mathcal{N}(0, I_d) \\
	Y & = & \sigma (w_0^T X)
\end{eqnarray*}
where $w_0 = (1, .., 1) \in \mathbb{R}^d$ and $\sigma$ is the ReLU activation function. In this experiment, we sample a training dataset of $n=100$ samples in dimension $d=10$. As previously, we train fully connected neural networks composed of a single hidden layer, with different node-scaling strategies. The width in all models is set to $P=2000$. The generalization error is computed on 5000 samples from the same single ReLU data-generating process. All experiments are repeated 5 times, the training and testing datasets are resampled for each run. The results are reported in Figure \ref{fig:simul_alignment}. We notice that at the end of the training, the error is near zero for all scalings. Examining the error on the test set, one can see that, as expected, the symmetrical $\gamma=1$ NTK scaling generalizes poorly in this setting. The asymmetrical scalings, on the other hand, perform significantly better, illustrating the benefits of this strategy in terms of generalization error. For completeness, we also report the results using the standard pytorch initialization. We can see that in this setting of a sparse generating process, the standard initialization also generalizes poorly compared to the asymmetrical model.

\begin{figure}
	\centering
	\includegraphics[width=0.4\linewidth]{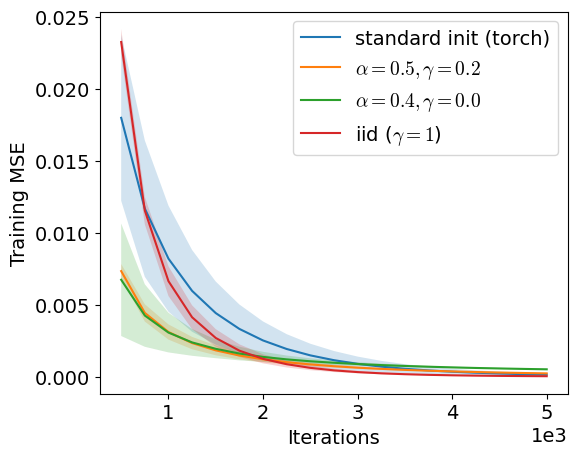}
	\includegraphics[width=0.4\linewidth]{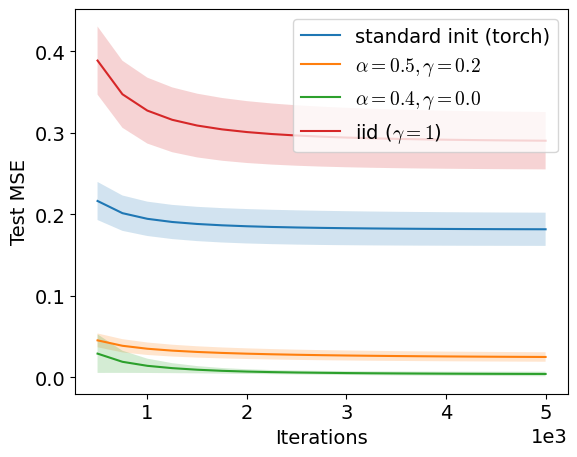}
	\caption{Results on simulated data from a single ReLU unit. Evolution of the training error (left) and test error (right) as a function of the training iteration.}
	\label{fig:simul_alignment}
\end{figure}

\subsection{Real data}

\subsubsection{Regression}
We also validate our model on four %real-world
regression datasets from the UCI repository\footnote{\url{https://archive.ics.uci.edu/ml/datasets.php}}: \texttt{concrete} ($(n,d)\,{=}\,(1030,9)$), \texttt{energy} ($(n,d)\,{=}\,(768,8)$), \texttt{airfoil} ($(n,d)\,{=}\,(1503,6)$), and \texttt{plant} ($(n,d)\,{=}\,(9568,4)$).
We split each dataset into training (40\%) , test (20\%), and validation sets (40\%), and use the validation set to test transfer learning. We use the same parameters as for the simulated data, but train our FFNNs for 100 000 steps in each run.
To further highlight the presence of feature learning in our model, we test the transferability of features learnt from our networks as follows. We first split the validation set into a held-out training set (50\%) and a test set (50\%), and extract features of the held-out training set using the FFNNs trained on the original training set. Features are taken to be the outputs of the hidden layers, so each data point in the validation set is represented with a $\nnodes=2000$ dimensional vector. Then, we sort feature dimensions with respect to feature importance measured as $(\lambda_{m,j}\Vert \mathbf{w}_{tj}\Vert^2)_{j \in [m]}$ and use the top-$k$ of these to train an external model. The chosen external model is a FFNN with a single hidden layer having 64 neurons and ReLU activation, which is trained for 5000 GD steps with a learning rate of 1.0. Our theory suggests that smaller $\gamma$ and $\alpha$ likely lead to better transfer learning.  A subset of our results appears in \cref{fig:regression}; see \cref{appendix_results} for additional results. In line with the simulated data experiments, we observe a stronger presence of feature learning, in terms of weight-norm changes and   NTG changes, for smaller values of $\gamma$ and $\alpha$. Also, we observe that models with smaller values of $\gamma$ have lower risks when a small number of features are used for the transfer. The interpretation is that those models are able to learn a sufficient number of representative features using relatively fewer neurons.

\begin{figure*}
	\centering
	\includegraphics[width=0.24\linewidth]{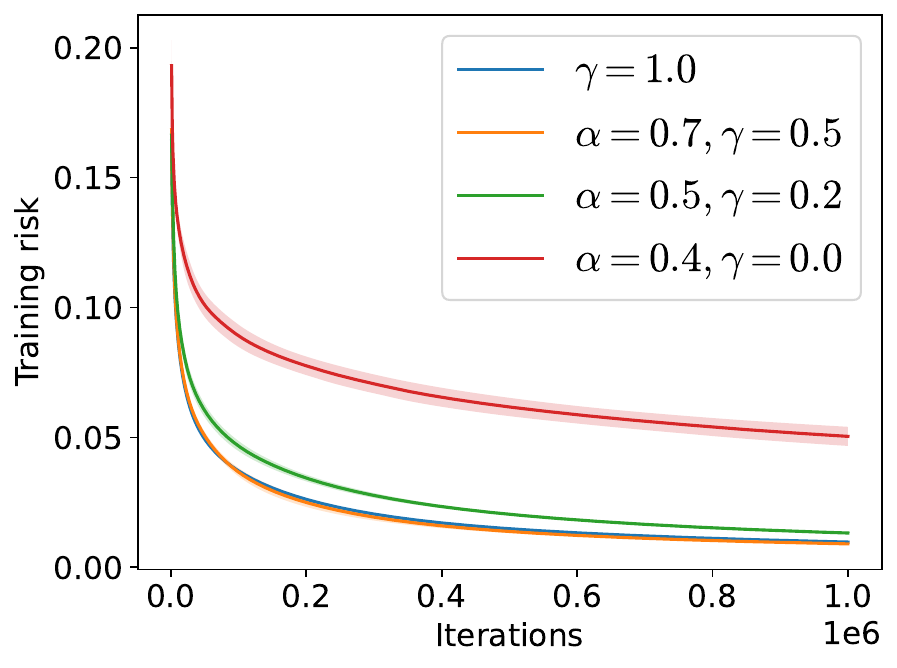}
	\includegraphics[width=0.24\linewidth]{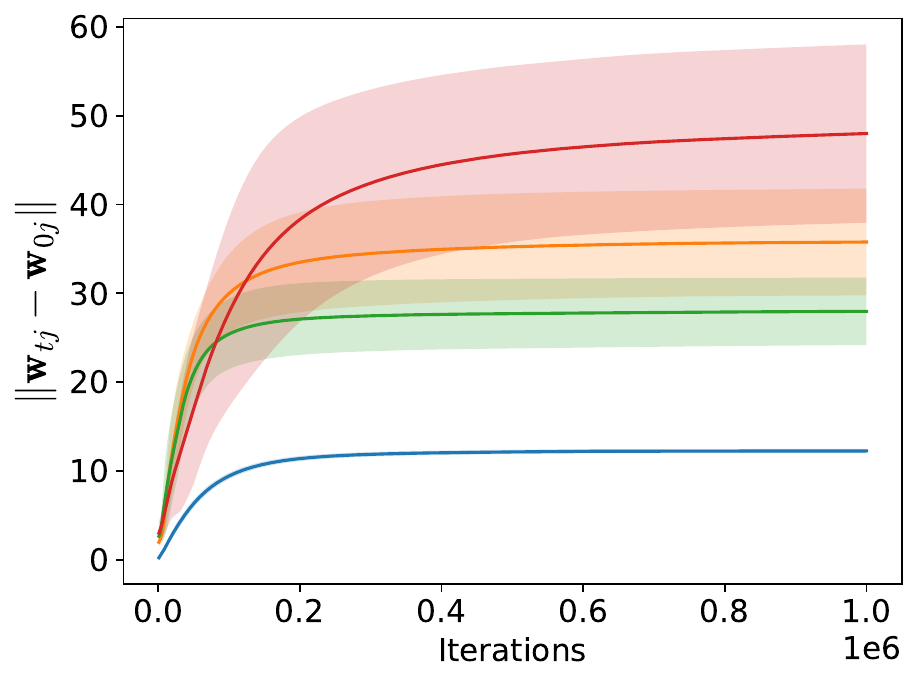}
	\includegraphics[width=0.24\linewidth]{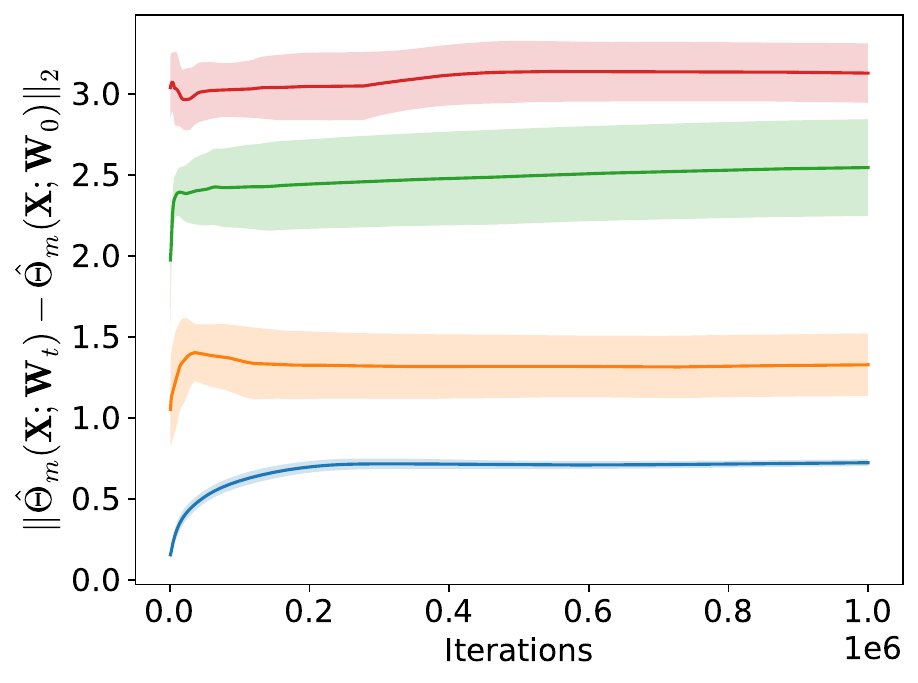}
	\includegraphics[width=0.24\linewidth]{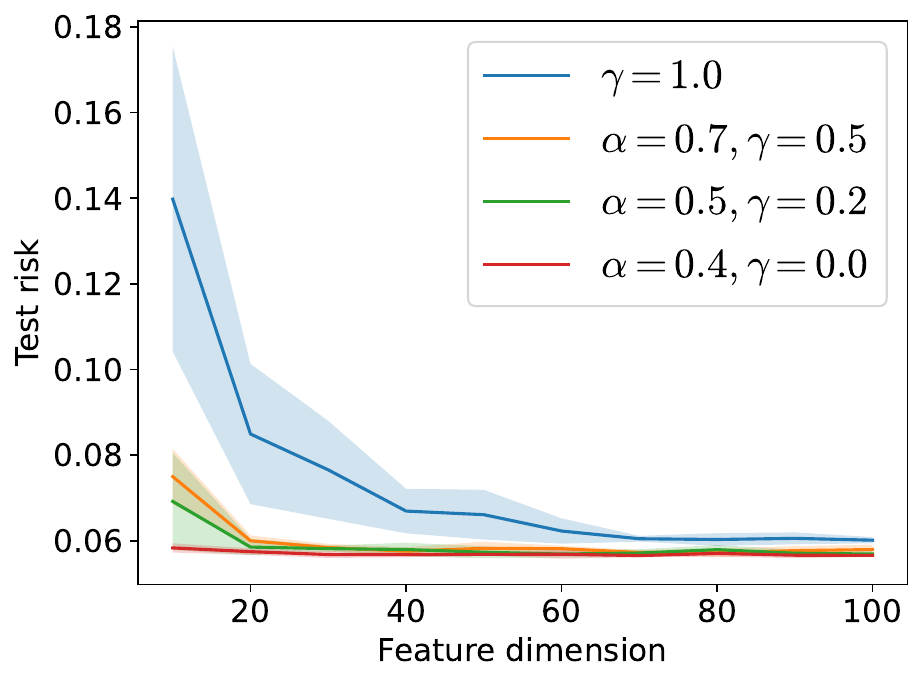}
	\caption{A subset of results for the regression experiments. From left to right, 1) training risks for dataset \texttt{concrete} , 2) differences in weight norms $\Vert \mathbf{w}_{tj} - \mathbf{w}_{0j}\Vert$ with the $j$'s being the neurons having the maximal difference at the end of the training for dataset \texttt{energy}, 3) differences in   NTG matrices for  dataset \texttt{airfoil}, and
		4) test risks of transferred models for dataset \texttt{plant}. }
	\label{fig:regression}
	\vspace{-1em}
\end{figure*}

\subsubsection{Classification}
We apply our model on two image classification tasks. The first is small-scale using the setting assumed in our theory, while the second is larger-scale using a more realistic setting. In addition to the transferability experiment described before, 
we test the prunability of the FFNNs. We gradually prune hidden nodes with small feature importance and measure risks after pruning. Feature importance is measured as above. Our theory suggests that models with smaller $\gamma$ and/or $\alpha$ values are likely to be more robust with respect to pruning, as long as $\gamma<1$. \citet{Wolinski2020a} had similar empirical findings on the benefits of asymmetrical scaling for network pruning when $\gamma=0$. \\[0.5ex]
\textbf{MNIST.} We take a subset of size 5000 from the MNIST dataset and train the same models used in the previous experiments. We also test pruning and transfer learning, where we use an additional subset of size 5000 to train an external FFNN having a single hidden layer with 128 nodes. To match our theory, instead of using cross-entropy loss, we use the MSE loss by treating one-hot class labels as continuous-valued targets. The outputs of the models are 10 dimensional, so we compute the   NTG matrices using only the first dimension of the outputs. In general, we get similar results in line with our previous experiments. The pruning and transfer learning results are displayed in \cref{fig:mnist}. For other results, see \cref{fig:mnist_all} in the Supplementary Material.\\[0.5ex]
\textbf{CIFAR.} We consider a more challenging image classification task of CIFAR--10 and CIFAR--100. The datasets have 60\,000 images with 50\,000 for used training and the rest used for testing. There are, respectively, 10 and 100 different classes. We show the benefits of asymmetrical node scaling hold for this more challenging problem. In many applications, one uses a  large model pre-trained on a general task and then performs fine-tuning or transfer learning to adapt it to the task at hand. We implement this approach on a ResNet-18 model, pre-trained on ImageNet data. With this model, we transform each original image to a vector of dimension $512$. We then train shallow FFNNs as described in \cref{sec:statisticalmodel}, with $\nnodes=2000$ and output dimension $10$ (resp. $100$). This experiment differs from previous cases as 1) we use stochastic GD with a mini-batch size of $64$ instead of full batch GD; 2) we use cross-entropy loss instead of MSE; and 3) both layers are trained. All experiments are run five times, with the learning rate $5.0$. \Cref{fig:cifar100} shows the pruning results for the same four values of pairs $(\gamma, \alpha)$ as above, for CIFAR--100. Similar results are obtained for CIFAR--10; see \cref{appendix_results}. Similar conclusions as before hold here, even though the theory does not apply directly.

\begin{figure*}
	\centering
	\includegraphics[width=0.24\linewidth]{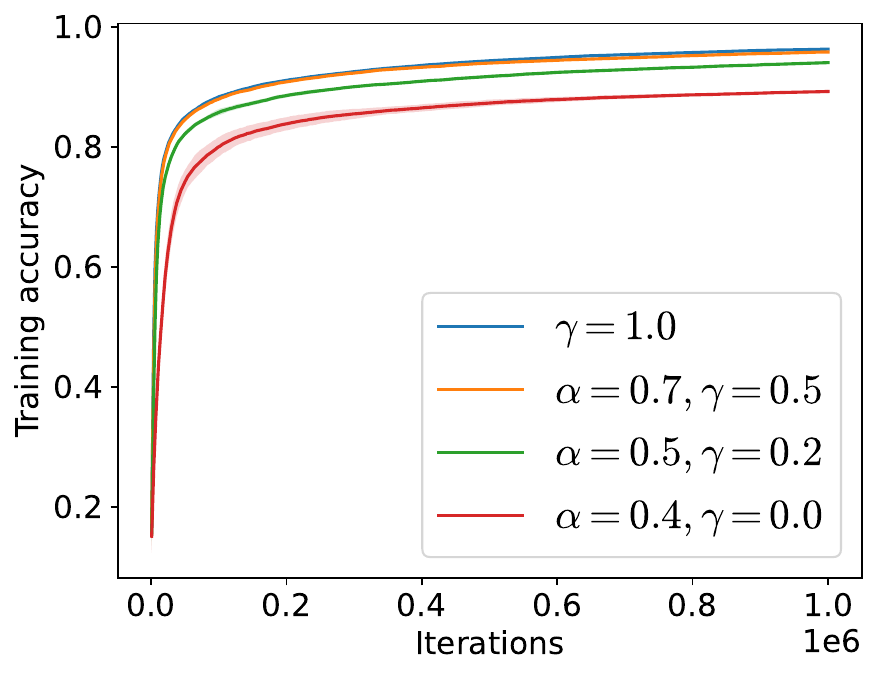}
	\includegraphics[width=0.24\linewidth]{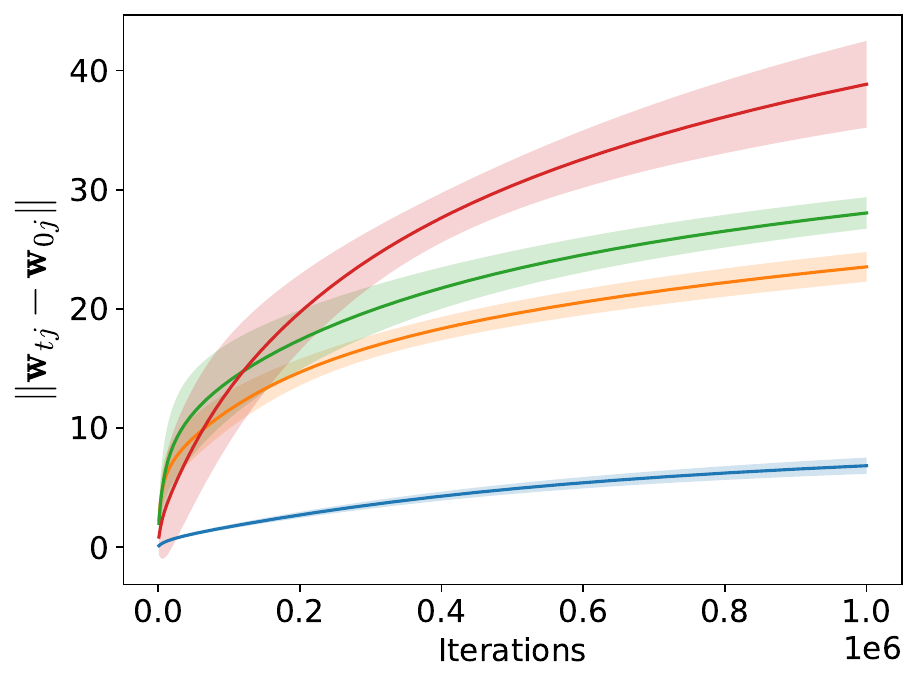}
	\includegraphics[width=0.24\linewidth]{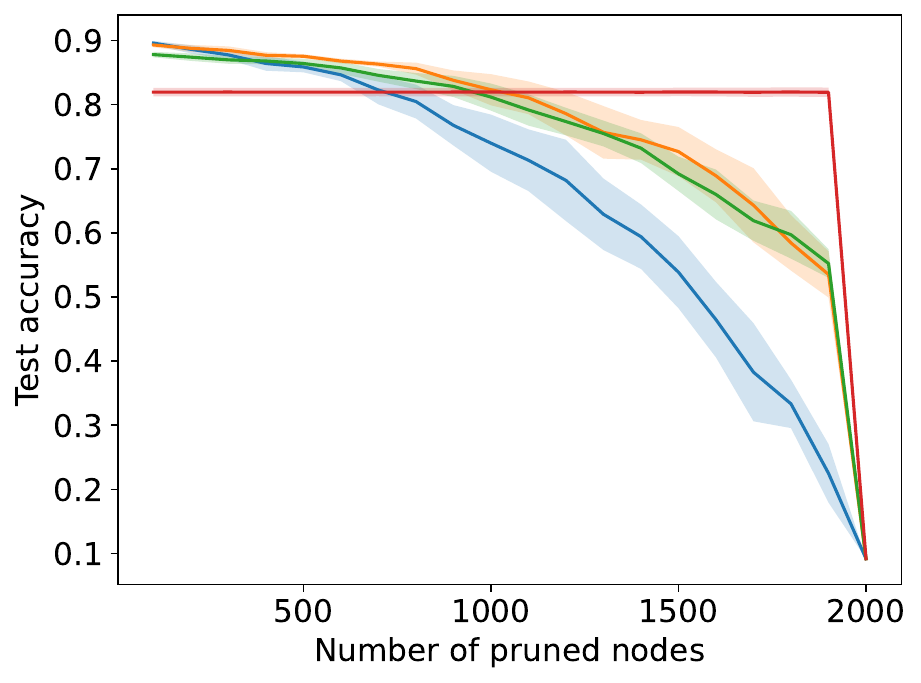}
	\includegraphics[width=0.24\linewidth]{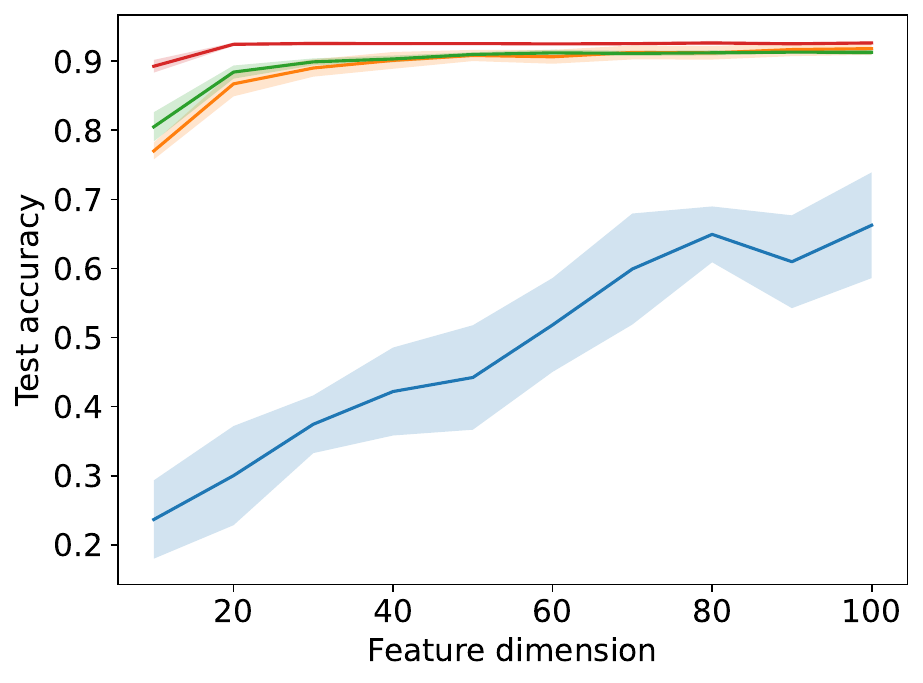}
	\caption{A subset of results for MNIST dataset. From left to right, 1) training risks, 2) differences in weight norms, 3) test accuracies of pruned models, and 4) test accuracies of transferred models.}
	\label{fig:mnist}
	\vspace{-0.5em}
\end{figure*}

\begin{figure*}
	\centering
	\includegraphics[width=0.24\linewidth]{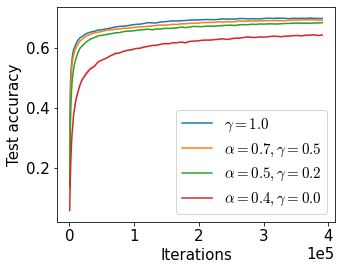}
	\includegraphics[width=0.24\linewidth]{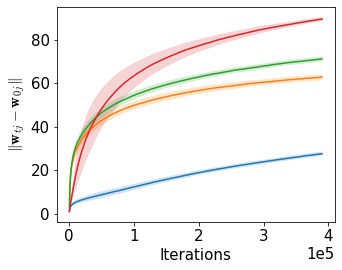}
	\includegraphics[width=0.24\linewidth]{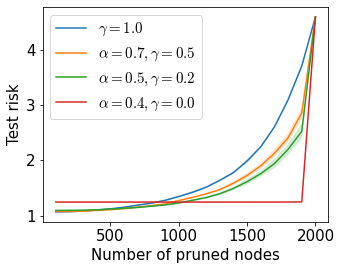}
	\includegraphics[width=0.24\linewidth]{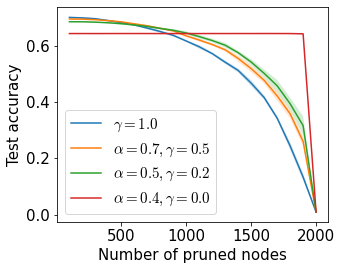}
	\caption{Results for \texttt{CIFAR--100}. From left to right, 1) test risk through training, 2) differences in weight norms $\Vert \mathbf{w}_{tj} - \mathbf{w}_{0j}\Vert$ with the $j$'s being the neurons having the maximal difference at the end of training, 3) test risks of pruned models, and 4) test accuracies of pruned models.}
	\label{fig:cifar100}
	\vspace{-1em}
\end{figure*}

\section{Discussion, limitations and further work}

We have shown that under an asymmetrical scaling of the nodes of a neural network,
it is possible to achieve both zero training error and feature learning, when the width of the neural network is sufficiently large.
We considered two definitions of feature learning. The first definition is a minor generalisation of the notion of feature learning from
\cite{Yang2021}, and it is defined as a change in the feature map.
We proposed a second definition, called non-uniform feature learning, which
additionally requires that the contributions of some individual nodes remain non-negligible in the asymptotic limit. We showed that under our asymmetric scaling and additional conditions, both definitions hold, whereas for the standard NTK, neither does,
and for the mean field, only the standard definition holds. We demonstrate empirically that having non-uniform feature learning is particularly important when we consider transfer learning and pruning.
Our definitions of feature learning relate to the change in the feature map.
As already noted by \cite{Yang2021}, it is a relatively weak definition of feature learning, as it does not connect the weight change with the generalisation properties.
Experimentally, we found that in some case (e.g. single ReLU), the asymmetrical, unpruned network provides the best test error,
while in others (MNIST and CIFAR), the unpruned symmetrical scaling gave the best test accuracy.
An interesting avenue of research is to investigate theoretically the generalisation properties of such asymmetrical scaling. We note that the approaches used for the symmetric NTK~\cite{Arora2019a}, which rely on the limiting kernel, cannot be applied to our setting, due to the evolving kernel.

In this article, we have assumed an iid Gaussian initialisation for the weights (\cref{assump:init_V}), which is a standard assumption in the analysis of large-width neural networks \cite{Du2019,Du2019a,Oymak2020,Nguyen2021}. Our results rely on the fact that the minimum eigenvalue $\kappa_n$ of the mean NTK at initialisation is strictly positive; this result was demonstrated by \cite[Proposition F.1]{Du2019a} under the iid Gaussian initialisation. An interesting direction of research would be to investigate whether the results derived in this paper hold under other, possibly non-iid, initialisation schemes. In particular, the case of orthogonal initialisations would be of particular interest \cite{Hu2020,Huang2021}.

The asymmetrical parameterisation in \Cref{eq:lamj} is rather general, and only requires the $\widetilde\lambda_j$ to be summable. A natural default choice, taken in this article, is to take a power function $\propto j^{-1/\alpha}$ where $0<\alpha<1$. Other parameterisations are also possible, such as $\widetilde\lambda_j = 1/K$ for $j=1,\ldots,K$ and 0 otherwise. One could also choose other scalings such as $(e-1)\exp(-j)$ or $C/(j\log^2 (j+1))$.

\section*{Acknowledgements}

We would like to thank Taeyoung Kim for helpful discussions, and the anonymous reviewers for their useful comments that helped improve the paper. HY was supported by the National
Research Foundation of Korea (NRF) grant funded by the
Korean Government (MSIT) (No. RS-2023-00279680). JL acknowledges support from Institute for Information \& communications Technology Planning \& Evaluation(IITP) grant funded by the Korea government(MSIT) (RS-2019-II190075, Artificial Intelligence Graduate School Program(KAIST)).

%\clearpage
\bibliography{scalemixture_ntk}
\bibliographystyle{tmlr}

\newpage

%%%%%%%%%%%%%%%%%%%%%%%%%%%%%%%%%%%%%%%%%%%%%%%%%
%% APPENDIX
%%%%%%%%%%%%%%%%%%%%%%%%%%%%%%%%%%%%%%%%%%%%%%%%%

\clearpage
\appendix

\renewcommand{\theequation}{S.\arabic{equation}} % Adds an S before equations
\renewcommand{\thefigure}{S.\arabic{figure}} % Adds an S before equations
\setcounter{figure}{0}  % Reset figure counter
\setcounter{equation}{0}  % Reset equation counter

{\vbox{\hsize\textwidth
{\LARGE\bf\sffamily Over-parameterised Shallow Neural Networks with Asymmetrical Node Scaling: Global Convergence Guarantees and Feature Learning: Supplementary Material\par} }}
\bigskip

This Supplementary Material is organised as follows. \Cref{sec:results-relu} presents additional results on global convergence and feature learning when the activation function is the (non-smooth) ReLU function. In particular, \cref{thm:upper-bound:relu} states conditions for the global convergence of gradient flow in the ReLU case, and is similar to \cref{thm:upper-bound:smooth-activation} (smooth case) in the main paper.
\cref{sec:discussion-relu} discusses some open problems in our framework when dealing with the ReLU activation function. Useful bounds and identities are presented in \cref{sec:usefulbounds}. \cref{sec:prooflimitingNTG} gives a proof of the proposition regarding the structure of the limiting NTG at initialisation while \cref{sec:secondarythm:init} provides a secondary proposition regarding the minimum eigenvalue of the NTG at initialisation. \cref{sec:secondarythm:dynamics} states and proves secondary lemmas on gradient flow dynamics. \cref{sec:proofboundrelu} and \ref{sec:proofboundsmooth} give details of the main proof for global convergence of gradient flow, respectively for the ReLU and smooth case. The proofs are rather short and mostly build on the secondary lemmas and propositions of \cref{sec:secondarythm:init,sec:secondarythm:dynamics}. \cref{appendix:global-convergence-gd} gives a detailed proof for global convergence of gradient descent in the smooth case. The proof builds on results of convergence of gradient flow. \cref{sec:featurelearning-proof} gives proofs of the feature-learning results for the smooth case in \cref{sec:featurelearning}, and
\cref{sec:featurelearning-relu-proofs} presents proofs of the corresponding feature-learning results for the ReLU case in \cref{sec:featurelearning-relu}. \cref{appendix_results} provides additional experiments to those of \cref{sec:experiments}, under a smooth activation. Finally, \cref{sec:additionalexperiments-relu} provides detailed results on the same experiments as in \cref{sec:experiments}, but with the ReLU activation instead of the Swish activation function used in the main paper.

\newpage

\section{Results for the ReLU activation function}
\label{sec:results-relu}

Although we assume a smooth activation function in the main text of the paper (\cref{assump:activation}), some of the results remain true when we drop this assumption and use the ReLU activation function instead. In this section, we explain these results for ReLU. Throughout the section, we assume a weak derivative $\sigma'(x) = \ind_{\{x > 0\}}$ of the ReLU activation function $\sigma$.

\subsection{Global convergence under gradient flow}

Our global convergence theorem under gradient flow in the main text (\cref{thm:upper-bound:smooth-activation}) has a
counterpart for the ReLU case, which is given below. This counterpart says that when we train the network with the ReLU activation, with high probability,
the loss decays exponentially fast with respect to $\kappa_n$ and the training time $t$, and the weights $\param_{tj}$ and the NTG matrix
respectively change by
\[
	O\left(\frac{n\lamj^{1/2}}{\kappa_n\din^{1/2}}\right)
	\quad\text{and}\quad
	O\left(\frac{n^2\sum_{j=1}^\nnodes \lamj^{3/2}}{\kappa_n\din^{3/2}}+\frac{n^{3/2}\sqrt{\sum_{j=1}^\nnodes \lamj^{3/2}}}{\kappa_n^{1/2} \din^{5/4}}\right).
\]
	\begin{theorem}[Global convergence, gradient flow, ReLU]
		\label{thm:upper-bound:relu}
		Consider $\delta \in (0,1)$. Let
		$D_0 = \sqrt{2C^2 + (2/\din)}$.
		Assume \cref{assump:data,assump:init_V}, and the use of the ReLU activation function. Also, assume $\gamma > 0$ and
		\begin{align*}
			\nnodes
			\geq
			\max \left(
			\frac{2^{3}n\log \frac{4n}{\delta}}{\kappa_n \din},\;
			\frac{2^{25}n^4 D_0^2}{\kappa_n^4\din^{3}\gamma^2 \delta^{5}},\;
			\frac{2^{35}n^6 D_0^2}{\kappa_n^6\din^{5}\gamma^2\delta^{5}}
			\right).
		\end{align*}
		Then, with probability at least $1-\delta$, the following properties hold for all $t \geq 0$:
		\begin{enumerate}[nosep]
			\item[(a)] $\eigmin(\ntgram(\bX;\paramall_t)) \geq \frac{\gamma \kappa_n}{4}$;
			\item[(b)] $L_{\nnodes}(\paramall_t) \leq e^{-(\gamma\kappa_n t)/2} L_{\nnodes}(\paramall_0)$;
			\item[(c)] $\|\param_{tj} - \param_{0j}\| \leq \frac{2^3nD_0 }{\kappa_n\din^{1/2}\gamma\delta^{1/2}} \sqrt{\lamj}$ for all $j \in [\nnodes]$;
			\item[(d)] $
			\|\ntgram(\bX;\paramall_t)-\ntgram(\bX;\paramall_0)\|_2
			\leq
			\Big(\frac{2^9n^2D_0}{\kappa_n\din^{3/2}\gamma\delta^{5/2}}
			\cdot \sum_{j=1}^{\nnodes}\lamj^{3/2}\Big)
			+
			\Big(\frac{2^6 n^{3/2}D_0^{1/2}}{\kappa_n^{1/2}\din^{5/4}\gamma^{1/2}\delta^{5/4}}
			\cdot
			\sqrt{\sum_{j=1}^{\nnodes}\lamj^{3/2}}\Big)$.
		\end{enumerate}
	\end{theorem}
The proof of the theorem is given in \cref{sec:proofboundrelu}, and uses \cref{lem:putting-together-via-contradiction-argument:ReLU,prop:ntk-initialisation,lem:exp decay,lem:3.2:ReLU}.

The theorem guarantees that whenever $\gamma>0$, the training error converges to 0 exponentially fast. Also, it implies that the weight change is bounded by a factor $\sqrt{\lamj}$, and the NTG change is bounded by a factor $\sqrt{\sum_{j=1}^{\nnodes}\lamj^{3/2}}$. As we show in \cref{sec:lambdasproperties}, as $m$ tends to $\infty$,
\[
\lamj\to(1-\gamma)\tillam_j\ \ \text{for every $j\geq 1$},\quad
\text{and}\quad
\sum_{j=1}^{\nnodes}\lamj^{3/2}\to (1-\gamma)^{3/2}\sum_{j=1}^{\infty}\tillam_j^{3/2}.
\]
Thus, when $\tillam_j>0$ (note that we necessarily have $\tillam_1>0$), the upper bound in (c) is vanishing in the infinite-width limit if and only if $\gamma=1$ (NTK regime); similarly, the upper bound in (d) is vanishing if and only if $\gamma=1$. In fact, both feature learning and non-uniform feature learning in high-probability versions of \cref{def:feature,def:nonuniformfeature}  occur whenever $\gamma < 1$, as we will show in the next subsection.

\subsection{Results on feature learning}
\label{sec:featurelearning-relu}

We present feature-learning results for the ReLU activation. The proofs of the theorems in this subsection appear in
\cref{sec:featurelearning-relu-proofs}

We start with a result that corresponds to \cref{thm:featurelearning-smooth} in the smooth-activation case.
The result says that if $\gamma < 1$ and the activation function is ReLU, then after the first step of gradient descent,
both feature learning and non-uniform feature learning occur in a slightly weaker sense
than that of \cref{def:feature,def:nonuniformfeature} where we have substituted the almost-sure conditions with corresponding high-probability conditions.
\begin{theorem}
	\label{thm:featurelearningrelu-general}
	Suppose that \cref{assump:data,assump:init_V,assump:zeroed-initialisation,assump:random-outputs} hold.
	Suppose also that $\gamma < 1$ and that the activation function $\sigma$ is ReLU. If $\tillam_k>0$, then with probability at least $1-(1/2)^k$, the following inequalities hold for all $i \in [n]$:
	\begin{align}
		\label{eqn:featurelearningrelu-general:0}
		&
		\liminf_{\nnodes \to \infty}
		\frac{
			\sum_{j = 1}^{\nnodes} \lamj \Big(\sigma(Z_j(\bx_i;\paramall_1)) - \sigma(Z_j(\bx_i;\paramall_0))\Big)^2
		}{
			\sum_{j = 1}^{\nnodes} \lamj \Big(\sigma(Z_j(\bx_i;\paramall_0))\Big)^2
		} > 0
		\\
		\label{eqn:featurelearningrelu-general:0b}
		& \qquad \text{and}\qquad
		\liminf_{\nnodes \to \infty}
		\frac{
			\max_{j \in [\nnodes]} \lamj \Big(\sigma(Z_j(\bx_i;\paramall_1)) - \sigma(Z_j(\bx_i;\paramall_0))\Big)^2
		}{
			\sum_{j = 1}^{\nnodes} \lamj \Big(\sigma(Z_j(\bx_i;\paramall_0))\Big)^2
		} > 0
	\end{align}
\end{theorem}

As we mentioned already, the proof of \cref{thm:featurelearningrelu-general} appears in \cref{sec:featurelearning-relu-proofs}.
Here we explain the key steps of the proof. Note that the condition for non-uniform feature learning in \cref{eqn:featurelearningrelu-general:0b}
implies that for feature learning in \cref{eqn:featurelearningrelu-general:0}. Thus, we focus on
proving the former condition.
The crux of proving the condition in \cref{eqn:featurelearningrelu-general:0b}
lies in the derivation of the following lower bound:
\begin{multline*}
	\liminf_{\nnodes \to \infty}
	\left(\max_{j \in [\nnodes]} \lamj \Big(\sigma(Z_j(\bx_i;\paramall_1)) - \sigma(Z_j(\bx_i;\paramall_0))\Big)^2\right)
	\\
	{} \geq
	\max_{j \in [k]}
	\left(
	\ind_{\left\{\param_{0j}^\top \bx_{i} \geq 0\right\}}
	\cdot
	\min\left\{
		\frac{\eta^2 c^2 (1-\gamma)^2\tillam_j^2}{d^2},\,
		\frac{(1-\gamma)\tillam_j (\param_{0j}^\top\bx_i)^2}{d}
	\right\}\right)
\end{multline*}
where $c$ is a positive real-valued continuous random variable that depends only on the outputs $y_1,\ldots,y_n$.
In particular, $c$ does not depend on $\paramall_0$ nor $\nnodes$, and moreover $c^2>0$ almost surely.
The assumptions of the theorem and the properties of $c$ imply that the above lower bound is strictly positive if
$\param_{0j}^\top \bx_i > 0$ for some $j$, and this latter condition happens with probability at least $1-(1/2)^k$,
which gives the claim of the theorem.

Recall that by definition, $\tillam_1$ is always positive. Thus, \cref{thm:featurelearningrelu-general}
implies that the inequalities in \cref{eqn:featurelearningrelu-general:0,eqn:featurelearningrelu-general:0b}
(which are the conditions for feature learning and non-uniform feature learning stated in \cref{def:feature,def:nonuniformfeature}
excepting the almost-sure condition)
hold with probability at least $1/2$ for \emph{any} choice of the node-scaling parameters.
Another immediate and perhaps more important consequence of the theorem is that,
if {\it all} the $\tillam_j$'s are positive,
then both feature learning and non-uniform feature learning occur, precisely in the
sense of \cref{def:feature,def:nonuniformfeature}.
This is because,
in that case, the inequalities
in \cref{eqn:featurelearningrelu-general:0,eqn:featurelearningrelu-general:0b} hold
with probability at least $1-(1/2)^k$ for all $k$ by \cref{thm:featurelearningrelu-general},
but this implies that both inequalities hold almost surely. The next corollary
states this consequence more explicitly.
\begin{corollary}
	\label{cor:featurelearningrelu}
	Suppose \cref{assump:data,assump:init_V,assump:zeroed-initialisation,assump:random-outputs} hold.
	Suppose also that $\gamma < 1$ and that the activation function $\sigma$ is ReLU. Let $i \in [n]$.
	If $\tillam_j > 0$ for all $j$, the following inequalities hold almost surely:
	\begin{align*}
		&
		\liminf_{\nnodes \to \infty}
		\frac{
			\sum_{j = 1}^{\nnodes} \lamj \Big(\sigma(Z_j(\bx_i;\paramall_1)) - \sigma(Z_j(\bx_i;\paramall_0))\Big)^2
		}{
			\sum_{j = 1}^{\nnodes} \lamj \Big(\sigma(Z_j(\bx_i;\paramall_0))\Big)^2
		} > 0
		\\
		& \qquad \text{and}\qquad
		\liminf_{\nnodes \to \infty}
		\frac{
			\max_{j \in [\nnodes]} \lamj \Big(\sigma(Z_j(\bx_i;\paramall_1)) - \sigma(Z_j(\bx_i;\paramall_0))\Big)^2
		}{
			\sum_{j = 1}^{\nnodes} \lamj \Big(\sigma(Z_j(\bx_i;\paramall_0))\Big)^2
		} > 0.
	\end{align*}
\end{corollary}

Our next result about the ReLU activation function is a counterpart of \cref{thm:weight-norm-square-change-smooth} in the smooth-activation case. It
 says that for all $j$, if $\gamma < 1$ and $\tillam_j > 0$,
then with probability at least $1/2$, the first step
of gradient descent induces a non-zero change in the squared norm of the weight vector
$\param_j$ in the infinite-width limit.
The result also suggests that the change in the squared norm is proportional to $\tillam_j$.
\begin{theorem}
	\label{thm:weight-norm-square-change-relu}
	Suppose \cref{assump:data,assump:init_V,assump:zeroed-initialisation,assump:random-outputs} hold.
	Suppose also that the activation function $\sigma$ is ReLU. Then, for all $j$, the following holds almost surely:
	\begin{equation}
		\label{eqn:weight-norm-square-change-relu:0}
		\liminf_{\nnodes \to \infty}
		\left\|\left.\nabla_{\param_{tj}} L(\paramall_t)\right|_{t=0}\right\|^2
		\geq
		\frac{(1-\gamma)\tillam_j}{d}
		\left|\sum_{i = 1}^n \sum_{i'=1}^n
			y_iy_{i'}
			\left(\bx_{i}^\top \bx_{i'}
				\ind_{\{\param_{0j}^\top \bx_i \geq 0\}}
				\ind_{\{\param_{0j}^\top \bx_{i'} \geq 0\}}
			\right)
		\right|.
	\end{equation}
	In particular, if $\gamma < 1$ and $\tillam_j > 0$, then with probability at least $1/2$,
	the above lower bound is positive so that
	\[
		\liminf_{\nnodes \to \infty}
		\left\|\left.\nabla_{\param_{tj}} L(\paramall_t)\right|_{t=0}\right\|^2 > 0.
	\]
\end{theorem}

\subsection{Discussion}
\label{sec:discussion-relu}

\cref{thm:upper-bound:relu} is the counterpart of \cref{thm:upper-bound:smooth-activation} for the global convergence of gradient flow with the ReLU activation function. Despite empirical evidence from \cref{sec:additionalexperiments-relu} suggesting that similar convergence results could potentially be applicable to GD in the ReLU context, we have yet to substantiate this with a comprehensive proof. The proof of the global convergence of GD with smooth activation provided in \cref{appendix:global-convergence-gd} relies on a Taylor approximation. This necessitates the activation function $\relu$ to be twice differentiable. It is worth noting that, in the symmetric NTK case, the global convergence of GD with the ReLU activation has been shown by \citet[Section 4]{Du2019}. Their proof, however, critically relies on the fact that the weights remain stationary throughout the iterations of GD, which is not the scenario we are dealing with here when $\gamma>0$. As such, it remains a compelling open question to determine whether the global convergence of GD can be proven within our specific framework when employing the ReLU activation function.

\section{Useful bounds and identities}
\label{sec:usefulbounds}	
	\subsection{Matrix Chernoff inequalities}

 The following matrix bounds can be found in \cite{tropp2012user}.
	\begin{proposition}\label{prop:matrixChernoff}
		Consider a finite sequence $(X_1,X_2,\ldots,X_p)$ of independent, random, positive semi-definite $n\times n$
		matrices with $\eigmax(X_j)\leq R$ almost surely for all $j \in [p]$, for some $R>0$. Define
		\begin{align*}
			\mu_{\min} =\eigmin \left (\sum_{j=1}^p \mathbb E[X_j]\right )~~\text{ and }~~\mu_{\max} &=\eigmax \left (\sum_{j=1}^p \mathbb E[X_j]\right ).
		\end{align*}
		Then, for all $\delta \in [0,1)$,
		\begin{align*}
			\Pr\left ( \eigmin\left (\sum_{j=1}^p X_j\right )\leq (1-\delta)\mu_{\min} \right )
            & {} \leq n\left [\frac{e^{-\delta}}{(1-\delta)^{1-\delta}}\right ]^{\mu_{\min}/R}
            \leq ne^{-\delta^2\mu_{\min}/(2R)}.
        \end{align*}
        Also, for all $\delta\geq0$,
        \begin{align*}
			\Pr\left ( \eigmax\left (\sum_{j=1}^p X_j\right )\geq (1+\delta)\mu_{\max} \right )
            & {} \leq n\left [\frac{e^{\delta}}{(1+\delta)^{1+\delta}}\right ]^{\mu_{\max}/R}
            {} \leq ne^{-\delta^2\mu_{\max}/((2+\delta)R)}.
		\end{align*}

  \end{proposition}

\subsection{Some identities on  $(\lamj)_{j\in[\nnodes]}$}
\label{sec:lambdasproperties}

	The following proposition summarises a number of useful properties on the scaling parameters defined by \eqref{eq:lamj}.
	\begin{proposition}
		For all $\nnodes\geq 1$,
		\begin{align}
			\label{eq:sum-lambdas-one}
			\sum_{j=1}^{\nnodes} \lamj =1,
			\\
			\label{eq:sum-lambdas-bounds}
			\sqrt{\gamma\nnodes}\leq \sum_{j=1}^{\nnodes} \sqrt{\lamj} \leq \sqrt{\nnodes}.
		\end{align}
		For every $r>1$, as $\nnodes\to\infty$,
		\begin{align}
			\label{eq:sum-powered-lambdas-asymptotic}
			\sum_{j=1}^{\nnodes} \lamj^r \sim \sum_{j=1}^{\nnodes} \left(\lamj^{(2)}\right)^r \to (1-\gamma)^r\sum_{j \geq 1} \tillam_j^r.
		\end{align}

Finally, we have, for all $\rho\in(0,1)$,
\begin{align}
\lim_{m\to \infty} \sum_{j = \lfloor \rho m \rfloor +1}^m \lambda_{m,j}=\gamma(1-\rho).
\label{eq:sum-lambda-abovegammam}
\end{align}
	\end{proposition}

	\begin{proof}
		\Cref{eq:sum-lambdas-one} follows from the definition of $\lamj$ as shown below:
		\[
		\sum_{j=1}^{\nnodes} \lamj
		=
		\sum_{j=1}^{\nnodes} \left(\frac{\gamma}{\nnodes}+(1-\gamma)  \frac{\tillam_j}{\sum_{k=1}^{\nnodes} \tillam_k}\right)
		=
		\gamma +
		(1-\gamma) \sum_{j=1}^{\nnodes}  \frac{\tillam_j}{\sum_{k=1}^{\nnodes} \tillam_k}
		=
		\gamma+ (1-\gamma)
		= 1.
		\]
		In \cref{eq:sum-lambdas-bounds}, the upper bound follows from Cauchy-Schwarz and \cref{eq:sum-lambdas-one}, and the lower bound from
		the definition of $\lamj$:
		\begin{align*}
			\sqrt{\gamma\nnodes }
			= \sum_{j=1}^\nnodes \sqrt{\frac{\gamma}{\nnodes}}
			\leq \sum_{j=1}^\nnodes \sqrt{\lamj}
			\leq \sqrt{\sum_{j=1}^\nnodes \lamj}\sqrt{\sum_{j=1}^\nnodes 1} = 1 \cdot \sqrt{\nnodes}.
		\end{align*}
		For \cref{eq:sum-powered-lambdas-asymptotic}, we note the following bounds on the sum of the $\lamj^r$ for all $r > 1$:
		\[
		\sum_{j=1}^\nnodes \left(\lambda^{(2)}_{\nnodes,j}\right)^r
		\leq
		\sum_{j=1}^\nnodes \left(\lambda_{\nnodes,j}\right)^r
		\leq
		\left(\left[\sum_{j=1}^\nnodes \left(\lambda^{(1)}_{\nnodes,j}\right)^r\right]^{1/r}
		+ \left[\sum_{j=1}^\nnodes \left(\lambda^{(2)}_{\nnodes,j}\right)^r\right]^{1/r}\right)^{r}
		\]
		where the second inequality uses the Minkowski inequality. But as $\nnodes \to \infty$,
		the term $\sum_{j = 1}^\nnodes (\lambda^{(1)}_{\nnodes,j})^r=\gamma^r\nnodes^{-(r-1)}\to 0$.
		Furthermore, as $\nnodes \to \infty$,
		\[
		\sum_{j=1}^\nnodes \left(\lambda^{(2)}_{\nnodes,j}\right)^r
		=
		\frac{(1-\gamma)^r}{\left(\sum_{k = 1}^\nnodes \tilde{\lambda}_k\right)^r} \sum_{j = 1}^\nnodes \tilde{\lambda}_j^r
		\to
		(1-\gamma)^r \sum_{j  \geq 1} \tilde{\lambda}_j^r
		\]
		because $(\sum_{k \geq 1} \tilde{\lambda}_k)^r = 1$.

Finally, we prove \cref{eq:sum-lambda-abovegammam}. For all $\rho\in(0,1)$, we have
\begin{align*}
\sum_{j = \lfloor \rho m \rfloor +1}^m \lambda_{m,j}= \frac{\gamma (m-\lfloor \rho m \rfloor)}{m} + (1-\gamma) \frac{\sum_{j=\lfloor \rho m \rfloor}^m \widetilde\lambda_{j}}{\sum_{j=1}^m \widetilde\lambda_j}.
\end{align*}
By sandwiching, $\frac{\gamma (m-\lfloor \rho m \rfloor)}{m}\to \gamma (1-\rho)$. Additionally, the series $\sum_{j=1}^m \widetilde\lambda_j$ converges to 1. Thus, its tail converges to 0 and  $\sum_{j=\lfloor \rho m \rfloor+1}^m \widetilde\lambda_{j}\to 0$.
\end{proof}

\cref{fig:zipflawalpha} shows the value of $\sum_{j\geq 1}\widetilde\lambda_j^2=\frac{\zeta(2/\alpha)}{\zeta(1/\alpha)^2}$ as a function of $\alpha$, when using Zipf weights \cref{eq:tildelambdazipf}.

 \begin{figure}
\begin{center}
\includegraphics[width=6cm]{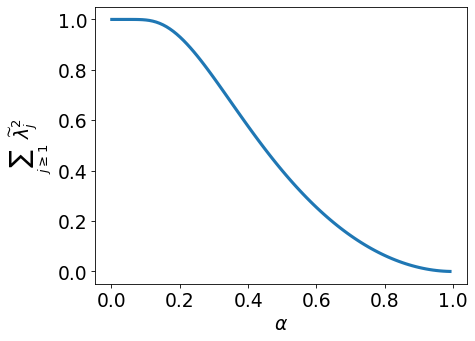}
\end{center}
\caption[$\sum_{j=1}^\infty\widetilde\lambda_j^2$ as a function of $\alpha$]{Value of $\sum_{j=1}^\infty\widetilde\lambda_j^2$ as a function of $\alpha$,  where $(\widetilde\lambda_j)_{j\geq 1}$ are defined as in~\cref{eq:tildelambdazipf},  As $\alpha\to 1$, it converges to 0, which corresponds to the kernel regime.}
\label{fig:zipflawalpha}
\end{figure}

\section{Proof of \cref{prop:limitingNTG} on the limiting NTG}
\label{sec:prooflimitingNTG}
This proposition holds also under the ReLU activation case. In what follows, we will give a proof that works for both the smooth activation function and ReLU.

It is sufficient to look at the convergence of individual entries of the NTG matrix; that is, to show that, for each pair $1\leq i,i'\leq n$,
\begin{equation}
\begin{aligned}
\ntker(\bx_i,\bx_{i'};\paramall_{0})
=
\frac{\bx_i^\top \bx_{i'}}{d}
\times\bigg( & \frac{\gamma}{m}\sum_{j=1}^\nnodes  \relu'(Z_{j}(\bx_i;\paramall_0))\relu'(Z_{j}(\bx_{i'};\paramall_0))
\\
&
{} +\frac{(1-\gamma)}{\sum_{k=1}^\nnodes \tillam_k}
\sum_{j=1}^\nnodes \tillam_j \relu'(Z_{j}(\bx_i;\paramall_0))\relu'(Z_{j}(\bx_{i'};\paramall_0))
\bigg)
\end{aligned}
\label{eq:limitingNTG:eq1}
\end{equation}
tends to
\begin{align}
\gamma \ntkerlim(\bx_i,\bx_{i'})+\frac{(1-\gamma)}{\din} \bx_i^\top \bx_{i'} \sum_{j=1}^\infty \tillam_j \relu'(Z_{j}(\bx_i;\paramall_0))\relu'(Z_{j}(\bx_{i'};\paramall_0))
\label{eq:limitingNTG:eq2}
\end{align}
almost surely as $\nnodes\to\infty$. Using the fact that $|\relu'(z)|\leq 1$ and the triangle inequality, the modulus of the difference between the RHS of \cref{eq:limitingNTG:eq1} and \cref{eq:limitingNTG:eq2} is upper bounded by
\begin{align*}
&
\left|\frac{\bx_i^\top \bx_{i'}}{d}\right|
\left(\gamma  \left |
\left(\frac{1}{m}\sum_{j=1}^\nnodes  \relu'(Z_{j}(\bx_i;\paramall_0))\relu'(Z_{j}(\bx_{i'};\paramall_0))\right)
- \mathbb E[\relu'(Z_{1}(\bx_i;\paramall_0))\relu'(Z_{1}(\bx_{i'};\paramall_0))]
\right | \right.\\
&
\qquad\qquad\qquad
\left.
{}
+(1-\gamma) \left[\left(\frac{1}{\sum_{j=1}^\nnodes\tillam_j}-1\right)\sum_{j=1}^\nnodes\tillam_j+\sum_{j=\nnodes+1}^\infty\tillam_j\right]\right)
\\
& =\Bigg|\frac{\bx_i^\top \bx_{i'}}{d}\Bigg|
\Bigg(\gamma  \Bigg|
\bigg( \frac{1}{m}\sum_{j=1}^\nnodes  \relu'(Z_{j}(\bx_i;\paramall_0))\relu'(Z_{j}(\bx_{i'};\paramall_0)) \bigg) - \mathbb E[\relu'(Z_{1}(\bx_i;\paramall_0))\relu'(Z_{1}(\bx_{i'};\paramall_0))]
\Bigg|
\\
&
\qquad\qquad\qquad
{}+2 (1-\gamma) \left[1-\sum_{j=1}^\nnodes\tillam_j\right]\Bigg)
\end{align*}
which tends to 0 almost surely as $\nnodes$ tends to infinity using the law of large numbers and the fact that $\sum_{j=1}^\infty \tillam_j= 1$.

\section{Secondary Proposition - NTG at initialisation}
\label{sec:secondarythm:init}

	The following proposition is a corollary of Lemma 4 in  \cite{Oymak2020}.  It holds under both the ReLU and smooth activation cases. A proof is included for completeness.
	
	\begin{proposition}
		\label{prop:ntk-initialisation}
		Let $\delta\in(0,1)$. Assume \cref{assump:data,assump:init_V}, $\gamma>0$, and $\nnodes\geq \frac{2^3n\log \frac{n}{\delta}}{\kappa_n\din}$. Also, assume that the activation function satisfies \cref{assump:activation} or it is ReLU. Then, with probability at least $1-\delta$,
		\begin{align*}
			\eigmin(\ntgram(\bX;\paramall_0))\geq \eigmin(\ntgram^{(1)}(\bX;\paramall_0))> \frac{\gamma\kappa_n}{2}>0.
		\end{align*}
	\end{proposition}
	
	\begin{proof}
		We follow here the proof of Lemma 4 in \cite{Oymak2020}.
		% (which also appears in the page 87 of \cite{Bartlett2021}). Recall that the NTG matrix takes the form
		\begin{align*}
			\ntgram(\bX;\paramall)&=\frac{1}{\din}\sum_{j=1}^\nnodes \lamj A_j\\
			&=\frac{1}{\din}\sum_{j=1}^\nnodes \lamj^{(1)} A_j+\frac{1}{\din}\sum_{j=1}^\nnodes \lamj^{(2)} A_j
		\end{align*}
		where
		$$
		A_j=\diag(\brelu'(\bX \paramj/\sqrt{\din}))\bX \bX^\top \diag(\brelu'(\bX \paramj/\sqrt{\din})).
		$$
		Let $\ntgram^{(1)}(\bX;\paramall)=\frac{1}{\din}\sum_{j=1}^\nnodes \lamj^{(1)} A_j=\frac{\gamma}{\nnodes\din}\sum_{j=1}^\nnodes  A_j$. Note that $\eigmin(\ntgram(\bX;\paramall))\geq\eigmin(\ntgram^{(1)}(\bX;\paramall))$ a.s., and
		$$
		\mathbb E[\ntgram^{(1)}(\bX;\paramall_0)]=\gamma\ntgramlim(\bX)
		$$
		where $\ntgramlim(\bX)$ is defined in \cref{eq:ntkerlim}. We have, for all $j\geq 1$,
		\begin{equation}
		\label{eqn:ntk-initialisation:0}
                \begin{aligned}
		\|A_j\|_2 =\eigmax(A_j)
                & \leq
                \eigmax(\diag(\brelu'(\bX \paramj/\sqrt{\din}))^2) \eigmax(\bX \bX^\top)
		\leq
		\eigmax(\bX \bX^\top)
                \\
                & \leq
		\trace(\bX \bX^\top)
		\leq
		n.
                \end{aligned}
		\end{equation}
		At initialisation, $A_1,A_2,\ldots,A_\nnodes$ are independent random matrices. Using matrix Chernoff inequalities (see \cref{prop:matrixChernoff}), we obtain, for all $\epsilon\in[0,1)$,
		\begin{align*}
			\Pr\left (\eigmin(\ntgram(\bX;\paramall_0))\leq (1-\epsilon)\gamma\kappa_n   \right )\leq n  e^{-\epsilon^2 \nnodes\kappa_n\din/(2n)}.
		\end{align*}
		
                Let $\delta\in(0,1)$. Taking $\epsilon=1/2$, we have that, if $\frac{\nnodes\kappa_n\din}{2^3n}\geq \log \frac{n}{\delta}$, then
		$$
		\Pr\left (\eigmin(\ntgram(\bX;\paramall_0))\leq \frac{\gamma\kappa_n}{2}   \right )\leq \delta.
		$$
	\end{proof}

	\section{Secondary Lemmas on gradient flow dynamics}
\label{sec:secondarythm:dynamics}

	The proof technique used to prove \cref{thm:upper-bound:relu,thm:upper-bound:smooth-activation} is similar to that of  \cite{Du2019} (NTK scaling). In particular, we provide in this section Lemmas similar to Lemmas 3.2, 3.3 and 3.4 in \cite{Du2019}, but adapted to our setting. \cref{lem:exp decay} is an adaptation of Lemma 3.3. \cref{lem:3.2:ReLU,lem:3.2:smooth-activation} are adaptations of Lemma 3.2, respectively for the ReLU and smooth activation cases. \cref{lem:putting-together-via-contradiction-argument:ReLU,lem:putting-together-via-contradiction-argument:smooth-activation} are adaptations of Lemma 3.4, respectively for the ReLU and smooth activation cases.

 \subsection{Lemma on exponential decay of the empirical risk and scaling of the weight changes}

 The following lemma is an adaptation of Lemma 3.3 of \cite{Du2019}, and applies to both the ReLU and smooth activation cases. It shows that, if the minimum eigenvalue of the NTG matrix is bounded away from 0, gradient flow converges to a global minimum exponentially fast.
 Recall that $\by=(y_1,\ldots,y_n)^\top \in \mathbb{R}^n$.
	
	\begin{lemma}\label{lem:exp decay}
		Let $t>0$ and $\zeta>0$. Assume \cref{assump:data} and $\eigmin(\ntgram(\bX;\paramall_s))\geq \frac{\zeta}{2}$ for all $0\leq s\leq t$. Also, assume that the activation function satisfies \cref{assump:activation} or it is ReLU. Then,
		$$
		L_\nnodes(\paramall_t)\leq e^{-\zeta t}L_\nnodes(\paramall_0),
		$$
		and for all $j \in [\nnodes]$,
		\begin{align}
			\|\param_{tj} -\param_{0j}\| \leq \sqrt{\frac{n\lamj}{\din}} \left\|\by-\bfu_0\right\|\frac{2}{\zeta},\label{eq:Vtjbound}
		\end{align}
		where $\bfu_0 = (f_\nnodes(\bx_1;\paramall_0),\ldots,f_\nnodes(\bx_n;\paramall_0))^\top \in \mathbb{R}^n$.
	\end{lemma}
	
	\begin{proof}
		For $0 \leq s \leq t$,
		write $\bfu_s=(f_\nnodes(\bx_1;\paramall_s),\ldots,f_\nnodes(\bx_n;\paramall_s))^\top$. We have
		\[
		\frac{d}{ds}\bfu_s = \ntgram(\bX;\paramall_s)(\bfy-\bfu_s).
		\]
		It follows that
		\begin{align*}
			\frac{dL_\nnodes(\paramall_s)}{ds}&=-(\bfy-\bfu_s)^\top\ntgram(\bX;\paramall_s)(\bfy-\bfu_s)\leq -\frac{\zeta}{2} (\bfy-\bfu_s)^\top(\bfy-\bfu_s)=-\zeta L_\nnodes(\paramall_s).
		\end{align*}
		Using Gr\"onwall's inequality, we obtain
		\begin{align*}
			L_\nnodes(\paramall_t)\leq e^{-\zeta t}L_\nnodes(\paramall_0).
		\end{align*}
		For $0\leq s\leq t$, using the Cauchy-Schwarz inequality, we get
		\begin{align*}
			\left \|\frac{d\paramsj}{ds} \right \|^2&=\left \| \sqrt{\lamj} \frac{\aj}{\sqrt{\din}} \sum_{i=1}^n  \relu'(Z_{sj}(\bx_i))\bx_i \cdot (y_i-f_\nnodes(\bx_i;\paramall_s))\right \|^2
			\\
			&= \frac{\lamj}{\din}\sum_{k=1}^{\din} \left(\sum_{i=1}^n  \relu'(Z_{sj}(\bx_i))x_{ik} \cdot (y_i-f_\nnodes(\bx_i;\paramall_s)) \right)^2
			\\
			&\leq \frac{\lamj}{\din}
			\sum_{k=1}^{\din}
			\left(\sum_{i=1}^n  x_{ik}^2\right)
			\left(\sum_{i=1}^n \relu'(Z_{sj}(\bx_i))^2 (y_i-f_\nnodes(\bx_i;\paramall_s))^2\right)
			\\
			&= \frac{\lamj}{\din}
			\left(\sum_{i=1}^n \relu'(Z_{sj}(\bx_i))^2 (y_i-f_\nnodes(\bx_i;\paramall_s))^2\right)
			\left(\sum_{k=1}^{\din} \sum_{i=1}^n  x_{ik}^2\right)
			\\
			&\leq \frac{\lamj}{\din}
			\left(\sum_{i=1}^n (y_i-f_\nnodes(\bx_i;\paramall_s))^2\right)
			\left(\sum_{i=1}^{n} \sum_{k=1}^\din x_{ik}^2\right)
			\\
			&\leq \frac{n\lamj}{\din} \|\by-\bfu_s\|^2
			\\
			&\leq \frac{n\lamj}{\din} \|\by-\bfu_0\|^2 e^{-\zeta s}.
		\end{align*}
		Integrating and using Minkowski's integral inequality,
		we obtain
		\begin{align*}
			\|\paramtj-\param_{0j}\|=\left\|\int_0^t \frac{d}{ds}\param_{sj} ds\right\|&\leq\int_0^t  \left\| \frac{d}{ds}\param_{sj} \right\|ds\\
			&\leq \sqrt{\frac{n\lamj}{\din}} \left\|\by-\bfu_0\right\|\int_0^t e^{-\zeta s/2}ds\\
			&\leq \sqrt{\frac{n\lamj}{\din}} \left\|\by-\bfu_0\right\|\frac{2}{\zeta}.
		\end{align*}
	\end{proof}

From now on, the proofs for the ReLU and smooth-activation cases slightly differ.

\subsection{Lemma bounding the NTK change and minimum eigenvalue - ReLU case}

	The next lemma and its proof are similar to Lemma 3.2 in \cite{Du2019} and its proof. Recall that
	$0 < \|\bx_i\| \leq 1$ for every $i \in [n]$, and the $\param_{0j}$ are iid $\mathcal N(0, \idmat_d)$.
	\begin{lemma}\label{lem:3.2:ReLU}
		Let $\delta \in (0,1)$, and $c_{\nnodes,j}>0$ for every $j \in [\nnodes]$. Assume that \cref{assump:data,assump:init_V} holds and the activation function is ReLU.
		Then, with probability at least $1-\delta$, the following holds. For every $\paramall=(\param_1^\top,\ldots,\param_\nnodes^\top)^\top$, if it satisfies
		\[
		\|\param_{0j}-\paramj\| \leq \frac{\delta^2 c_{\nnodes,j}}{4}\quad \text{for all $j \in [\nnodes]$,}
		\]
		we have
		\[
		\left\| \ntgram^{(s)}(\mathbf{X};\paramall)-\ntgram^{(s)}(\mathbf{X};\paramall_{0})\right\|_2
		\leq
		\frac{n}{\din}\sum_{j=1}^{\nnodes}\lamj^{(k)}c_{\nnodes,j}
		+
		\frac{2 n}{\din} \,
		\sqrt{\sum_{j=1}^{\nnodes}\lamj^{(k)}c_{\nnodes,j}}
		\qquad\text{for all $k \in [2]$}
		\]
		and
		\begin{align}
			\label{eqn:3.2:ReLU:0}
			\eigmin(\ntgram(\bX;\paramall))\geq
			\eigmin(\ntgram^{(1)}(\bX;\paramall_0))
			-
			\left(\frac{n\gamma}{\din \nnodes}\sum_{j=1}^{\nnodes} c_{\nnodes,j}
			+
			\frac{2 n \gamma}{\din \nnodes^{1/2}}
			\sqrt{\sum_{j=1}^{\nnodes} c_{\nnodes,j}}\right).
		\end{align}
	\end{lemma}
	
	\begin{proof}
		For $k \in [2]$, let
		\begin{align*}
			f_{\nnodes}^{(k)}(-;\paramall) & : \R^\din \to \R,
			&
			f_{\nnodes}^{(k)}(\mathbf{x};\paramall) & =\sum_{j=1}^{\nnodes}\sqrt{\lamj^{(k)}}\aj\relu(Z_{j}(\mathbf{x};\paramall)).
		\end{align*}
		Define $\nabla_{\paramall}f_{\nnodes}^{(k)}(\mathbf{X};\paramall)$ to be the $n$-by-$(\nnodes d)$
		matrix whose $i$-th row is the $\nnodes d$-dimensional row vector
		$(\nabla_{\paramall}f_{\nnodes}^{(k)}(\mathbf{x}_{i};\paramall))^{\top}$.
		
		Note that for all $k \in [2]$,
		\begin{align}
                    \nonumber
				& \left\|\ntgram^{(k)}(\mathbf{X};\paramall)-\ntgram^{(k)}(\mathbf{X};\paramall_{0})\right\|_2
                    \\
                    \nonumber
				& \qquad {} =
				\left\|
				\nabla_{\paramall}f_{\nnodes}^{(k)}(\mathbf{X};\paramall)\nabla_{\paramall}f_{\nnodes}^{(k)}(\mathbf{X};\paramall)^{\top}
				-\nabla_{\paramall}f_{\nnodes}^{(k)}(\mathbf{X};\paramall_{0})\nabla_{\paramall}f_{\nnodes}^{(k)}(\mathbf{X};\paramall_{0})^{\top}
				\right\|_2
				\\
                    \label{eqn:3.2:ReLU:1}
				& \qquad {} \leq
				\left\|
				\nabla_{\paramall}f_{\nnodes}^{(k)}(\mathbf{X};\paramall)-
				\nabla_{\paramall}f_{\nnodes}^{(k)}(\mathbf{X};\paramall_{0})\right\|_2^{2}
				\\
                    \nonumber
				& \qquad
				\phantom{{} \leq {}}
				{} + 2
				\left\| \nabla_{\paramall}f_{\nnodes}^{(k)}(\mathbf{X};\paramall_{0})\right\|_2
				\left\| \nabla_{\paramall}f_{\nnodes}^{(k)}(\mathbf{X};\paramall)-\nabla_{\paramall}f_{\nnodes}^{(k)}(\mathbf{X};\paramall_{0})\right\|_2.
		\end{align}
		The justification of the inequality from above is given below (which is an expanded version
		of the three equations (364-366) in \cite{Bartlett2021}): for all $n$-by-$(pd)$ matrices $A$ and $B$,
		\begin{align*}
			\left\| AA^{\top}-BB^{\top}\right\|_2
			& {} =
			\left\|\frac{1}{2}(A-B)(A+B)^{\top}+\frac{1}{2}(A+B)(A-B)^{\top}\right\|_2
			\\
			& {} \leq
			\frac{1}{2}\left(\left\| (A-B)(A+B)^{\top}\right\|_2 +\left\|(A+B)(A-B)^{\top}\right\|_2 \right)
			\\
			& {} \leq
			\frac{1}{2}\left(\left\| A-B\right\|_2 \times\left\|(A+B)^{\top}\right\|_2 +\left\| A+B\right\|_2 \times\left\|(A-B)^{\top}\right\|_2 \right)
			\\
			& {} =
			\left\| A-B\right\|_2 \times\left\| A+B\right\|_2
			\\
			& {} \leq
			\left\| A-B\right\|_2 \times\left(  \left\| A-B+B\right\|_2 +\left\| B\right\|_2 \right)
			\\
			& {} \leq
			\left\| A-B\right\|_2 \times\left(  \left\| A-B\right\|_2 +2\left\| B\right\|_2 \right).
		\end{align*}

		Coming back to the inequality in \cref{eqn:3.2:ReLU:1}, we next bound the two terms
		$\left\| \nabla_{\paramall}f_{\nnodes}^{(k)}(\mathbf{X};\paramall_{0})\right\|_2 $ and
		$\left\| \nabla_{\paramall}f_{\nnodes}^{(k)}(\mathbf{X};\paramall)-\nabla_{\paramall}f_{\nnodes}^{(k)}(\mathbf{X};\paramall_{0})\right\|_2$
		there.
		
		We bound the first term as follows:
		\begin{align}
			\nonumber
			\left\| \nabla_{\paramall}f_{\nnodes}^{(k)}(\mathbf{X};\paramall_{0})\right\|_2^{2}
			\leq
			\left\| \nabla_{\paramall}f_{\nnodes}^{(k)}(\mathbf{X};\paramall_{0})\right\|_{F}^{2}
			&
			{} =
			\sum_{i=1}^{n}\sum_{j=1}^{\nnodes}\left\|{\nabla_{\paramj} f_{\nnodes}^{(k)}(\mathbf{x}_{i};\paramall_{0})}\right\|^{2}
			\\
			\nonumber
			&
			{} =
			\sum_{i=1}^{n}\sum_{j=1}^{\nnodes}\lamj^{(k)}\left| \relu'(Z_{j}(\mathbf{x}_{i};\paramall_{0}))\right|^{2}\frac{\left\Vert
				\mathbf{x}_{i}\right\Vert ^{2}}{d}
			\\
			\label{eqn:3.2:ReLU:2}
			&  \leq\frac{n}{d}\sum_{j=1}^{\nnodes}\lamj^{(k)}\leq\frac{n}{d}\gamma_k
		\end{align}
  where $\gamma_1=\gamma$ and $\gamma_2=1-\gamma$.
		The second inequality uses the assumption that $|\relu'(x)| \leq 1$ for all $x \in \mathbb{R}$ and $\|\bx_i\| \leq 1$ for all $i \in [n]$.
		The third inequality follows from the fact that $\sum_{j = 1}^{\nnodes} \lamj^{(k)} \leq \sum_{j=1}^{\nnodes} \lamj = 1$.
		
		For the second term, we recall that
		$Z_{j}(\mathbf{x}_{i};\paramall)=\frac{1}{\sqrt{d}}\param_{j}^{\top}\mathbf{x}_{i}$. Using this fact, we
		derive an upper bound for the second term as follows:
		\begin{align}
			\nonumber
			& \left\| \nabla_{\paramall}f_{\nnodes}^{(k)}(\mathbf{X};\paramall)-\nabla_{\paramall} f_{\nnodes}^{(k)}(\mathbf{X};\paramall_{0})\right\|_2^{2}
                \\
                \nonumber
			& \qquad
			{} \leq
			\left\| \nabla_{\paramall}f_{\nnodes}^{(k)}(\mathbf{X};\paramall)-\nabla_{\paramall}f_{\nnodes}^{(k)}(\mathbf{X};\paramall_{0})\right\| _{F}^{2}
			\\
			\nonumber
			& \qquad
			{} =
			\sum_{i=1}^{n}\sum_{j=1}^{\nnodes}
			\left\|
			{\nabla_{\paramj} f_{\nnodes}^{(k)}(\mathbf{x}_{i};\paramall)}
			- {\nabla_{\paramj} f_{\nnodes}^{(k)}(\mathbf{x}_{i};\paramall_{0})}
			\right\|^{2}
			\\
			\nonumber
			& \qquad
			{} =
			\sum_{i=1}^{n}\sum_{j=1}^{\nnodes}
			\left\|
			\sqrt{\lamj^{(k)}}\aj\frac{\mathbf{x}_{i}}{\sqrt{d}}
			\left[  \relu'(Z_{j}(\mathbf{x}_{i};\paramall))-\relu'(Z_{j}(\mathbf{x}_{i};\paramall_{0}))\right]
			\right\|^{2}
			\\
			\label{eqn:3.2:ReLU:3}
			& \qquad
			{} =
			\frac{1}{d}\sum_{i=1}^{n}
			\sum_{j=1}^{\nnodes}
			\left\| \mathbf{x}_{i} \right\|^{2}
			\lamj^{(k)}\left|  \relu'(Z_{j}(\mathbf{x}_{i};\paramall))-\relu'(Z_{j}(\mathbf{x}_{i};\paramall_{0}))\right|^{2}.
		\end{align}
		In the rest of the proof, we will derive a probabilistic bound on the upper bound just obtained, and show the conclusions claimed in the lemma.
		
		For any $\epsilon>0$, $i \in [n]$, and $j \in [\nnodes]$, we define the event
		$$
		A_{i,j}(\epsilon)=\left\{\exists \paramj \text{ s.t. } \left\|\param_{0j} -\paramj\right\|\leq \epsilon \text{ and } \relu'(\param_j^{\top} \bx_i)\neq \relu'(\param_{0j}^{\top} \bx_i)\right\}.
		$$
		If this event happens, we have $|\param_{0j}^{\top} \bx_i| \leq \epsilon$. To see this, assume that $A_{i,j}(\epsilon)$ holds with $\paramj$ as a witness of the existential quantification, and note that since
		the norm of $\bx_i$ is at most $1$,
		\[
		\left|\param_{0j}^\top \bx_i - \param_j^\top \bx_i \right|
		\leq
		\left\|\param_{0j} - \paramj \right\| \left\|\bx_i\right\|
		\leq
		\epsilon.
		\]
		If $\param_{0j}^\top \bx_i > 0$, then $\param_j^\top \bx_i \leq 0$ and thus
		\[
		\param_{0j}^\top \bx_i \leq \epsilon + \param_j^\top \bx_i < \epsilon.
		\]
	 Alternatively, if $\param_{0j}^\top \bx_i \leq 0$, then $\param_j^\top \bx_i > 0$ and thus
		\[
		-\param_{0j}^\top \bx_i \leq \epsilon - \param_j^\top \bx_i \leq \epsilon.
		\]
		 In both cases, we have the desired $|\param_{0j}^{\top} \bx_i| \leq \epsilon$.
		
		Using the observation that we have just explained and the fact that $\param_{0j}^{\top}\bx_i\sim\mathcal{N}(0,\|\bx_i\|^2)$, we obtain, for a random variable $N \sim \mathcal{N}(0,1)$,
		\begin{align}
			\nonumber
			\Pr(A_{i,j}(\epsilon))
			\leq \Pr\left (|N|\leq \frac{\epsilon}{\|\bx_i\|}\right )
			& {} =
			\mathrm{erf}\left(\frac{\epsilon}{\|\bx_i\| \sqrt{2}}\right)
			\\
			\nonumber
			& {}
			\leq
			\sqrt{1 - \exp\left(-\left(4\left(\frac{\epsilon}{ \|\bx_i\| \sqrt{2}}\right)^2\right)/\pi\right)}
			\\
			\label{eqn:3.2:ReLU:4}
			&
			\leq
			\sqrt{\frac{2\epsilon^2}{\|\bx_i\|^2\pi}}
			\leq \frac{\epsilon}{\|\bx_i\|},
		\end{align}
		where the second inequality uses $\mathrm{erf}(x) \leq \sqrt{1-\exp(-(4x^2)/\pi)}$. Let $\Psi(\paramall_0)$ be the
		constraint on $\paramall = (\param_1^\top,\ldots,\param_m^\top)^\top$ defined by
		\[
		\paramall\in \Psi(\paramall_0) \iff
		\|\param_{0{j'}} - \param_{j'}\| \leq \frac{\delta^2c_{\nnodes,j'}}{4} \ \text{for all $j' \in [\nnodes]$.}
		\]
		Then, for all $k = 1,2$, we have
		\begin{align*}
		        &
			\mathbb{E}\left[
			\sup_{\paramall\in \Psi(\paramall_0)}
			\left\| \nabla_{\paramall}f_{\nnodes}^{(k)}(\mathbf{X};\paramall)-\nabla_{\paramall} f_{\nnodes}^{(k)}(\mathbf{X};\paramall_{0})\right\|_2^{2}\right]
			\\
			&
			\qquad {} \leq
			\frac{1}{d}\sum_{i=1}^{n}\sum_{j=1}^{\nnodes}\|\bx_i\|^2 \lamj^{(k)}
			\mathbb{E}\left[\sup_{\paramall\in \Psi(\paramall_0)} \left|\relu'(Z_{j}(\mathbf{x}_{i};\paramall))-\relu'(Z_{j}(\mathbf{x}_{i};\paramall_{0}))\right|^{2}\right]
			\\
			&
			\qquad {} \leq
			\frac{1}{d}\sum_{i=1}^{n}\sum_{j=1}^{\nnodes}\|\bx_i\|^2 \lamj^{(k)}
			\Pr\left(\exists \paramall \in \Psi(\paramall_0) \ \text{s.t.}\ \relu'(Z_{j}(\mathbf{x}_{i};\paramall)) \neq \relu'(Z_{j}(\mathbf{x}_{i};\paramall_{0}))\right)
			\\
   			&
			\qquad {} =
			\frac{1}{d}\sum_{i=1}^{n}\sum_{j=1}^{\nnodes}\|\bx_i\|^2 \lamj^{(k)}
			\Pr\left(\exists \param_j\ \text{s.t.}\
   \|\param_{0j} - \param_j\| \leq \frac{\delta^2c_{\nnodes,j}}{4} \ \text{and}\ \relu'(\param_j^\top \mathbf{x}_i) \neq \relu'(\param_{0j}^\top \mathbf{x}_i)\right)
			\\
			&
			\qquad {} \leq
			\frac{1}{d}\sum_{i=1}^{n}\sum_{j=1}^{\nnodes} \|\bx_i\|^2 \lamj^{(k)}
			\Pr\left(A_{i,j}(\delta^2 c_{\nnodes,j}/4)\right)
			\\
			&
			\qquad {} \leq
			\frac{(\delta^2/4)}{\din}\sum_{i=1}^{n}\sum_{j=1}^{\nnodes}\|\bx_i\| \lamj^{(k)}c_{\nnodes,j}
			\\
			&
			\qquad {} \leq
			\frac{n(\delta^2/4)}{\din}\sum_{j=1}^{\nnodes}\lamj^{(k)}c_{\nnodes,j}.
		\end{align*}
		The first inequality uses the bound in \cref{eqn:3.2:ReLU:3}, and the fourth inequality uses the inequality derived in \cref{eqn:3.2:ReLU:4}.
		
		We bring together the bound on the expectation just shown and also the bounds proved in \cref{eqn:3.2:ReLU:1,eqn:3.2:ReLU:2}. Recalling that $\gamma_1=\gamma$ and $\gamma_2=1-\gamma$, we have
		\begin{align*}
			&
			\mathbb{E}\left[
			\sup_{\paramall\in \Psi(\paramall_0)}
			\left\|\ntgram^{(k)}(\mathbf{X};\paramall)-\ntgram^{(k)}(\mathbf{X};\paramall_{0})\right\|_2
			\right]
			\\
			&  {} \leq
			\mathbb{E}\left[
			\sup_{\paramall\in \Psi(\paramall_0)}
			\left\|
			\nabla_{\paramall}f_{\nnodes}^{(k)}(\mathbf{X};\paramall)-
			\nabla_{\paramall}f_{\nnodes}^{(k)}(\mathbf{X};\paramall_{0})\right\|_2^{2}\right]
			\\
			& \qquad
			{} + 2 \,
			\mathbb{E}\left[
			\sup_{\paramall\in \Psi(\paramall_0)}
			\left\| \nabla_{\paramall}f_{\nnodes}^{(k)}(\mathbf{X};\paramall_{0})\right\|_2
			\left\| \nabla_{\paramall}f_{\nnodes}^{(k)}(\mathbf{X};\paramall)-\nabla_{\paramall}f_{\nnodes}^{(k)}(\mathbf{X};\paramall_{0})\right\|_2\right]
			\\
			& {} \leq
			\mathbb{E}\left[
			\sup_{\paramall\in \Psi(\paramall_0)}
			\left\|
			\nabla_{\paramall}f_{\nnodes}^{(k)}(\mathbf{X};\paramall)-
			\nabla_{\paramall}f_{\nnodes}^{(k)}(\mathbf{X};\paramall_{0})\right\|_2^{2}\right]
			\\
			& \qquad
			{} + 2
			\sqrt{\frac{n}{d}\gamma_k} \,
			\mathbb{E}\left[
			\sup_{\paramall\in \Psi(\paramall_0)}
			\left\| \nabla_{\paramall}f_{\nnodes}^{(k)}(\mathbf{X};\paramall)-\nabla_{\paramall}f_{\nnodes}^{(k)}(\mathbf{X};\paramall_{0})\right\|_2\right]
			\\
			& {} \leq
			\mathbb{E}\left[
			\sup_{\paramall\in \Psi(\paramall_0)}
			\left\|
			\nabla_{\paramall}f_{\nnodes}^{(k)}(\mathbf{X};\paramall)-
			\nabla_{\paramall}f_{\nnodes}^{(k)}(\mathbf{X};\paramall_{0})\right\|_2^{2}\right]
			\\
			& \qquad
			{} + 2
			\sqrt{\frac{n}{d}\gamma_k} \,
			\sqrt{
				\mathbb{E}\left[
				\sup_{\paramall\in \Psi(\paramall_0)}
				\left\| \nabla_{\paramall}f_{\nnodes}^{(k)}(\mathbf{X};\paramall)-\nabla_{\paramall}f_{\nnodes}^{(k)}(\mathbf{X};\paramall_{0})\right\|^2_2\right]}
			\\
			& {} \leq
			\frac{n(\delta^2/4)}{d}\sum_{j=1}^{\nnodes}\lamj^{(k)}c_{\nnodes,j}
			+
			2
			\sqrt{\frac{n}{d}\gamma_k} \,
			\sqrt{\frac{n(\delta^2/4)}{d}\sum_{j=1}^{\nnodes}\lamj^{(k)}c_{\nnodes,j}}
			\\
			& {} \leq
			\frac{\delta}{2} \left(
			\frac{n}{d}\sum_{j=1}^{\nnodes}\lamj^{(k)}c_{\nnodes,j}
			+
			\frac{2 n}{d}\,
			\sqrt{\gamma_k\sum_{j=1}^{\nnodes}\lamj^{(k)}c_{\nnodes,j}}\right).
		\end{align*}
		The third inequality uses Jensen's inequality, and the last uses the fact that $\delta/2 \geq (\delta/2)^2$.
		Hence, for each $k = 1,2$, by Markov inequality, we have, with probability at least $1-(\delta/2)$,
		\[
		\sup_{\paramall\in \Psi(\paramall_0)}
		\left\|\ntgram^{(k)}(\mathbf{X};\paramall)-\ntgram^{(k)}(\mathbf{X};\paramall_{0})\right\|_2
		\leq
		\frac{n}{d}\sum_{j=1}^{\nnodes}\lamj^{(k)}c_{\nnodes,j}
		+
		\frac{2 n}{d}\sqrt{\gamma_k} \,
		\sqrt{\sum_{j=1}^{\nnodes}\lamj^{(k)}c_{\nnodes,j}}.
		\]
		By union bound, the conjunction of the above inequalities for the $k=1$ and $k=2$ cases holds with probability at least $1-\delta$.
		
		We prove the last remaining claim using the following lemma.

		If $A$ and $B$ are real symmetric matrices, then
		\begin{align*}
			\eigmin(A)\geq \eigmin(B)-\|A-B\|_2,
		\end{align*}
		which holds because
		\begin{align*}
		\eigmin(A) = \eigmin(B + (A - B))
        & {} \geq \eigmin(B) + \eigmin(A-B)
        \\
        & {} \geq \eigmin(B) - \eigmax(B-A)
        \\
        & {} \geq \eigmin(B) - \|B-A\|_2 = \eigmin(B) - \|A-B\|_2.
		\end{align*}
		Thus,
		\begin{align*}
			&
			\inf_{\paramall\in \Psi(\paramall_0)}
			\left(\eigmin(\ntgram^{(1)}(\bX;\paramall))\right)
			\\
			& \qquad {}\geq
			\eigmin(\ntgram^{(1)}(\bX;\paramall_0))
			- \sup_{\paramall\in \Psi(\paramall_0)} \left\|\ntgram^{(1)}(\bX;\paramall) - \ntgram^{(1)}(\bX;\paramall_0)\right\|_2
			\\
			&\qquad {} \geq
			\eigmin(\ntgram^{(1)}(\bX;\paramall_0))
			-
			\left(\frac{n}{d}\sum_{j=1}^{\nnodes}\lamj^{(1)}c_{\nnodes,j}
			+
			\frac{2 n}{d}
			\sqrt{\gamma\sum_{j=1}^{\nnodes}\lamj^{(1)}c_{\nnodes,j}}\right)
			\\
			& \qquad {} =
			\eigmin(\ntgram^{(1)}(\bX;\paramall_0))
			-
			\left(\frac{n\gamma}{d\nnodes}\sum_{j=1}^{\nnodes} c_{\nnodes,j}
			+
			\frac{2 n \gamma}{d \nnodes^{1/2}}
			\sqrt{\sum_{j=1}^{\nnodes} c_{\nnodes,j}}\right)
		\end{align*}
		holds with probability at least $1-\delta$. \cref{eqn:3.2:ReLU:0} then follows from the fact that for all $\paramall$, $\eigmin(\ntgram(\bX;\paramall))\geq\eigmin(\ntgram^{(1)}(\bX;\paramall))$.
	\end{proof}

\subsection{Lemma on a sufficient condition for \cref{thm:upper-bound:relu} - ReLU case}

	We now bring together the results from \cref{prop:ntk-initialisation,lem:exp decay,lem:3.2:ReLU}, and identify a sufficient condition for \cref{thm:upper-bound:relu}, which corresponds to the condition in Lemma 3.4 in~\cite{Du2019}.

 \begin{lemma}
	\label{lem:putting-together-via-contradiction-argument:ReLU}
        Consider $\delta \in (0,1)$.
		Assume that \cref{assump:data,assump:init_V} hold, the activation function is ReLU, and $c_{\nnodes,j} > 0$ for all $j \in [\nnodes]$. Also, assume that $\gamma > 0$ and
		\[
		\nnodes \geq \max\left(
		\left(\frac{8n \log \frac{4n}{\delta}}{d \kappa_n}\right),\;
		\left(\frac{8n}{d\kappa_n}\sum_{j=1}^{\nnodes} c_{\nnodes,j}\right),\;
		\left(\frac{16^2n^2}{d^2\kappa_n^2} \sum_{j=1}^{\nnodes} c_{\nnodes,j}\right)\right).
		\]
		Define
		\[
		R'_{\nnodes,j} = \sqrt{\frac{n\lamj}{\din}} \left\|\by-\bfu_0\right\|\frac{4}{\gamma \kappa_n}
		\quad\text{and}\quad
		R_{\nnodes,j} = \frac{\delta^2 c_{\nnodes,j}}{64}.
		\]
If $R'_{\nnodes,j} <  R_{\nnodes,j}$ for all $j \in [\nnodes]$ with probability at least $1-\frac{\delta}{2}$, then on an event with probability at least $1-\delta$, we have that for all $j \in [\nnodes]$, $R'_{\nnodes,j} <  R_{\nnodes,j}$ and the following properties also hold for all $t \geq 0$:
		\begin{enumerate}
			\item[(a)] $\eigmin(\ntgram(\bX;\paramall_t)) \geq \frac{\gamma \kappa_n}{4}$;
			\item[(b)] $L_\nnodes(\paramall_t) \leq e^{-(\gamma \kappa_n t)/2} L_\nnodes(\paramall_0)$;
			\item[(c)] $\|\param_{tj} - \param_{0j}\| \leq R'_{\nnodes,j}$ for all $j \in [\nnodes]$; and
			\item[(d)] $\|\ntgram(\bX;\paramall_t)-\ntgram(\bX;\paramall_0)\|_2 \leq \frac{n}{d}\sum_{j=1}^{\nnodes}\lamj c_{\nnodes,j}
			+
			\frac{2\sqrt{2} \cdot n}{d} \,
			\sqrt{\sum_{j=1}^{\nnodes}\lamj c_{\nnodes,j}}$.
		\end{enumerate}
\end{lemma}
	\begin{proof}
		Suppose $R'_{\nnodes,j} < R_{\nnodes,j}$ for all $j \in [\nnodes]$ on some event $A'$ having probability  at least $1 - \frac{\delta}{2}$. Also, we would like to instantiate \cref{prop:ntk-initialisation} and \cref{lem:3.2:ReLU} with $\delta/4$, so that each of their claims holds with probability at least $1 - \frac{\delta}{4}$.
  Let $A$ be the intersection of $A'$ with the event that the conjunction of the two claims in \cref{prop:ntk-initialisation} and \cref{lem:3.2:ReLU} hold with $\delta/4$. By the union bound, $A$ has probability at least $1-\delta$.
    We will show that on the event $A$, the four claimed properties of the lemma hold.
		
		It will be sufficient to show that
		\begin{equation}
			\label{eqn:putting-together-via-contradiction-argument:ReLU:0}
			\|\param_{sj} -\param_{0j}\| \leq R_{\nnodes,j} \quad\text{for all $j \in [\nnodes]$ and $s \geq 0$}.
		\end{equation}
		To see why doing so is sufficient, pick an arbitrary $t_0 \geq 0$, and assume the above inequality for all $s \geq 0$.
		Then, by event $A$ and \cref{lem:3.2:ReLU}, for all $0 \leq s \leq t_0$, we have the following upper bound on the change of the Gram matrix from time $0$ to $s$,
		and the following lower bound on the smallest eigenvalue of
		$\ntgram(\bX;\paramall_s)$:
		\begin{align*}
			\left\|\ntgram(\bX;\paramall_s)-\ntgram(\bX;\paramall_0)\right\|_2
			& {} \leq
			\sum_{k = 1}^2
			\left\|\ntgram^{(k)}(\bX;\paramall_s)-\ntgram^{(k)}(\bX;\paramall_0)\right\|_2
			\\
			& {} \leq
			\sum_{k = 1}^2
			\left(\frac{n}{d}\sum_{j=1}^{\nnodes}\lamj^{(k)} c_{\nnodes,j}
			+
			\frac{2 n}{d} \,
			\sqrt{\sum_{j=1}^{\nnodes}\lamj^{(k)}c_{\nnodes,j}}\right)
			\\
			& {} \leq
			\frac{n}{d}\sum_{j=1}^{\nnodes}\lamj c_{\nnodes,j}
			+
			\frac{2\sqrt{2} \cdot n}{d} \,
			\sqrt{\sum_{j=1}^{\nnodes}\lamj c_{\nnodes,j}}
		\end{align*}
		and
		\begin{align*}
			\eigmin(\ntgram(\bX;\paramall_s))
			& {} \geq  \eigmin(\ntgram^{(1)}(\bX;\paramall_0))-
			\left(\frac{n\gamma}{d \nnodes}\sum_{j=1}^{\nnodes} c_{\nnodes,j}
			+
			\frac{2 n \gamma}{d \nnodes^{1/2}}
			\sqrt{\sum_{j=1}^{\nnodes} c_{\nnodes,j}}\right)
			\\
			& {} \geq \frac{\gamma\kappa_n}{2} -
			\frac{\gamma\kappa_n}{4} \cdot
			\left(\frac{1}{\nnodes} \cdot \frac{4n}{d\kappa_n}\sum_{j=1}^{\nnodes} c_{\nnodes,j}
			+
			\frac{1}{\nnodes^{1/2}} \cdot
			\frac{8n}{d\kappa_n}
			\sqrt{\sum_{j=1}^{\nnodes} c_{\nnodes,j}}\right)
			\\
			& {} \geq \frac{\gamma\kappa_n}{2} -
			\frac{\gamma \kappa_n}{4}
			= \frac{\gamma \kappa_n}{4}.
		\end{align*}
		We now apply \cref{lem:exp decay} with $\zeta$ being set to $\frac{\gamma \kappa_n}{2}$, which gives
		$$
		L_\nnodes(\paramall_{t_0})\leq e^{- (\gamma\kappa_n t_0)/2}L_\nnodes(\paramall_0)
		$$
		and
		\begin{align*}
			\|\param_{t_0j} -\param_{0j}\| \leq \sqrt{\frac{n\lamj}{\din}} \left\|\by-\bfu_0\right\|\frac{4}{\gamma\kappa_n} = R'_{\nnodes,j} \quad\text{for all $j \in [\nnodes]$}.
		\end{align*}
		We have just shown that all the four properties in the lemma hold for $t_0$.
		
		It remains to prove \cref{eqn:putting-together-via-contradiction-argument:ReLU:0} under the event $A$ and the assumption that $R'_{\nnodes,j} < R_{\nnodes,j}$ for all $j \in [\nnodes]$ holds on this event.
		Suppose that  \cref{eqn:putting-together-via-contradiction-argument:ReLU:0} fails for some $j \in [\nnodes]$.
		Let
		\[
		t_1 = \inf \left\{t   \;\left|\;  \|\paramj -\param_{0j}\| > R_{\nnodes,j} \text{ for some $j \in [\nnodes]$} \right.\right\}.
		\]
		Then, by the continuity of $\param_{tj}$ on $t$, we have
		\[
		\|\param_{sj} - \param_{0j}\| \leq R_{\nnodes,j}\quad \text{for all $j \in [\nnodes]$ and $0\leq s\leq t_1$}
		\]
		and for some $j_0 \in [\nnodes]$,
		\begin{equation}
			\label{eqn:putting-together-via-contradiction-argument:ReLU:1}
			\|\param_{t_1j_0} - \param_{0j_0}\| = R_{\nnodes,j_0}.
		\end{equation}
		Thus, by the argument that we gave in the previous paragraph, we have
		\[
		\|\param_{t_1j} - \param_{0j}\| \leq R'_{\nnodes,j} \quad \text{for all $j \in [\nnodes]$}.
		\]
		In particular, $\|\param_{t_1j_0} - \param_{0j_0}\| \leq R'_{\nnodes,j_0}$.
		But this contradicts our assumption $R'_{\nnodes,j_0} < R_{\nnodes,j_0}$.
	\end{proof}

 \subsection{Lemma bounding the NTK change and minimum eigenvalue  - Smooth activation case}

We now give  a version of \cref{lem:3.2:ReLU} for the smooth activation case (that is, under \cref{assump:activation}). The proof of this version is similar to the one for Lemma 5 in \cite{Oymak2020}, and uses the three equations (364-366) in \cite{Bartlett2021}.
	\begin{lemma}\label{lem:3.2:smooth-activation}
		Assume that \cref{assump:data,assump:activation,assump:init_V} hold. Let $c_{\nnodes,j}>0$ for every $j \in [\nnodes]$. Then, for any fixed $\paramall=(\param_1^\top,\ldots,\param_\nnodes^\top)^\top$, if it satisfies
		$$
		\|\param_{0j}-\paramj\| \leq \frac{c_{\nnodes,j}}{2}\quad \text{for all $j \in [\nnodes]$,}
		$$
		we have
		\[
		\left\| \ntgram^{(k)}(\mathbf{X};\paramall)-\ntgram^{(k)}(\mathbf{X};\paramall_{0})\right\|_2
		\leq
		\frac{nM^{2}}{4d^{2}}
		\sum_{j=1}^{\nnodes}\lamj^{(k)}c_{\nnodes,j}^{2}+\frac{nM}{d^{3/2}}
		\sqrt{\sum_{j=1}^{\nnodes}\lamj^{(k)}c_{\nnodes,j}^{2}}
		\qquad\text{for all $k \in [2]$}
		\]
		and
		\begin{align}\label{eq:smooth-activation}
			\eigmin(\ntgram(\bX;\paramall))\geq
			\eigmin(\ntgram^{(1)}(\bX;\paramall_0)) -
			\left(\frac{nM^{2}\gamma}{4d^{2}\nnodes}
			\sum_{j=1}^{\nnodes}c_{\nnodes,j}^{2}+\frac{nM \gamma}{d^{3/2}\nnodes^{1/2}}
			\sqrt{\sum_{j=1}^{\nnodes}c_{\nnodes,j}^{2}}\right).
		\end{align}
	\end{lemma}
	Note that this lemma has a deterministic conclusion, although its original counterpart (\cref{lem:3.2:ReLU}) has a probabilistic one.
	\begin{proof}
	        The beginning part of the proof is essentially an abbreviated version of the beginning part of the proof of \cref{lem:3.2:ReLU}. This repetition is intended to help the reader by not forcing her or him to look at the proof of \cref{lem:3.2:ReLU} beforehand.
	
		For $k \in [2]$, let
		\begin{align*}
			f_{\nnodes}^{(k)}(-;\paramall) & : \R^\din \to \R,
			&
			f_{\nnodes}^{(k)}(\mathbf{x};\paramall) & =\sum_{j=1}^{\nnodes}\sqrt{\lamj^{(k)}}\aj\relu(Z_{j}(\mathbf{x};\paramall)),
		\end{align*}
		and define $\nabla_{\paramall}f_{\nnodes}^{(k)}(\mathbf{X};\paramall)$ to be the $n$-by-$(pd)$
		matrix whose $i$-th row is the $md$-dimensional row vector
		$(\nabla_{\paramall}f_{\nnodes}^{(k)}(\mathbf{x}_{i};\paramall))^{\top}$.
		
		For all $k \in [2]$, we have
		\begin{align}
                    \nonumber
                    &
				\left\|\ntgram^{(k)}(\mathbf{X};\paramall)-\ntgram^{(k)}(\mathbf{X};\paramall_{0})\right\|_2
                    \\
                    \nonumber
				& \qquad {} =
				\left\|
				\nabla_{\paramall}f_{\nnodes}^{(k)}(\mathbf{X};\paramall)\nabla_{\paramall}f_{\nnodes}^{(k)}(\mathbf{X};\paramall)^{\top}
				-\nabla_{\paramall}f_{\nnodes}^{(k)}(\mathbf{X};\paramall_{0})\nabla_{\paramall}f_{\nnodes}^{(k)}(\mathbf{X};\paramall_{0})^{\top}
				\right\|_2
				\\
                    \label{eq:ntker-k-bound}
				& \qquad {} \leq
				\left\|
				\nabla_{\paramall}f_{\nnodes}^{(k)}(\mathbf{X};\paramall)-
				\nabla_{\paramall}f_{\nnodes}^{(k)}(\mathbf{X};\paramall_{0})\right\|_2^{2}
				\\
                    \nonumber
				& \qquad
				\phantom{{} \leq {}}
				{} + 2
				\left\| \nabla_{\paramall}f_{\nnodes}^{(k)}(\mathbf{X};\paramall_{0})\right\|_2
				\left\| \nabla_{\paramall}f_{\nnodes}^{(k)}(\mathbf{X};\paramall)-\nabla_{\paramall}f_{\nnodes}^{(k)}(\mathbf{X};\paramall_{0})\right\|_2.
		\end{align}
		To see why this inequality holds, see the proof of \cref{lem:3.2:ReLU}.
		We bound the two terms
		$\left\| \nabla_{\paramall}f_{\nnodes}^{(k)}(\mathbf{X};\paramall_{0})\right\|_2 $ and
		$\left\| \nabla_{\paramall}f_{\nnodes}^{(k)}(\mathbf{X};\paramall)-\nabla_{\paramall}f_{\nnodes}^{(k)}(\mathbf{X};\paramall_{0})\right\|_2$
		in \cref{eq:ntker-k-bound}. We bound the first term as follows:
		\begin{align*}
			\left\| \nabla_{\paramall}f_{\nnodes}^{(k)}(\mathbf{X};\paramall_{0})\right\|_2^{2}
			\leq
			\left\| \nabla_{\paramall}f_{\nnodes}^{(k)}(\mathbf{X};\paramall_{0})\right\|_{F}^{2}
			&
			{} =
			\sum_{i=1}^{n}\sum_{j=1}^{\nnodes}\left\|{\nabla_{\paramj} f_{\nnodes}^{(k)}(\mathbf{x}_{i};\paramall_{0})}\right\|^{2}
			\\
			&
			{} =
			\sum_{i=1}^{n}\sum_{j=1}^{\nnodes}\lamj^{(k)}\left| \relu'(Z_{j}(\mathbf{x}_{i};\paramall_{0}))\right|^{2}\frac{\left\Vert
				\mathbf{x}_{i}\right\Vert ^{2}}{d}\\
			&  \leq\frac{n}{d}\sum_{j=1}^{\nnodes}\lamj^{(k)}\leq\frac{n}{d}\gamma_k
		\end{align*}
  { where $\gamma_1=\gamma$ and $\gamma_2=1-\gamma$.}
		The second inequality uses the assumption that $|\relu'(x)| \leq 1$ for all $x \in \mathbb{R}$ and $\|\bx_i\| \leq 1$ for all $i \in [n]$.
		The third inequality holds because $\sum_{j = 1}^{\nnodes} \lamj^{(k)} \leq \sum_{j=1}^{\nnodes} \lamj = 1$.
		For the second term, we recall that
		$\left|\relu''(x)\right| \leq M$ and so $\relu'$ is $M$-Lipschitz, and also that
		$Z_{j}(\mathbf{x}_{i};\paramall)=\frac{1}{\sqrt{d}}\param_{j}^{\top}\mathbf{x}_{i}$. Using these facts, we
		derive an upper bound for the second term as follows:
		\begin{align*}
			& \left\| \nabla_{\paramall}f_{\nnodes}^{(k)}(\mathbf{X};\paramall)-\nabla_{\paramall} f_{\nnodes}^{(k)}(\mathbf{X};\paramall_{0})\right\|_2^{2}
                \\
			& \qquad
			{} \leq
			\left\| \nabla_{\paramall}f_{\nnodes}^{(k)}(\mathbf{X};\paramall)-\nabla_{\paramall}f_{\nnodes}^{(k)}(\mathbf{X};\paramall_{0})\right\| _{F}^{2}
			\\
			& \qquad
			{} =
			\sum_{i=1}^{n}\sum_{j=1}^{\nnodes}
			\left\|
			{\nabla_{\paramj} f_{\nnodes}^{(k)}(\mathbf{x}_{i};\paramall)}
			- {\nabla_{\paramj} f_{\nnodes}^{(k)}(\mathbf{x}_{i};\paramall_{0})}
			\right\|^{2}
			\\
			& \qquad
			{} =
			\sum_{i=1}^{n}\sum_{j=1}^{\nnodes}
			\left\|
			\sqrt{\lamj^{(k)}}\aj\frac{\mathbf{x}_{i}}{\sqrt{d}}
			\left[  \relu'(Z_{j}(\mathbf{x}_{i};\paramall))-\relu'(Z_{j}(\mathbf{x}_{i};\paramall_{0}))\right]
			\right\|^{2}
			\\
			& \qquad
			{} =
			\frac{1}{d}\sum_{i=1}^{n}
			\left\| \mathbf{x}_{i} \right\|^{2}
			\sum_{j=1}^{\nnodes}
			\lamj^{(k)}\left[  \relu'(Z_{j}(\mathbf{x}_{i};\paramall))-\relu'(Z_{j}(\mathbf{x}_{i};\paramall_{0}))\right]^{2}
			\\
			& \qquad
			{} \leq
			\frac{1}{d}\sum_{i=1}^{n}\sum_{j=1}^{\nnodes}\lamj^{(k)}
			\left[\relu'(Z_{j}(\mathbf{x}_{i};\paramall))-\relu'(Z_{j}(\mathbf{x}_{i};\paramall_{0}))\right]^{2}
			\\
			& \qquad
			{} \leq
			\frac{M^{2}}{d^{2}}\sum_{i=1}^{n}\sum_{j=1}^{\nnodes}
			\lamj^{(k)}\left(\left(\param_{j}-\param_{0j}\right)^{\top}\mathbf{x}_{i}\right)^{2}
			\\
			& \qquad
			{} \leq
			\frac{nM^{2}}{d^{2}}\sum_{j=1}^{\nnodes}\lamj^{(k)}
			\left\| \param_{j}-\param_{0j}\right\|^{2}
			\\
			& \qquad
			{} \leq
			\frac{nM^{2}}{4d^{2}}\sum_{j=1}^{\nnodes}
			\lamj^{(k)} c_{\nnodes,j}^{2}.
		\end{align*}
		The second to last step uses the Cauchy-Schwartz inequality, and the last step uses
		our assumption that $\left\|\param_{j}-\param_{0j}\right\| \leq\frac{ c_{\nnodes,j}}{2}$
		for all $j \in [\nnodes]$. From the derived bounds on the first and second terms in the last line of \cref{eq:ntker-k-bound}, it follows that
		\begin{align*}
			\left\|
			\ntgram^{(k)}(\mathbf{X};\paramall)-\ntgram^{(k)}(\mathbf{X};\paramall_{0})
			\right\|_2
			&
			{} \leq
			\frac{nM^{2}}{4 d^{2}} \sum_{j=1}^{\nnodes}\lamj^{(k)}c_{\nnodes,j}^{2}
			+
			2 \sqrt{\frac{n}{d}\gamma_k}
			\sqrt{\frac{nM^2}{4d^2}\sum_{j=1}^{\nnodes}\lamj^{(k)}c_{\nnodes,j}^2}
			\\
			&
			{} =
			\frac{nM^{2}}{4 d^{2}} \sum_{j=1}^{\nnodes}\lamj^{(k)}c_{\nnodes,j}^{2}
			+
			\frac{nM}{d^{3/2}}
			\sqrt{\gamma_k\sum_{j=1}^{\nnodes}\lamj^{(k)}c_{\nnodes,j}^2}.
		\end{align*}
		
		Finally, as noted in the proof of \cref{lem:3.2:ReLU}, we have
		\begin{align*}
			\eigmin(\ntgram(\bX;\paramall))
			& {} \geq
			\eigmin(\ntgram^{(1)}(\bX;\paramall))
			\\
			& {} \geq
			\eigmin(\ntgram^{(1)}(\bX;\paramall_0)) -
			\left\|\ntgram^{(1)}(\bX;\paramall) - \ntgram^{(1)}(\bX;\paramall_0)\right\|_2.
		\end{align*}
		Thus,
		\[
		\eigmin(\ntgram(\bX;\paramall))
		\geq
		\eigmin(\ntgram^{(1)}(\bX;\paramall_0)) -
		\left(
		\frac{nM^{2} \gamma}{4 d^{2}\nnodes} \sum_{j=1}^{\nnodes}c_{\nnodes,j}^{2}
		+
		\frac{nM \gamma}{d^{3/2}\nnodes^{1/2}}
		\sqrt{\sum_{j=1}^{\nnodes}c_{\nnodes,j}^2}
		\right).
		\]
	\end{proof}

\subsection{Lemma on a sufficient condition for \cref{thm:upper-bound:smooth-activation} - Smooth activation case}

We now give  a version of \cref{lem:putting-together-via-contradiction-argument:ReLU} for the smooth activation case (i.e., under \cref{assump:activation}). It brings together the results from \cref{prop:ntk-initialisation,lem:exp decay,lem:3.2:smooth-activation}, and identifies a sufficient condition for \cref{thm:upper-bound:relu}, which corresponds to the condition in Lemma 3.4 in~\cite{Du2019}.
 \begin{lemma}
		\label{lem:putting-together-via-contradiction-argument:smooth-activation}
		Assume that \cref{assump:data,assump:activation,assump:init_V} hold. Let $\delta \in (0,1)$, and $c_{\nnodes,j} > 0$ for all $j \in [\nnodes]$. Assume that $\gamma > 0$ and
		\[
		\nnodes \geq
		\max\left(\frac{8n\log \frac{2n}{\delta}}{d \kappa_n},\;
		\frac{nM^{2}\delta^2}{8d^{2}\kappa_n} \sum_{j=1}^{\nnodes}c_{\nnodes,j}^{2},\;
		\frac{4 n^2M^2\delta^2}{d^{3}\kappa^2_n}\sum_{j=1}^{\nnodes}c_{\nnodes,j}^2
		\right).
		\]
		For each $j \in [\nnodes]$, define
		\[
		R'_{\nnodes,j} = \sqrt{\frac{n\lamj}{\din}} \left\|\by-\bfu_0\right\|\frac{4}{\gamma \kappa_n}
		\quad\text{and}\quad
		R_{m,j} = \frac{\delta c_{\nnodes,j}}{8}.
		\]
  If $R'_{\nnodes,j} <  R_{\nnodes,j}$ for all $j \in [\nnodes]$ with probability at least $1-\frac{\delta}{2}$, then on an event with probability at least $1-\delta$, we have that for all $j \in [\nnodes]$, $R'_{\nnodes,j} <  R_{\nnodes,j}$ and the following properties also hold for all $t \geq 0$:
		\begin{enumerate}
			\item[(a)] $\eigmin(\ntgram(\bX;\paramall_t)) \geq \frac{\gamma \kappa_n}{4}$;
			\item[(b)] $L_\nnodes(\paramall_t) \leq e^{-(\gamma \kappa_n t)/2} L_\nnodes(\paramall_0)$;
			\item[(c)] $\|\param_{tj} - \param_{0j}\| \leq R'_{\nnodes,j}$ for all $j \in [\nnodes]$; and
			\item[(d)] $\|\ntgram(\bX;\paramall_t)-\ntgram(\bX;\paramall_0)\|_2 \leq
			\frac{nM^{2}\delta^2}{8^2d^{2}}
			\sum_{j=1}^{\nnodes}\lamj c_{\nnodes,j}^{2}+ \frac{nM\delta}{2^{3/2}d^{3/2}}
			\sqrt{\sum_{j=1}^{\nnodes}\lamj c_{\nnodes,j}^{2}}$.
		\end{enumerate}
	\end{lemma}
	\begin{proof}
		The proof is very similar to that of \cref{lem:putting-together-via-contradiction-argument:ReLU}, although the concrete bounds in these proofs
		differ due to the differences between \cref{lem:3.2:ReLU} and \cref{lem:3.2:smooth-activation}.

Suppose $R'_{\nnodes,j} < R_{\nnodes,j}$ for all $j \in [\nnodes]$ on some event $A'$ having probability  at least $1 - \frac{\delta}{2}$. Also, we would like to instantiate \cref{prop:ntk-initialisation} with $\delta/2$, so that its claim holds with probability at least $1 - \frac{\delta}{2}$.
  Let $A$ be the intersection of $A'$ with the event that  the claim in \cref{prop:ntk-initialisation} holds with $\delta/2$. By the union bound, $A$ has probability at least $1-\delta$.
    We will show that on the event $A$, the four claimed properties of the lemma hold.
		
		It will be sufficient to show that
		\begin{equation}
			\label{eqn:putting-together-via-contradiction-argument:smooth-activation:0}
			\|\param_{sj} -\param_{0j}\| \leq R_{m,j} \quad\text{for all $s \geq 0$}.
		\end{equation}
		To see why doing so is sufficient, pick an arbitrary $t_0 \geq 0$, and assume the above inequality for all $s \geq 0$.
		Then, by the event $A$ and \cref{lem:3.2:smooth-activation},  for all $0 \leq s \leq t_0$, we have the following upper bound on the change of the Gram matrix from
		time $0$ to $s$, and the following lower bound on the smallest eigenvalue of $\ntgram(\bX;\paramall_s)$:
		\begin{align*}
			& \left\|\ntgram(\bX;\paramall_s)-\ntgram(\bX;\paramall_0)\right\|_2
			\\
                & \qquad {} \leq
			\left\|\ntgram^{(1)}(\bX;\paramall_s)-\ntgram^{(1)}(\bX;\paramall_0)\right\|_2
			+
			\left\|\ntgram^{(2)}(\bX;\paramall_s)-\ntgram^{(2)}(\bX;\paramall_0)\right\|_2
			\\
			& \qquad {} \leq
			\frac{nM^{2}\delta^2}{64 d^{2}}\sum_{j=1}^{\nnodes}\lamj c_{\nnodes,j}^{2}
			+
			\frac{nM\delta}{2^{3/2}d^{3/2}}\sqrt{\sum_{j=1}^{\nnodes}\lamj c_{\nnodes,j}^{2}}
		\end{align*}
		and
		\begin{align*}
			\eigmin(\ntgram(\bX;\paramall_s))
			& {} \geq  \eigmin(\ntgram^{(1)}(\bX;\paramall_0))-
			\left(
			\frac{nM^{2}\delta^2 \gamma}{64 d^{2}\nnodes} \sum_{j=1}^{\nnodes}c_{\nnodes,j}^{2}
			+
			\frac{nM \delta {\color{blue}\gamma}}{4d^{3/2}\nnodes^{1/2}}
			\sqrt{\sum_{j=1}^{\nnodes}c_{\nnodes,j}^2}
			\right)
			\\
			& {} > \frac{\gamma\kappa_n}{2} - \frac{\gamma\kappa_n}{4}
			\left(
			\frac{1}{\nnodes} \cdot \frac{nM^{2}\delta^2}{16d^{2}\kappa_n} \sum_{j=1}^{\nnodes}c_{\nnodes,j}^{2}
			+
			\frac{1}{\nnodes^{1/2}} \cdot
			\frac{ nM\delta}{d^{3/2}\kappa_n}
			\sqrt{\sum_{j=1}^{\nnodes}c_{\nnodes,j}^2}
			\right)
			\\
			& {} \geq \frac{\gamma\kappa_n}{2} -
			\frac{\gamma \kappa_n}{4} \left(\frac{1}{2} + \frac{1}{2}\right)
			= \frac{\gamma \kappa_n}{4}.
		\end{align*}
		We now apply the version of \cref{lem:exp decay} for the analytic activation $\relu$, with $\zeta$ being set to $\frac{\gamma \kappa_n}{2}$. This application gives
		$$
		L_\nnodes(\paramall_{t_0})\leq e^{- (\gamma\kappa_n t_0)/2}L_\nnodes(\paramall_0)
		$$
		and
		\begin{align*}
			\|\param_{t_0j} -\param_{0j}\| \leq \sqrt{\frac{n\lamj}{\din}} \left\|\by-\bfu_0\right\|\frac{4}{\gamma\kappa_n} = R'_{\nnodes,j} \quad\text{for all $j \in [\nnodes]$}.
		\end{align*}
		We have just shown that all the four properties in the lemma hold for $t_0$.
		
		It remains to prove \cref{eqn:putting-together-via-contradiction-argument:smooth-activation:0} under the event $A$. Suppose that  \cref{eqn:putting-together-via-contradiction-argument:smooth-activation:0} fails for some $j \in [\nnodes]$. Let
		\[
		t_1 = \inf \left\{t   \;\left|\;  \|\paramtj -\param_{0j}\| > R_{m,j} \text{ for some $j \in [\nnodes]$}\right.\right\}.
		\]
		Then, by the continuity of $\param_{tj}$ on $t$, we have
		\[
		\|\param_{sj} - \param_{0j}\| \leq R_{m,j}\quad \text{for all $j \in [\nnodes]$ and $0\leq s\leq t_1$}
		\]
		and for some $j_0 \in [\nnodes]$,
		\begin{equation}
			\label{eqn:putting-together-via-contradiction-argument:1}
			\|\param_{t_1j_0} - \param_{0j_0}\| = R_{\nnodes,j_0}.
		\end{equation}
		Thus, by the argument that we gave in the previous paragraph, we have
		\[
		\|\param_{t_1j} - \param_{0j}\| \leq R'_{\nnodes,j} \quad \text{for all $j \in [\nnodes]$}.
		\]
		In particular, $\|\param_{t_1j_0} - \param_{0j_0}\| \leq R'_{\nnodes,j_0}$.
		But this contradicts our assumption $R'_{\nnodes,j_0} < R_{\nnodes,j_0}$.
	\end{proof}

 \section{Proof of \cref{thm:upper-bound:relu} on the global convergence of gradient flow (ReLU case)}
 \label{sec:proofboundrelu}
 The proof of \cref{thm:upper-bound:relu} essentially follows \cref{lem:putting-together-via-contradiction-argument:ReLU}, which itself follows from the secondary \cref{prop:ntk-initialisation,lem:exp decay,lem:3.2:ReLU}, derived  in \cref{sec:secondarythm:init,sec:secondarythm:dynamics}.  Pick $\delta \in (0,1)$.
	Let
	\[
	D =
	\sqrt{n^2\left(C^2 + \frac{1}{d}\right) \frac{2 \cdot 512^2}{\gamma^2\delta^5 \kappa_n^2 d}}
	\]
	where $C$ is the assumed upper bound on the $|y_i|$'s.
	Assume $\gamma > 0$ and
	\begin{align*}
		\nnodes &
		\geq
		\max\left(
		\left(\frac{8n \log \frac{4n}{\delta}}{\kappa_n d}\right),\;
		\left(\frac{8nD}{d\kappa_n }\right)^2,\;
		\left(\frac{16^2n^2D}{d^2\kappa_n^2}\right)^2
		\right)
	\end{align*}
	and set $c_{\nnodes,j}$ as follows:
	\[
	c_{\nnodes,j}
	=
	\sqrt{\lamj} \cdot
	\sqrt{n^2\left(C^2 + \frac{1}{d}\right) \frac{2 \cdot 512^2}{\gamma^2\delta^5 \kappa_n^2 d}}
	=
	\sqrt{\lamj} \cdot D.
	\]
	Note that
	\begin{align*}
		\left(\frac{8n}{d\kappa_n}\sum_{j=1}^{\nnodes} c_{\nnodes,j}\right)^2
		=
		\left(\frac{8nD}{d\kappa_n}\right)^2 \cdot \left(\sum_{j=1}^{\nnodes} \sqrt{\lamj}\right)^2
		& {} \leq
		\left(\frac{8nD}{d\kappa_n}\right)^2 \cdot \left(\sum_{j=1}^{\nnodes} \lamj\right) \cdot \nnodes
		\\
		& {} =
		\left(\frac{8nD}{d\kappa_n}\right)^2 \cdot \nnodes
		{} \leq \nnodes^2,
	\end{align*}
	and also that
	\begin{align*}
		\left(
		\frac{16^2n^2}{d^2\kappa_n^2}
		\sum_{j=1}^{\nnodes} c_{\nnodes,j}\right)^2
		=
		\left(
		\frac{16^2n^2D}{d^2\kappa_n^2}
		\right)^2
		\cdot
		\left(\sum_{j=1}^{\nnodes} \sqrt{\lamj}\right)^2
		& \leq
		\left(
		\frac{16^2n^2D}{d^2\kappa_n^2 }
		\right)^2
		\cdot
		\left(\sum_{j=1}^{\nnodes} \lamj\right)
		\cdot
		\nnodes
		\\
		& {} =
		\left(
		\frac{16^2n^2D}{d^2\kappa_n^2 }
		\right)^2
		\cdot
		\nnodes
		\leq
		\nnodes^2.
	\end{align*}
	Thus,
	\[
	\nnodes \geq \max\left(
	\left(\frac{8n \log \frac{4n}{\delta}}{d \kappa_n}\right),\;
	\left(\frac{8n}{d\kappa_n}\sum_{j=1}^{\nnodes} c_{\nnodes,j}\right),\;
	\left(\frac{16^2n^2}{d^2\kappa_n^2} \sum_{j=1}^{\nnodes} c_{\nnodes,j}\right)\right).
	\]
	As a result, we can now employ \cref{lem:putting-together-via-contradiction-argument:ReLU}. Thus,
	if we find an event $A'$ such that the probability of $A'$ is at least $1-(\delta/2)$ and
	under $A'$, we have $R'_{\nnodes,j} < R_{\nnodes,j}$, then the conclusion of  \cref{lem:putting-together-via-contradiction-argument:ReLU} holds.
In particular, we may further calculate conclusions (c) and (d)  of  \cref{lem:putting-together-via-contradiction-argument:ReLU} as
	\begin{align*}
		\|\param_{tj} - \param_{0j}\|
		\leq
		R'_{\nnodes,j}	<
		R_{\nnodes,j}
		=
		\frac{\delta^2 c_{\nnodes,j}}{64}
		& {} =
		\frac{\delta^2}{64} \cdot
		\sqrt{\lamj} \cdot
		\sqrt{n^2\left(C^2 + \frac{1}{d}\right) \frac{2 \cdot 512^2}{\gamma^2\delta^5 \kappa_n^2 d}}
		\\
		& {} =
		\frac{8n}{\kappa_n d^{1/2}} \cdot
		\sqrt{\left(C^2 + \frac{1}{d}\right) \frac{2}{\gamma^2\delta}} \cdot
		\sqrt{\lamj},
	\end{align*}
	and
	\begin{align*}
		\|\ntgram(\bX;\paramall_t)-\ntgram(\bX;\paramall_0)\|_2
		& {} \leq \frac{n}{d}\sum_{j=1}^{\nnodes}\lamj c_{\nnodes,j}
		+
		\frac{2\sqrt{2} \cdot n}{d} \,
		\sqrt{\sum_{j=1}^{\nnodes}\lamj c_{\nnodes,j}}
		\\
		& {} = \frac{n}{d} \cdot D \cdot \sum_{j=1}^{\nnodes}\lamj^{3/2}
		+
		\frac{2\sqrt{2} \cdot n}{d} \cdot \sqrt{D} \cdot
		\sqrt{\sum_{j=1}^{\nnodes}\lamj^{3/2}}
		\\
		& {} = \frac{n}{d} \cdot \sqrt{n^2\left(C^2 + \frac{1}{d}\right) \frac{2 \cdot 512^2}{\gamma^2\delta^5 \kappa_n^2 d}}
		\cdot \sum_{j=1}^{\nnodes}\lamj^{3/2}
		\\
		& \quad {} +
		\frac{2\sqrt{2} \cdot n}{d} \cdot \left(n^2\left(C^2 + \frac{1}{d}\right) \frac{2 \cdot 512^2}{\gamma^2\delta^5 \kappa_n^2 d}\right)^{1/4}
		\cdot
		\sqrt{\sum_{j=1}^{\nnodes}\lamj^{3/2}}
		\\
		& {} = \frac{512n^2}{\kappa_nd^{3/2}} \cdot \sqrt{\left(C^2 + \frac{1}{d}\right) \frac{2}{\gamma^2\delta^5}}
		\cdot \sum_{j=1}^{\nnodes}\lamj^{3/2}
		\\
		& \quad {} +
		\frac{64 n^{3/2}}{\kappa_n^{1/2}d^{5/4}} \cdot \left(\left(C^2+ \frac{1}{d}\right) \frac{2}{\gamma^2\delta^5}\right)^{1/4}
		\cdot
		\sqrt{\sum_{j=1}^{\nnodes}\lamj^{3/2}}.
	\end{align*}
	
	It remains to find such an event $A'$. Start by noting that
	\begin{align*}
		\mathbb{E}[\left\|\by-\bfu_0\right\|^2]
		& = \sum_{i = 1}^n \left(y_i^2 - 2y_i\mathbb{E}[f_{\nnodes}(\bfx_i;\paramall_0)] + \mathbb{E}[f_{\nnodes}(\bfx_i;\paramall_0)^2]\right)
		\\
		& = \sum_{i = 1}^n \left(y_i^2 - 2y_i \cdot 0 + \mathbb{E}\left[\frac{1}{d} \sum_{j=1}^{\nnodes} \lamj  (\param_j^\top \bfx_i)^2 \ind_{\{\param_j^\top \bfx_i \geq 0\}} \right]\right)
		\\
		& = \sum_{i = 1}^n \left(y_i^2 + \frac{1}{d} \sum_{j=1}^{\nnodes} \lamj \mathbb{E}\left[(\param_j^\top \bfx_i)^2 \ind_{\{\param_j^\top \bfx_i \geq 0\}} \right]\right)
		\\
		& \leq n\left(C^2 + \frac{1}{d}\right).
	\end{align*}
	Thus, by Markov inequality,  with probability at least $1 - (\delta/2)$,
	\[
	\left\|\by-\bfu_0\right\|^2 < n\left(C^2 + \frac{1}{d}\right) \frac{2}{\delta}.
	\]
	Let $A'$ be the corresponding event for the above inequality. Then, under $A'$, we have
	\begin{align*}
		R'_{\nnodes,j}
		& = \sqrt{\frac{n\lamj}{\din}} \left\|\by-\bfu_0\right\|\frac{4}{\gamma \kappa_n}
		\\
		& <
		\sqrt{\frac{n\lamj}{\din}} \cdot  \sqrt{n\left(C^2 + \frac{1}{d}\right) \frac{2}{\delta}} \cdot \frac{4}{\gamma \kappa_n}
		\\
		& =
		\sqrt{\lamj} \cdot  \sqrt{n^2\left(C^2 + \frac{1}{d}\right) \frac{2 \cdot 4^2}{\gamma^2 \delta \kappa_n^2 d}}
		\\
		& =
		\frac{\delta^2 c_{\nnodes,j}}{128}
		< \frac{\delta^2 c_{\nnodes,j}}{64} = R_{\nnodes,j}.
	\end{align*}
	Thus, $A'$ is the desired event.

\section{Proof of \cref{thm:upper-bound:smooth-activation} on the global convergence of gradient flow (smooth case)}
 \label{sec:proofboundsmooth}

		The proof of the theorem is similar to that of \cref{thm:upper-bound:relu}. It derives from \cref{lem:putting-together-via-contradiction-argument:smooth-activation}, which itself follows from the secondary \cref{prop:ntk-initialisation,lem:exp decay,lem:3.2:smooth-activation}, derived  in \cref{sec:secondarythm:init,sec:secondarythm:dynamics}.
  Recall that
	\begin{equation*}
	C_1 = \sup_{c \in (0,1]}\mathbb{E}[\relu(cz)^2]
	\end{equation*}
	where the expectation is taken over the real-valued random variable $z$ with the distribution $\mathcal{N}(0,1/d)$.
	To see that $C_1$ is finite, note that since $|\relu'(x)| \leq 1$ for all $x \in \mathbb{R}$, we have
	\[
	|\relu(cz) - \relu(0)| \leq |cz| \ \text{for all $c \in (0,1]$}.
	\]
	Thus, for every $c \in (0,1]$,
	\[
	\relu(0) - |cz| \leq \relu(cz) \leq \relu(0) + |cz|,
	\]
	which implies that
	\begin{align*}
	\mathbb{E}[\relu(cz)^2]
	& {} \leq \relu(0)^2 + 2 |\relu(0)| \cdot |c| \cdot \mathbb{E}[|z|] + c^2\mathbb{E}[z^2]
	\\
	& {} \leq \relu(0)^2 + 2 |\relu(0)| \cdot \mathbb{E}[|z|] + \mathbb{E}[z^2].
	\end{align*}
	As a result, $\mathbb{E}[\relu(cz)^2]$ is bounded, so $C_1$ is finite.

  Pick $\delta \in (0,1)$.
		Assume $\gamma > 0$ and
		\begin{align*}
			\nnodes &
			\geq
			\max\left(
			\left(\frac{8n}{\kappa_n\din} \cdot \log \frac{2n}{\delta}\right),\;
			\left(\frac{2^{10} n^3M^{2}}{\kappa_n^3d^3} \cdot  \frac{C^2 + C_1}{\gamma^2\delta}\right),\;
			\left(\frac{2^{15} n^4M^2}{\kappa^4_n d^4} \cdot \frac{C^2 + C_1}{\gamma^2\delta}\right)
			\right)
		\end{align*}
		and instantiate \cref{lem:putting-together-via-contradiction-argument:smooth-activation} using the below $c_{\nnodes,j}$:
		\[
		c_{\nnodes,j}
		=
		\sqrt{\lamj} \cdot
		\sqrt{n^2\left(C^2 + C_1\right) \frac{2 \cdot 64^2}{\gamma^2\delta^3 \kappa_n^2 d}}
		\]
		where $C$ is the assumed upper bound on the $|y_i|$'s. Note that
		\begin{align*}
			\frac{nM^{2}\delta^2}{8d^{2}\kappa_n} \sum_{j=1}^{\nnodes}c_{\nnodes,j}^{2}
			& {} =
			\frac{nM^{2}\delta^2}{8d^{2}\kappa_n} \sum_{j=1}^{\nnodes} \left(\lamj \cdot n^2\left(C^2 + C_1\right) \frac{2 \cdot 64^2}{\gamma^2\delta^3 \kappa_n^2 d}\right)
			\\
			& {} =
			\frac{nM^{2}\delta^2}{8d^{2}\kappa_n} \cdot n^2\left(C^2 + C_1\right) \frac{2 \cdot 64^2}{\gamma^2\delta^3 \kappa_n^2 d} \cdot \sum_{j=1}^{\nnodes} \lamj
			\\
			& {} =
			\frac{2^{10} n^3M^{2}}{\kappa_n^3d^3} \times  \frac{C^2 + C_1}{\gamma^2\delta}
		\end{align*}
		and
		\begin{align*}
			\frac{4n^2M^2\delta^2}{d^{3}\kappa^2_n}\sum_{j=1}^{\nnodes}c_{\nnodes,j}^2
			& {} =
			\frac{4n^2M^2\delta^2}{ d^{3}\kappa^2_n} \sum_{j=1}^{\nnodes} \left(\lamj \cdot n^2\left(C^2 + C_1\right) \frac{2 \cdot 64^2}{\gamma^2\delta^3 \kappa_n^2 d}\right)
			\\
			& {} =
			\frac{4n^2M^2\delta^2}{d^{3}\kappa^2_n} \cdot n^2\left(C^2 + C_1\right) \frac{2 \cdot 64^2}{\gamma^2\delta^3 \kappa_n^2 d} \cdot \sum_{j=1}^{\nnodes} \lamj
			\\
			& {} =
			\frac{2^{15} n^4M^2}{\kappa^4_n d^4} \times \frac{C^2 + C_1}{\gamma^2\delta}.
		\end{align*}
		Thus,
		\[
		\nnodes \geq
		\max\left(\frac{8n\log \frac{2n}{\delta}}{d \kappa_n},\;
		\frac{nM^{2}\delta^2}{8d^{2}\kappa_n} \sum_{j=1}^{\nnodes}c_{\nnodes,j}^{2},\;
		\frac{4 n^2M^2\delta^2}{ d^{3}\kappa^2_n}\sum_{j=1}^{\nnodes}c_{\nnodes,j}^2
		\right).
		\]
		This allows us to employ \cref{lem:putting-together-via-contradiction-argument:smooth-activation}. Hence, it is sufficient
		to find an event $A'$ such that the probability of $A'$ is at least $1-(\delta/2)$ and
		under $A'$, we have $R'_{\nnodes,j} < R_{m,j}$.
		The desired conclusion then follows from the conclusion of  \cref{lem:putting-together-via-contradiction-argument:smooth-activation},
		and the below calculations: if  $\|\param_{tj} - \param_{0j}\| \leq R'_{\nnodes,j}$ and $R'_{\nnodes,j} < R_{m,j}$, then
		\begin{align*}
			\|\param_{tj} - \param_{0j}\|
			<
			R_{m,j}
			& {} = \frac{\delta c_{\nnodes,j}}{8}
			\\
			& {} = \frac{\delta}{8} \cdot  \sqrt{\lamj} \cdot
			\sqrt{n^2\left(C^2 + C_1\right) \frac{2 \cdot 64^2}{\gamma^2\delta^3 \kappa_n^2 d}}
			\\
			& {} =   \sqrt{\lamj} \times \frac{n}{\kappa_n d^{1/2}}
			\sqrt{\frac{128(C^2 + C_1)}{\gamma^2\delta}},
		\end{align*}
		and the upper bound on $\|\ntgram(\bX;\paramall_t)-\ntgram(\bX;\paramall_0)\|_2$ in the conclusion of  \cref{lem:putting-together-via-contradiction-argument:smooth-activation} can be rewritten to
		\begin{align*}
			& \|\ntgram(\bX;\paramall_t)-\ntgram(\bX;\paramall_0)\|_2
			\\
			& \qquad {} \leq
			\frac{nM^{2}\delta^2}{8^2d^{2}}
			\sum_{j=1}^{\nnodes}\lamj c_{\nnodes,j}^{2}+
			\frac{nM\delta}{2^{3/2}d^{3/2}}
			\sqrt{\sum_{j=1}^{\nnodes}\lamj c_{\nnodes,j}^{2}}
			\\
			& \qquad {} =
			\frac{nM^{2}\delta^2}{4^3d^{2}}
			\sum_{j=1}^{\nnodes}\lamj\left(\lamj n^2 \frac{(C^2 + C_1) 2 \cdot 64^2}{\gamma^2\delta^3 \kappa_n^2 d}\right)
                \\
                & \qquad\qquad
			+
			\frac{nM\delta}{2^{3/2}d^{3/2}}
			\sqrt{\sum_{j=1}^{\nnodes}\lamj\left(\lamj n^2 \frac{(C^2+C_1)2 \cdot 64^2}{\gamma^2\delta^3 \kappa_n^2 d}\right)}
			\\
			& \qquad {} =
			\left(
			\frac{n^3M^{2}}{\kappa_n^2 d^{3}}
			\sum_{j=1}^{\nnodes}\lamj^2 \frac{2^7 (C^2 + C_1)}{\gamma^2\delta}\right)
			+
			\frac{n^2M}{\kappa_n d^2}
			\sqrt{\sum_{j=1}^{\nnodes}\lamj^2 \frac{2^{10}(C^2+C_1)}{\gamma^2\delta}}.
		\end{align*}
		
		Note that
		\begin{align*}
			\mathbb{E}[\left\|\by-\bfu_0\right\|^2]
			& = \sum_{i = 1}^n \left(y_i^2 - 2y_i\mathbb{E}[f_{\nnodes}(\bfx_i;\paramall_0)] + \mathbb{E}[f_{\nnodes}(\bfx_i;\paramall_0)^2]\right)
			\\
			& = \sum_{i = 1}^n \left(y_i^2 - 2y_i \cdot 0 + \mathbb{E}\left[\sum_{j=1}^{\nnodes} \lamj  \relu(Z_j(\bfx_i;\paramall_0))^2 \right]\right)
			\\
			& = \sum_{i = 1}^n \left(y_i^2 + \sum_{j=1}^{\nnodes} \lamj \mathbb{E}\left[ \relu(Z_j(\bfx_i;\paramall_0))^2 \right]\right)
			\\
			& \leq n\left(C^2 + C_1\right).
		\end{align*}
		Thus, by Markov inequality,  with probability at least $1 - (\delta/2)$,
		\[
		\left\|\by-\bfu_0\right\|^2 < n\left(C^2 + C_1\right) \frac{2}{\delta}.
		\]
		Let $A'$ be the corresponding event for the above inequality. Then, under $A'$, we have
		\begin{align*}
			R'_{\nnodes,j}
			& = \sqrt{\frac{n\lamj}{\din}} \left\|\by-\bfu_0\right\|\frac{4}{\gamma \kappa_n}
			\\
			& <
			\sqrt{\frac{n\lamj}{\din}} \cdot  \sqrt{n\left(C^2 + C_1\right) \frac{2}{\delta}} \cdot \frac{4}{\gamma \kappa_n}
			\\
			& =
			\sqrt{\lamj} \cdot  \sqrt{n^2\left(C^2 + C_1\right) \frac{2 \cdot 4^2}{\gamma^2 \delta \kappa_n^2 d}}
			\\
			& =
			\frac{\delta c_{\nnodes,j}}{16}
			< \frac{\delta c_{\nnodes,j}}{8} = R_{m,j}.
		\end{align*}
		Thus, $A'$ is the desired event.

\section{Proof of \cref{thm:convergence-gd:smooth-case} on the global convergence of gradient descent (smooth activation)}
\label{appendix:global-convergence-gd}

Our convergence proof follows the structure of the convergence proof of \cite[Theorem 5.1]{Du2019a} with necessary modifications, which in particular account for the changing weights and Gram matrices in our setup.

\subsection{Sketch of the proof}

The proof is by induction on the number of gradient-update steps $s$.
Here is a sketch of the proof for the inductive case. We start by decomposing the error at step $s+1$:
\begin{align}
\nonumber
\Vert \mathbf{y}-\mathbf{u}_{s+1}\Vert^{2}
& {} =
\Vert(\mathbf{y}-\mathbf{u}_s)-(\mathbf{u}_{s+1}-\mathbf{u}_{s})\Vert^{2}
\\
\nonumber
& {} =
\Vert \mathbf{y}-\mathbf{u}_s \Vert^{2}
-2(\mathbf{y}-\mathbf{u}_s)^{\top}(\mathbf{u}_{s+1}-\mathbf{u}_{s})
+ \Vert \mathbf{u}_{s+1}-\mathbf{u}_{s} \Vert^{2}
\\
\label{eqn:gd-inductive:top-level-decomposition}
& {} =
\Vert \mathbf{y}-\mathbf{u}_s\Vert^{2}
-2(\mathbf{y}-\mathbf{u}_s)^{\top}\mathbf{I}_{1}
-2(\mathbf{y}-\mathbf{u}_s)^{\top}\mathbf{I}_{2}+\left\Vert \mathbf{u}_{s+1}-\mathbf{u}_{s}\right\Vert ^{2},
\end{align}
where $\mathbf{I}_1 = \eta \ntgram(\bX;\paramall_s)(\mathbf{y}-\mathbf{u}_s)$
and $\mathbf{I}_2 = (\mathbf{u}_{s+1} - \mathbf{u}_{s} - \mathbf{I}_1)$. We can then show that with high probability, both the third and the fourth terms in \cref{eqn:gd-inductive:top-level-decomposition} are $O(\eta^2)\Vert \mathbf{y} - \mathbf{u}_s\Vert^2$, so that the sum of these terms can be bounded from above by $(\eta \gamma \kappa_n/4)\Vert \mathbf{y} - \mathbf{u}_s\Vert^2$ if $\eta$ is sufficiently small.
On the other hand,
the second term can be bounded using the minimum eigenvalue of the positive definite Gram matrix:
\begin{align*}
-2 (\mathbf{y}-\mathbf{u}_s)^{\top}\mathbf{I}_{1}
& {} =
\left(
-2\eta (\mathbf{y}-\mathbf{u}_s)^{\top}\ntgram(\bX;\paramall_s)(\mathbf{y} - \mathbf{u}_s)
\right)
\\
& {} \leq
-2 \eta\eigmin(\ntgram(\bX;\paramall_s))\Vert \mathbf{y}-\mathbf{u}_s \Vert^{2}.
\end{align*}
We will show that if the network is large enough, with high probability,
$-2 \eta\eigmin(\ntgram(\bX;\paramall_s))$ in the above upper bound is at most
$-3\eta \gamma \kappa_n / 4$. Putting all these together gives the required bound: with high probability,
\begin{align*}
\Vert \mathbf{y} - \mathbf{u}_{s+1}\Vert^2
& {} \leq
\Vert \mathbf{y}-\mathbf{u}_s\Vert^{2}
-2(\mathbf{y}-\mathbf{u}_s)^{\top}\mathbf{I}_{1}
-2(\mathbf{y}-\mathbf{u}_s)^{\top}\mathbf{I}_{2}+\left\Vert \mathbf{u}_{s+1}-\mathbf{u}_{s}\right\Vert ^{2}
\\
& {} \leq
\Vert \mathbf{y}-\mathbf{u}_s\Vert^{2}
-\frac{3\eta\gamma\kappa_n}{4}\Vert \mathbf{y} - \mathbf{u}_s\Vert^2
+\frac{\eta\gamma\kappa_n}{4}\Vert \mathbf{y} - \mathbf{u}_s\Vert^2
\\
& {} \leq
\left(1 - \frac{\eta \gamma \kappa_n}{2}\right) \Vert \mathbf{y} - \mathbf{u}_s\Vert^2
\\
& {} \leq
\left(1 - \frac{\eta \gamma \kappa_n}{2}\right)^{s+1} \Vert \mathbf{y} - \mathbf{u}_0\Vert^2.
\end{align*}
The step of upper-bounding $-2 \eta\eigmin(\ntgram(\bX;\paramall_s))$ by $-3\eta\gamma \kappa_n / 4$ is where we have to account for the changing weights and Gram matrix, and this is where the difference between our proof and that of \cite{Du2019a} lies.

As mentioned already, the Gram matrix $\ntgram(\bX;\paramall_s)$ changes during gradient descent even when the network is very wide, but we will show that despite these changes, its minimum eigenvalue remains lower-bounded by $3 \gamma \kappa_n / 8$ with high probability. This can be done using the decomposition $\ntgram=\ntgrampart1+\ntgrampart2$ in \cref{eq:decompositionNTG} from our proof sketch of the global convergence of gradient flow. At a high level, the reasoning goes like this. The induction hypothesis implies that the weight change $\Vert \paramsj - \param_{0j} \Vert$ is $O(\sqrt{\lamj})$, which is small enough to guarantee that $\ntgrampart1(\bX;\paramall_{s'})$ remains almost constant during training for a large network. This, in turn, implies that the minimum eigenvalue of $\ntgrampart1(\bX;\paramall_{s})$ is lower-bounded by $3 \gamma \kappa_n / 8$ with high probability. Since $\eigmin(\ntgram(\bX;\paramall_s)) \geq \eigmin(\ntgrampart1(\bX;\paramall_s))$, we get the desired upper bound.

\subsection{Two key lemmas}

Before proving the theorem, we show two useful facts. Let
$\mathbf{u}(\paramall)$ be the $n$-dimensional vector
\[
(f_\nnodes(\mathbf{x}_1;\paramall),\ldots,f_\nnodes(\mathbf{x}_n;\paramall))^\top
\]
which consists of the network outputs on the training inputs under the parameters $\paramall$. Note that for each gradient-update step $s \in \mathbb{N} \cup \{0\}$, the vector $\mathbf{u}(\paramall_s)$ is equal to $\mathbf{u}_s$, the notation that we have been using in the main text of the paper. We also define $\mathbf{u}'(\paramall)$ to be the following $n$-by-$\nnodes$ matrix:
\[
\mathbf{u}^\prime(\paramall) = \frac{\partial \mathbf{u}}{\partial \paramall}.
\]
For each $s \in \mathbb{N} \cup \{0\}$, let $\ntgram(s) = \ntgram(\mathbf{X};\paramall_s)$ and
\[
\widetilde{\mathbf{u}}_{s+1}
= \left(\mathbf{u}_s-\eta\frac{d\mathbf{u}_t}{dt}\Big|_{\paramall_t=\paramall_s}\right)
= \left(\mathbf{u}_s-\eta\ntgram(s)(\mathbf{u}_s-\mathbf{y})\right)
\]
be the Euler discretisation of the gradient
flow of the output. Here $\eta > 0$ is the learning rate.

\begin{lemma}
\label{lemma:loss-derivative-bound}
For all $\paramall$ and $j \in [\nnodes]$,
\[
\left\Vert \frac{\partial L_\nnodes(\paramall)}{\partial \param_j} \right\Vert
\leq
\frac{\sqrt{\lamj n}}{\sqrt{d}} \Vert \mathbf{y} - \mathbf{u}(\paramall)\Vert.
\]
\end{lemma}
\begin{proof}
\begin{align*}
\left\Vert \frac{\partial L_\nnodes(\paramall)}{\partial \param_j} \right\Vert
& {} =
\left\Vert
  \sum_{i = 1}^n (u(\paramall)_i - y_i)
    \times \sqrt{\lamj} a_j
    \times \sigma'\left(\frac{\param_j^\top \mathbf{x}_i}{\sqrt{d}} \right)
    \times \frac{\mathbf{x}_i}{\sqrt{d}}
\right\Vert
\\
& {} \leq
  \sum_{i = 1}^n
  \left\Vert (u(\paramall)_i - y_i)
    \times \sqrt{\lamj} a_j
    \times \sigma'\left(\frac{\param_j^\top \mathbf{x}_i}{\sqrt{d}} \right)
    \times \frac{\mathbf{x}_i}{\sqrt{d}}
\right\Vert
\\
& {} \leq
  \frac{\sqrt{\lamj}}{\sqrt{d}} \times
  \sum_{i = 1}^n
  |u(\paramall)_i - y_i|
\\
& {} \leq
  \frac{\sqrt{\lamj n}}{\sqrt{d}} \Vert \mathbf{y} - \mathbf{u}(\paramall)\Vert.
\end{align*}
\end{proof}

The next lemma gives an upper bound on $\Vert \mathbf{y} - \mathbf{u}_{s+1}\Vert$. As we will show shortly, this upper bound will play a crucial role in the proof of \cref{thm:convergence-gd:smooth-case}.

\begin{lemma}
\label{lemma:gd-smooth-inductivecase}
Assume \cref{assump:data,assump:activation,assump:init_V}. Then, for all $s \in \mathbb{N} \cup \{0\}$, we have
\begin{equation}
\label{eqn:gd-smooth-inductivecase0}
\Vert \mathbf{y} - \mathbf{u}_{s+1}\Vert^2
 \leq
\left(1 - 2\eta\eigmin(\ntgram(s)) + \frac{2 \eta^2 M n^{3/2}}{d^2} \Vert \mathbf{y} - \mathbf{u}_s\Vert + \frac{\eta^2n^2}{d^2} \right) \times \Vert \mathbf{y} - \mathbf{u}_s\Vert^2.
\end{equation}
\end{lemma}
\begin{proof}
Write
\[
\mathbf{u}_{s+1}-\mathbf{u}_s
=
\underset{\mathbf{I}_{1}}{\underbrace{\widetilde{\mathbf{u}}_{s+1}-\mathbf{u}_s}}
+
\underset{\mathbf{I}_{2}}{\underbrace{\mathbf{u}_{s+1}-\widetilde{\mathbf{u}}_{s+1}}}.
\]
Then, we have
\begin{align*}
\left\Vert \mathbf{y}-\mathbf{u}_{s+1}\right\Vert^{2}
& {} =
\left\Vert(\mathbf{y}-\mathbf{u}_s)-(\mathbf{u}_{s+1}-\mathbf{u}_s)\right\Vert^{2}
\\
& {} =
\left\Vert \mathbf{y}-\mathbf{u}_s\right\Vert^{2}
-
2(\mathbf{y}-\mathbf{u}_s)^{\top}(\mathbf{u}_{s+1}-\mathbf{u}_s)
+
\left\Vert \mathbf{u}_{s+1} - \mathbf{u}_s\right\Vert^{2}
\\
&{} =
\left\Vert \mathbf{y}-\mathbf{u}_s\right\Vert^{2}
-
2(\mathbf{y}-\mathbf{u}_s)^{\top}\mathbf{I}_{1}
-
2(\mathbf{y}-\mathbf{u}_s)^{\top}\mathbf{I}_{2}
+
\left\Vert \mathbf{u}_{s+1}-\mathbf{u}_s\right\Vert^{2}.
\end{align*}
Since the Gram matrix $\ntgram(s)$ is positive definite and $\eta > 0$, we have
\begin{align*}
(\mathbf{y}-\mathbf{u}_s)^{\top}\mathbf{I}_{1}
 {} =
(\mathbf{y}-\mathbf{u}_s)^{\top}\left(\widetilde{\mathbf{u}}_{s+1}-\mathbf{u}_s\right)
& {} =
\eta(\mathbf{y}-\mathbf{u}_s)^{\top}\ntgram(s)(\mathbf{y}-\mathbf{u}_s)
\\
& {} \geq
\eta\eigmin(\ntgram(s))\left\Vert \mathbf{y}-\mathbf{u}_s\right\Vert^{2}.
\end{align*}
We now get a bound on $\mathbf{I}_{2}$. Note that
$\ntgram(s) = \mathbf{u}^{\prime}_s(\mathbf{u}^{\prime}_s)^{\top}$
where $\mathbf{u}^{\prime}_s=\mathbf{u}^\prime(\paramall_s) = \frac{\partial\mathbf{u}}{\partial\paramall}\big|_{\paramall=\paramall_s}$. Let
\[
L_\nnodes'(\paramall)
= \frac{\partial L_\nnodes(\paramall)}{\partial \paramall}
= \sum_{i = 1}^n (u(\paramall)_{i} - y_i) u'(\paramall)_{i}
= \mathbf{u}'(\paramall)^\top (\mathbf{u}(\paramall) - \mathbf{y})
\]
and
\[
L_\nnodes'(s) = L_\nnodes'(\paramall_s).
\]
Then,
\begin{align*}
\mathbf{I}_{2}
& {} =
\mathbf{u}_{s+1}-\mathbf{u}_s + \eta\mathbf{u}^{\prime}_s(\mathbf{u}^{\prime}_s)^{\top}(\mathbf{u}_s-\mathbf{y})
\\
& {} =
\left(-\int_{r=0}^{\eta}\Big(\mathbf{u}^{\prime}\big(\paramall_s-rL_\nnodes^{\prime}(s)\big)\Big) L_\nnodes^{\prime}(s)\,dr\right)
+
\eta\mathbf{u}^{\prime}_s(\mathbf{u}^{\prime}_s)^{\top}(\mathbf{u}_s-\mathbf{y})
\\
& {} =
\int_{r=0}^{\eta}
\Big( \mathbf{u}^{\prime}_s -\mathbf{u}^{\prime}(\paramall_s - rL_\nnodes^{\prime}(s))\Big)L_\nnodes^{\prime}(s)\,  dr.
\end{align*}
Also,
\[
\left\Vert L_\nnodes^{\prime}(s)\right\Vert
= \left\Vert \sum_{i=1}^{n} (y_{i}-u_{si})u_{si}^\prime\right\Vert
 \leq\sum_{i=1}^{n}\left\vert y_{i}-u_{si}\right\vert \left\Vert u_{si}^{\prime}\right\Vert
\]
and
\[
\left\Vert u_{si}^{\prime} \right\Vert ^{2}
=
\sum_{j=1}^{m}
\lamj a_{j}^{2}\left(\sigma^{\prime}\left(\frac{\param_{sj}^\top \mathbf{x}_i}{\sqrt{d}}\right)\right)^{2}\frac{\Vert \mathbf{x}_{i}\Vert^{2}}{d}
\leq
\frac{1}{d},
\]
since $\sum_j \lamj = 1$, $a_j \in \{-1,+1\}$,  $\sigma'$ is $1$-Lipschitz, and $\Vert \mathbf{x}_i\Vert \leq 1$.
Hence, by Cauchy-Schwarz,
\[
\left\Vert L^{\prime}_\nnodes(s)\right\Vert
\leq
\frac{1}{\sqrt{d}}\sum_{i=1}^{n}\left\vert y_{i}-u_{si}\right\vert
\leq
\frac{\sqrt{n}}{\sqrt{d}}\left\Vert \mathbf{y}-\mathbf{u}_s\right\Vert.
\]

Let $\paramall_{(s,r)}=\paramall_s - rL^{\prime}_\nnodes(s)$. For $j \in [\nnodes]$, write $\param_{(s,r)j}$ for the part of $\paramall_{(s,r)}$ going to the $j$-th node.
Then, for all $i \in [n]$,
\begin{align*}
\left\Vert u^{\prime}_{si} - u^{\prime}(\paramall_{(s,r)})_i\right\Vert ^{2}
& {}=
\sum_{j=1}^{m}\lamj a_{j}^{2}
\left(
  \sigma^{\prime}\left(\frac{\param_{sj}^\top \mathbf{x}_{i}}{\sqrt{d}}\right)
  - \sigma^{\prime}\left(\frac{\param_{(s,r)j}^\top \mathbf{x}_i}{\sqrt{d}}\right)
\right)^2
\frac{\left\Vert \mathbf{x}_{i} \right\Vert^{2}}{d}
\\
&  {} \leq
M^{2}
\sum_{j=1}^{m}\lamj a_{j}^{2}
\left(\left(\param_{sj} -\param_{(s,r)j}\right)^{\top}\mathbf{x}_{i}\right)^{2}
\frac{\Vert \mathbf{x}_i\Vert^2}{d^{2}}
\\
&  {} \leq
\frac{M^{2}}{d^2}
\sum_{j=1}^{m}\lamj
\left\Vert \param_{sj} -\param_{(s,r)j} \right\Vert^2
\\
& {} \leq
\frac{M^{2}}{d^{2}}\left\Vert \paramall_s-\paramall_{(s,r)}\right\Vert^{2}.
\end{align*}
The first inequality follows from the $M$-Lipschitz continuity of $\sigma'$, and the next inequality
from $a_j \in \{-1,1\}$, $\|\mathbf{x}_i\|\le 1$, and Cauchy-Schwartz. The last inequality
uses the fact that $\sum_j\lamj= 1$.
Finally, for all $0 \leq r \leq \eta$,
\[
\left\Vert \paramall_s -\paramall_{(s,r)}\right\Vert =r\left\Vert L^{\prime}_\nnodes(s)\right\Vert
\leq\eta\frac{\sqrt{n}}{\sqrt{d}}\left\Vert \mathbf{y}-\mathbf{u}_s\right\Vert.
\]
Thus,
\begin{align*}
\Vert \mathbf{I}_{2} \Vert^2
& {} =
\sum_{i = 1}^n
\left(
\int_{r=0}^{\eta}
\Big( u_{si}^{\prime}- u^{\prime}(\paramall_{(s,r)})_i\Big)^\top L^{\prime}_\nnodes(s)\,  dr
\right)^2
\\
& {} \leq
\sum_{i=1}^n
\left(
\int_{r=0}^{\eta}
\left|\Big( u_{si}^{\prime}-u^{\prime}(\paramall_{(s,r)})_i\Big)^\top L^{\prime}_\nnodes(s)\right|  dr
\right)^2
\\
& {} \leq
\sum_{i=1}^n
\left(
\int_{r=0}^{\eta}
\Vert u_{si}^{\prime} - u^{\prime}(\paramall_{(s,r)})_i\Vert \times \Vert L^{\prime}_\nnodes(s)\Vert\,  dr
\right)^2
\\
& {} \leq
\sum_{i=1}^n
\left(
\int_{r=0}^{\eta}
\frac{\eta M \sqrt{n}}{d^{3/2}} \Vert \mathbf{y} - \mathbf{u}_s\Vert
\times \frac{\sqrt{n}}{\sqrt{d}} \Vert \mathbf{y} - \mathbf{u}_s\Vert\, dr
\right)^2
\\
& {} =
\frac{\eta^4 M^2n^3}{d^4} \Vert \mathbf{y} - \mathbf{u}_s\Vert^4
\\
& {} =
\left( \frac{\eta^2 M n^{3/2}}{d^2} \left\Vert \mathbf{y}-\mathbf{u}_s\right\Vert ^{2}\right)^2.
\end{align*}
As the upper bound depends quadratically on $\eta$, we can choose it
small enough for gradient descent to converge, as we will show in the proof of \cref{thm:convergence-gd:smooth-case} in the next subsection.

Recall that $\Vert \mathbf{y} - \mathbf{u}_{s+1}\Vert^2$ can be expressed as the sum of four terms:
\begin{equation}
\label{eqn:gd-smooth-inductivecase1}
\left\Vert \mathbf{y}-\mathbf{u}_{s+1}\right\Vert ^{2}
=
\left\Vert \mathbf{y}-\mathbf{u}_s\right\Vert ^{2}
-2(\mathbf{y}-\mathbf{u}_s)^{\top}\mathbf{I}_{1}
-2(\mathbf{y}-\mathbf{u}_s)^{\top}\mathbf{I}_{2}+\left\Vert \mathbf{u}_{s+1}-\mathbf{u}_s\right\Vert ^{2}.
\end{equation}
Thus far, we have bounded the second and third terms on the RHS of \cref{eqn:gd-smooth-inductivecase1}:
\begin{align*}
-2(\mathbf{y}-\mathbf{u}_s)^{\top}\mathbf{I}_{1}
& \leq
- 2\eta \eigmin(\ntgram(s)) \Vert \mathbf{y} - \mathbf{u}_s \Vert^2,
\\
-2(\mathbf{y}-\mathbf{u}_s)^{\top}\mathbf{I}_{2}
& \leq
2 \Vert \mathbf{y} - \mathbf{u}_s\Vert \Vert \mathbf{I}_2\Vert
\leq
\left(\frac{2\eta^2 M n^{3/2}}{d^2} \Vert \mathbf{y} - \mathbf{u}_s\Vert^3\right).
\end{align*}
These bounds lead to the first three terms in the claimed upper bound of \cref{eqn:gd-smooth-inductivecase0}. It remains to get an  appropriate upper bound of the fourth term on the RHS of \cref{eqn:gd-smooth-inductivecase1}.

Using the bound on the derivative of the loss in \cref{lemma:loss-derivative-bound}, we complete the proof:
\begin{align*}
\left\Vert \mathbf{u}_{s+1}-\mathbf{u}_s\right\Vert^{2}
& {} =
\sum_{i=1}^n ( u_{(s+1)i} - u_{si})^2
\\
& {} =
\sum_{i=1}^n
  \left( \sum_{j = 1}^\nnodes \sqrt{\lamj} a_j
    \left(
      \sigma\left(\frac{\param_{(s+1)j}^\top \mathbf{x}_i}{\sqrt{d}}\right)
      - \sigma\left(\frac{\param_{sj}^\top \mathbf{x}_i}{\sqrt{d}}\right)
    \right)
  \right)^2
\\
& {} \leq
\sum_{i=1}^n
  \left( \sum_{j = 1}^\nnodes \sqrt{\lamj} a_j
    \left|
      \sigma\left(\frac{\param_{(s+1)j}^\top \mathbf{x}_i}{\sqrt{d}}\right)
      - \sigma\left(\frac{\param_{sj}^\top \mathbf{x}_i}{\sqrt{d}}\right)
    \right|
  \right)^2
\\
& {} \leq
\sum_{i=1}^n
  \left( \sum_{j = 1}^\nnodes \sqrt{\lamj} a_j
    \left|
      \frac{\param_{(s+1)j}^\top \mathbf{x}_i}{\sqrt{d}}
      - \frac{\param_{sj}^\top \mathbf{x}_i}{\sqrt{d}}
    \right|
  \right)^2
\\
& {} \leq
\sum_{i=1}^n
  \left( \sum_{j = 1}^\nnodes  \frac{\sqrt{\lamj} a_j}{\sqrt{d}}
    \Vert \param_{(s+1)j} - \param_{sj} \Vert \Vert \mathbf{x}_i \Vert
  \right)^2
\\
& {} \leq
\left(\sum_{i=1}^n \Vert \mathbf{x}_i \Vert^2\right) \times
  \left( \sum_{j = 1}^\nnodes  \frac{\sqrt{\lamj} a_j}{\sqrt{d}}
    \times
    \Vert \param_{(s+1)j} - \param_{sj} \Vert
  \right)^2
\\
& {} \leq
n \times
  \left( \sum_{j = 1}^\nnodes  \frac{\sqrt{\lamj} a_j}{\sqrt{d}}
    \times
    \left\Vert \eta \frac{\partial L_\nnodes(\paramall_s)}{\partial \param_{sj}} \right\Vert
  \right)^2
\\
& {} \leq
n \times
  \left( \sum_{j = 1}^\nnodes  \frac{\sqrt{\lamj} a_j}{\sqrt{d}}
    \times
    \frac{\eta\sqrt{\lamj n}}{\sqrt{d}} \Vert \mathbf{y} - \mathbf{u}_s\Vert
  \right)^2
\\
& \quad {} \leq
\frac{\eta^2 n^2}{d^2} \Vert \mathbf{y} - \mathbf{u}_s\Vert^2
  \left( \sum_{j = 1}^\nnodes  \lamj
  \right)^2
\\
& \quad {} =
\frac{\eta^2 n^2}{d^2} \Vert \mathbf{y} - \mathbf{u}_s\Vert^2.
\end{align*}
\end{proof}

\subsection{Proof of \cref{thm:convergence-gd:smooth-case}}
Using the lemmas we have just shown, we will prove global convergence of gradient descent.
Recall the assumed bound $C$ on $|y_i|$ for every $i \geq 1$ in \cref{assump:data}, and also
\[
C_1 = \sup_{c \in (0,1]} \mathbb{E}[\sigma(cz)^2]
\]
where the expectation is taken over the real-valued random variable $z$ distributed as $\mathcal{N}(0,1/d)$. As shown in \cref{sec:proofboundsmooth}, $C_1$ is finite.

By the argument in \cref{sec:proofboundsmooth} again, there exists an event $E_1$ such that $E_1$
happens with probability at least $1 - (\delta/2)$ and conditioned on $E_1$, we have
\begin{equation}
\label{eqn:convergence-gd:smooth-case:pf1}
\Vert \mathbf{y} - \mathbf{u}_0 \Vert < \sqrt{n (C^2 + C_1) \frac{2}{\delta}}.
\end{equation}
Meanwhile, by \cref{prop:ntk-initialisation}, there is an event $E_2$ such that $E_2$ happens with probability at least $1 - (\delta/2)$ and
conditioned on $E_2$, we have
\begin{equation}
\label{eqn:convergence-gd:smooth-case:pf0}
\eigmin(\ntgram(0)) > \frac{\gamma \kappa_n}{2}.
\end{equation}
Let $E_3$ be the event that is the conjunction of $E_1$ and $E_2$. This event happens with probability at least $1 - \delta$, and under this event, \cref{eqn:convergence-gd:smooth-case:pf0,eqn:convergence-gd:smooth-case:pf1} both hold.

Condition on $E_3$. We prove the inequality in \cref{eqn:convergence-gd:smooth-case:0} by induction on $s$. The base case of $s = 0$ is immediate. To prove the inductive case, assume that $s \geq 1$, and that the inequality in  \cref{eqn:convergence-gd:smooth-case:0} holds for all $s' = 0,1,\ldots, s-1$.

Let $\alpha = \eta \gamma \kappa_n/2$ and $\beta = (1 - \alpha)^{1/2}$ and
\begin{align*}
c_{\nnodes,j} & = \frac{\eta n}{1-\beta} \sqrt{\frac{8\lamj(C^2+C_1)}{\delta d}}.
\end{align*}
Then,
\begin{align*}
\sum_{j = 1}^\nnodes c_{\nnodes,j}^2
& {} =
\left(\frac{\eta^2 n^2}{(1-\beta)^2} \frac{8(C^2+C_1)}{\delta d} \sum_{j = 1}^\nnodes \lamj\right)
=
\left(\frac{\eta^2 n^2}{(1-\beta)^2} \frac{8(C^2+C_1)}{\delta d}\right).
\end{align*}
Note that for all $j \in [\nnodes]$,
\begin{align*}
\Vert \param_{sj} - \param_{0j} \Vert
& {} \leq
\sum_{s' = 0}^{s-1} \Vert \param_{(s'+1)j} - \param_{s'j} \Vert
\\
& {} \leq
\sum_{s' = 0}^{s-1} \eta \left\Vert \frac{\partial L_\nnodes(\paramall_{s'})}{\partial \param_{s'j}} \right\Vert
\\
& {} \leq
\sum_{s' = 0}^{s-1} \eta \sqrt{\frac{\lamj n}{d}} \Vert \mathbf{y} - \mathbf{u}_{s'} \Vert
\\
& {} \leq
\eta \sqrt{\frac{\lamj n}{d}}
\sum_{s' = 0}^{s-1}  (1-\alpha)^{s'/2}\Vert \mathbf{y} - \mathbf{u}_0 \Vert
\\
& {} \leq
\frac{\eta}{1-\beta}
 \sqrt{\frac{\lamj n}{d}}
\Vert \mathbf{y} - \mathbf{u}_0 \Vert
\\
& {} \leq
\frac{\eta}{1-\beta}
\sqrt{\frac{\lamj n}{d}}
\sqrt{n (C^2 + C_1) \frac{2}{\delta}}
\\
& {} =
\frac{1}{2}
\times
\frac{\eta n}{1-\beta}
\sqrt{\frac{8\lamj(C^2+C_1)}{\delta d}}
{} =
\frac{c_{\nnodes,j}}{2}
\end{align*}
where the third inequality uses the bound shown in \cref{lemma:loss-derivative-bound}, the fourth inequality follows
from the induction hypothesis, and the sixth inequality uses the bound in \eqref{eqn:convergence-gd:smooth-case:pf1}.
Thus,
by \cref{lem:3.2:smooth-activation} with $c_{\nnodes,j}$ from above and the lower bound on the minimum eigenvalue in \cref{eqn:convergence-gd:smooth-case:pf0}, we have
\begin{align*}
& \eigmin(\ntgram(s))
\\
& \qquad {} \geq
\eigmin(\ntgram^{(1)}(\mathbf{X};\mathbf{W}_0))
-
\left(
\frac{n M^2 \gamma}{4d^2 m} \sum_{j=1}^\nnodes c_{\nnodes,j}^2 + \frac{nM\gamma}{d^{3/2}m^{1/2}} \sqrt{\sum_{j=1}^\nnodes c_{\nnodes,j}^2}
\right)
\\
& \qquad {} =
\frac{\gamma \kappa_n}{2}
-
\left(
\frac{n M^2 \gamma}{4d^2 m} \left(\frac{\eta^2 n^2}{(1-\beta)^2} \frac{8(C^2+C_1)}{\delta d}\right)
+ \frac{nM\gamma}{d^{3/2}m^{1/2}} \sqrt{\frac{\eta^2 n^2}{(1-\beta)^2} \frac{8(C^2+C_1)}{\delta d}}
\right)
\\
& \qquad {} =
\frac{\gamma \kappa_n}{2}
-
\left(
\frac{2 \eta^2 n^3 M^2 \gamma (C^2+C_1)}{d^3 m (1 - \beta)^2 \delta}
+ \frac{\sqrt{8} \eta n^2M\gamma (C^2 + C_1)^{1/2}}{d^2m^{1/2}(1-\beta)\delta^{1/2}}
\right).
\end{align*}
Meanwhile, by \cref{lemma:gd-smooth-inductivecase}, the induction hypothesis, and \cref{eqn:convergence-gd:smooth-case:pf1},
\begin{align*}
& \Vert \mathbf{y} - \mathbf{u}_{s+1}\Vert^2
\\
&
\ \
{} \leq
\left(1 - 2\eta\eigmin(\ntgram(s)) + \frac{2 \eta^2 M n^{3/2}}{d^2} \Vert \mathbf{y} - \mathbf{u}_s\Vert + \frac{\eta^2n^2}{d^2} \right) \Vert \mathbf{y} - \mathbf{u}_s\Vert^2
\\
&
\ \
{} \leq
\left(1 - 2\eta\eigmin(\ntgram(s)) + \frac{2 \eta^2 M n^{3/2}}{d^2} (1-\alpha)^{s/2}\Vert \mathbf{y} - \mathbf{u}_0\Vert + \frac{\eta^2n^2}{d^2} \right)  \Vert \mathbf{y} - \mathbf{u}_s\Vert^2
\\
&
\ \
{} \leq
\left(1 - 2\eta\eigmin(\ntgram(s)) + \frac{2 \eta^2 M n^{3/2}}{d^2} (1-\alpha)^{s/2}\sqrt{n (C^2 + C_1) \frac{2}{\delta}} + \frac{\eta^2n^2}{d^2} \right)  \Vert \mathbf{y} - \mathbf{u}_s\Vert^2.
\end{align*}
Thus, we can complete the proof of this inductive case if we show that
\[
\left(
 2\eta\eigmin(\ntgram(s)) - \frac{2 \eta^2 M n^{3/2}}{d^2} (1-\alpha)^{s/2}\sqrt{n (C^2 + C_1) \frac{2}{\delta}} -
 \frac{\eta^2n^2}{d^2}\right)
 \geq
 \frac{\eta \gamma \kappa_n}{2}
\]
which is equivalent to
\[
 \eigmin(\ntgram(s))
 \geq
 \left(
 \frac{\eta M n^{3/2}}{d^2} (1-\alpha)^{s/2}\sqrt{n (C^2 + C_1) \frac{2}{\delta}}
 +
 \frac{\eta n^2}{2d^2}
 +
 \frac{\gamma \kappa_n}{4}\right).
\]
We will show this sufficient condition by proving the following stronger inequality (stronger because of the lower bound on
$\eigmin(\ntgram(s))$ that we have derived above):
\begin{multline*}
\frac{\gamma \kappa_n}{2}
-
\left(
\frac{2 \eta^2 n^3 M^2 \gamma (C^2+C_1)}{d^3 \nnodes (1 - \beta)^2 \delta}
+ \frac{\sqrt{8} \eta n^2M\gamma (C^2 + C_1)^{1/2}}{d^2\nnodes^{1/2}(1-\beta)\delta^{1/2}}
\right)
\\
{} \geq
 \left(
 \frac{\eta M n^{3/2}}{d^2} (1-\alpha)^{s/2}\sqrt{n (C^2 + C_1) \frac{2}{\delta}}
 +
 \frac{\eta n^2}{2d^2}
 +
 \frac{\gamma \kappa_n}{4}\right),
\end{multline*}
which is equivalent to
\begin{align*}
\frac{\gamma \kappa_n}{4}
\geq
 \bigg(
 \frac{2 \eta^2 n^3 M^2 \gamma (C^2+C_1)}{d^3 \nnodes (1 - \beta)^2 \delta}
 & {} +
 \frac{\sqrt{8} \eta n^2M\gamma (C^2 + C_1)^{1/2}}{d^2\nnodes^{1/2}(1-\beta)\delta^{1/2}}
 \\
 & {} +
 \frac{\eta M n^{3/2}}{d^2} (1-\alpha)^{s/2}\sqrt{n (C^2 + C_1) \frac{2}{\delta}}
 +
 \frac{\eta n^2}{2d^2}\bigg).
\end{align*}
But the four summands on the RHS of the above inequality are at most $\gamma \kappa_n / 16$ by
the assumed upper bound on $\eta$, the assumed lower bound on $\nnodes$, and the fact that $(1-\alpha) \leq 1$.
Thus, the inequality from above holds, as desired.

\section{Proofs of the results of \cref{sec:featurelearning} on feature learning (smooth case)}
\label{sec:featurelearning-proof}

\subsection{Proofs of \Cref{sec:featurelearninglinear} (linear activation)}

\subsubsection{Proof of \cref{thm:featurelearninglinear}}

\label{sec:app:linear}

Consider a linear activation $\sigma(z)=z$. The model is therefore defined as
\begin{align*}
f_{\nnodes}(\bx;\paramall) &= \frac{1}{\sqrt{\din}}\sum_{j=1}^{\nnodes} \sqrt{\lamj}  \aj \paramj^\top \bx.
\end{align*}

The objective function in \cref{eq:objectivefunction} can be written as
\begin{equation}
L_\nnodes(\paramall) =\frac{1}{2}\|\by - \bfA \paramall  \|^2\label{eq:objectivefunctionlinear}
\end{equation}
where $\by=(y_1,\ldots,y_n)^\top$ and $\bfA$ is the $n\times \nnodes \din$ matrix defined by
\begin{align*}
\bfA
&  =
\frac{1}{\sqrt{d}}\left(
\begin{array}[c]{ccc}%
\sqrt{\lambda_{m,1}}a_1 \bx_{1}^{\top} & \ldots & \sqrt{\lambda_{m,m}}a_m \bx_{1}^{\top}\\
\vdots &  & \vdots \\
\sqrt{\lambda_{m,1}}a_1 \bx_{n}^{\top} & \ldots & \sqrt{\lambda_{m,m}}a_m \bx_{n}^{\top}%
\end{array}
\right)
=
\frac{1}{\sqrt{d}}
(\bfB \otimes \bX),
\end{align*}
where $\otimes$ denotes the Kronecker product and $\bfB=(\sqrt{\lambda_{m,1}}a_1 \ldots \sqrt{\lambda_{m,m}}a_m) \in \mathbb{R}^{1\times m}$. We sometimes view $\bfB$ as a row vector and write $\bfB^\top$ to mean the corresponding $m$-dimensional (column) vector. Let
\[
\bX = \bfU \bfD \bfV^\top
\]
be a reduced SVD of the data matrix $\bX$, where $\bfU$ is a $n\times k$ matrix with orthonormal columns, $\bfD$ is a diagonal $k\times k$ matrix, $\bfV$ is a $d\times k$ matrix with orthonormal columns, and $k\leq \min(n,\din)$ is the rank of $\bfX$. Define
\[
	\bfV'={\frac{1}{\sqrt{\sum_{j=1}^m \lamj}}} \,(\bfB^\top \otimes \bfV) \in \mathbb{R}^{md \times k}.
\]
Note that $\bfV'$ has orthonormal columns as
\[
	(\bfV')^\top \bfV' ={\frac{1}{{\sum_{j=1}^m \lamj}}}\sum_{j=1}^m \lamj a_j^2 (\bfV^\top \bfV) = I_k.
\]
Therefore,
\[
\bfA = \bfU \left(\frac{{ \sqrt{\sum_{j=1}^m \lamj}}}{\sqrt{d}} \bfD \right) (\bfV')^\top
\]
is the reduced SVD of $\bfA$, and
\[
	\bfA \bfA^\top
	= \frac{\sum_{j=1}^m \lamj}{\din} \bX\bX^\top
	= \frac{\sum_{j=1}^m \lamj}{\din} \bfU \bfD^2 \bfU^\top.
\]
If $k<md$, let $\bfV'_\bot$ be a matrix in $\mathbb{R}^{md \times (md - k)}$ that makes the
$md\times md$ matrix $(\bfV',\bfV'_\bot)$ orthonormal; otherwise, let $\bfV'_\bot$ be the $md$ dimensional zero vector.

The solution of \cref{eq:objectivefunctionlinear} under gradient flow or gradient descent with the initialisation $\paramall_0$ is given by
\[
\paramall_\infty
= \bfA^\dagger \bfy + \bfV'_\bot (\bfV'_\bot)^\top \paramall_0
= \frac{\sqrt{d}}{{\sqrt{\sum_j \lamj}}} \bfV'\bfD^{-1} \bfU^\top \bfy + \bfV'_\bot (\bfV'_\bot)^\top \paramall_0
\]
where $(-)^\dagger$ is the Moore-Penrose inverse operator. Also,
\[
\paramall_0 = \bfV'(\bfV')^\top \paramall_0 + \bfV'_\bot (\bfV'_\bot)^\top \paramall_0.
\]
From these facts, we can derive a formula that describes the changes in weights during the training based on gradient flow or gradient descent:
\begin{align*}
\paramall_\infty - \paramall_0
& = \frac{\sqrt{d}}{{\sqrt{\sum_j \lamj}}} \left(\bfV'\bfD^{-1} \bfU^\top \bfy\right) + (\bfV'_\bot (\bfV'_\bot)^\top \bfW_0) - \bfW_0
\\
& = \left(\bfB^\top \otimes \frac{\sqrt{d}}{{{\sum_j \lamj}}} \left(\bfV \bfD^{-1} \bfU^\top \bfy\right)\right) - (\bfV' (\bfV')^\top \paramall_0)
\\
& = {\frac{1}{\sum_j \lamj}} \left(\bfB^\top \otimes \left(\bfbeta_\infty - \bfV \bfV^\top \bfbeta_0\right)\right)
\end{align*}
where $\bfbeta_0 = \sum_{j = 1}^m \sqrt{\lambda_{m,j}} a_j \param_{0 j}$ and
\[
\bfbeta_\infty = \sqrt{d} \bfX^\dagger \bfy = \sqrt{d} \left(\bfV \bfD^{-1} \bfU^\top \bfy\right)
\]
is the minimum-norm minimiser of $\frac{1}{2}\|\bfy - \frac{1}{\sqrt{d}} \bfX \bfbeta\|^2$. It follows that
\[
\param_{\infty j} - \param_{0 j} =\frac{\sqrt{\lamj}}{{\sum_k \lambda_{m,k}}} a_j (\bfbeta_\infty - \bfV \bfV^\top \bfbeta_0).
\]

\subsubsection{Proof of \Cref{thm:featurelearninglinear-infinite}}

First note that, under the scaling \eqref{eq:lamj}, $$\sum_{k=1}^\nnodes \lambda_{m,k} \Bigr(\relu(Z_{k}(\bx;\paramall_0))\Bigr)^2\to
\frac{\gamma}{d} \mathbb E\Bigr[\Bigr(\bx^\top \param_{01}\Bigr)^2\Bigr] + \frac{(1-\gamma)}{d}\sum_{k=1}^\infty  \widetilde\lambda_{k} \Bigr(\bx^\top \param_{0k}\Bigr)^2 $$ almost surely as $\nnodes\to\infty$; hence the denominator in \eqref{def:fl} is of order 1. Similarly, under the mean-field scaling,
$$\nnodes\times \sum_{k=1}^\nnodes \frac{1}{m^2} \Bigr(\relu(Z_{k}(\bx;\paramall_0))\Bigr)^2\to
\frac{1}{d} \mathbb E\Bigr[\Bigr(\bx^\top \param_{01}\Bigr)^2\Bigr]
$$ almost surely as $\nnodes\to\infty$; hence the denominator in \eqref{def:fl} is of order $1/\nnodes$. For the numerator, from \cref{eq:featurelearninglinear-weights}, we have
\begin{align*}
\sum_{j=1}^\nnodes \lamj \Bigr( \relu(Z_{j}(\bx;\paramall_\infty))- \relu(Z_{j}(\bx;\paramall_0))\Bigr)^2&=\frac{1}{d}\frac{\sum_{j=1}^\nnodes \lamj^2}{(\sum_{k=1}^m\lambda_{m,k})^2} \Bigr( \bx^\top(\bfbeta_\infty-\bfV\bfV^\top\bfbeta_0)\Bigr)^2.
\end{align*}
Under the scaling \eqref{eq:lamj}, $\frac{\sum_{j=1}^\nnodes \lamj^2}{(\sum_{k}\lambda_{m,k})^2}=\sum_{j=1}^\nnodes \lamj^2\to (1-\gamma)^2\sum_{j\geq 1}\widetilde\lambda_j^2$. Hence feature learning occurs if and only if $\gamma<1$. Under the mean-field scaling, $\frac{\sum_{j=1}^\nnodes \lamj^2}{(\sum_{k}\lambda_{m,k})^2}=1/m$. Hence feature learning occurs. Additionally, as the $(\lamj)_{j\geq 1}$ are ordered, we have
\begin{align*}
\max_{j=1,\ldots,\nnodes} \lamj \Bigr( \relu(Z_{j}(\bx;\paramall_\infty))- \relu(Z_{j}(\bx;\paramall_0))\Bigr)^2 &=\max_{j=1,\ldots,\nnodes} \frac{1}{d}\frac{\lambda_{m,j}^2}{{(\sum_{k}\lambda_{m,k})^2}}  \Bigr( \bx^\top(\bfbeta_\infty-\bfV\bfV^\top\bfbeta_0)\Bigr)^2\\
&= \frac{\lambda_{m,1}^2}{d{(\sum_{k}\lambda_{m,k})^2}}  \Bigr( \bx^\top(\bfbeta_\infty-\bfV\bfV^\top\bfbeta_0)\Bigr)^2.
\end{align*}
Under the scaling \eqref{eq:lamj}, $\lambda_{m,1}^2\to (1-\gamma)\widetilde\lambda_1$, with $\widetilde\lambda_1>0$, hence non-uniform feature learning occurs if and only if $\gamma<1$. Under mean-field scaling, $\lambda_{m,1}^2/{(\sum_{k}\lambda_{m,k})^2}=1/m^2=o(1/m)$, hence non-uniform feature learning does not occur. Additionally, by \cref{eq:featurelearninglinear-weights} in \cref{thm:featurelearninglinear}, we have
		\begin{align*}
			\sqrt{\lambda_{m,j}} a_j \param_{\infty j}
			& {} =
			\sqrt{\lambda_{m,j}} a_j\param_{0 j} + {{\frac{\lamj}{\sum_k \lambda_{m,k}}} }  (\bfbeta_\infty - \bfV \bfV^\top \bfbeta_0)\\
			&{} =
			{{\frac{\lamj}{\sum_k \lambda_{m,k}}} }\left(\bfbeta_\infty - \bfV \bfV^\top \sum_{k\neq j} \sqrt{\lambda_{m,k}} a_k \param_{0k}\right)
			+
			\sqrt{\lambda_{m,j}} \left(I_\din-{{\frac{\lamj}{\sum_k \lambda_{m,k}}} }\lamj \bfV \bfV^\top\right) a_j\param_{0 j}.
		\end{align*}
		The right-hand side is the sum of two independent Gaussian random vectors, and is therefore a Gaussian random vector, with mean ${{\frac{\lamj}{\sum_k \lambda_{m,k}}} } \bfbeta_\infty$ and covariance matrix ${{\frac{\lambda_{m,j}^2}{(\sum_k \lambda_{m,k})^2}} }({{\sum_k \lambda_{m,k}}}-\lambda_{m,j})(\bfV\bfV^\top)^2+\lambda_{m,j}\left(I_\din-{{\frac{\lamj}{\sum_k \lambda_{m,k}}} } \bfV \bfV^\top\right)^2=\lamj (I_\din-{{\frac{\lamj}{\sum_k \lambda_{m,k}}} } \bfV\bfV^\top )$. The distributional convergence in \cref{eq:convergencewinf} then follows from Slutsky's theorem.

\subsubsection{Proof of \Cref{thm:featurelearninglinearpruning}}

		Using  Markov and Cauchy-Schwarz inequalities,
		\[
		\Pr\left(\left|\widetilde f_{m,\rho}(\bx;\paramall_\infty)-f_{m}(\bx;\paramall_\infty)\right| > \varepsilon\right)
		\;\leq\;
		\frac{\|\bx\|}{\varepsilon\sqrt{\din}} \times
		\mathbb{E}\left[\left\|\sum_{j>\lfloor \rho m \rfloor} \sqrt{\lambda_{m,j}} a_j \param_{\infty j} \right\|\right].
		\]
		Meanwhile, we have
		\begin{align*}
			&
			\left\|\sum_{j>\lfloor \rho m \rfloor} \sqrt{\lambda_{m,j}} a_j \param_{\infty j} \right\|
			\\
			&
			\qquad {}\leq
			\left\|\sum_{j>\lfloor \rho m \rfloor} \lamj  \left(\bfbeta_\infty - \bfV \bfV^\top \sum_{k\neq j} \sqrt{\lambda_{m,k}} a_k \param_{0k}\right)\right\|
			+
			\left\|\sum_{j>\lfloor \rho m \rfloor}\sqrt{\lambda_{m,j}} \left(I_{\din}-\lamj \bfV \bfV^\top\right) a_j\param_{0 j}\right\|
			\\
			&
			\qquad {}\leq
			\sum_{j>\lfloor \rho m \rfloor} \lamj  \left(\left\| \bfbeta_\infty \right\| + \left\| \bfV \bfV^\top \sum_{k\neq j} \sqrt{\lambda_{m,k}} a_k \param_{0k}\right\|\right)
			+
			\left\|\sum_{j>\lfloor \rho m \rfloor}\sqrt{\lambda_{m,j}} \left(I_{\din}-\lamj \bfV \bfV^\top\right) a_j\param_{0 j}\right\|.
		\end{align*}
		Also,
		\begin{align*}
			\mathbb{E}\left[ \left\| \bfV \bfV^\top \sum_{k\neq j} \sqrt{\lambda_{m,k}} a_k \param_{0k}\right\| \right]
			& {} \leq
			\sqrt{\mathbb{E}\left[ \left\| \bfV \bfV^\top \sum_{k\neq j} \sqrt{\lambda_{m,k}} a_k \param_{0k}\right\|^2 \right]}
			\\
			& {} =
			\sqrt{(1-\lamj)\trace(\bfV \bfV^\top)}
			\\
			& {} \leq
			\sqrt{\din},
		\end{align*}
		and
		\begin{align*}
			\mathbb{E} \left[ \left\|\sum_{j>\lfloor \rho m \rfloor}\sqrt{\lambda_{m,j}} \left(I_{\din}-\lamj \bfV \bfV^\top\right) a_j\param_{0 j}\right\|\right]
			& {} \leq
			\sqrt{\mathbb{E}\left[ \left\|\sum_{j>\lfloor \rho m \rfloor}\sqrt{\lambda_{m,j}} \left(I_{\din}-\lamj \bfV \bfV^\top\right) a_j\param_{0 j}\right\|^2
				\right]}
			\\
			& {} =
			\sqrt{\sum_{j>\lfloor \rho m \rfloor} \lambda_{m,j} \trace\left(\left(I_{\din}-\lamj \bfV \bfV^\top\right)^2\right)}
			\\
			& {} =
			\sqrt{\sum_{j>\lfloor \rho m \rfloor} \lambda_{m,j}  \trace\left(I_{\din} -2\lamj \bfV\bfV^\top + \lamj^2 \bfV \bfV^\top\right)}
			\\
			& {} \leq
			\sqrt{\sum_{j>\lfloor \rho m \rfloor} \lambda_{m,j} \times \din \times (1-\lamj)^2}
			\\
			& {} \leq
			\sqrt{\din \sum_{j>\lfloor \rho m \rfloor} \lambda_{m,j}}.
		\end{align*}
		By combining the above inequalities, we obtain the desired result:
		\begin{align*}
			\Pr\left(\left|\widetilde f_{m,\rho}(\bx;\paramall_\infty)-f_{m}(\bx;\paramall_\infty)\right| > \varepsilon\right)
			\;\leq\;
			\frac{\|\bx\|}{\varepsilon\sqrt{\din}}
			\left( \left(\left\| \bfbeta_\infty \right\|+\sqrt{d}\right)\left(\sum_{j>\lfloor \rho m \rfloor} \lambda_{m,j}\right)
			+ \sqrt{{\din}\sum_{j>\lfloor \rho m \rfloor} \lambda_{m,j}}
			\right).
		\end{align*}

\subsection{Proofs of \cref{sec:featurelearningnonlinear} (nonlinear activation)}

\subsubsection{Proof of \cref{thm:featurelearning-smooth}}

Our proof of \cref{thm:featurelearning-smooth} relies on the following observation on the linear combinations
of continuous independent real-valued random variables.

\begin{lemma}
	\label{lem:random-outputs:asli}
	Let $z_1,\ldots,z_n$ be independent continuous real-valued random variables.
	Let $\mathcal{B} \subset \mathbb{R}$ be a finite subset of the real numbers such that $\mathcal{B} \not= \{0\}$.
	Then, almost surely,
	\[
	\min_{\mathbf{b} \in \mathcal{B}^n \setminus \{ 0,...,0\}} \left| \sum_{i=1}^n b_i z_i \right| > 0.
	\]
\end{lemma}
\begin{proof}
	Denote $\mathcal{S} =\mathcal{B}^n \setminus \{ 0,...,0\}$. For any $\mathbf{b}=(b_1,...,b_n) \in \mathcal{S}$,
	$ \sum_{i=1}^n b_i z_i $, so that $\Pr(\sum_{i=1}^n b_i z_i =0)=0$. Hence, since $\mathcal{S}$ is finite,
	\begin{align*}
	\Pr\left( \min_{ \mathbf{b} \in \mathcal{S}} \left| \sum_{i=1}^n b_i z_i \right| = 0 \right)
	& =
	\Pr\left( \bigcup_{\mathbf{b}\in\mathcal{S}} \left\{ \sum_{i=1}^n b_i z_i = 0\right\} \right)
	\\
	&\leq
	\sum_{\mathbf{b}\in\mathcal{S}} \Pr\left(\sum_{i=1}^n b_i z_i = 0 \right) \\
	&=
	0.
	\end{align*}
\end{proof}

The proof also uses the our globally-made standard assumption that for every random variable $Z \sim \mathcal{N}(0,s^2)$ for some $s > 0$, the expectation
$\mathbb{E}[\sigma(Z)^2]$ is finite and greater than $0$.

\paragraph{Proof of \cref{thm:featurelearning-smooth}.}
	Since non-uniform feature learning implies feature learning, we prove the former only.
	We start by showing that the denominator
	$\sum_{j=1}^\nnodes \lamj (\sigma(Z_j(\bx_i;\paramall_0)))^2$ in the condition for
	non-uniform feature learning converges to a
	positive finite value almost surely as $m$ tends to $\infty$. To
	see this, note
	\begin{align*}
	& \lim_{\nnodes \to \infty} \sum_{j=1}^\nnodes \lamj (\sigma(Z_j(\bx_i;\paramall_0)))^2
	 = \lim_{\nnodes \to \infty}
	\sum_{j=1}^\nnodes
		\left(\gamma \cdot \frac{1}{\nnodes} + (1-\gamma) \cdot \frac{\tillam_j}{\sum_{j' = 1}^\nnodes \tillam_{j'}}\right)
		\sigma(Z_j(\bx_i;\paramall_0))^2
	\\
	& \qquad\qquad\qquad {} =
	\left(\gamma \cdot
	\lim_{\nnodes \to \infty}
	\sum_{j=1}^\nnodes \frac{1}{\nnodes} \sigma(Z_j(\bx_i;\paramall_0))^2
	\right)
	+
	\left((1-\gamma) \cdot
	\frac{
		\lim_{\nnodes \to \infty}
		\sum_{j = 1}^\nnodes \tillam_j \sigma(Z_j(\bx_i;\paramall_0))^2
	}{
		\lim_{\nnodes \to \infty}
		\sum_{j' = 1}^\nnodes \tillam_{j'}
	}\right)
	\\
	& \qquad\qquad\qquad {} =
	\gamma \cdot \mathbb{E}_{Z\sim \mathcal{N}(0,\|\bx_i\|^2/d)}\left[\sigma(Z)^2\right]
	+
	(1-\gamma) \cdot
		\sum_{j = 1}^\infty \tillam_j \sigma(Z_j(\bx_i;\paramall_0))^2.
	\end{align*}
	The expectation in the first summand is positive and finite by
	our globally-made assumption on the activation function $\sigma$. Also, the infinite sum in the second
	summand is positive almost surely because it is greater than
	$\tillam_1\sigma(Z_1(\bx_i;\paramall_0))^2$
	but  $\tillam_1\sigma(Z_1(\bx_i;\paramall_0))^2$
	 is almost surely positive; $\tillam_1 > 0$ and $\sigma(Z_1(\bx_i;\paramall_0))$
	 is almost surely non-zero due to the injectivity of $\sigma$
	and the continuity of the random variable $Z_1(\bx;\paramall_0)$.
	Furthermore, the sum is almost surely finite as well, because its expectation is
	$\mathbb{E}_{Z\sim \mathcal{N}(0,\|\bx_i\|^2/d)}[\sigma(Z)^2]$ which is finite
	by our globally-made assumption on the activation function $\sigma$. Thus, the limit of the denominator
	is positive and finite almost surely.
	
	Since the denominator in the condition of non-uniform feature learning
	converges to a positive finite value almost surely, the condition holds if
	\begin{equation}
        \label{eqn:featurelearning-smooth:proof:0}
        \liminf_{\nnodes \to \infty} \left(\max_{j \in [\nnodes]}\lamj\left(\sigma(Z_j(\bx_i;\paramall_1)) - \sigma(Z_j(\bx_i;\paramall_0))\right)^2\right) > 0
        \quad\text{almost surely}.
	\end{equation}
	Note that
	the limit here is not redundant since $\paramall_1$ depends on $\nnodes$.
	The new condition in \cref{eqn:featurelearning-smooth:proof:0} can be simplified further.
	It holds whenever
	\begin{equation}
		\label{eq:featurelearning-smooth-activation0}
		\liminf_{\nnodes \to \infty} \Big(Z_1(\bx_i;\paramall_1) - Z_1(\bx_i;\paramall_0)\Big)^2 > 0.
	\end{equation}
	To see this, note that by the assumption of the theorem and the inverse function theorem,
	$\sigma^{-1}$ is a well-defined continuous function and also that
	$Z_1(\bx_i;\paramall_0)$ does not depend on $\nnodes$.
	As a result, the inequality in \cref{eq:featurelearning-smooth-activation0} implies
	\begin{equation}
		\label{eq:featurelearning-smooth-activation1}
		\liminf_{\nnodes \to \infty} \Big(\sigma(Z_1(\bx_i;\paramall_1)) - \sigma(Z_1(\bx_i;\paramall_0))\Big)^2 > 0,
	\end{equation}
	because otherwise some subsequence of $(\sigma(Z_1(\bx_i;\paramall_1)))_{\nnodes}$
	would converge to $\sigma(Z_1(\bx_i;\paramall_0))$
	as $\nnodes$ tends to $\infty$, but then by the continuity of $\sigma^{-1}$,
	the corresponding subsequence of $(Z_1(\bx_i;\paramall_1))_{\nnodes}$
	would converge to $Z_1(\bx_i;\paramall_0)$,
	which contradicts \cref{eq:featurelearning-smooth-activation0}. Now
	using \cref{eq:featurelearning-smooth-activation1}, the assumption $\gamma > 0$,
	and the fact that $\tillam_1 > 0$,
	we can prove the condition in \cref{eqn:featurelearning-smooth:proof:0} as follows:
	\begin{align*}
		\liminf_{\nnodes \to \infty} \max_{j \in [\nnodes]} \lamj \Big(\sigma(Z_j(\bx_i;\paramall_1)) - \sigma(Z_j(\bx_i;\paramall_0))\Big)^2
		& {} \geq
		\liminf_{\nnodes \to \infty} \lambda_{\nnodes,1}\Big(\sigma(Z_1(\bx_i;\paramall_1)) - \sigma(Z_1(\bx_i;\paramall_0))\Big)^2
		\\
		& {} \geq
		\liminf_{\nnodes \to \infty} (1-\gamma)\tillam_1\Big(\sigma(Z_1(\bx_i;\paramall_1)) - \sigma(Z_1(\bx_i;\paramall_0))\Big)^2
		\\
		& {} > 0.
	\end{align*}
	
	We now show that \cref{eq:featurelearning-smooth-activation0} holds almost surely.
	Note that
	\begin{align*}
		\Big(Z_1(\bx_i;\paramall_1) - Z_1(\bx_i;\paramall_0)\Big)^2
		& {} =
		\left(\frac{\param_{11}^\top \bx_i}{\sqrt{\din}} - \frac{\param_{01}^\top \bx_i}{\sqrt{\din}}\right)^2
		\\
		& {} =
		\frac{1}{\din} \left(\eta \left(\left. \nabla_{\param_{tj}} L(\paramall_t)\right|_{t = 0}\right)^\top \bx_i\right)^2
		\\
		& {} =
		\frac{\eta^2}{\din}\left(
		\sum_{i' = 1}^n y_{i'} \sqrt{\lambda_{\nnodes,1}} a_1
        \sigma'\left(\frac{\param_{0j}^\top \bx_{i'}}{\sqrt{d}}\right)
		\frac{\bx_{i'}^\top \bx_i}{\sqrt{d}}
		\right)^2
		\\
		& {} =
		\frac{\eta^2\lambda_{\nnodes,1}}{d^2}\left(
			\sum_{i' = 1}^n y_{i'}
			\left(\sigma'\left(\frac{\param_{0j}^\top \bx_{i'}}{\sqrt{d}}\right)
			\bx_{i'}^\top \bx_i\right)
		\right)^2
		\\
		& {} \geq
		\frac{\eta^2(1-\gamma)\tillam_{1}}{d^2}\left(
			\sum_{i' = 1}^n y_{i'}
			\left(\sigma'\left(\frac{\param_{0j}^\top \bx_{i'}}{\sqrt{d}}\right)
			\bx_{i'}^\top \bx_i\right)
		\right)^2.
	\end{align*}
	Since the lower bound from above does not depend on $\nnodes$, we have
	\[
	\liminf_{\nnodes \to \infty} \Big(Z_1(\bx_i;\paramall_1) - Z_1(\bx_i;\paramall_0)\Big)^2
	\geq	
	\frac{\eta^2(1-\gamma)\tillam_{1}}{d^2}\left(
		\sum_{i' = 1}^n y_{i'}
		\left(\sigma'\left(\frac{\param_{0j}^\top \bx_{i'}}{\sqrt{d}}\right)
		\bx_{i'}^\top \bx_i\right)
	\right)^2.
	\]
	Since $\eta^2(1-\gamma)\tillam_{1} / d^2$ is positive,
	this lower bound is positive almost surely whenever the summation
	inside the square is positive almost surely.
	
	Conditioning on
	$\param_{0j}$ and noting that $
	\sigma'\left(\frac{\param_{0j}^\top \bx_{i}}{\sqrt{d}}\right)
	\|\bx_{i}\|^2 > 0$,
	we have by \cref{lem:random-outputs:asli} that almost surely
	\[	\sum_{i' = 1}^n y_{i'}
	\left(\sigma'\left(\frac{\param_{0j}^\top \bx_{i'}}{\sqrt{d}}\right)
	\bx_{i'}^\top \bx_i\right) > 0.\]
	We may use this lemma since the $y_{i'}$'s are independent from $\param_{0j}$ and so their distributions are unaffected by the conditioning. Now note that this almost-sure positivity of the summation
	holds regardless of which value the conditioned $\param_{0j}$ takes.
	Thus, the summation is positive almost surely without the conditioning. This completes the proof.

\subsubsection{Proof of \cref{thm:weight-norm-square-change-smooth}}

The proof of the theorem uses the following lemma on quadratic combinations of continuous independent random variables.

\begin{lemma}
	\label{lem:random-outputs:as-quadratic}
	Let $z_1,\ldots,z_n$ be continuous independent real-valued random variables.
	Let $B$ be an $n$-by-$n$ real-valued matrix such that $B_{ii} \neq 0$ for some $i \in [n]$. Then, almost surely,
	\[
		\left| \sum_{i=1}^n \sum_{i' = 1}^n z_iz_{i'} B_{ii'} \right| > 0.
	\]
\end{lemma}
\begin{proof}
	Let $i \in [n]$ such that $B_{ii} \neq 0$. Then, when viewed as a polynomial on $z_i$,
	\[
		\sum_{i=1}^n \sum_{i' = 1}^n z_iz_{i'} B_{ii'}
	\]
	is a quadratic polynomial with a non-zero coefficient for the term $z_i^2$. As a result, the
	zero set of this polynomial on $z_i$ has measure zero with respect to Lebesgue measure, that is,
	the Lebesgue measure of the set
	\[
	\left\{z_i \ \left| \ \sum_{i=1}^n \sum_{i' = 1}^n z_iz_{i'} B_{ii'} = 0\right.\right\} \subseteq \R
	\]
	is zero (because the zero set of any analytic function has zero Lebesgue measure).
	Furthermore, $z_i$ is a continuous random variable, and so we have
	\[
	\mathbb{E}\left[\left.\ind_{\left\{ \sum_{i=1}^n \sum_{i' = 1}^n z_iz_{i'} B_{ii'} = 0\right\}}\ \right|\ \left\{\left.z_{i'} \,\right|\, i' \in [n], i' \neq i\right\}\right] = 0.
	\]
	As a result,
	\begin{align*}
		\Pr\left(\sum_{i=1}^n \sum_{i' = 1}^n z_iz_{i'} B_{ii'} = 0\right)
		& =
		\mathbb{E}\left[\ind_{\left\{ \sum_{i=1}^n \sum_{i' = 1}^n z_iz_{i'} B_{ii'} = 0\right\}}\right]
		\\
		& =
		\mathbb{E}\left[
		\mathbb{E}\left[\left.\ind_{\left\{ \sum_{i=1}^n \sum_{i' = 1}^n z_iz_{i'} B_{ii'} = 0\right\}}\ \right|\ \left\{\left.z_{i'} \,\right|\, i' \in [n], i' \neq i\right\}\right]\right]
		\\
		& =
		\mathbb{E}[0] = 0.
	\end{align*}
	This proves the claim of the lemma.
\end{proof}

\paragraph{Proof of \cref{thm:weight-norm-square-change-smooth}.}
We first compute a lower bound of the squared norm of the gradient, which does not depend on $\nnodes$.
\begin{align*}
	\left\|\left.\nabla_{\param_{tj}} L(\paramall_t)\right|_{t=0}\right\|^2
    & {} =
	\left\|\sum_{i = 1}^n y_i \sqrt{\lamj} \aj
        \sigma'\left(\frac{\param_{0j}^\top \bx_i}{\sqrt{d}}\right)
		\frac{\bx_{i}}{\sqrt{d}}\right\|^2
    \\
	& {} =
	\frac{\lamj}{d}
	\left\|\sum_{i = 1}^n y_i
        \sigma'\left(\frac{\param_{0j}^\top \bx_i}{\sqrt{d}}\right)
		\bx_{i}\right\|^2
	\\
	& {} \geq
	\frac{(1-\gamma)\tillam_j}{d}
	\left\|\sum_{i = 1}^n y_i
		\sigma'\left(\frac{\param_{0j}^\top \bx_i}{\sqrt{d}}\right)
		\bx_{i}\right\|^2
	\\
    & {} =
    \frac{(1-\gamma)\tillam_j}{d}
	\left|
	\sum_{k = 1}^\din \sum_{i = 1}^n \sum_{i'=1}^n
		y_iy_{i'}
		\sigma'\left(\frac{\param_{0j}^\top \bx_i}{\sqrt{d}}\right)
		\sigma'\left(\frac{\param_{0j}^\top \bx_{i'}}{\sqrt{d}}\right)
		x_{ik}x_{i'k}
	\right|
	\\
    & {} =
    \frac{(1-\gamma)\tillam_j}{d}
	\left|\sum_{i = 1}^n \sum_{i'=1}^n
		y_iy_{i'}
		\left(\bx_{i}^\top \bx_{i'}
			\sigma'\left(\frac{\param_{0j}^\top \bx_i}{\sqrt{d}}\right)
			\sigma'\left(\frac{\param_{0j}^\top \bx_{i'}}{\sqrt{d}}\right)
		\right)
	\right|.
\end{align*}
Thus,
\[
	\liminf_{\nnodes \to \infty}\left\|\left.\nabla_{\param_{tj}} L(\paramall_t)\right|_{t=0}\right\|^2
	\geq
    \frac{(1-\gamma)\tillam_j}{d}
	\left|\sum_{i = 1}^n \sum_{i'=1}^n
		y_iy_{i'}
		\left(\bx_{i}^\top \bx_{i'}
			\sigma'\left(\frac{\param_{0j}^\top \bx_i}{\sqrt{d}}\right)
			\sigma'\left(\frac{\param_{0j}^\top \bx_{i'}}{\sqrt{d}}\right)
		\right)
	\right|.
\]
But by assumption, $((1-\gamma)\tillam_j) / d$ is positive. The other factor in the lower bound
is also positive with probability one. To see this, note that the $y_i$'s in the factor are continuous independent random variables, independent also from $\param_{0j}$, and
\[
\|\bx_i\|^2 \sigma'\left(\frac{\param_{0j}^\top \bx_i}{\sqrt{d}}\right)^2 > 0\ \text{ for all $i \in [n]$},
\]
due to \cref{assump:data} and the assumption that $\sigma'>0$. As a result,
conditioned on $\param_{0j}$,
by \cref{lem:random-outputs:as-quadratic},
the factor is positive almost surely with respect to the
conditional distributions of the $y_i$'s, which are the same as
the original unconditional distributions of them due to
the independence of the $y_i$'s with respect to $\param_{0j}$. Since
this positivity holds regardless of which value $\param_{0j}$ takes,
it also holds without the conditioning on $\param_{0j}$.
This completes the proof.

\section{Proofs of the results of \cref{sec:featurelearning-relu} on feature learning (ReLU case)}
\label{sec:featurelearning-relu-proofs}

\subsection{Proof of \cref{thm:featurelearningrelu-general}}

Our proof relies on a few lemmas.

\begin{lemma}
	\label{lem:relu-zeroed-initialisation:gradient}
	Assume \cref{assump:zeroed-initialisation}. If the activation function $\sigma$ is ReLU, we have
	\[
		\left.\nabla_{\paramtj} L(\paramall_t)\right|_{t=0}
		=
		\left(-
		\sqrt{\frac{\lamj}{d}}\aj \sum_{i = 1}^\nnodes y_i \sigma'(Z_j(\bx_i;\paramall_0))\bx_i\right)
		=
		\left(-
		\sqrt{\frac{\lamj}{d}}\aj \sum_{i = 1}^\nnodes \ind_{\{\param_{0j}^\top \bx_i \geq 0\}} y_i\bx_i\right)
	\]
\end{lemma}
\begin{proof}
	The lemma follows from a straightforward calculation using the fact that $f_\nnodes(\bx;\paramall_0) = 0$ for all $\bx \in \R^d$.
\end{proof}

\begin{lemma}
	\label{lem:relu-zeroed-initialisation:error}
	Assume \cref{assump:data,assump:zeroed-initialisation,assump:random-outputs}.
	Then, we have that for all $\nnodes$, $j \in [\nnodes]$, and $i \in [n]$,
	\begin{align*}
		\lamj \Big(\sigma(Z_j(\bx_i;\paramall_1)) - \sigma(Z_j(\bx_i;\paramall_0))\Big)^2
		& {} \geq	
		\ind_{\{\param_{0j}^\top \bx_{i} \geq 0\}}
		\cdot
		\min\left\{
			\frac{\eta^2 c^2 (1-\gamma)^2\tillam_j^2}{d^2},
			\frac{(1-\gamma)\tillam_j (\param_{0j}^\top\bx_i)^2}{d}
		\right\}
	\end{align*}
	where $c$ depends only on the inputs/outputs (in particular, not depending on $\nnodes$)
	and is almost surely strictly positive (almost surely, with respect to the input/output).
\end{lemma}
\begin{proof}
	Using \cref{lem:relu-zeroed-initialisation:gradient}, we can compute
	\begin{align*}
		& {} \sigma(Z_j(\bx_i;\paramall_1)) - \sigma(Z_j(\bx_i;\paramall_0))
		\\
		& \qquad {} =
		\sigma(Z_j(\bx_i;\paramall_1)) - \sigma(Z_j(\bx_i;\paramall_0))
		\\
		& \qquad {} =
		\sigma\left(
			\frac{1}{\sqrt{d}}\left(
					\param_{0j}
					+ \eta\sqrt{\frac{\lamj}{d}}a_j
						\sum_{i' = 1}^n \ind_{\{\param_{0j}^\top \bx_{i'} \geq 0\}} y_{i'} \bx_{i'}
			\right)^\top \bx_i
		\right)
		-
		\sigma\left(\frac{\param_{0j}^\top \bx_i}{\sqrt{d}}\right)
		\\
		& \qquad {} =
		\sigma\left(
			\frac{\param_{0j}^\top \bx_i}{\sqrt{d}}
				+
				\eta\frac{\sqrt{\lamj}}{d}a_j
				\sum_{i' = 1}^n
					\left(\ind_{\{\param_{0j}^\top \bx_{i'} \geq 0\}} \bx_{i'}^\top \bx_i\right)
					y_{i'}
		\right)
		-
		\sigma\left(\frac{\param_{0j}^\top \bx_i}{\sqrt{d}}\right).	
	\end{align*}
	Denoting $\delta_i = ((\eta\sqrt{\lamj}a_j)/d)\sum_{i' = 1}^n (\ind_{\{\param_{0j}^\top \bx_{i'} \geq 0\}} \bx_{i'}^\top \bx_i)y_{i'}$, we have that
	\begin{align*}
		\left| \sigma(Z_j(\bx_i;\paramall_1)) - \sigma(Z_j(\bx_i;\paramall_0)) \right|
		&{} = \left| \sigma\left(\frac{\param_{0j}^\top \bx_i}{\sqrt{d}} + \delta_i \right)
		-
		\sigma\left(\frac{\param_{0j}^\top \bx_i}{\sqrt{d}}\right)	\right| \\
		&{} \geq  \ind_{\{\param_{0j}^\top \bx_{i} \geq 0\}} \left| \sigma\left(\frac{\param_{0j}^\top \bx_i}{\sqrt{d}} + \delta_i \right)
		-
		\sigma\left(\frac{\param_{0j}^\top \bx_i}{\sqrt{d}}\right)	\right| \\
		&{} \geq \ind_{\{\param_{0j}^\top \bx_{i} \geq 0\}} \left( \ind_{\{\param_{0j}^\top \bx_{i} + \sqrt{d} \delta_i \geq 0\}} \left| \delta_i \right| + \ind_{\{\param_{0j}^\top \bx_{i} + \sqrt{d} \delta_i < 0\}} \left| \frac{\param_{0j}^\top \bx_i}{\sqrt{d}} \right| \right).
	\end{align*}
	We then get
	\begin{align*}
		\lamj \Big(\sigma(Z_j(\bx_i;\paramall_1)) - \sigma(Z_j(\bx_i;\paramall_0))\Big)^2 \geq \ind_{\{\param_{0j}^\top \bx_{i} \geq 0\}} \lambda_{m,j} \min\left\{
				\delta_i^2,
				\frac{(\param_{0j}^\top\bx_i)^2}{d}.
			\right\}
	\end{align*}
	Now, notice that
	\begin{align*}
		\ind_{\{\param_{0j}^\top \bx_{i} \geq 0\}} \delta_i^2
		&{} =
		\frac{\eta^2}{d^2} \lambda_{m,j} \ind_{\{\param_{0j}^\top \bx_{i} \geq 0\}}  \left( \sum_{i' = 1}^n \left(\ind_{\{\param_{0j}^\top \bx_{i'} \geq 0\}} \bx_{i'}^\top \bx_i\right) y_{i'} \right)^2 	
		\\
		&{} =
		\frac{\eta^2}{d^2} \lambda_{m,j} \ind_{\{\param_{0j}^\top \bx_{i} \geq 0\}} \left(\|\bx_i\|^2y_i +  \sum_{i' \in \{1,...,n\}\setminus\{i\} } \left(\ind_{\{\param_{0j}^\top \bx_{i'} \geq 0\}}\bx_{i'}^\top \bx_i\right)y_{i'} \right)^2
		\\
		&{} \geq
		\frac{\eta^2}{d^2} \lambda_{m,j} \ind_{\{\param_{0j}^\top \bx_{i} \geq 0\}} \left( \min_{\mathbf{b} \in \mathcal{S}} \left| \sum_{i'=1}^n b_{i'} y_{i'} \right| \right)^2
	\end{align*}
	where $\mathcal{S} = \mathcal{B}^n \setminus \{0,...,0\}$ and $\mathcal{B}$ is given by
	\[
	\mathcal{B} = \left\{\left.u \cdot v\ \right|\ u \in \{0,1\}, v \in \left\{\bx_1^\top \bx_{i}, \ldots, \bx_n^\top \bx_{i} \right\}\right\}
	\]
	Note that the $y_i$'s are continuous and independent random variables by \cref{assump:random-outputs}.
	Thus, by \cref{lem:random-outputs:asli},
	with probability one,
	\[
		c = \min_{\mathbf{b} \in \mathcal{S}} \left| \sum_{i'=1}^n b_{i'} y_{i'} \right| > 0.
	\]
	Putting everything together, we finally get
	\begin{align*}
	\lamj \Big(\sigma(Z_j(\bx_i;\paramall_1)) - \sigma(Z_j(\bx_i;\paramall_0))\Big)^2	
	& \geq	
	\ind_{\{\param_{0j}^\top \bx_{i} \geq 0\}}
	\cdot
	\min\left\{
		\frac{\eta^2 c^2 \lambda_{m,j}^2}{d^2},
		\frac{\lambda_{m,j}(\param_{0j}^\top\bx_i)^2}{d}
	\right\}
	\\
	& \geq
	\ind_{\{\param_{0j}^\top \bx_{i} \geq 0\}}
	\cdot
	\min\left\{
		\frac{\eta^2 c^2 (1-\gamma)^2\tillam_j^2}{d^2},
		\frac{(1-\gamma)\tillam_j(\param_{0j}^\top\bx_i)^2}{d}
	\right\},
	\end{align*}
	which concludes the proof.
\end{proof}

\paragraph{Proof of \cref{thm:featurelearningrelu-general}.}
	Note that the condition for non-uniform feature learning in \cref{eqn:featurelearningrelu-general:0b}
	implies that for feature learning in \cref{eqn:featurelearningrelu-general:0}. Thus, we
	will prove only the former condition.

	By our setup, we have that $\lambda_{\nnodes,1}\ge \ldots \ge \lambda_{\nnodes,k} \geq (1-\gamma)\tillam_{k} > 0$ for all $\nnodes$.
	Also, by \cref{lem:relu-zeroed-initialisation:error}, we have that for $1\le j\le k$,
	\[
		\lambda_{m,j} \Big(\sigma(Z_j(\bx_i;\paramall_1)) - \sigma(Z_j(\bx_i;\paramall_0))\Big)^2
		\geq
		\ind_{\left\{\param_{0j}^\top \bx_{i} \geq 0\right\}}
		\cdot
		\min\left\{
			\frac{\eta^2 c^2 (1-\gamma)^2\tillam_j^2}{d^2},\,
			\frac{(1-\gamma)\tillam_j (\param_{0j}^\top\bx_i)^2}{d}
		\right\}.
	\]
	Thus, for all $\nnodes \geq k$,
	\begin{multline*}
		\max_{j \in [\nnodes]} \left(\lamj \Big(\sigma(Z_j(\bx_i;\paramall_1)) - \sigma(Z_j(\bx_i;\paramall_0))\Big)^2\right)
		\\
		{} \geq
		\max_{j \in [k]}
		\left(\ind_{\left\{\param_{0j}^\top \bx_{i} \geq 0\right\}}
		\cdot
		\min\left\{
			\frac{\eta^2 c^2 (1-\gamma)^2\tillam_j^2}{d^2},\,
			\frac{(1-\gamma)\tillam_j (\param_{0j}^\top\bx_i)^2}{d}\right\}\right),
	\end{multline*}
	which implies
	\begin{multline}
		\label{eq:featurelearningrelu-general:lower-bound1}
		\liminf_{\nnodes \to \infty} \left(\max_{j \in [\nnodes]}
			\left(\lamj \Big(\sigma(Z_j(\bx_i;\paramall_1)) - \sigma(Z_j(\bx_i;\paramall_0))\Big)^2\right)\right)
		\\
		{} \geq
		\max_{j \in [k]}\left(
		\ind_{\left\{\param_{0j}^\top \bx_{i} \geq 0\right\}}
		\cdot
		\min\left\{
			\frac{\eta^2 c^2 (1-\gamma)^2\tillam_j^2}{d^2},\,
			\frac{(1-\gamma)\tillam_j (\param_{0j}^\top\bx_i)^2}{d}
		\right\}\right).
	\end{multline}
	We will show that for all $\delta \in (0,1/2)$, with probability at least $1-(1/2 + \delta)^k$,
	the lower bound in \cref{eq:featurelearningrelu-general:lower-bound1}
	is positive and
	\begin{equation}
		\label{eq:featurelearningrelu-general:lower-bound1b}
		0 < \sum_{j = 1}^\infty \lamj \sigma(Z_j(\bx_i;\paramall_0))^2 < \infty.
	\end{equation}
	This will prove the claim of the theorem.
	
	Pick $\delta \in (0,1/2)$. Let $E$ be the event $\{c > 0\}$. Then, $\Pr(E) = 1$ by \cref{lem:random-outputs:asli}.
	Note that the first argument of the minimum in the lower bound of \cref{eq:featurelearningrelu-general:lower-bound1}
	is positive on the event $E$.
	Let $\epsilon > 0$ be a positive constant such that
	\begin{equation}
		\label{eq:featurelearningrelu-general:lower-bound2}
		\Pr(\param_{0j}^\top \bx_{i} \geq \epsilon) \geq \frac{1}{2} - \delta
		\ \text{for all $j\le k$},
	\end{equation}
	which is possible since each $\param_{0j}^\top \bx_{i}$ is a centred normal random variable with variance $\|\bx_i\|^2 > 0$.
	Define $E'_\delta$ be the event $\bigcup_{j=1}^k\{\param_{0j}^\top \bx_{i} \geq \epsilon\}$. Then, since $\param_{01}^\top \bx_{i},
	\ldots,\param_{0k}^\top \bx_i$ are independent and the lower bound in \cref{eq:featurelearningrelu-general:lower-bound2}
	holds, we have
	\[
		\Pr(E \cap E'_\delta) \geq 1-(1/2 + \delta)^k.
	\]
	Now condition on $E \cap E'_\delta$.
	Then, there exists some $j \leq k$ such that $\param_{0j}^\top \bx_{i} \geq \epsilon$. Thus,
	the lower bound in \cref{eq:featurelearningrelu-general:lower-bound1} is positive as shown below:
	\begin{align*}
		&
		\max_{j' \in [k]}\left(
			\ind_{\left\{\param_{0j'}^\top \bx_{i} \geq 0\right\}}
			\cdot
			\min\left\{
				\frac{\eta^2 c^2 (1-\gamma)^2\tillam_{j'}^2}{d^2},\,
				\frac{(1-\gamma)\tillam_{j'} (\param_{0j'}^\top\bx_i)^2}{d}
			\right\}
		\right)
		\\
		& \qquad\qquad {} \geq
			\ind_{\left\{\param_{0j}^\top \bx_{i} \geq 0\right\}}
			\cdot
			\min\left\{
				\frac{\eta^2 c^2 (1-\gamma)^2\tillam_j^2}{d^2},
				\frac{(1-\gamma)\tillam_j (\param_{0j}^\top\bx_i)^2}{d}
			\right\}
		\\
		& \qquad\qquad {} \geq
			\min\left\{
				\frac{\eta^2 c^2 (1-\gamma)^2\tillam_j^2}{d^2},
				\frac{(1-\gamma)\tillam_j \epsilon^2}{d}
				\right\}
		\\
		& \qquad\qquad {} > 0.
	\end{align*}
	Thus, with probability at least $1-(1/2 + \delta)^k$, we have
	\[
		\liminf_{\nnodes \to \infty} \left(\max_{j \in [\nnodes]} \lamj \Big(\sigma(Z_j(\bx_i;\paramall_1)) - \sigma(Z_j(\bx_i;\paramall_0))\Big)^2\right) > 0.
	\]
	Also, under the same conditioning, we have
	\begin{align*}
		\sum_{j' = 1}^\infty \lambda_{\nnodes,j'} \cdot \sigma(Z_{j'}(\bx_i;\paramall_0))^2
		& {} \geq
		\lamj \cdot \sigma(Z_j(\bx_i;\paramall_0))^2
		\\
		& {} =
		\frac{\lamj}{d} \cdot \ind_{\{\paramj^\top \bx_i \geq 0\}} \cdot (\paramj^\top \bx_i)^2
		\\
		& {} \geq
		\frac{(1-\gamma)\tillam_j}{d} \cdot \ind_{\{\paramj^\top \bx_i \geq 0\}} \cdot (\paramj^\top \bx_i)^2
		\\
		& {} \geq
		\frac{(1-\gamma)\tillam_j}{d} \cdot \epsilon^2
		\\
		& > 0.
	\end{align*}
	Furthermore, without any conditioning, we have
	\[
		\sum_{j' = 1}^\infty \lambda_{\nnodes,j'} \cdot \sigma(Z_{j'}(\bx_i;\paramall_0))^2 < \infty
	\]
	almost surely, because again without any conditioning, the usual expectation of the
	right-hand side of the above inequality is finite as shown below:
	\begin{align*}
		\mathbb{E}\left[\sum_{j' = 1}^\infty \lambda_{\nnodes,j'} \cdot \sigma(Z_{j'}(\bx_i;\paramall_0))^2\right]
		=
		\sum_{j' = 1}^\infty \lambda_{\nnodes,j'} \cdot \mathbb{E}\left[\sigma(Z_{j'}(\bx_i;\paramall_0))^2\right]
		& {} =
		\sum_{j' = 1}^\infty \lambda_{\nnodes,j'} \cdot \frac{\|\bx_i\|^2}{2d}
		\\
		& {} = \frac{\|\bx_i\|^2}{2d} < \infty.
	\end{align*}
	Thus, \cref{eq:featurelearningrelu-general:lower-bound1b} holds with probability
	at least $1-(1/2 + \delta)^k$. This completes the proof of the theorem.

\subsection{Proof of \cref{thm:weight-norm-square-change-relu}}

	We first compute a lower bound for the squared norm of the gradient, which does not depend on $\nnodes$.
	\begin{align*}
		\left\|\left.\nabla_{\param_{tj}} L(\paramall_t)\right|_{t=0}\right\|^2
		& {} =
		\left\|\sum_{i = 1}^n y_i \sqrt{\lamj} \aj
			\sigma'\left(\frac{\param_{0j}^\top \bx_i}{\sqrt{d}}\right)
			\frac{\bx_{i}}{\sqrt{d}}\right\|^2
		\\
		& {} =
		\frac{\lamj}{d}
		\left\|\sum_{i = 1}^n y_i
			\sigma'\left(\frac{\param_{0j}^\top \bx_i}{\sqrt{d}}\right)
			\bx_{i}\right\|^2
		\\
		& {} \geq
		\frac{(1-\gamma)\tillam_j}{d}
		\left\|\sum_{i = 1}^n y_i
			\sigma'\left(\frac{\param_{0j}^\top \bx_i}{\sqrt{d}}\right)
			\bx_{i}\right\|^2
		\\
		& {} =
		\frac{(1-\gamma)\tillam_j}{d}
		\left|
		\sum_{k = 1}^\din \sum_{i = 1}^n \sum_{i'=1}^n
			y_iy_{i'}
			\sigma'\left(\frac{\param_{0j}^\top \bx_i}{\sqrt{d}}\right)
			\sigma'\left(\frac{\param_{0j}^\top \bx_{i'}}{\sqrt{d}}\right)
			x_{ik}x_{i'k}
		\right|
		\\
		& {} =
		\frac{(1-\gamma)\tillam_j}{d}
		\left|\sum_{i = 1}^n \sum_{i'=1}^n
			y_iy_{i'}
			\left(\bx_{i}^\top \bx_{i'}
				\sigma'\left(\frac{\param_{0j}^\top \bx_i}{\sqrt{d}}\right)
				\sigma'\left(\frac{\param_{0j}^\top \bx_{i'}}{\sqrt{d}}\right)
			\right)
		\right|
		\\
		\\
		& {} =
		\frac{(1-\gamma)\tillam_j}{d}
		\left|\sum_{i = 1}^n \sum_{i'=1}^n
			y_iy_{i'}
			\left(\bx_{i}^\top \bx_{i'}
				\ind_{\left\{\param_{0j}^\top \bx_i \geq 0\right\}}
				\ind_{\left\{\param_{0j}^\top \bx_{i'} \geq 0\right\}}
			\right)
		\right|.
	\end{align*}
	Thus,
	\[
		\liminf_{\nnodes \to \infty}\left\|\left.\nabla_{\param_{tj}} L(\paramall_t)\right|_{t=0}\right\|^2
		\geq
		\frac{(1-\gamma)\tillam_j}{d}
		\left|\sum_{i = 1}^n \sum_{i'=1}^n
			y_iy_{i'}
			\left(\bx_{i}^\top \bx_{i'}
				\ind_{\{\param_{0j}^\top \bx_i \geq 0\}}
				\ind_{\{\param_{0j}^\top \bx_{i'} \geq 0\}}
			\right)
		\right|.
	\]
	But by assumption, the factor $((1-\gamma)\tillam_j) / d$ in the lower bound is always positive.
	The claim of the theorem follows from the property that the other factor in the lower bound
	is also positive with probability at least $1/2$. In the rest of the proof, we will show
	why this is so.

	 Note that
	\[
	\|\bx_{i}\|^2
	\ind_{\left\{\param_{0j}^\top \bx_i \geq 0\right\}} > 0.
	\]
	if and only if $\param_{0j}^\top \bx_i \geq 0$.

	Condition on $\param_{0j}$ and recall that
	the $y_i$'s are continuous independent
	real-valued random variables which are also independent from $\param_{0j}$, thus their distributions are unaffected by the conditioning.
	If $\param_{0j}^\top \bx_i \geq 0$, the inequality
	\[
	\left|\sum_{i = 1}^n \sum_{i'=1}^n
	y_iy_{i'}
	\left(\bx_{i}^\top \bx_{i'}
		\ind_{\{\param_{0j}^\top \bx_i \geq 0\}}
		\ind_{\{\param_{0j}^\top \bx_{i'} \geq 0\}}
	\right)
	\right|
	> 0
	\]
	holds almost surely. Since $\param_{0j}^\top \bx_i$
	holds with probability $1/2$, the above inequality holds unconditionally with probability at least $1/2$, as desired.

\section{Additional experimental results (smooth activation)}
\label{appendix_results}

We provide here additional results for the experiments described in \cref{sec:experiments}.

\subsection{Regression}
In Figures \ref{fig:concrete_all}, \ref{fig:energy_all}, \ref{fig:airfoil_all} and \ref{fig:plant_all} we respectively provide the detailed results for the datasets \texttt{concrete}, \texttt{energy},  \texttt{airfoil} and \texttt{plant}.

\begin{figure}
    \centering
    \includegraphics[width=0.4\linewidth]{figures/swish/concrete/trainrisk.pdf}
    \includegraphics[width=0.4\linewidth]{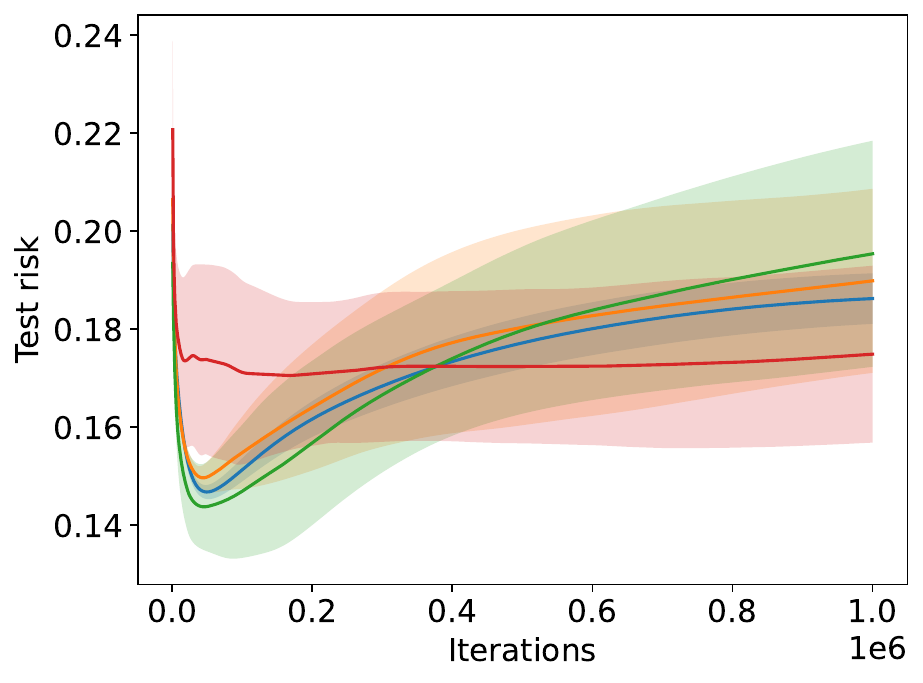}
    \includegraphics[width=0.3\linewidth]{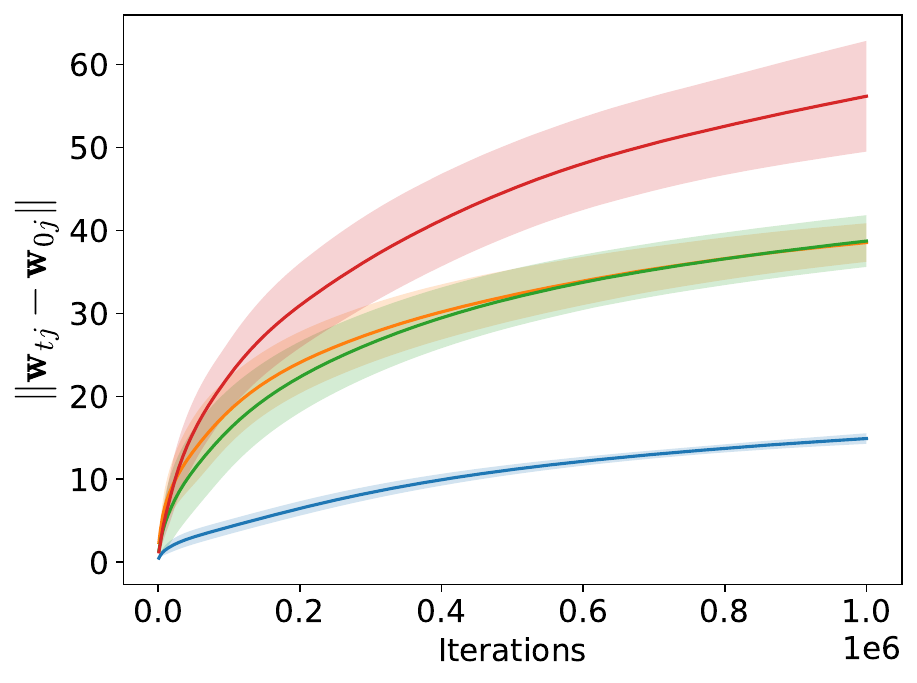}
    \includegraphics[width=0.3\linewidth]{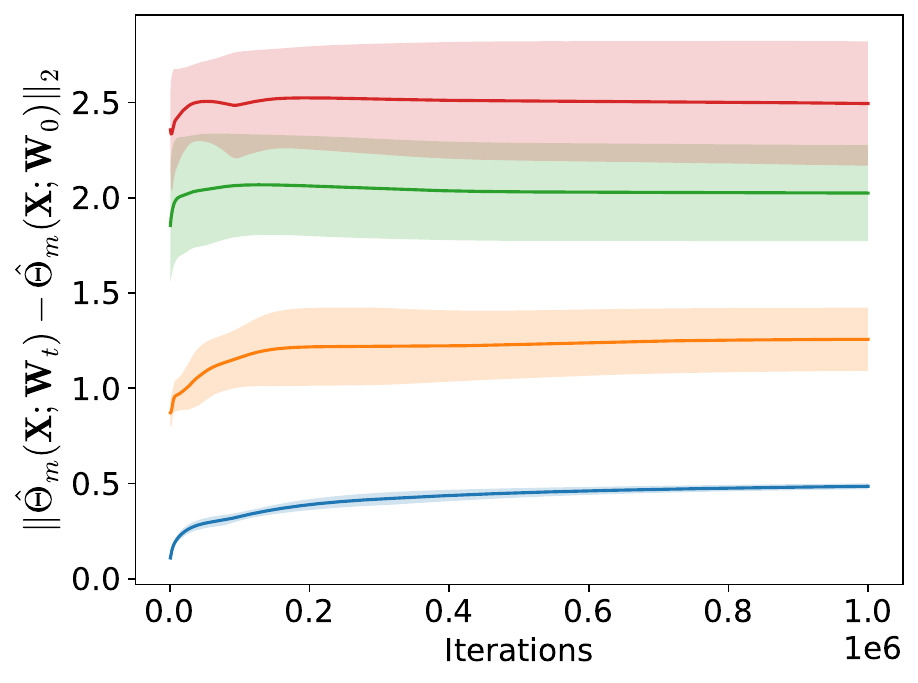}
    \includegraphics[width=0.3\linewidth]{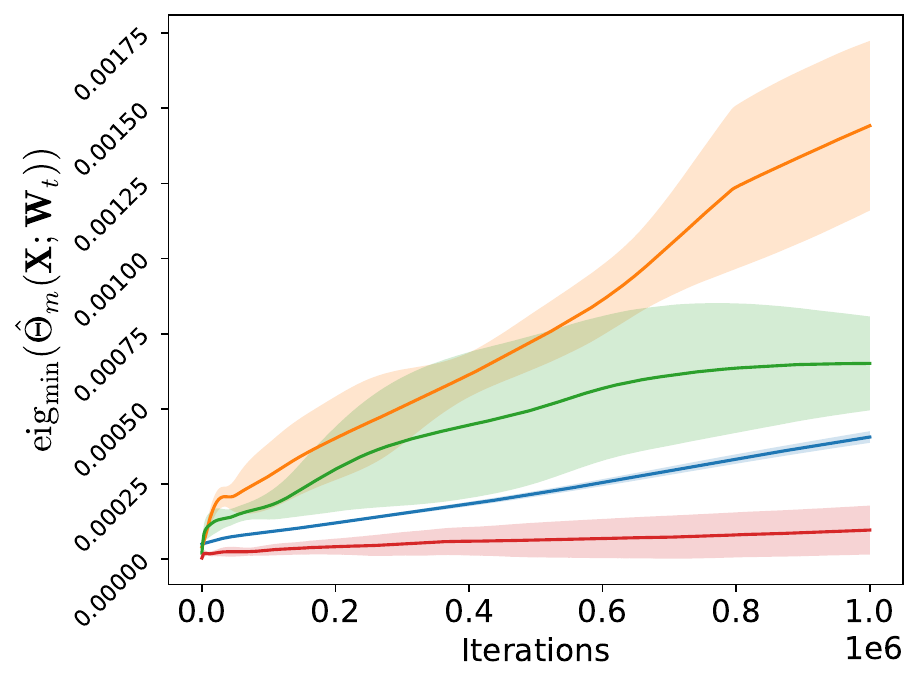}
    \includegraphics[width=0.4\linewidth]{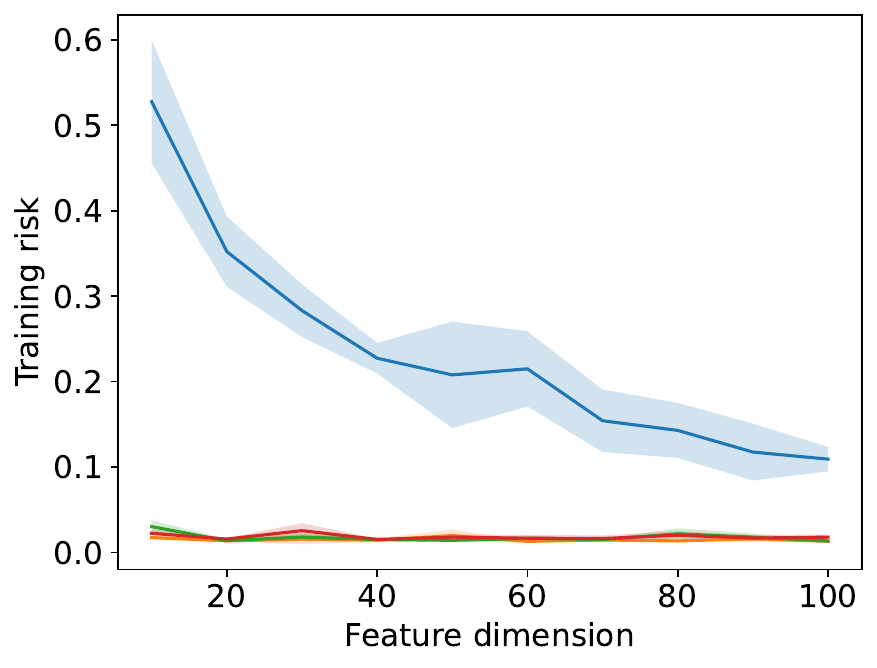}
    \includegraphics[width=0.4\linewidth]{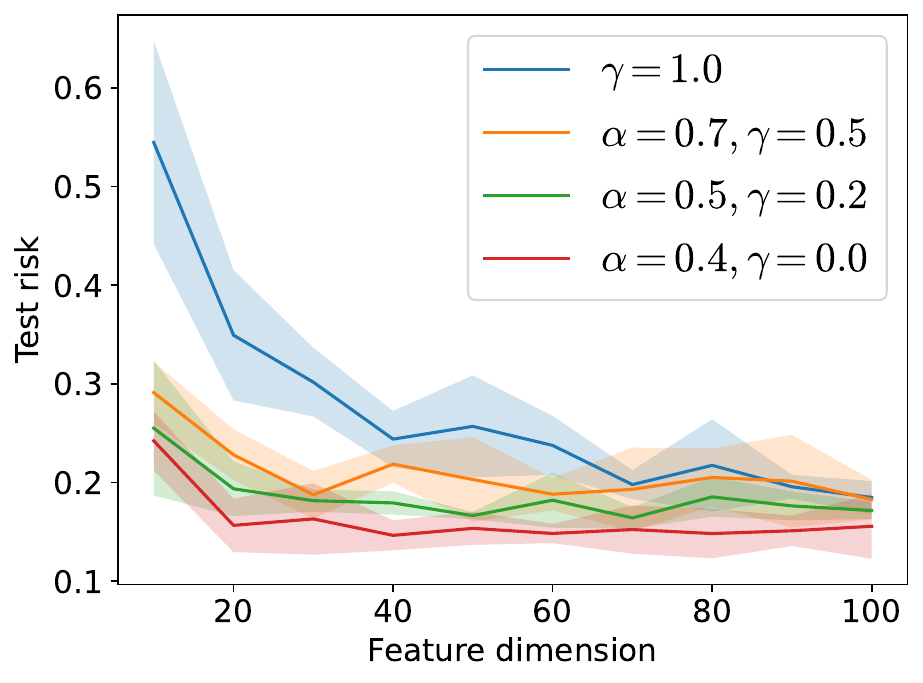}
    \caption[Results for the \texttt{concrete} dataset (swish)]{Results for the \texttt{concrete} dataset (swish). From left to right and top to bottom, 1) training risks, 2) test risks, 3) differences in weight norms $\Vert \mathbf{w}_{tj} - \mathbf{w}_{0j}\Vert$ with $j$'s being the neurons having the maximum difference at the end of the training, 4) difference in NTG matrices, 5) minimum NTG eigenvalues, 6) training risks for transfer learning, and 7) test risks for transfer learning.}
    \label{fig:concrete_all}
\end{figure}

\begin{figure}
    \centering
    \includegraphics[width=0.4\linewidth]{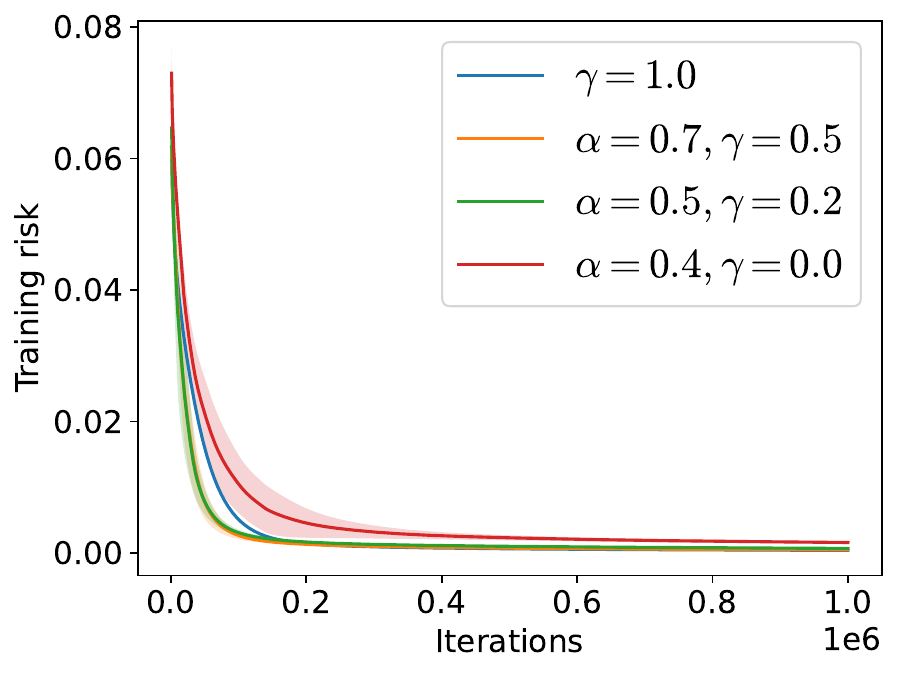}
    \includegraphics[width=0.4\linewidth]{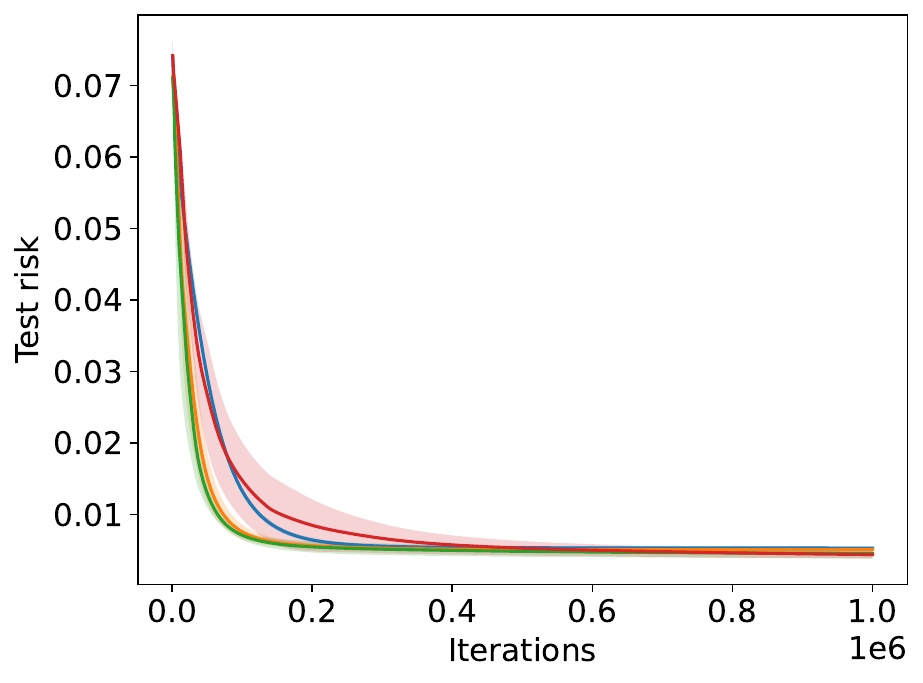}
    \includegraphics[width=0.3\linewidth]{figures/swish/energy/wdiff.pdf}
    \includegraphics[width=0.3\linewidth]{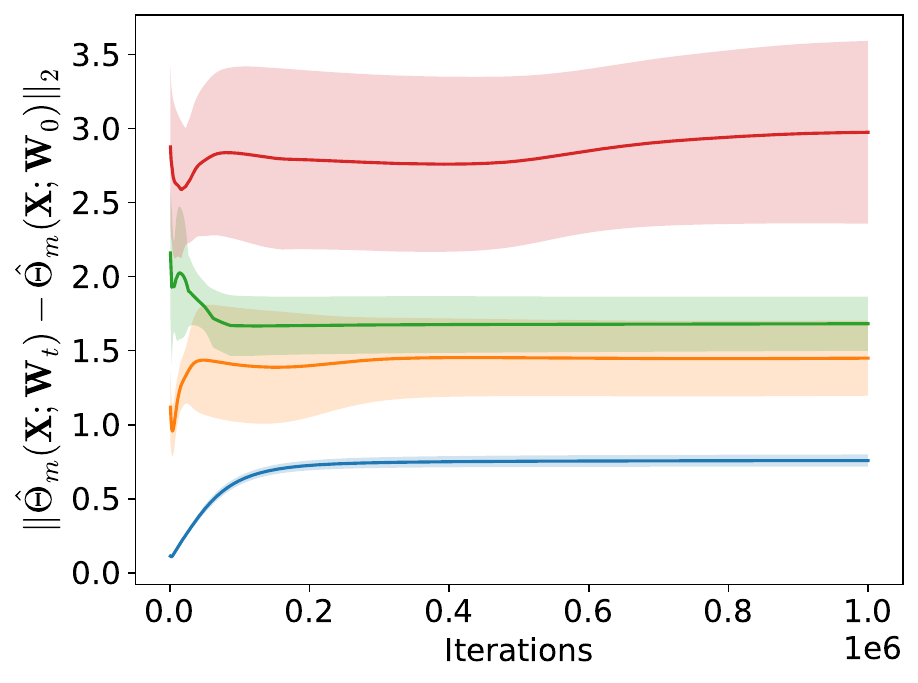}
    \includegraphics[width=0.3\linewidth]{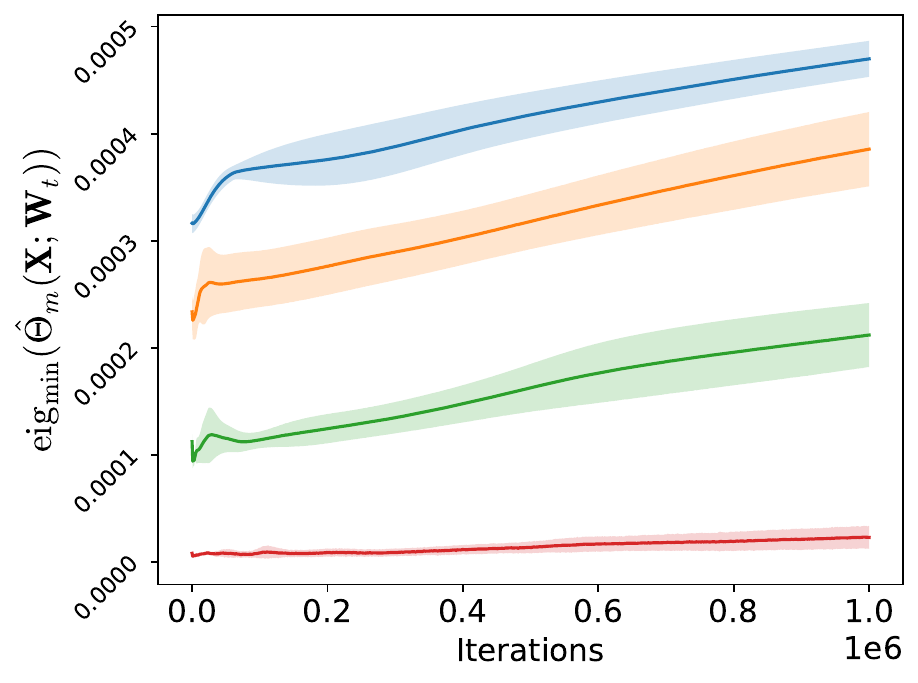}
    \includegraphics[width=0.4\linewidth]{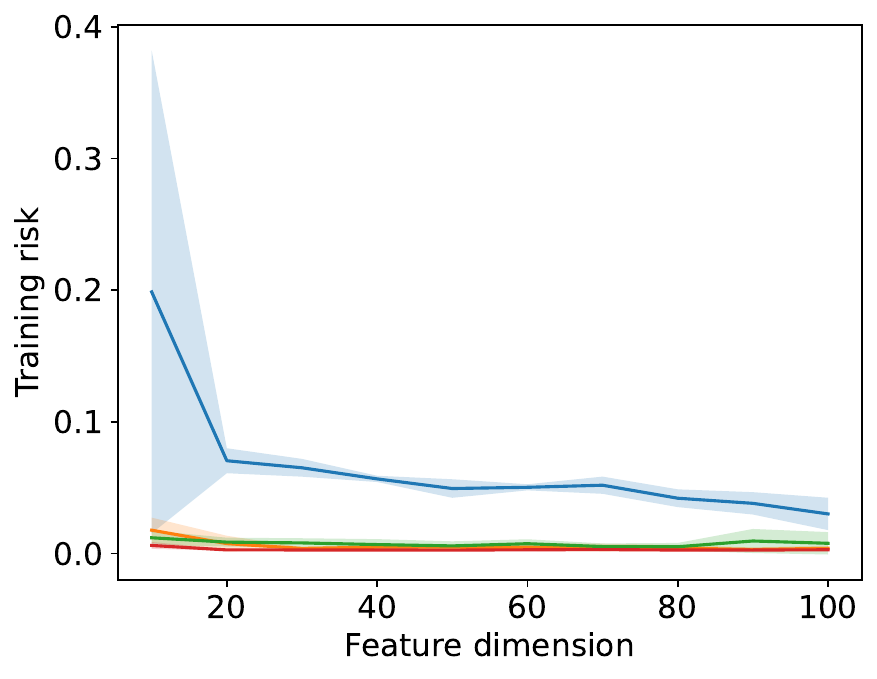}
    \includegraphics[width=0.4\linewidth]{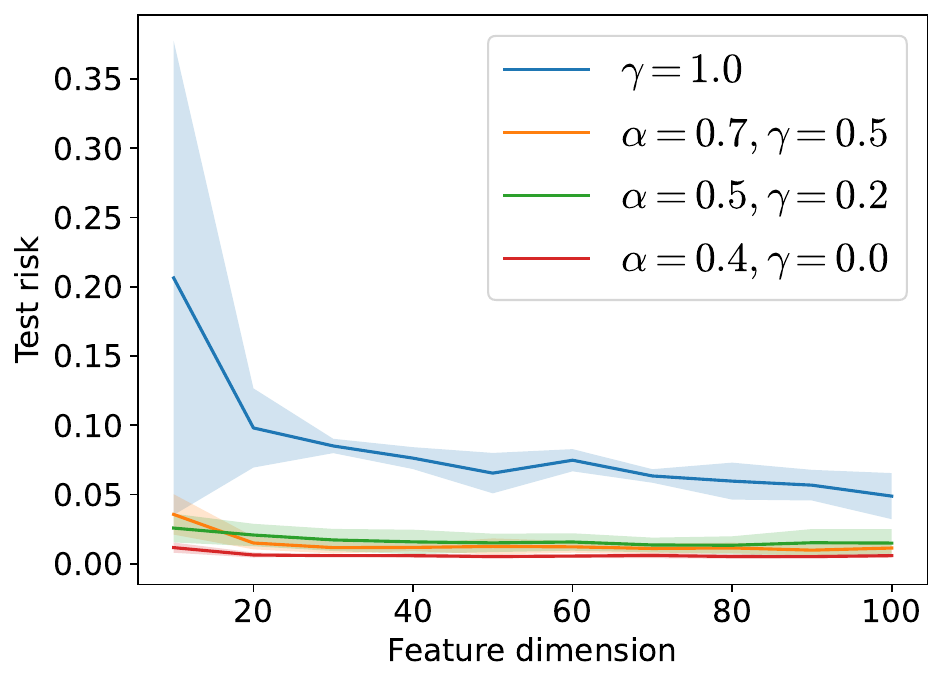}
    \caption[Results for the \texttt{energy} dataset (swish)]{Results for the \texttt{energy} dataset (swish).  From left to right and top to bottom, 1) training risks, 2) test risks, 3) differences in weight norms $\Vert \mathbf{w}_{tj} - \mathbf{w}_{0j}\Vert$ with $j$'s being the neurons having the maximum difference at the end of the training, 4) difference in NTG matrices, 5) minimum NTG eigenvalues, 6) training risks for transfer learning, and 7) test risks for transfer learning.}
    \label{fig:energy_all}
\end{figure}

\begin{figure}
    \centering
    \includegraphics[width=0.4\linewidth]{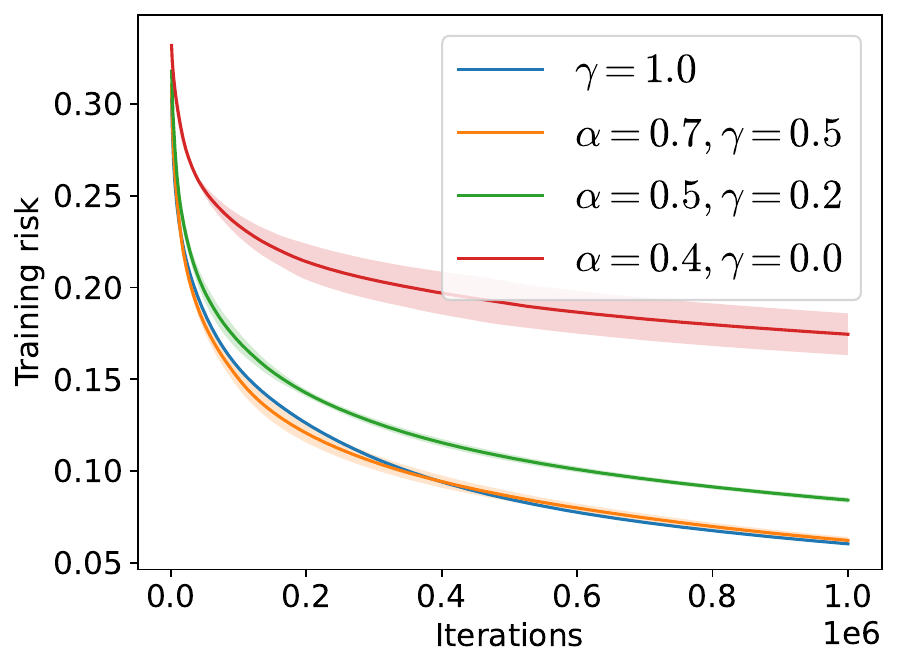}
    \includegraphics[width=0.4\linewidth]{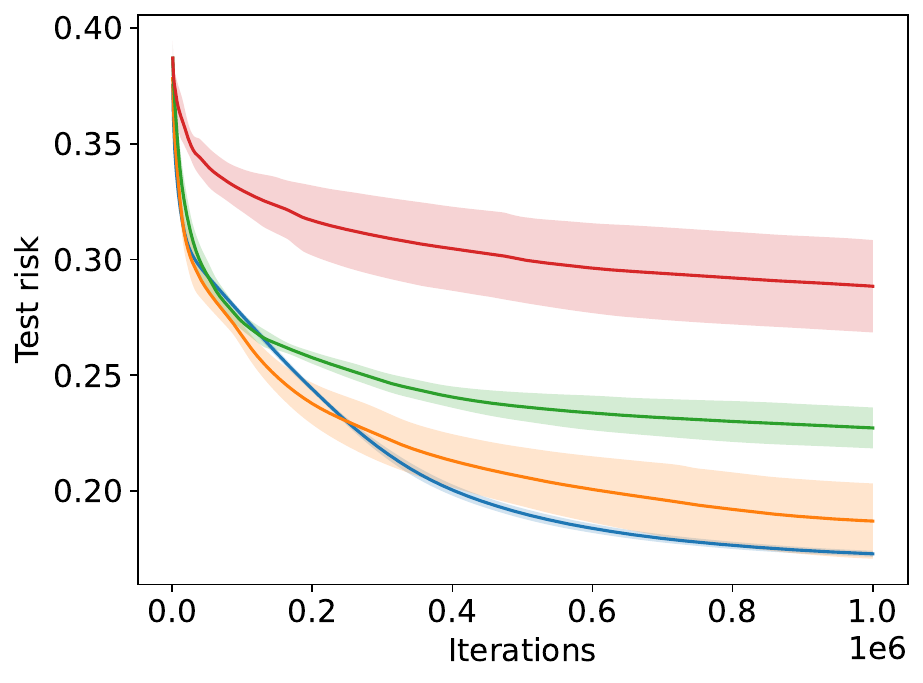}
    \includegraphics[width=0.3\linewidth]{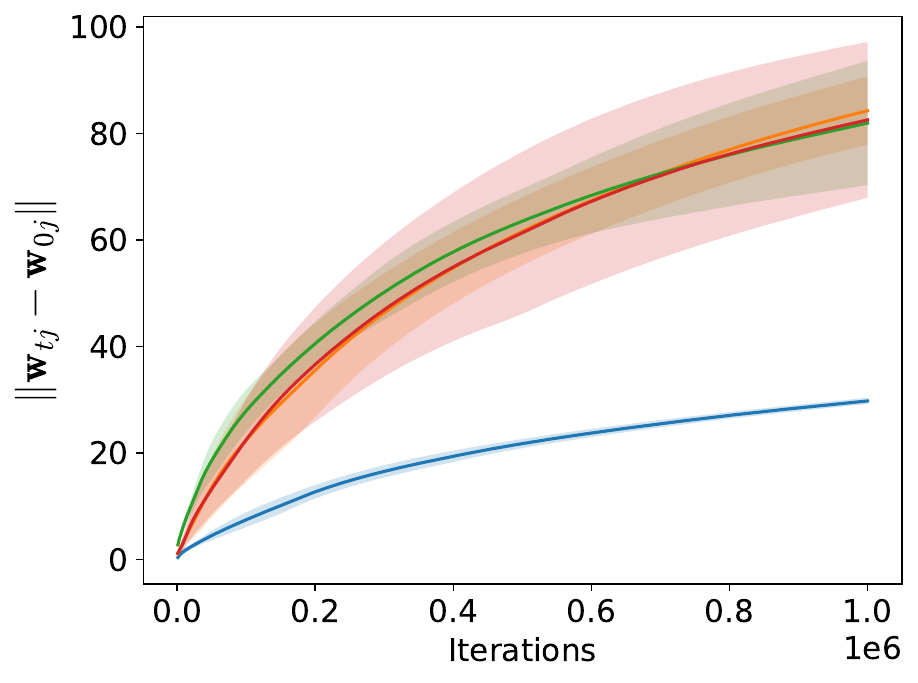}
    \includegraphics[width=0.3\linewidth]{figures/swish/airfoil/ntkdiff.pdf}
    \includegraphics[width=0.3\linewidth]{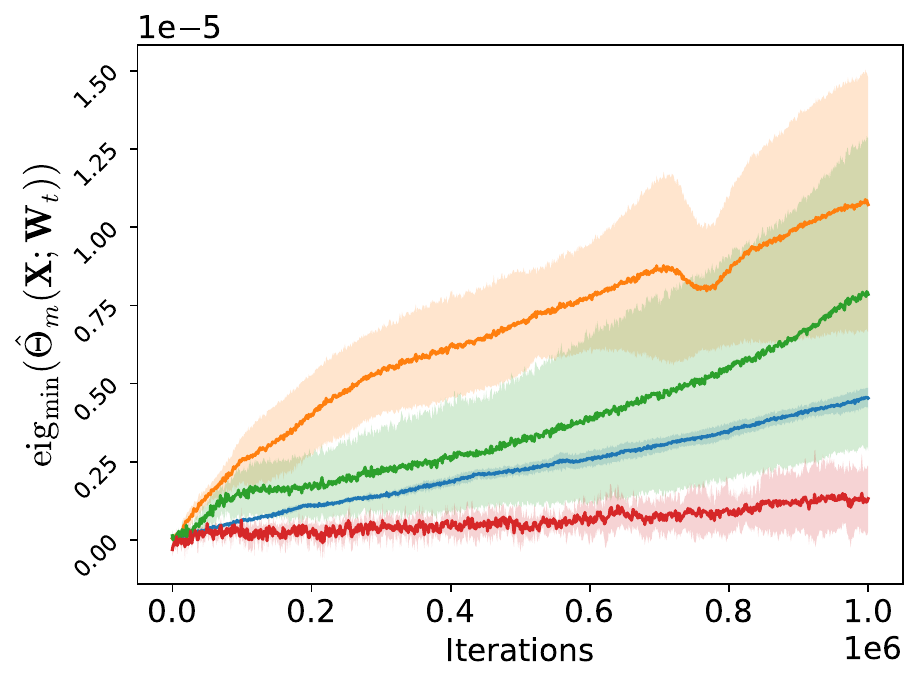}
    \includegraphics[width=0.4\linewidth]{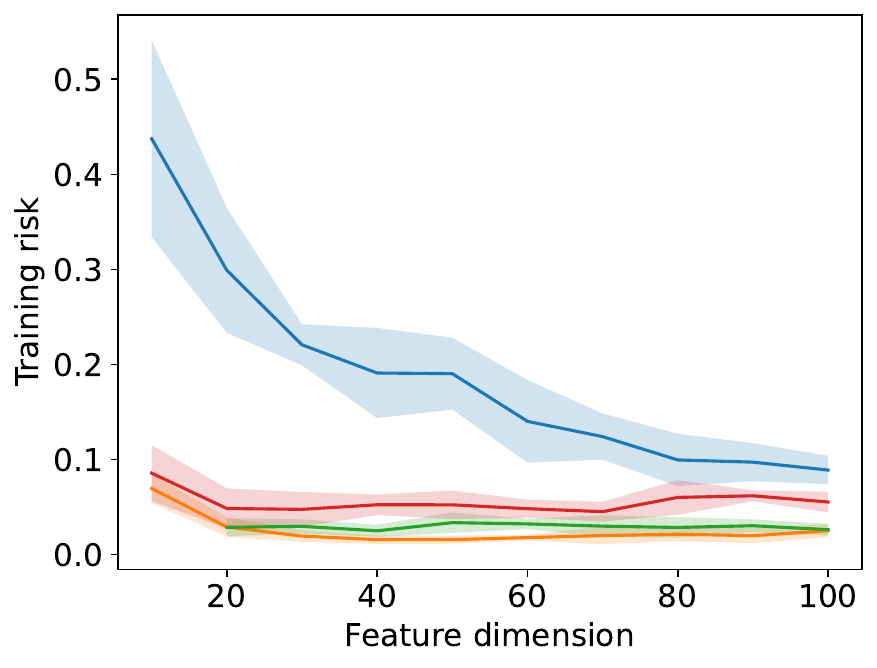}
    \includegraphics[width=0.4\linewidth]{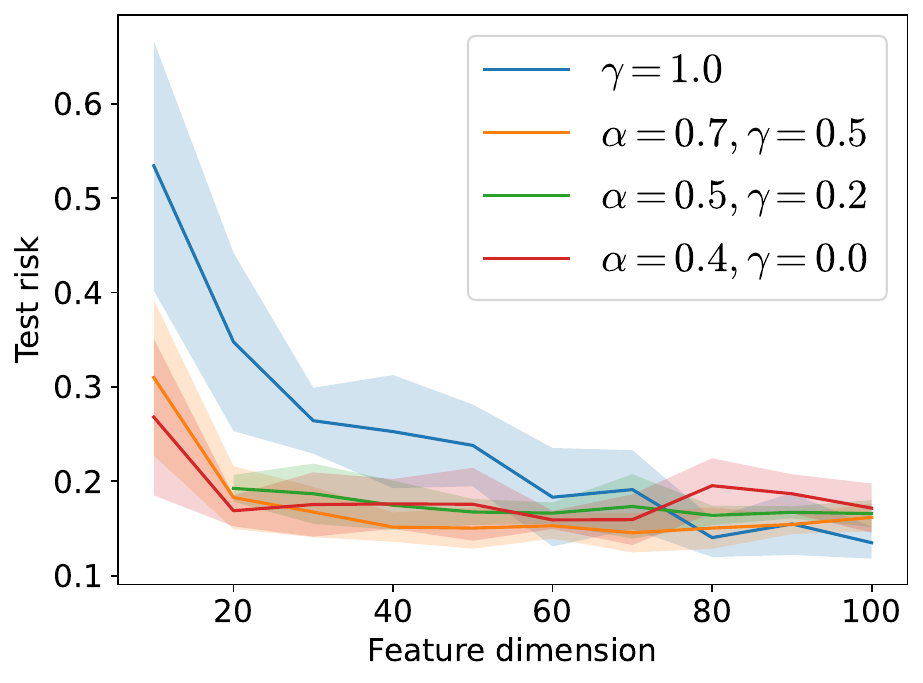}
    \caption[Results for the \texttt{airfoil} dataset (swish)]{Results for the \texttt{airfoil} dataset (swish).  From left to right and top to bottom, 1) training risks, 2) test risks, 3) differences in weight norms $\Vert \mathbf{w}_{tj} - \mathbf{w}_{0j}\Vert$ with $j$'s being the neurons having the maximum difference at the end of the training, 4) difference in NTG matrices, 5) minimum NTG eigenvalues, 6) training risks for transfer learning, and 7) test risks for transfer learning.}
    \label{fig:airfoil_all}
\end{figure}

\begin{figure}
    \centering
    \includegraphics[width=0.4\linewidth]{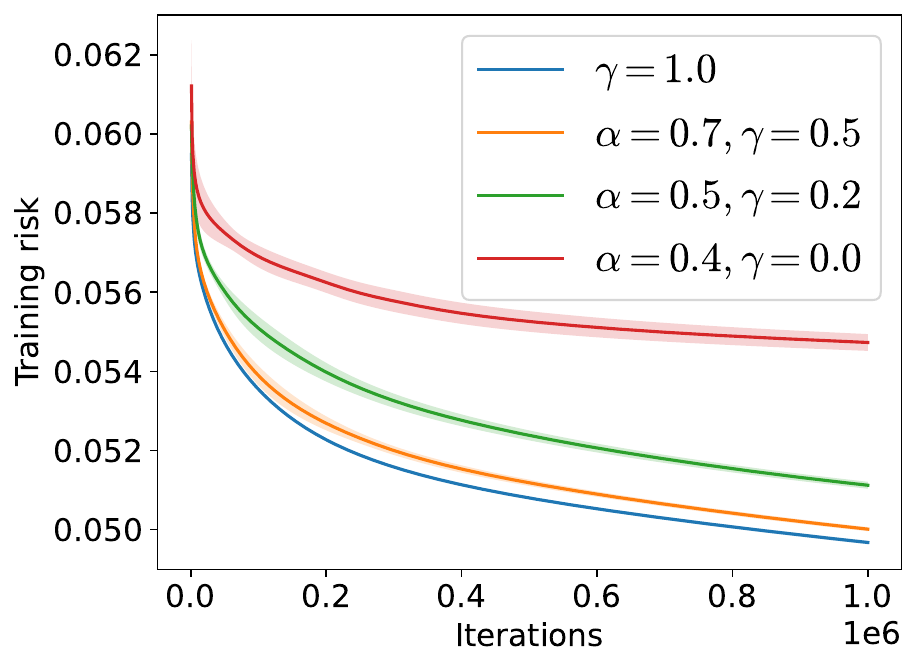}
    \includegraphics[width=0.4\linewidth]{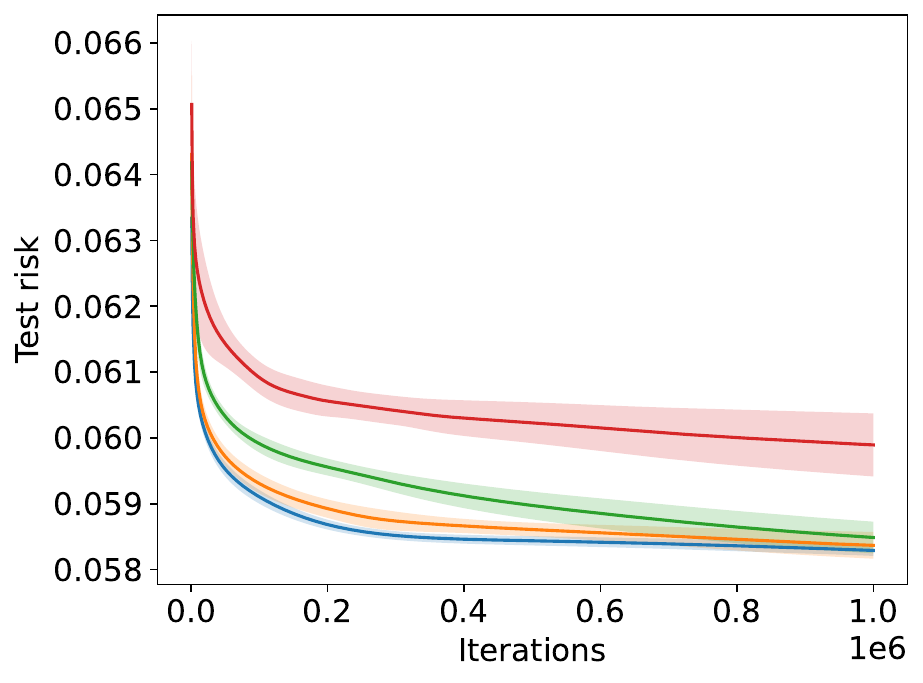}
    \includegraphics[width=0.3\linewidth]{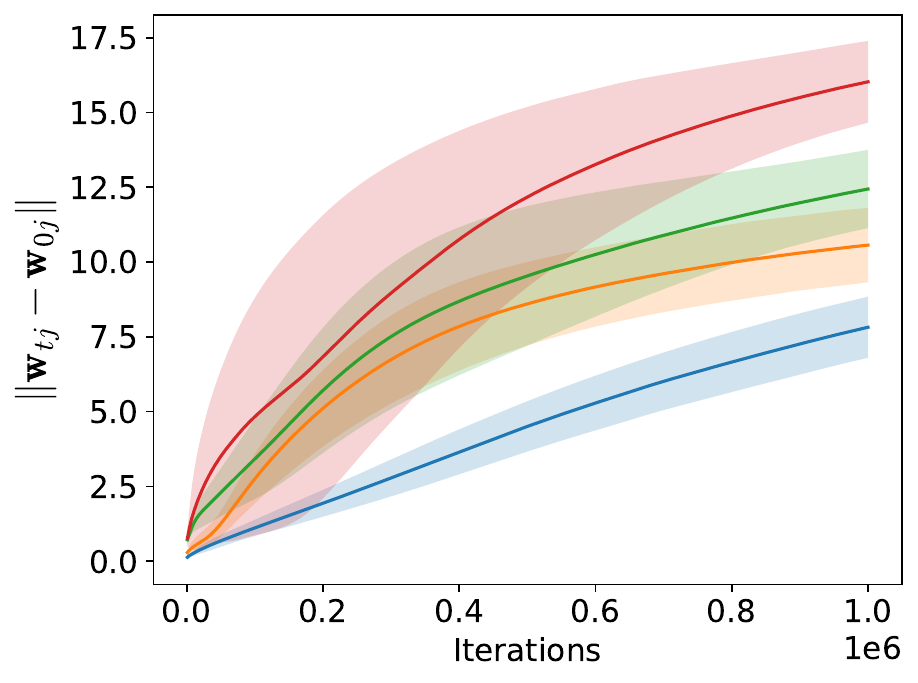}
    \includegraphics[width=0.3\linewidth]{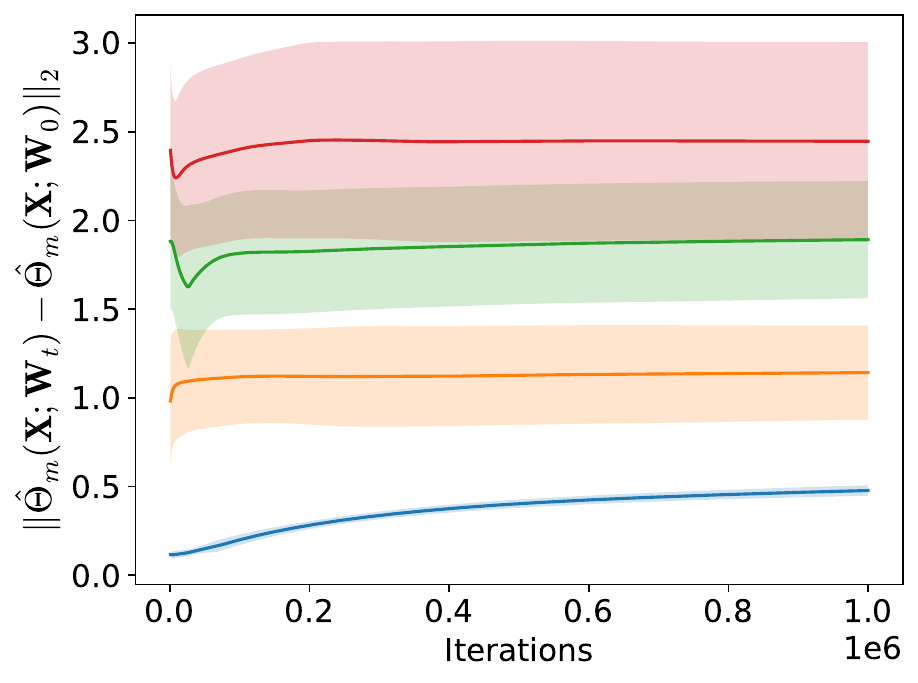}
    \includegraphics[width=0.3\linewidth]{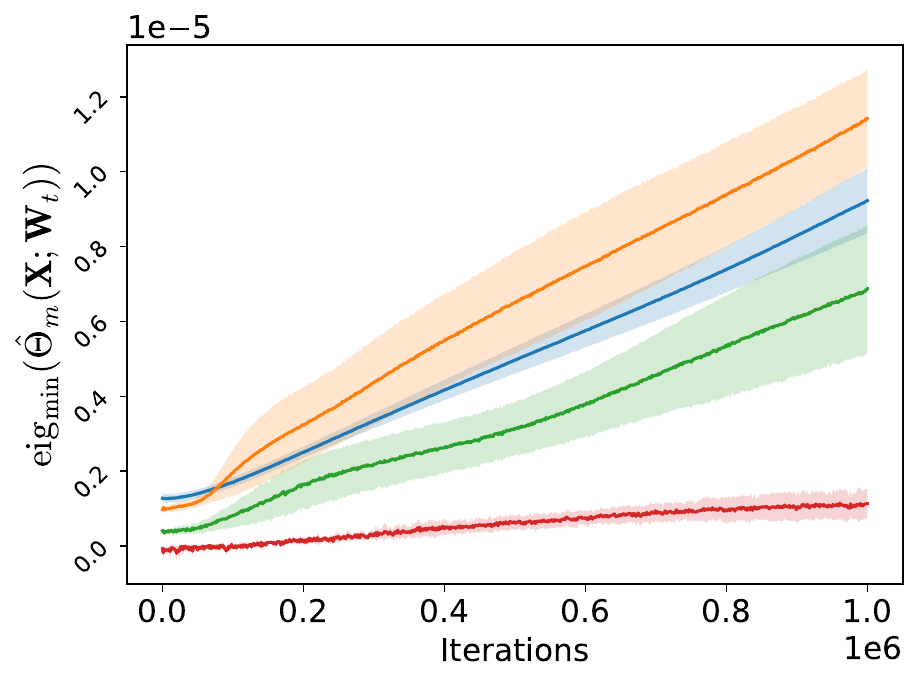}
    \includegraphics[width=0.4\linewidth]{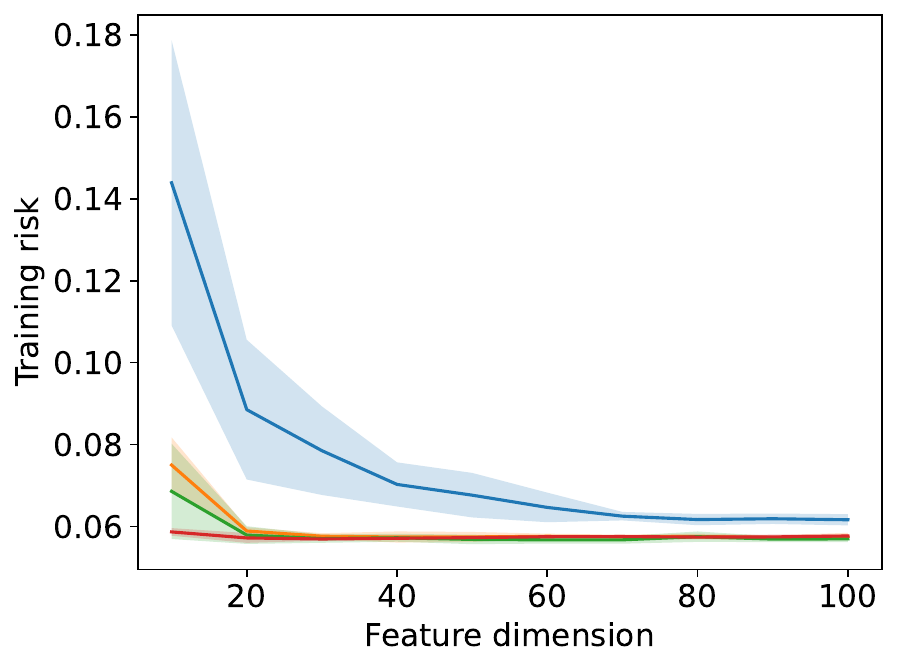}
    \includegraphics[width=0.4\linewidth]{figures/swish/plant/transfer_testrisk.pdf}
    \caption[Results for the \texttt{plant} dataset (swish)]{Results for the \texttt{plant} dataset (swish).  From left to right and top to bottom, 1) training risks, 2) test risks, 3) differences in weight norms $\Vert \mathbf{w}_{tj} - \mathbf{w}_{0j}\Vert$ with $j$'s being the neurons having the maximum difference at the end of the training, 4) difference in NTG matrices, 5) minimum NTG eigenvalues, 6) training risks for transfer learning, and 7) test risks for transfer learning.}
    \label{fig:plant_all}
\end{figure}

\subsection{Classification}
We provide in \cref{fig:mnist_all} detailed results for the MNIST dataset, and in \cref{fig:cifar10} results for the CIFAR--10 dataset. In \cref{fig:different_gamma}, we provide further details on the individual impact of the parameter $\gamma\in[0,1]$. Recall that the smaller the value of $\gamma$, the more asymmetry is introduced, where $\gamma=1$ recovers the iid model. We can see from the experiments that pruning performance is improved as $\gamma$ becomes smaller.

% In \cref{fig:diff_alpha}, we study the impact of $\alpha$ for a fixed value of $\gamma=0.5$ with the CIFAR--10 dataset. In early iterations, larger weight variations happen for smaller values of $\alpha$ as predicted by the theory. However, by the end of the training, this is not true anymore, and all examined values of $\alpha$ lead to similar variations of the weights. We can also notice that larger values of $\alpha$ seem to give a slight advantage in terms of prunability in this setting.

% In \cref{fig:cifar10_alpha05}, we study the impact of the $\gamma$ parameter on CIFAR10. We recall that this parameter controls the weight of the symmetric part of the lambdas. We can see that these results corroborates the theory, a smaller value of $\gamma$ leads to improved pruned models, as the asymmetrical part, responsible for feature learning becomes more prevalent.

\begin{figure}
    \centering
    \includegraphics[width=0.4\linewidth]{figures/swish/mnist/trainacc.pdf}
    \includegraphics[width=0.4\linewidth]{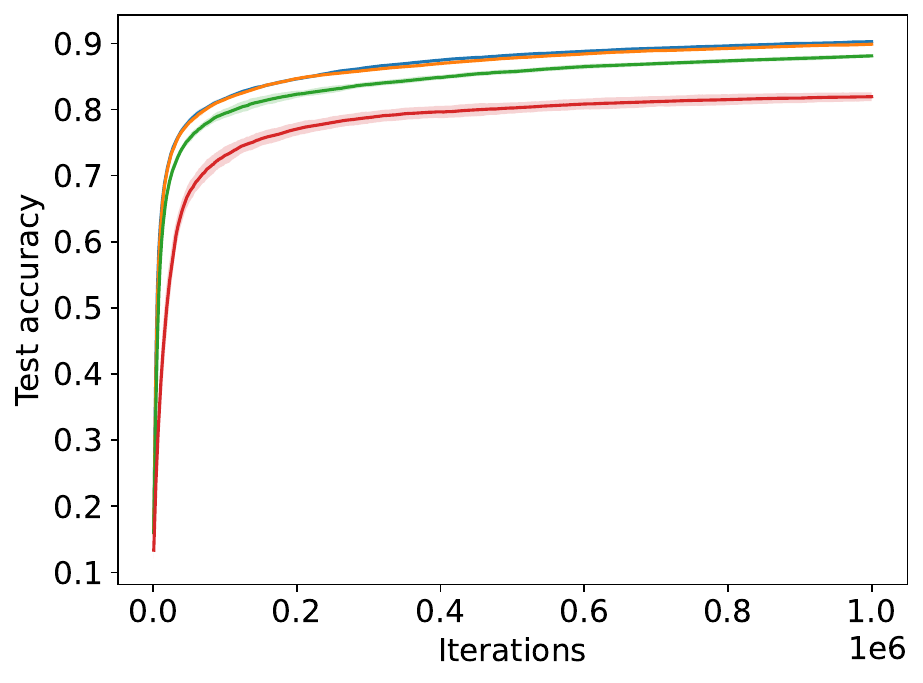}
    \includegraphics[width=0.3\linewidth]{figures/swish/mnist/wdiff.pdf}
    \includegraphics[width=0.3\linewidth]{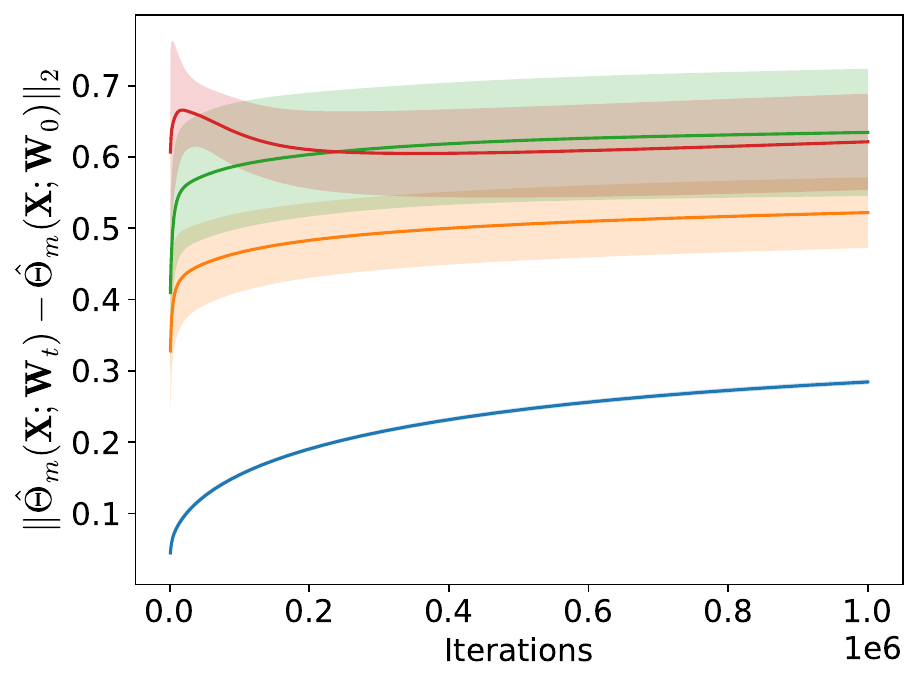}
    \includegraphics[width=0.3\linewidth]{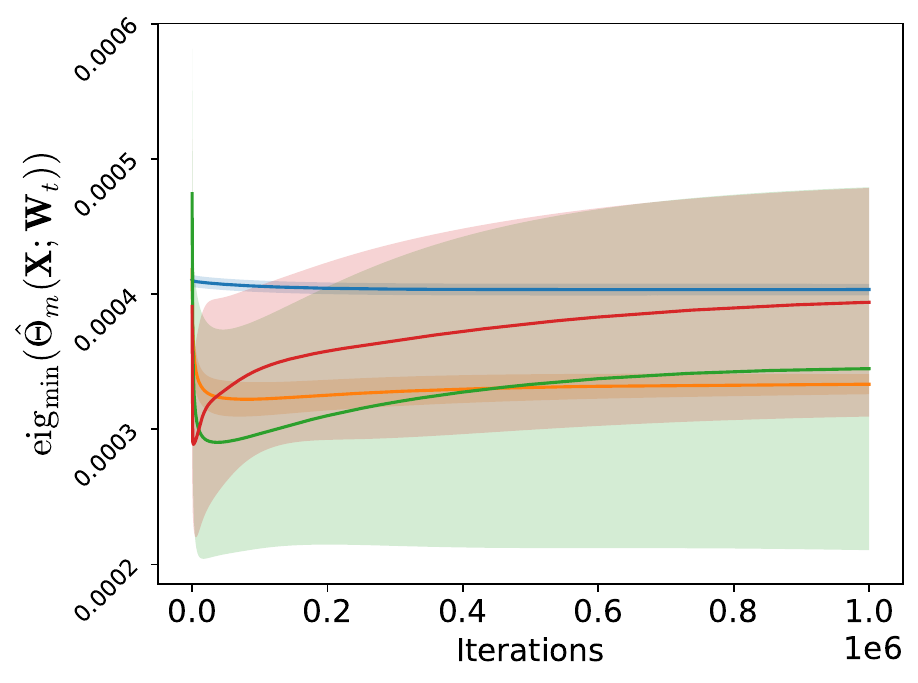}
    \includegraphics[width=0.4\linewidth]{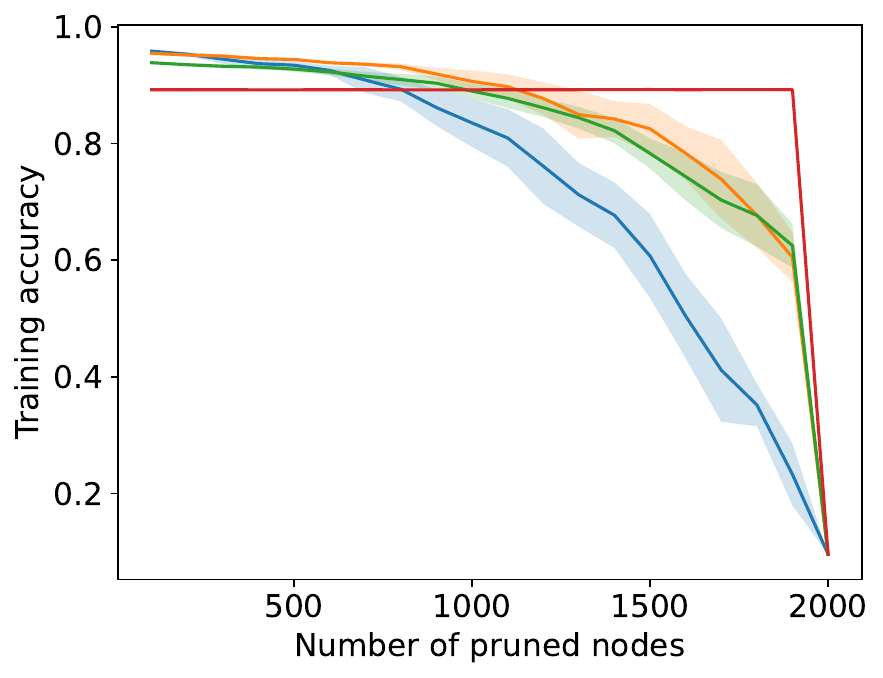}
    \includegraphics[width=0.4\linewidth]{figures/swish/mnist/prune_testacc.pdf}
    \includegraphics[width=0.4\linewidth]{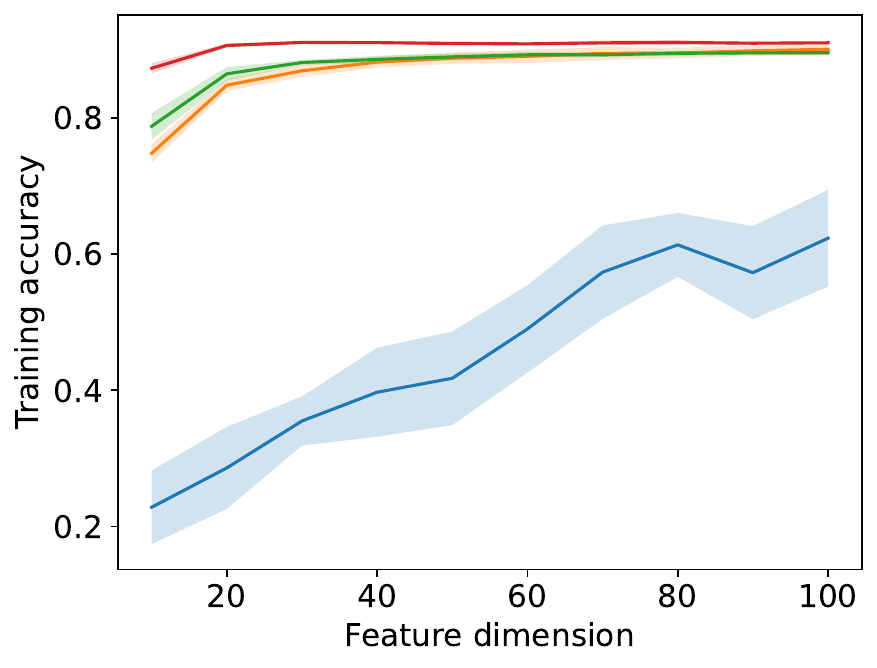}
    \includegraphics[width=0.4\linewidth]{figures/swish/mnist/transfer_testacc.pdf}
    \caption[Results for the \texttt{MNIST} dataset (swish)]{Results for the \texttt{MNIST} dataset (swish). From left to right and top to bottom, 1) training risks, 2) test risks, 3) differences in weight norms $\Vert \mathbf{w}_{tj} - \mathbf{w}_{0j}\Vert$ with $j$'s being the neurons having the maximum difference at the end of the training, 4) difference in NTG matrices, 5) minimum NTG eigenvalues, 6) training accuracies for pruning, 7) test accuracies for pruning, 8) training accuracies for transfer learning, and 9) test accuracies for transfer learning.}
    \label{fig:mnist_all}
\end{figure}

\begin{figure*}
    \centering
    \includegraphics[width=0.4\linewidth]{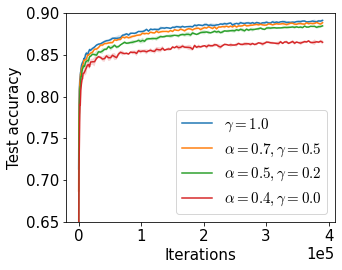}
    \includegraphics[width=0.4\linewidth]{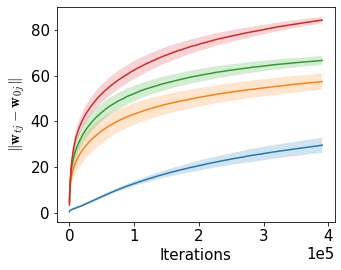}
    \includegraphics[width=0.4\linewidth]{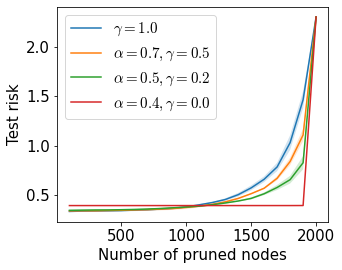}
    \includegraphics[width=0.4\linewidth]{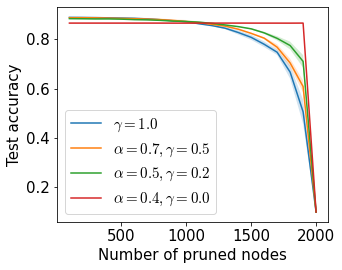}
    \caption[Results for the \texttt{CIFAR--10} dataset (swish)]{Results for the \texttt{CIFAR--10} dataset (swish). From left to right and top to bottom, 1) test accuracies through training, 2) differences in weight norms $\Vert \mathbf{w}_{tj} - \mathbf{w}_{0j}\Vert$ with $j$'s being the neurons having the maximum difference at the end of the training, 3) test risks of the pruned models, and 4) test accuracies of the pruned models.}
    \label{fig:cifar10}
    \vspace{-1em}
\end{figure*}

\begin{figure*}
    \centering
    \includegraphics[width=0.4\linewidth]{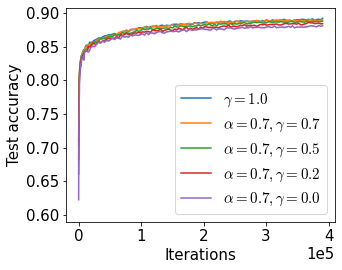}
    \includegraphics[width=0.4\linewidth]{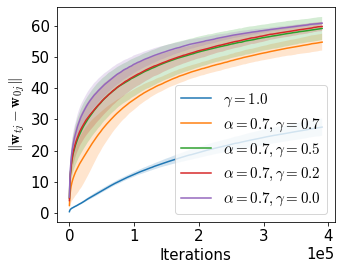}
    \includegraphics[width=0.4\linewidth]{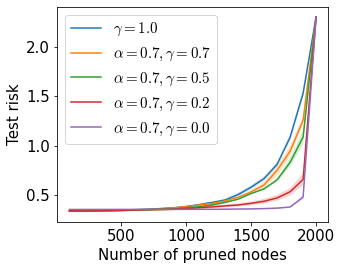}
    \includegraphics[width=0.4\linewidth]{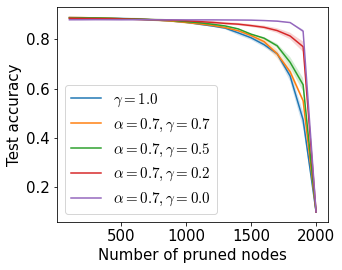}
    \caption[Results for the \texttt{CIFAR--10} dataset (swish). Impact of the parameter $\gamma$.]{Results for the \texttt{CIFAR--10} dataset (swish). Impact of the parameter $\gamma$.  From left to right and top to bottom, 1) test accuracies through training, 2) differences in weight norms $\Vert \mathbf{w}_{tj} - \mathbf{w}_{0j}\Vert$ with $j$'s being the neurons having the maximum difference at the end of the training, 3) test risks of the pruned models, and 4) test accuracies of the pruned models.}
    \label{fig:different_gamma}
    \vspace{-1em}
\end{figure*}

\section{Experimental results for the ReLU activation function}
\label{sec:additionalexperiments-relu}

We provide here additional experimental results, as in \cref{appendix_results}, but with a different activation function. The experimental setting is the same as described in \cref{sec:experiments}, except that the swish activation function is replaced by the ReLU function. Although our theory does not cover the convergence of GD with the ReLU, the experimental results obtained in this section are quantitatively similar to those obtained with the swish function.

\subsection{Regression}
In Figures \ref{fig:concrete_all_relu}, \ref{fig:energy_all_relu}, \ref{fig:airfoil_all_relu} and \ref{fig:plant_all_relu} we respectively provide detailed results for the datasets \texttt{concrete}, \texttt{energy},  \texttt{airfoil} and \texttt{plant}.

\begin{figure}
    \centering
    \includegraphics[width=0.4\linewidth]{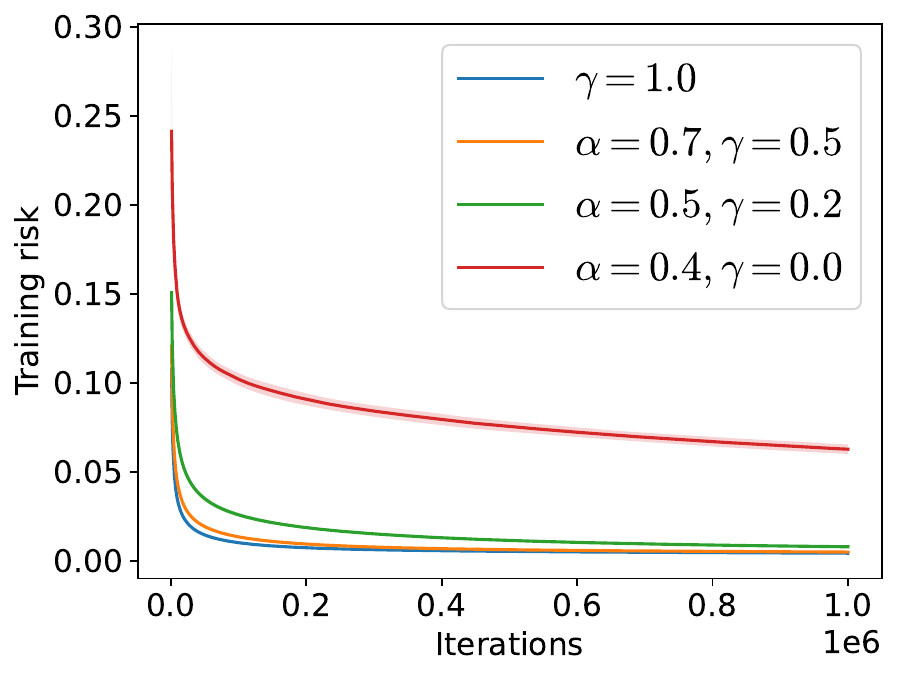}
    \includegraphics[width=0.4\linewidth]{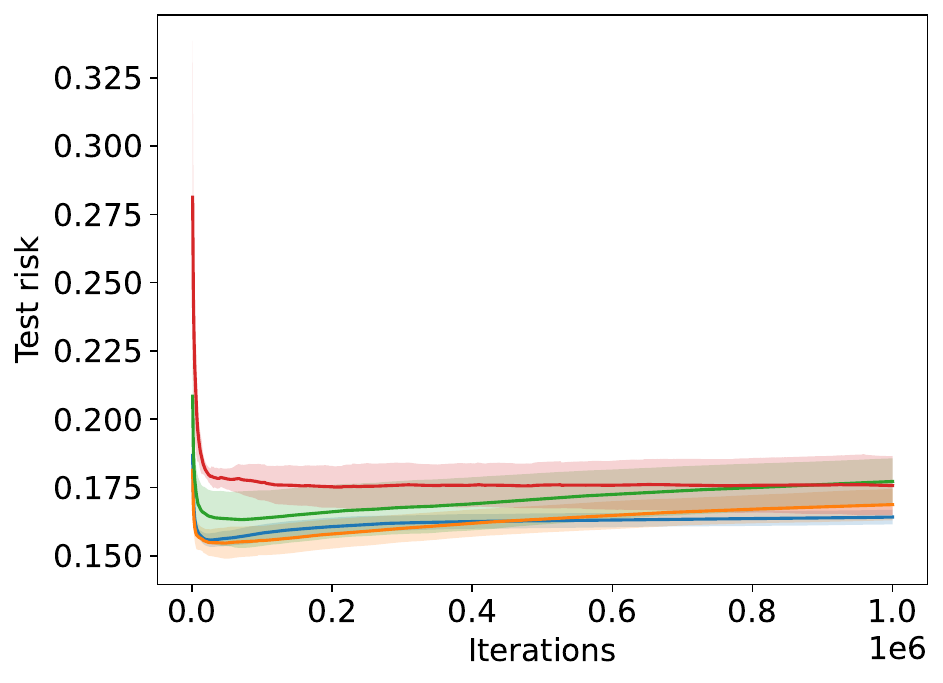}
    \includegraphics[width=0.32\linewidth]{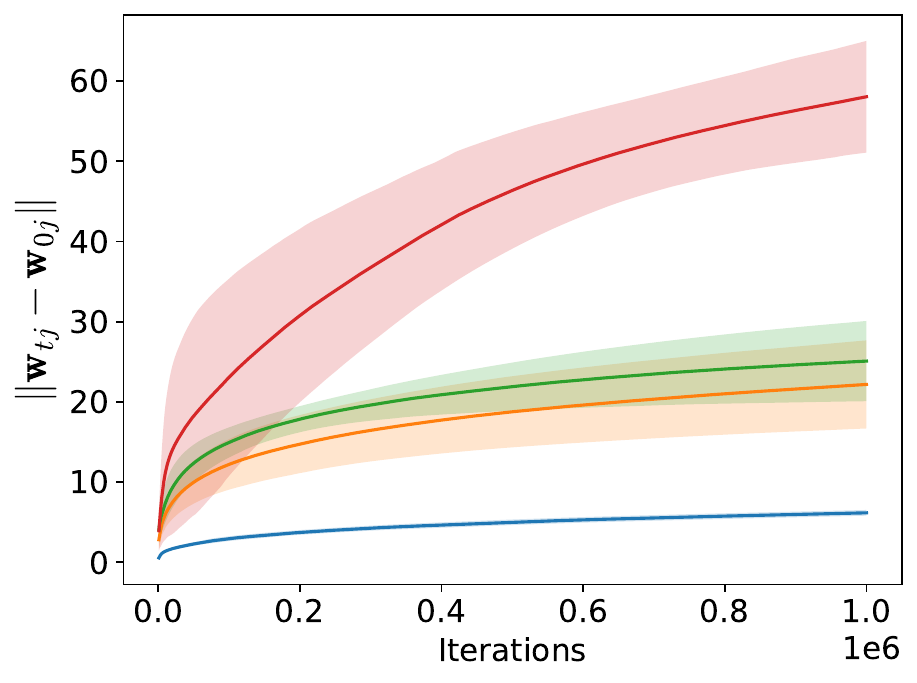}
    \includegraphics[width=0.32\linewidth]{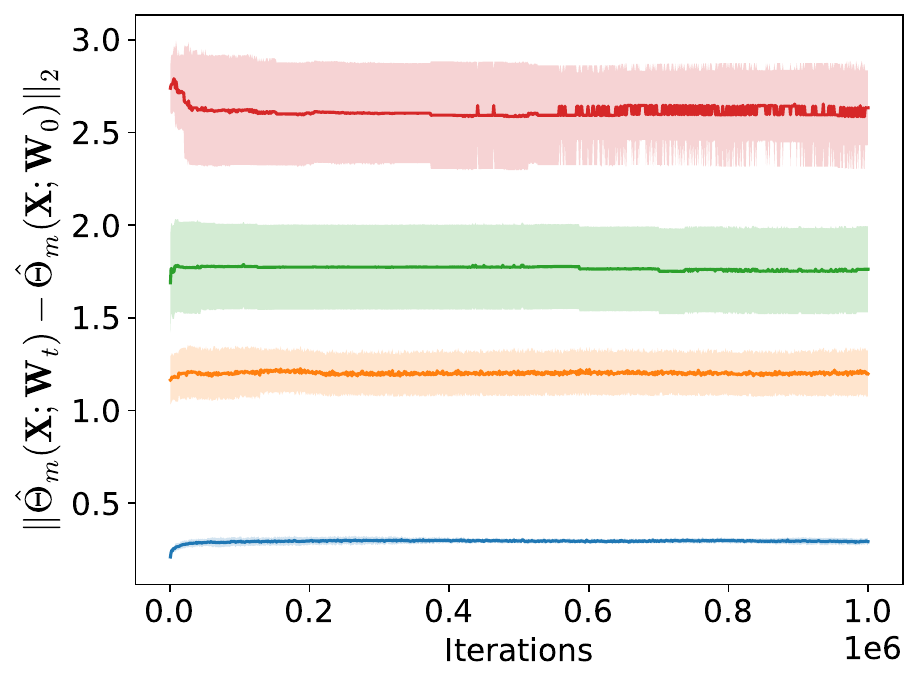}
    \includegraphics[width=0.32\linewidth]{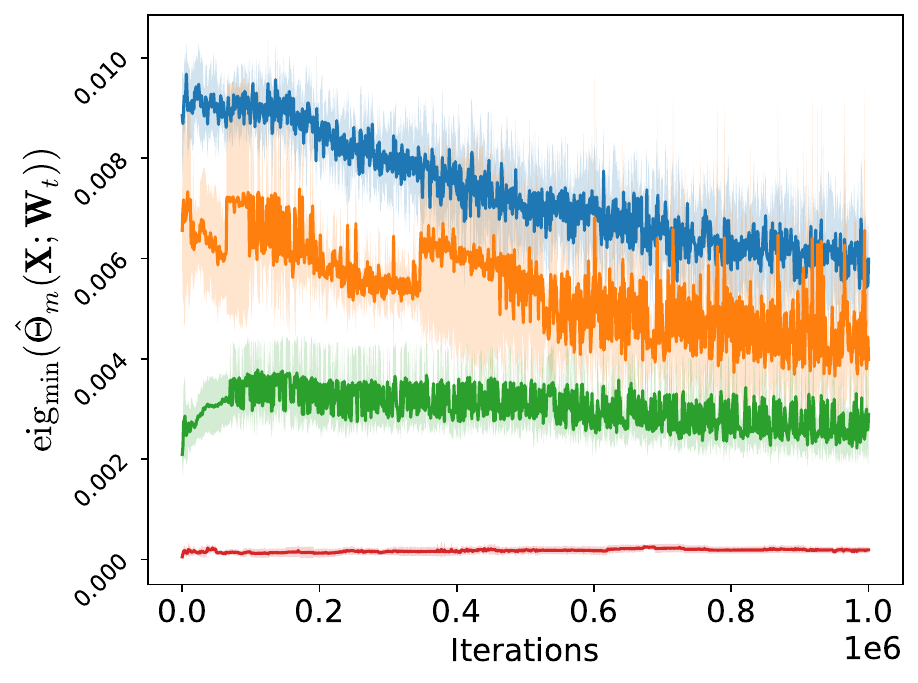}
    \includegraphics[width=0.4\linewidth]{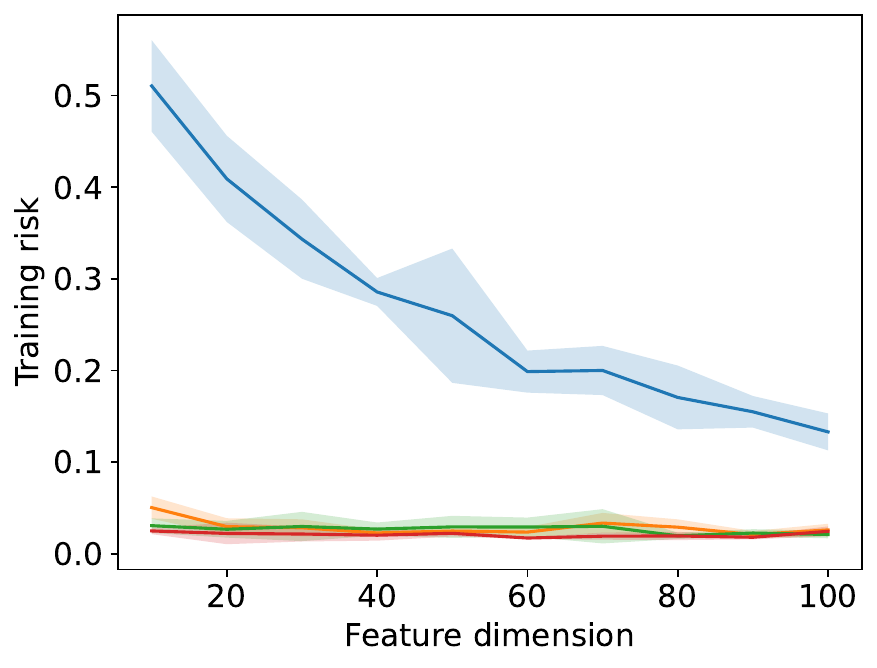}
    \includegraphics[width=0.4\linewidth]{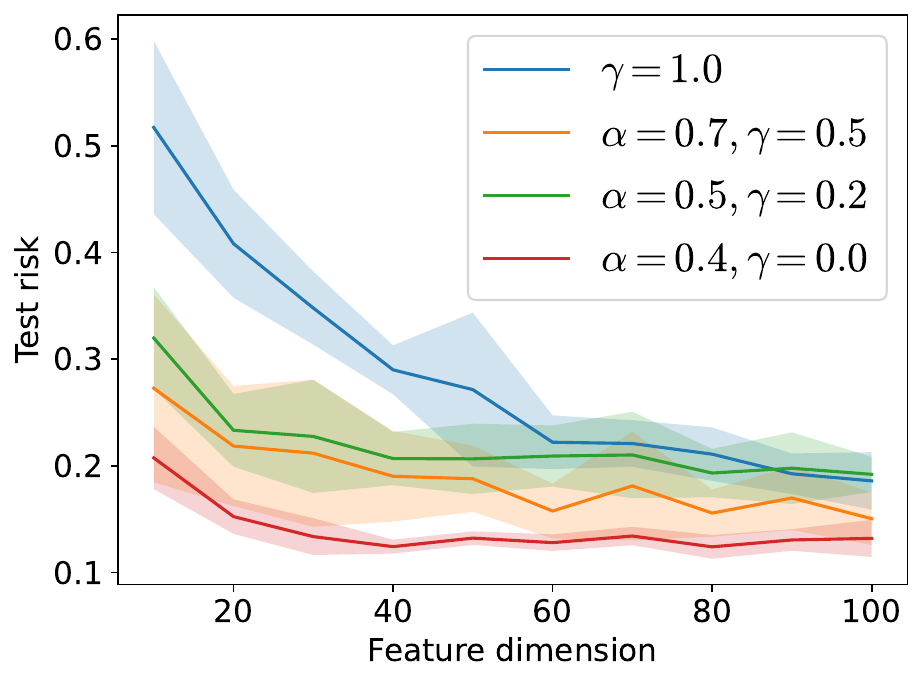}
    \caption[Results for the \texttt{concrete} dataset (ReLU)]{Results for the \texttt{concrete} dataset (ReLU).  From left to right and top to bottom, 1) training risks, 2) test risks, 3) differences in weight norms $\Vert \mathbf{w}_{tj} - \mathbf{w}_{0j}\Vert$ with $j$'s being the neurons having the maximum difference at the end of the training, 4) difference in NTG matrices, 5) minimum NTG eigenvalues, 6) training risks for transfer learning, and 7) test risks for transfer learning.}
    \label{fig:concrete_all_relu}
\end{figure}

\begin{figure}
    \centering
    \includegraphics[width=0.4\linewidth]{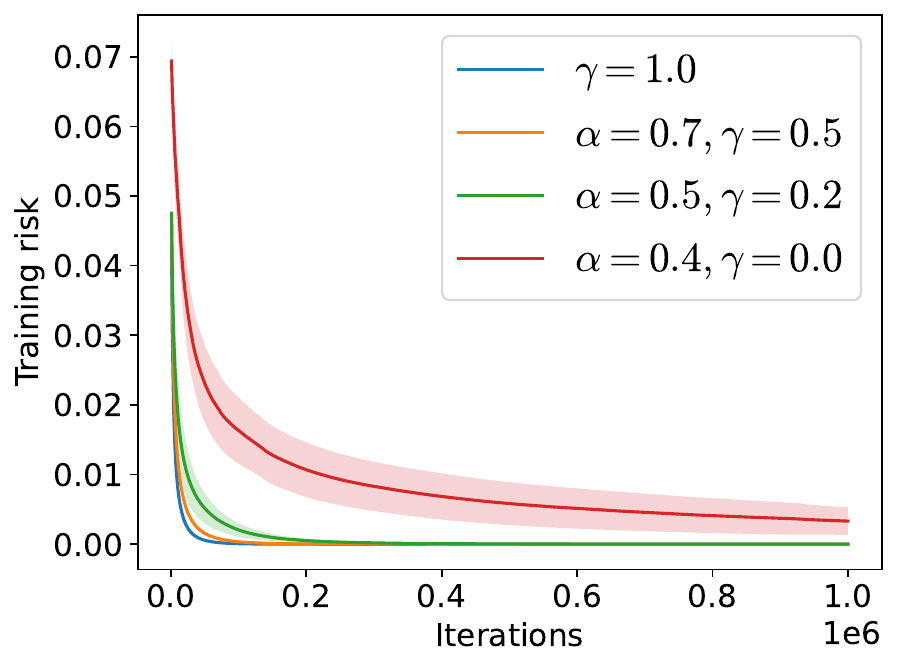}
    \includegraphics[width=0.4\linewidth]{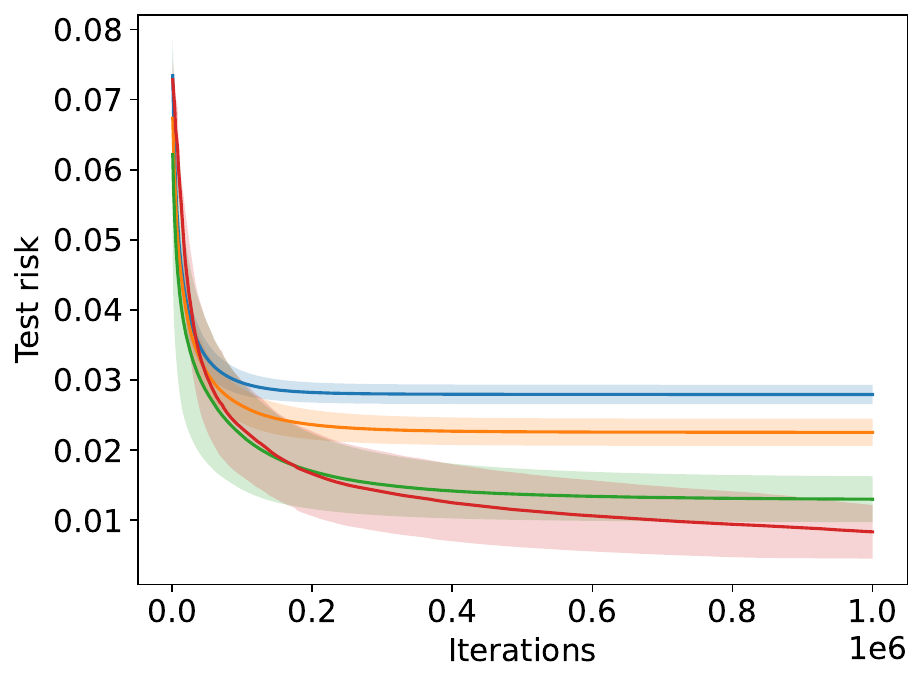}
    \includegraphics[width=0.32\linewidth]{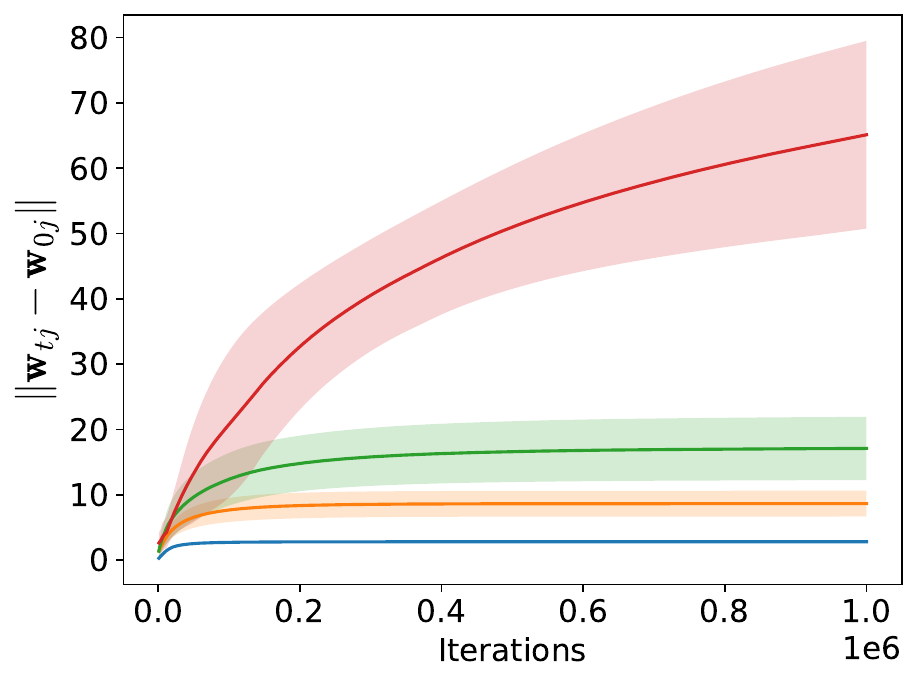}
    \includegraphics[width=0.32\linewidth]{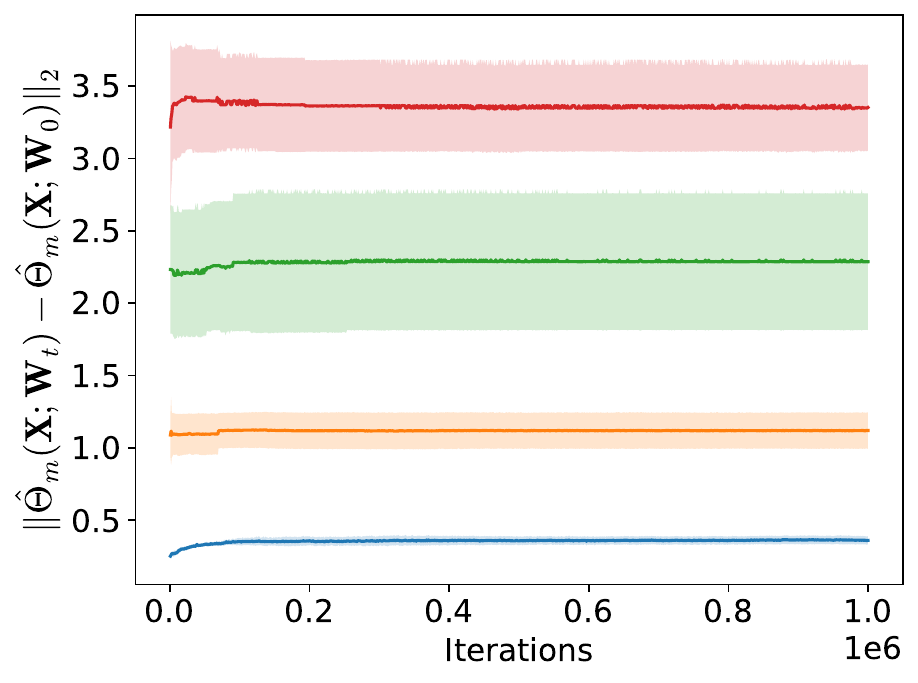}
    \includegraphics[width=0.32\linewidth]{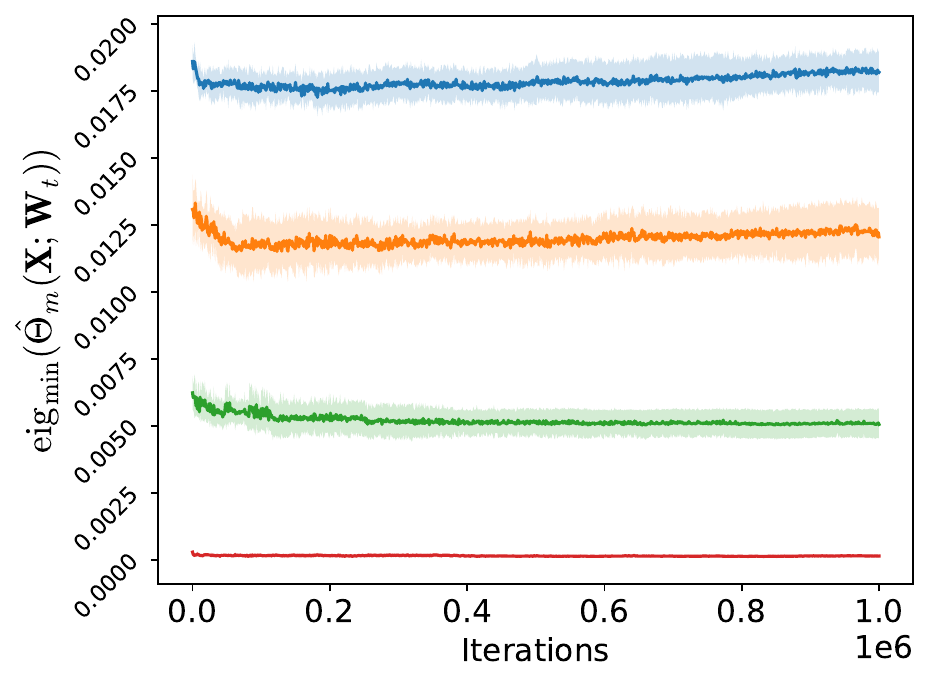}
    \includegraphics[width=0.4\linewidth]{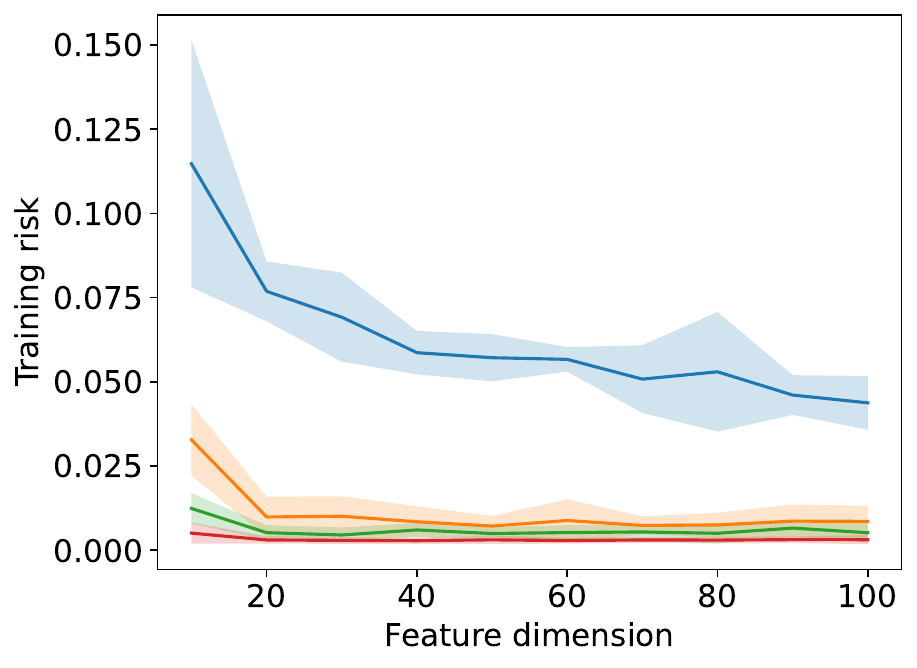}
    \includegraphics[width=0.4\linewidth]{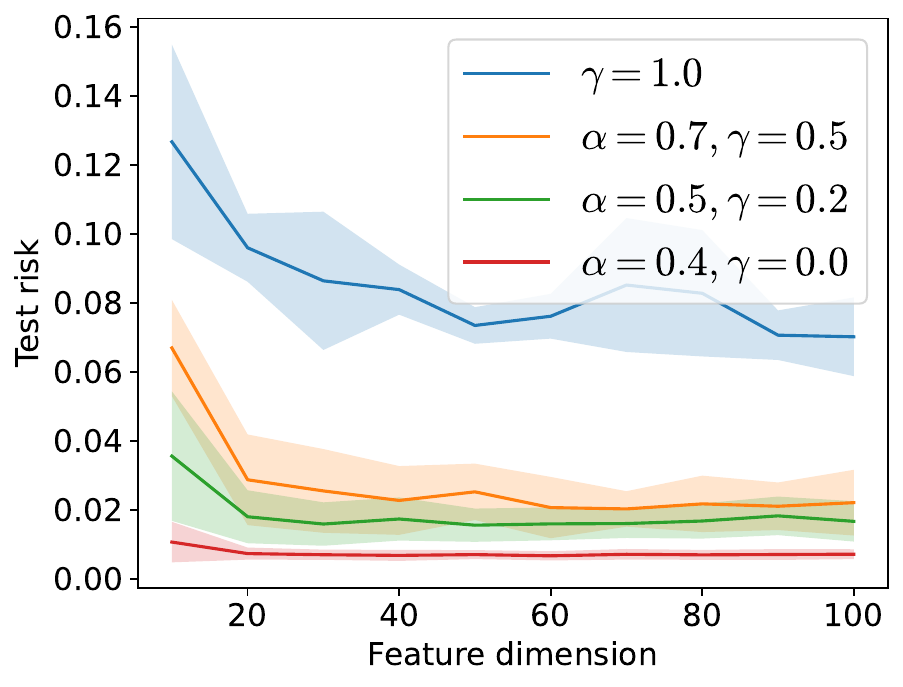}
    \caption[Results for the \texttt{energy} dataset (ReLU)]{Results for the \texttt{energy} dataset (ReLU).  From left to right and top to bottom, 1) training risks, 2) test risks, 3) differences in weight norms $\Vert \mathbf{w}_{tj} - \mathbf{w}_{0j}\Vert$ with $j$'s being the neurons having the maximum difference at the end of the training, 4) difference in NTG matrices, 5) minimum NTG eigenvalues, 6) training risks for transfer learning, and 7) test risks for transfer learning.}
    \label{fig:energy_all_relu}
\end{figure}

\begin{figure}
    \centering
    \includegraphics[width=0.4\linewidth]{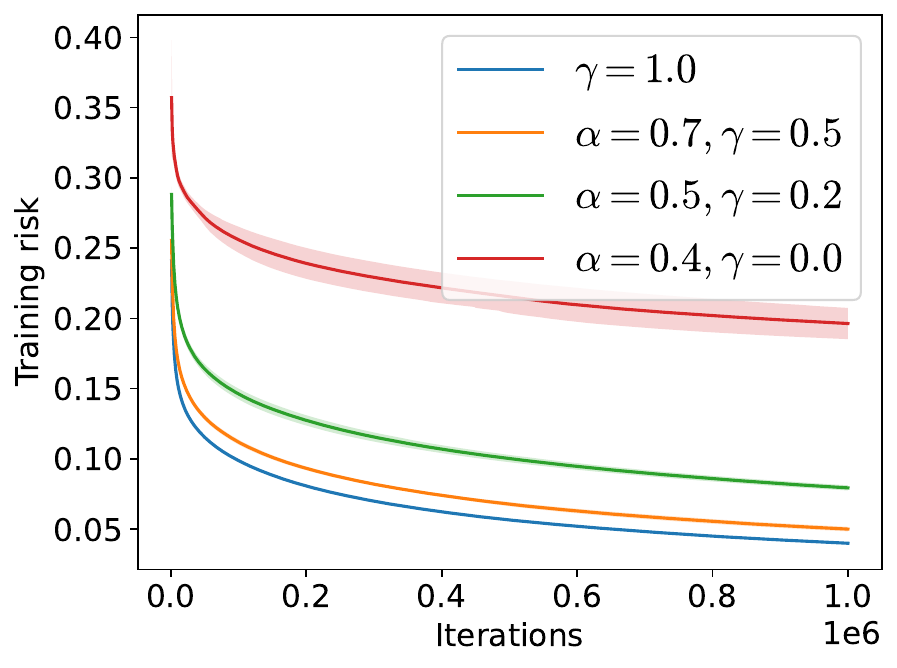}
    \includegraphics[width=0.4\linewidth]{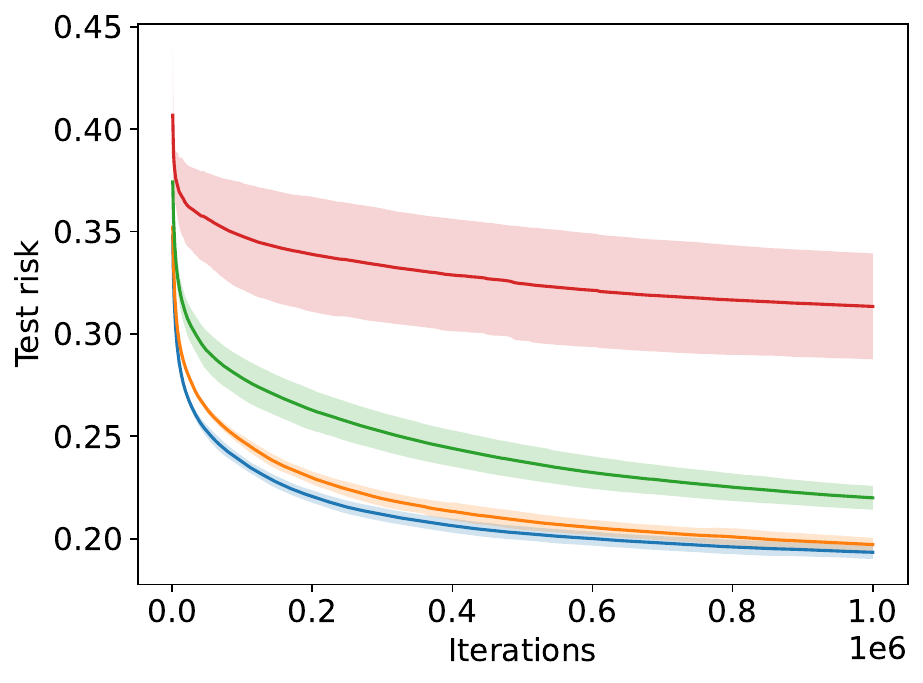}
    \includegraphics[width=0.32\linewidth]{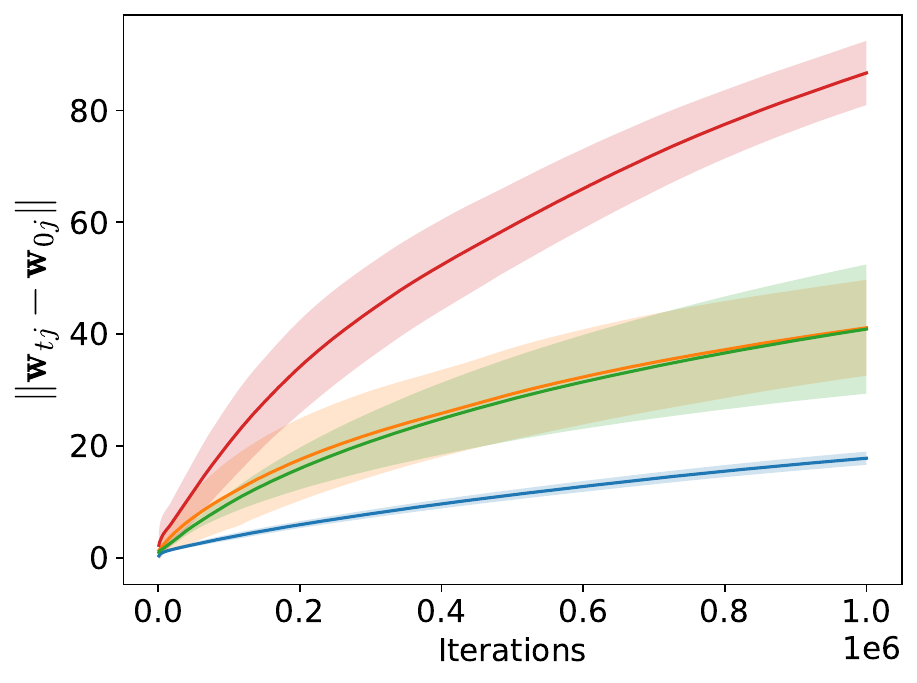}
    \includegraphics[width=0.32\linewidth]{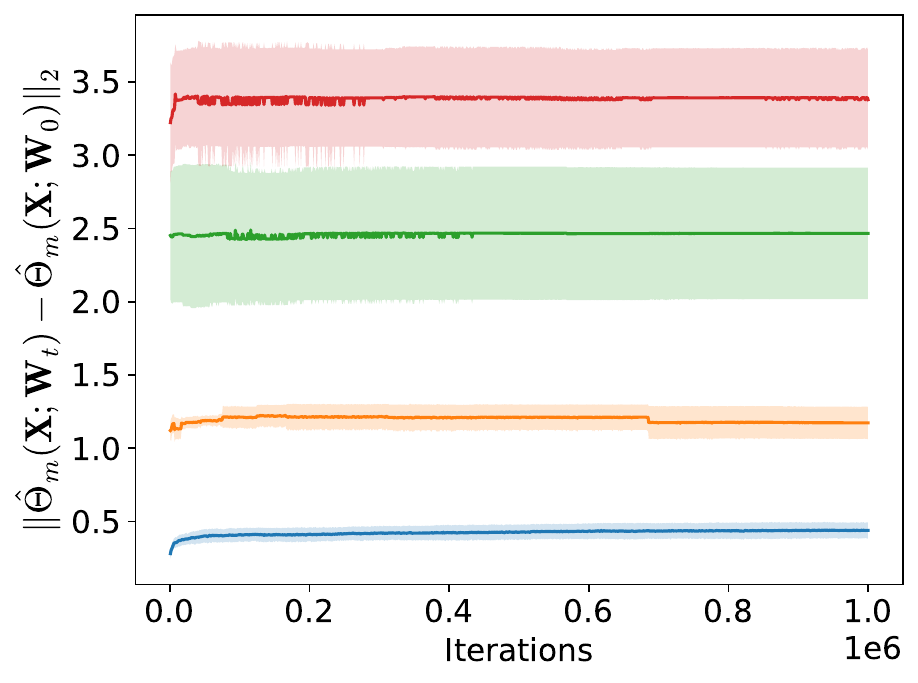}
    \includegraphics[width=0.32\linewidth]{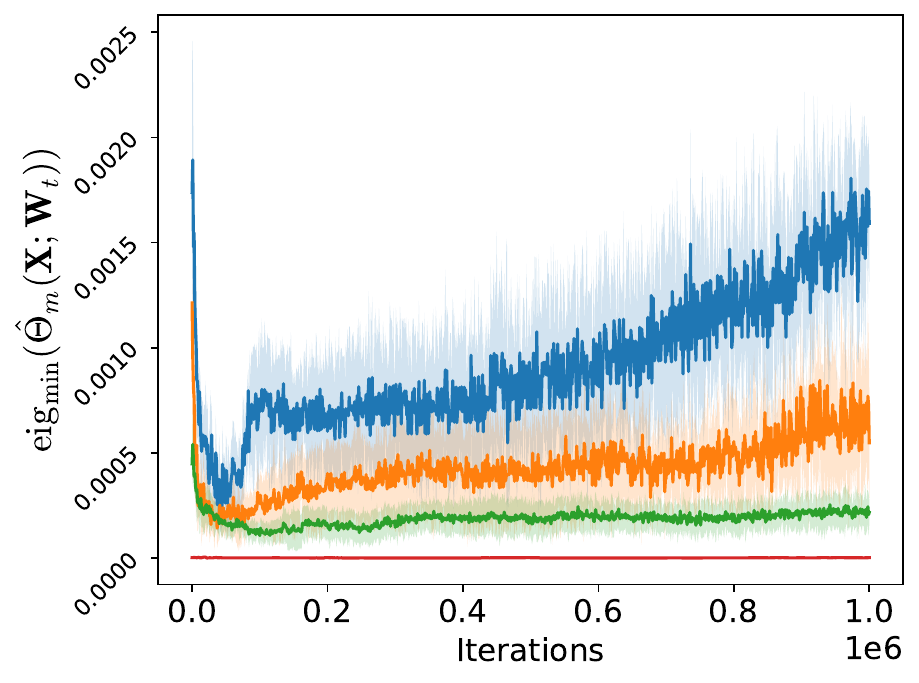}
    \includegraphics[width=0.4\linewidth]{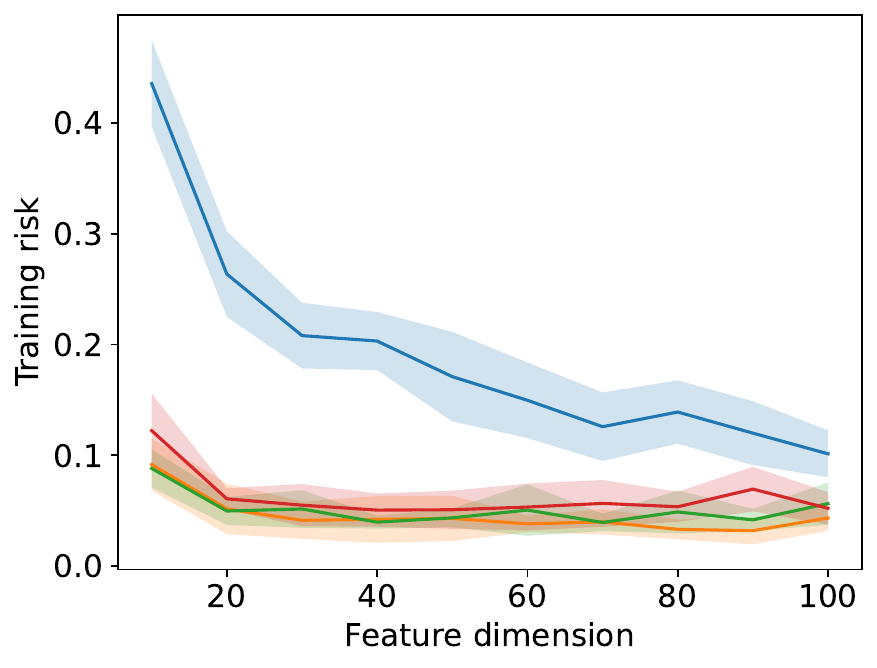}
    \includegraphics[width=0.4\linewidth]{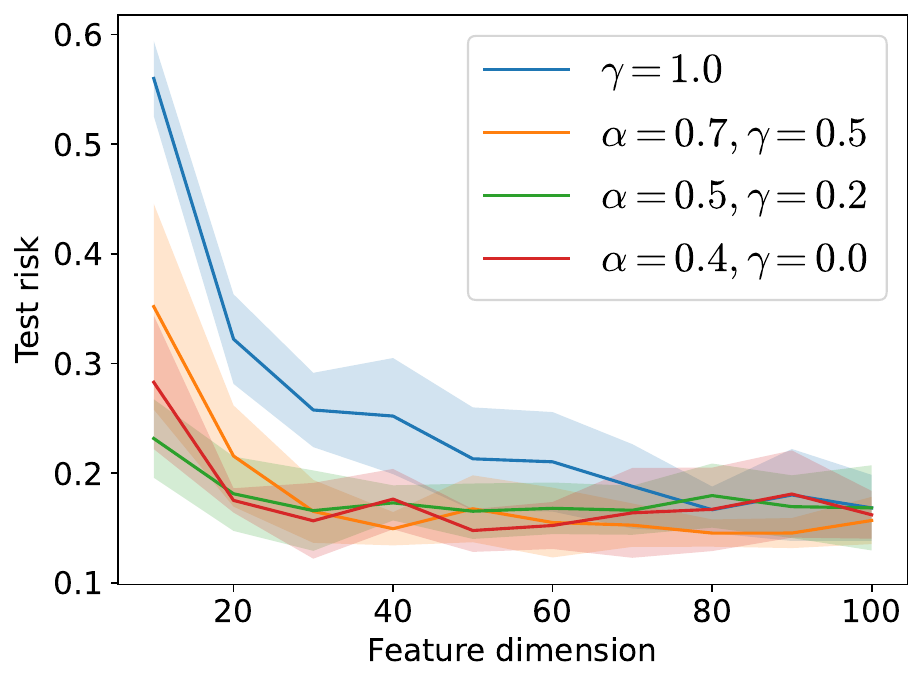}
    \caption[Results for the \texttt{airfoil} dataset (ReLU)]{Results for the \texttt{airfoil} dataset (ReLU).  From left to right and top to bottom, 1) training risks, 2) test risks, 3) differences in weight norms $\Vert \mathbf{w}_{tj} - \mathbf{w}_{0j}\Vert$ with $j$'s being the neurons having the maximum difference at the end of the training, 4) difference in NTG matrices, 5) minimum NTG eigenvalues, 6) training risks for transfer learning, and 7) test risks for transfer learning.}
    \label{fig:airfoil_all_relu}
\end{figure}

\begin{figure}
    \centering
    \includegraphics[width=0.4\linewidth]{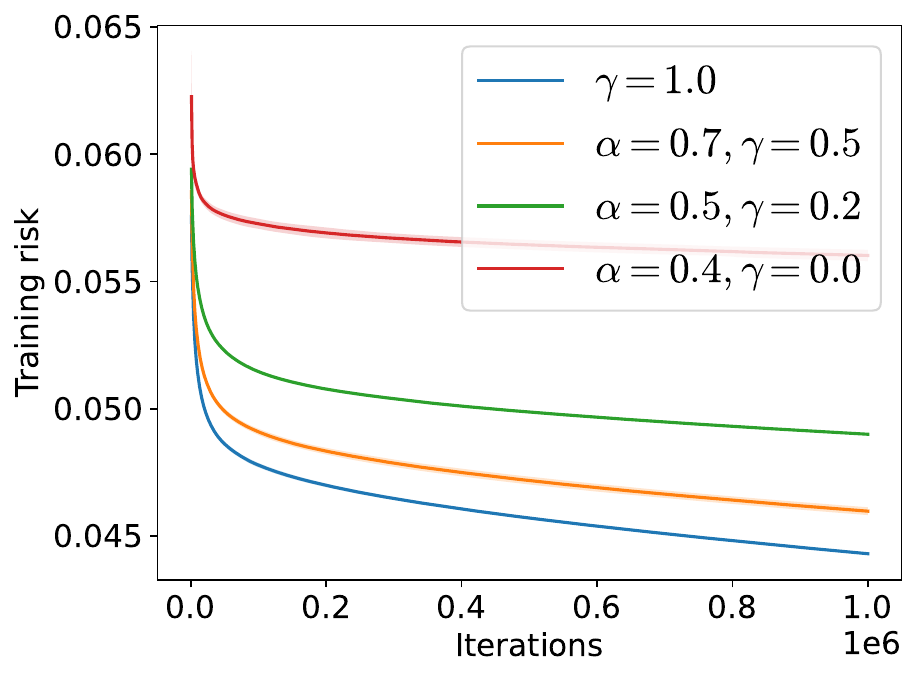}
    \includegraphics[width=0.4\linewidth]{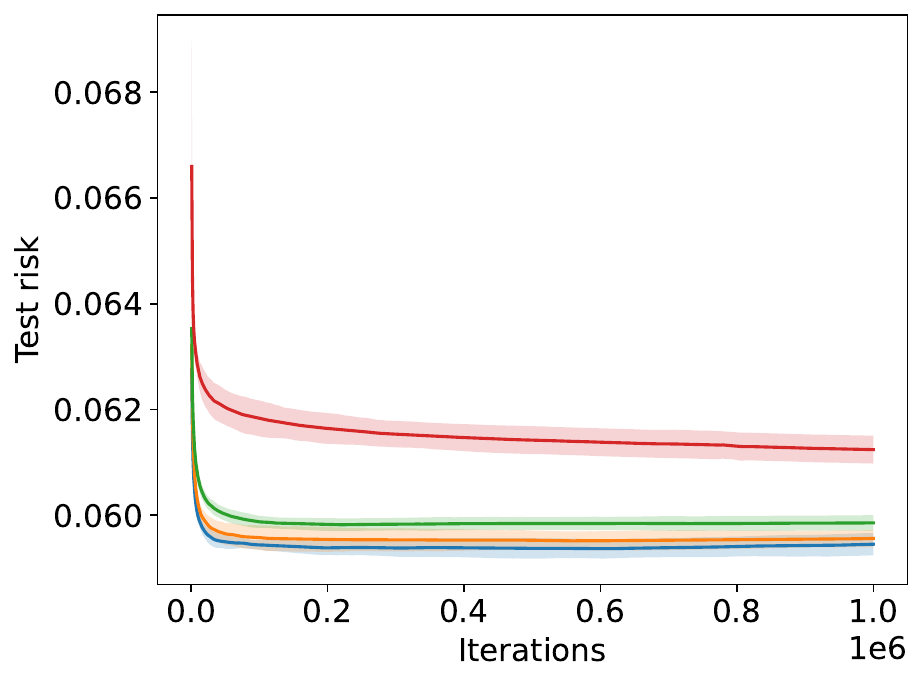}
    \includegraphics[width=0.32\linewidth]{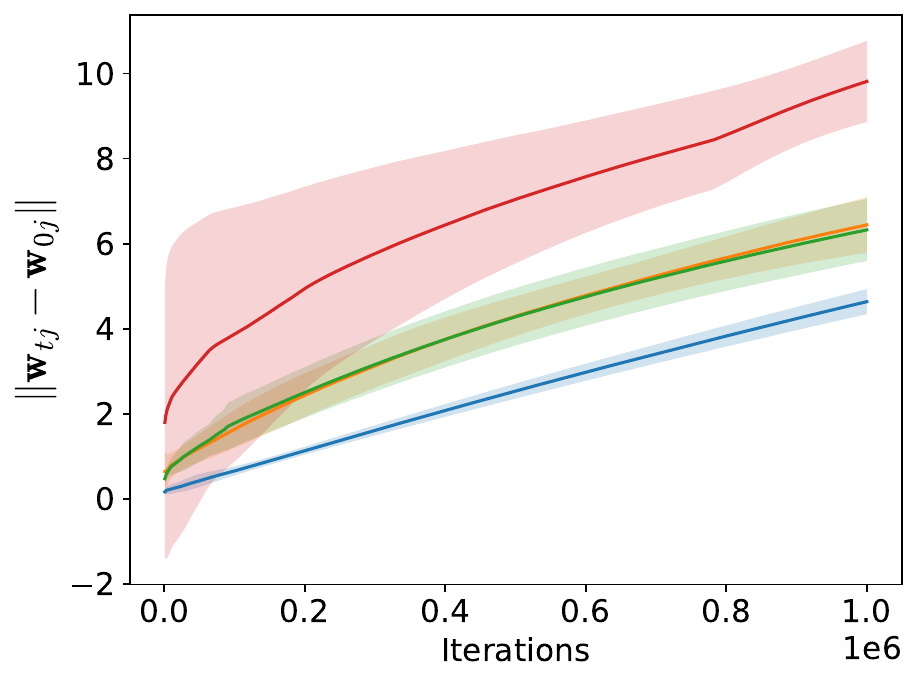}
    \includegraphics[width=0.32\linewidth]{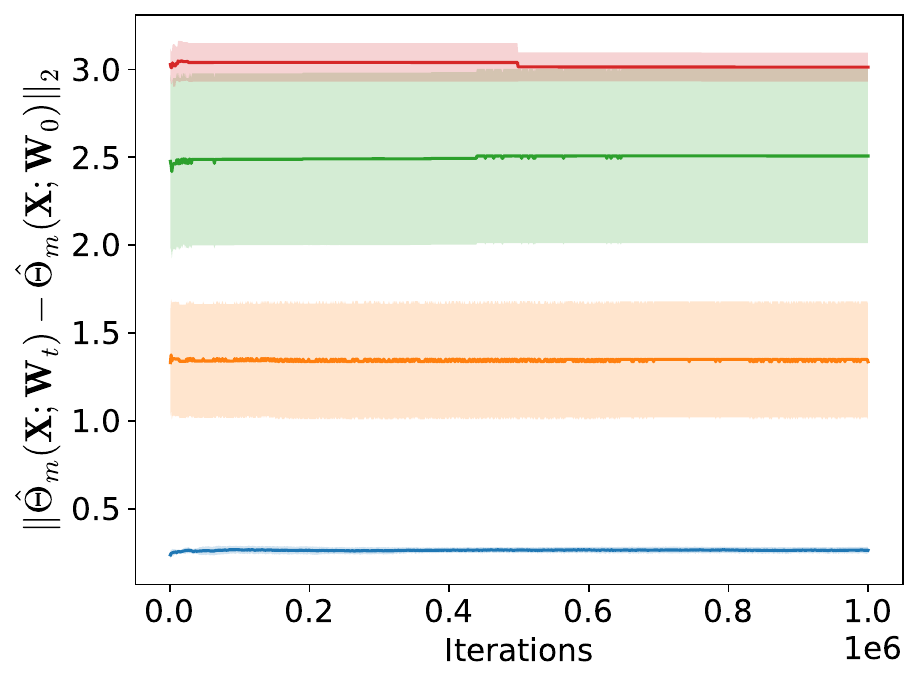}
    \includegraphics[width=0.32\linewidth]{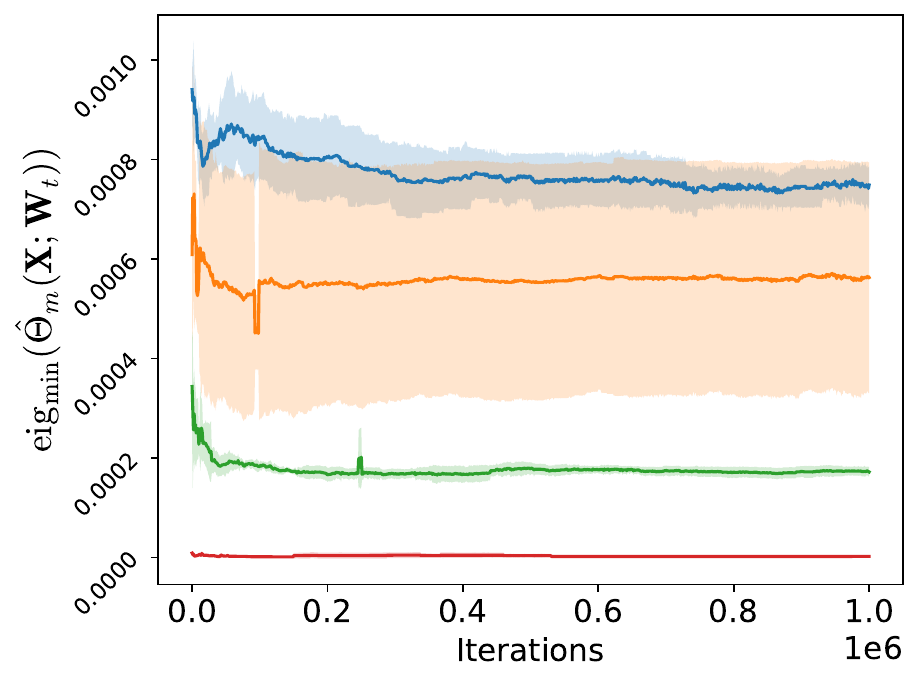}
    \includegraphics[width=0.4\linewidth]{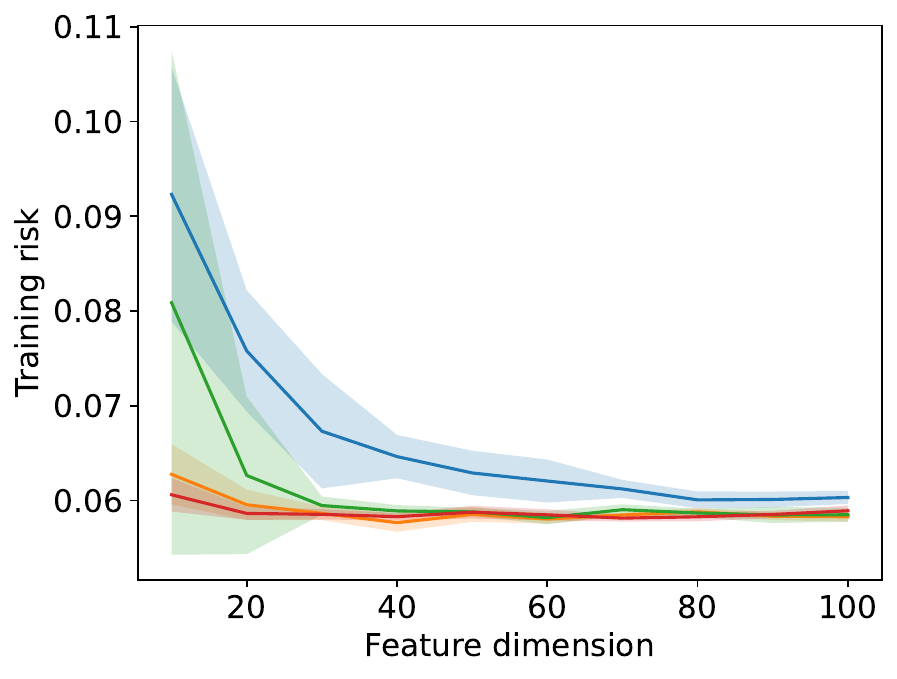}
    \includegraphics[width=0.4\linewidth]{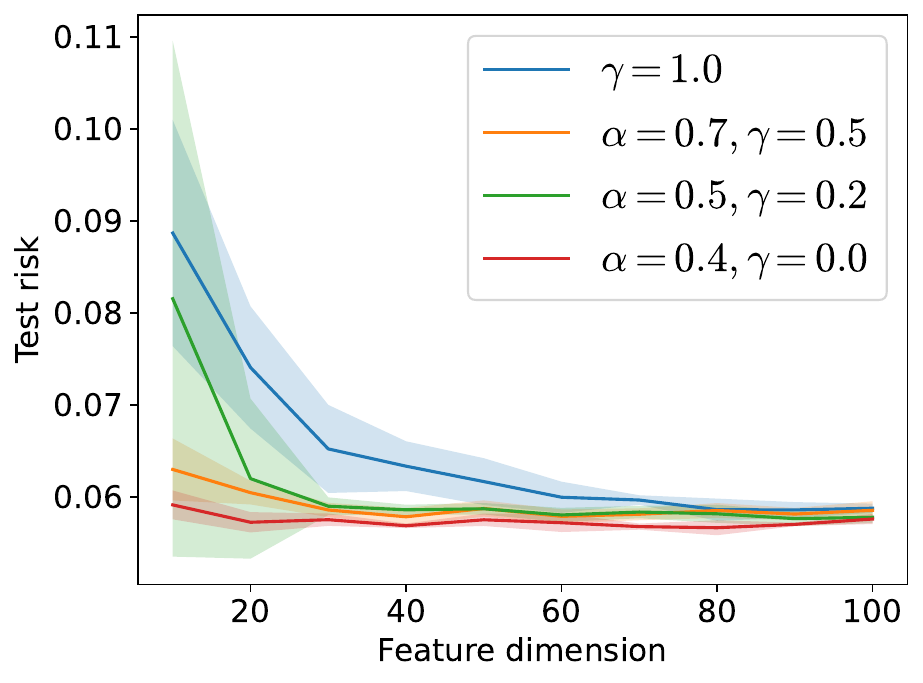}
    \caption[Results for the \texttt{plant} dataset (ReLU)]{Results for the \texttt{plant} dataset (ReLU).  From left to right and top to bottom, 1) training risks, 2) test risks, 3) differences in weight norms $\Vert \mathbf{w}_{tj} - \mathbf{w}_{0j}\Vert$ with $j$'s being the neurons having the maximum difference at the end of the training, 4) difference in NTG matrices, 5) minimum NTG eigenvalues, 6) training risks for transfer learning, and 7) test risks for transfer learning.}
    \label{fig:plant_all_relu}
\end{figure}

\subsection{Classification}

We provide in Figures \ref{fig:mnist_all_relu}, \ref{fig:cifar10_relu} and \ref{fig:cifar100_relu} detailed results for respectively the MNIST, CIFAR10 and CIFAR100 experiments.

\begin{figure}
    \centering
    \includegraphics[width=0.4\linewidth]{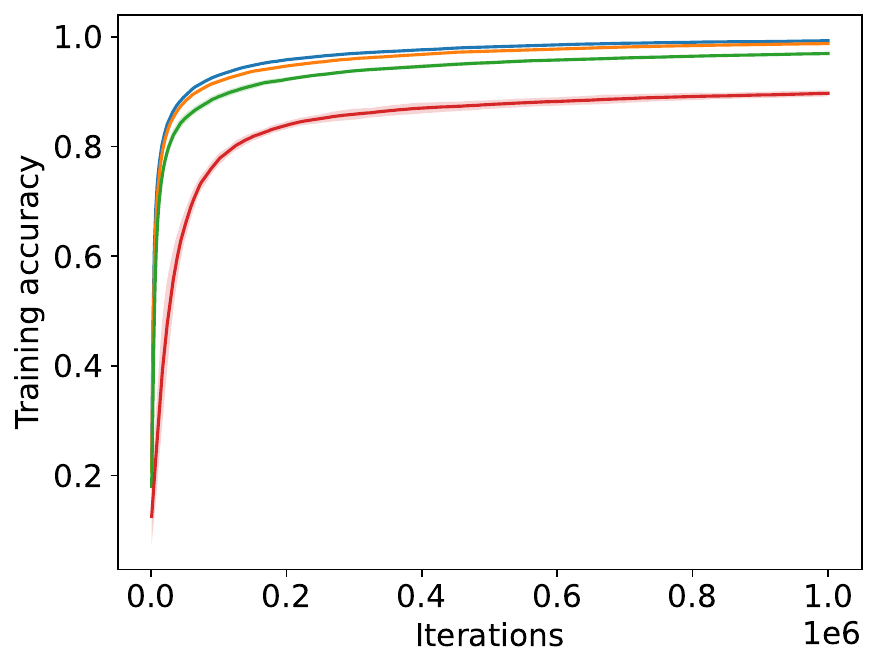}
    \includegraphics[width=0.4\linewidth]{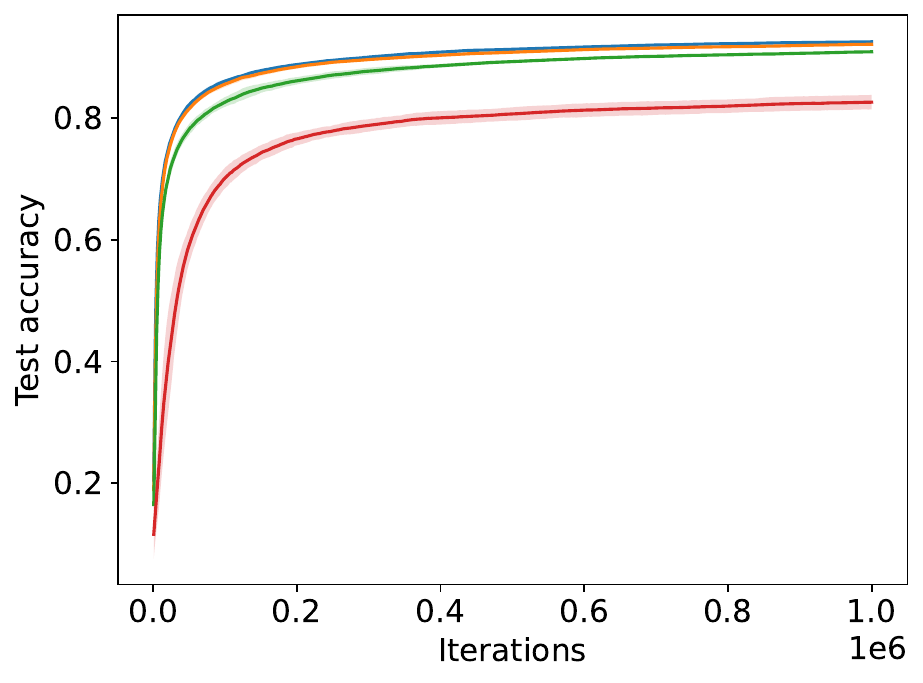}
    \includegraphics[width=0.32\linewidth]{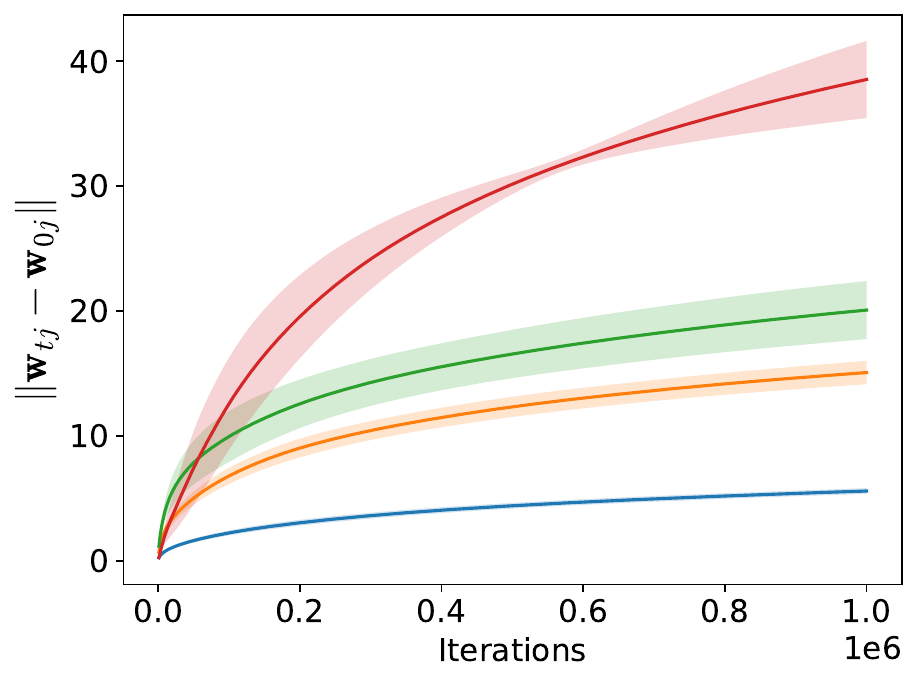}
    \includegraphics[width=0.32\linewidth]{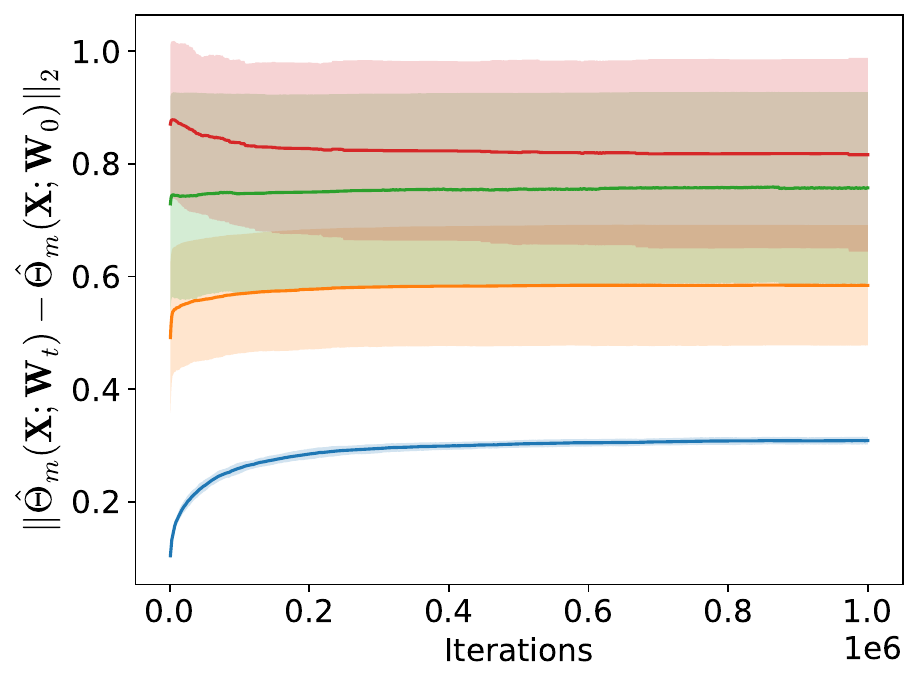}
    \includegraphics[width=0.32\linewidth]{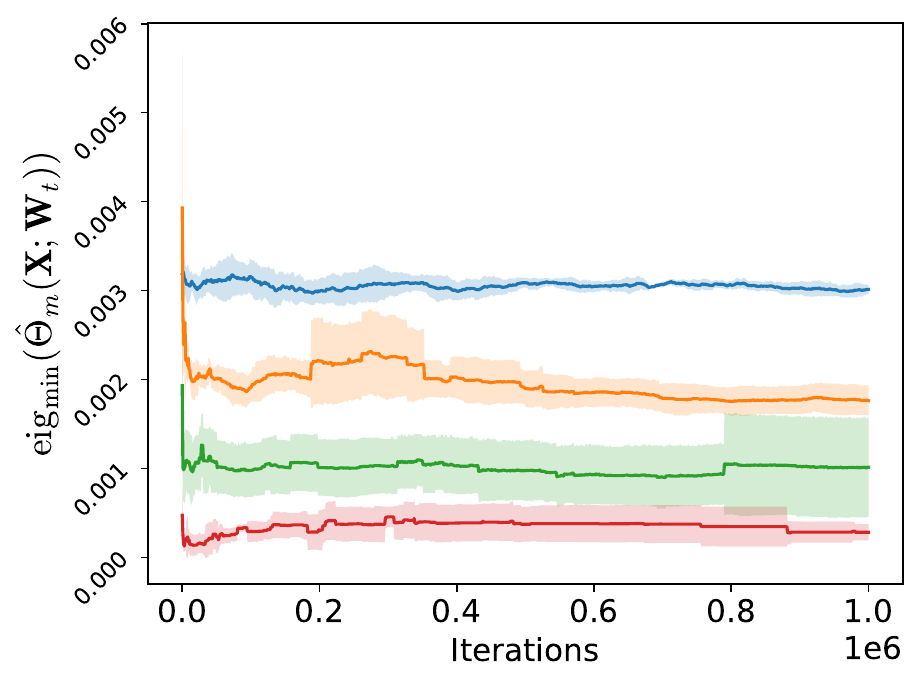}
    \includegraphics[width=0.4\linewidth]{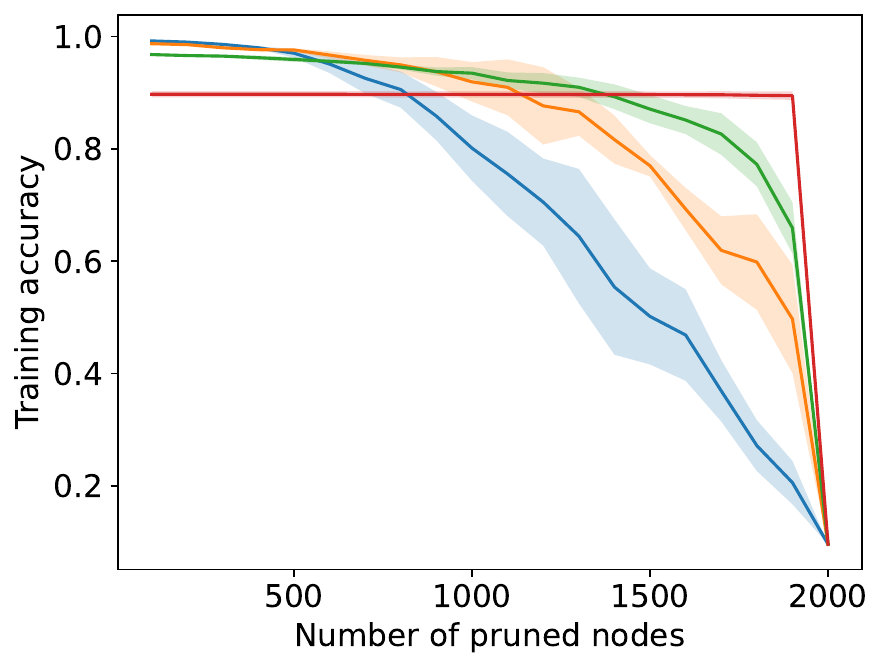}
    \includegraphics[width=0.4\linewidth]{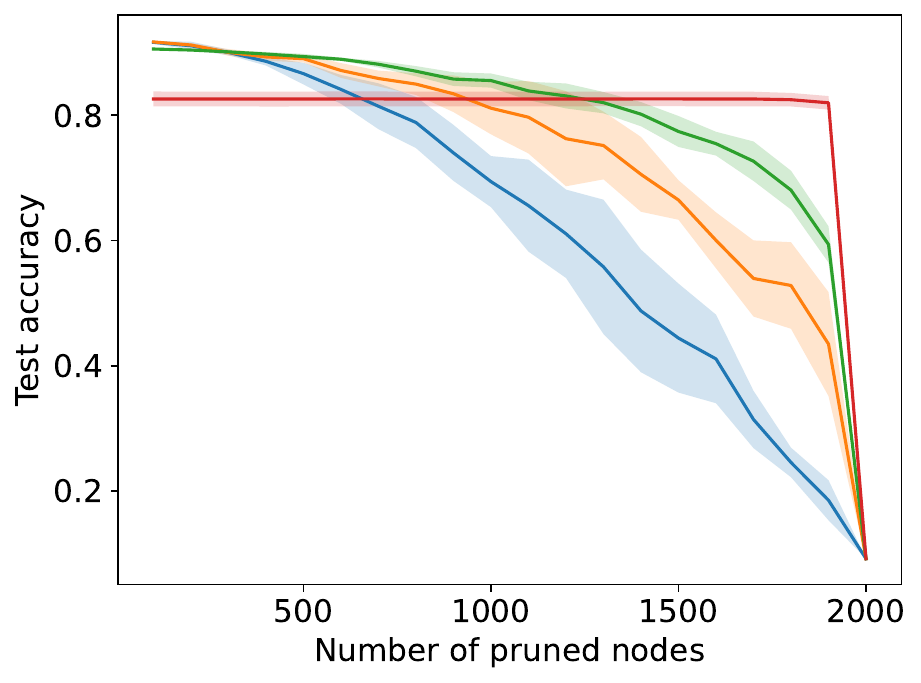}
    \includegraphics[width=0.4\linewidth]{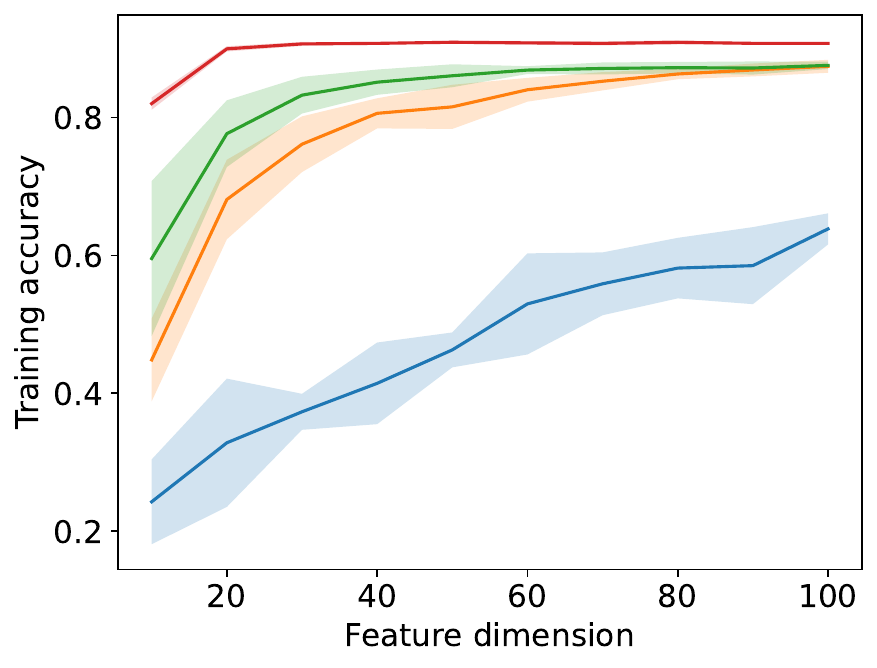}
    \includegraphics[width=0.4\linewidth]{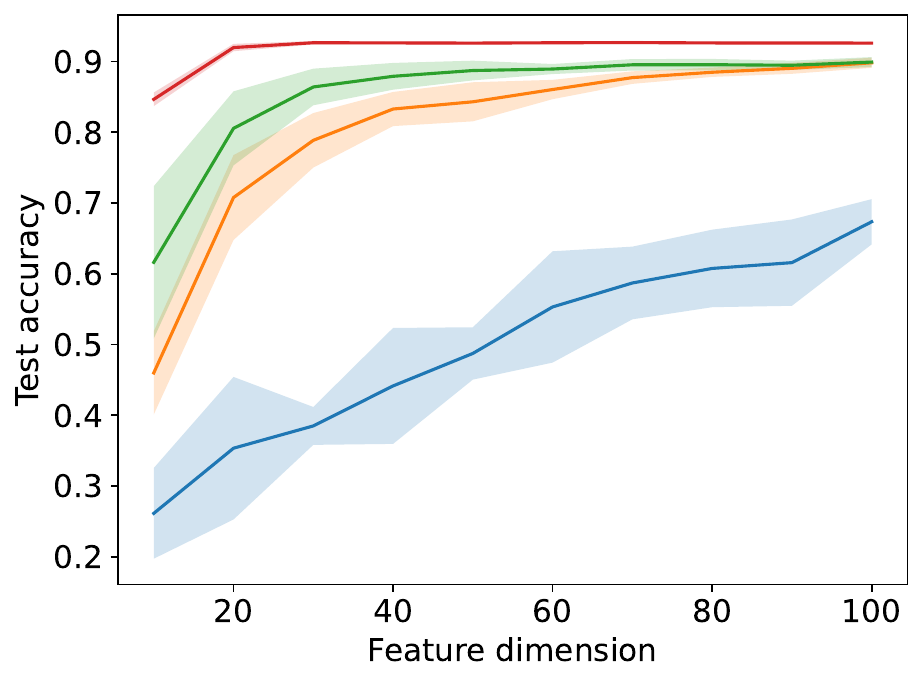}
    \caption[Results for the \texttt{MNIST} dataset (ReLU)]{Results for the \texttt{MNIST} dataset (ReLU). From left to right and top to bottom, 1) training risks, 2) test risks, 3) differences in weight norms $\Vert \mathbf{w}_{tj} - \mathbf{w}_{0j}\Vert$ with $j$'s being the neurons having the maximum difference at the end of the training, 4) difference in NTG matrices, 5) minimum NTG eigenvalues, 6) training accuracies for pruning, 7) test accuracies for pruning, 8) training accuracies for transfer learning, and 9) test accuracies for transfer learning.}
    \label{fig:mnist_all_relu}
\end{figure}

\begin{figure*}
    \centering
    \includegraphics[width=0.4\linewidth]{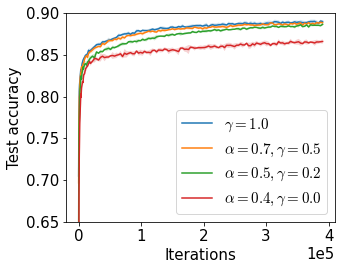}
    \includegraphics[width=0.4\linewidth]{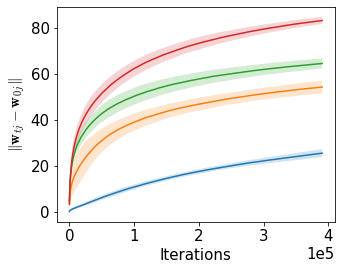}
    \includegraphics[width=0.4\linewidth]{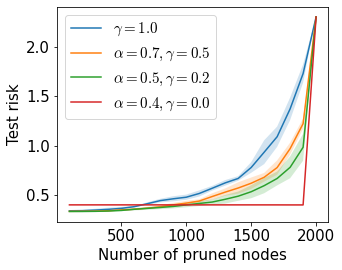}
    \includegraphics[width=0.4\linewidth]{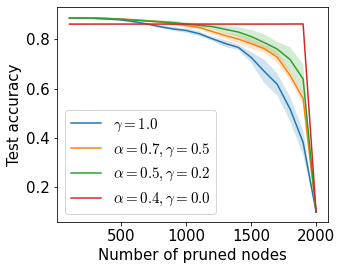}
    \caption[Results for the \texttt{CIFAR--10} dataset (ReLU)]{Results for the \texttt{CIFAR--10} dataset (ReLU). From left to right and top to bottom, 1) test accuracies through training, 2) differences in weight norms $\Vert \mathbf{w}_{tj} - \mathbf{w}_{0j}\Vert$ with $j$'s being the neurons having the maximum difference at the end of the training, 3) test risks of the pruned models, and 4) test accuracies of the pruned models.}
    \label{fig:cifar10_relu}
    \vspace{-1em}
\end{figure*}

\begin{figure*}
    \centering
    \includegraphics[width=0.4\linewidth]{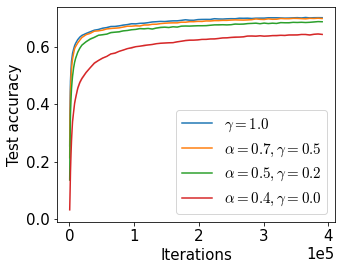}
    \includegraphics[width=0.4\linewidth]{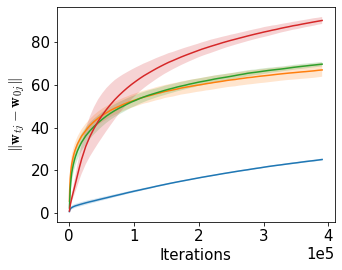}
    \includegraphics[width=0.4\linewidth]{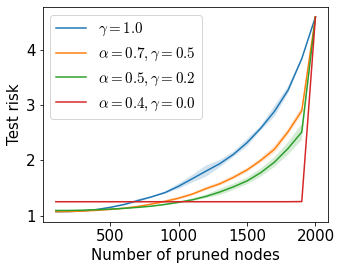}
    \includegraphics[width=0.4\linewidth]{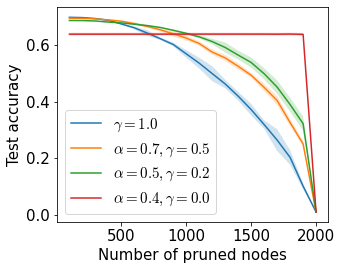}
    \caption[Results for the \texttt{CIFAR--100} dataset (ReLU)]{Results for the \texttt{CIFAR--100} dataset (ReLU). From left to right and top to bottom, 1) test accuracies through training, 2) differences in weight norms $\Vert \mathbf{w}_{tj} - \mathbf{w}_{0j}\Vert$ with $j$'s being the neurons having the maximum difference at the end of the training, 3) test risks of the pruned models, and 4) test accuracies of the pruned models.}
    \label{fig:cifar100_relu}
    \vspace{-1em}
\end{figure*}

\newpage

\section{Visualisation of the learned features.}

This section aims at visualizing the main features learned in the MNIST and CIFAR experiments reported in the main text. Inspired by \cite{Yang2021}, we plot the first two PCA components of the learned features for MNIST (Figure \ref{fig:mnist_features}) and CIFAR10 (Figure \ref{fig:cifar_features}) datasets.  For the MNIST dataset, as in \cite{Yang2021}, the figures show that the features are quasi-random with the symmetric NTK setting, while there is more separation under the asymmetric scaling. For the CIFAR10 experiment, which uses pre-trained features on ImageNet, the features of the symmetric NTK are similar to those of the pre-trained features. The features obtained by PCA better differentiates between the class.

\begin{figure}
    \centering
    \includegraphics[width=0.9\linewidth]{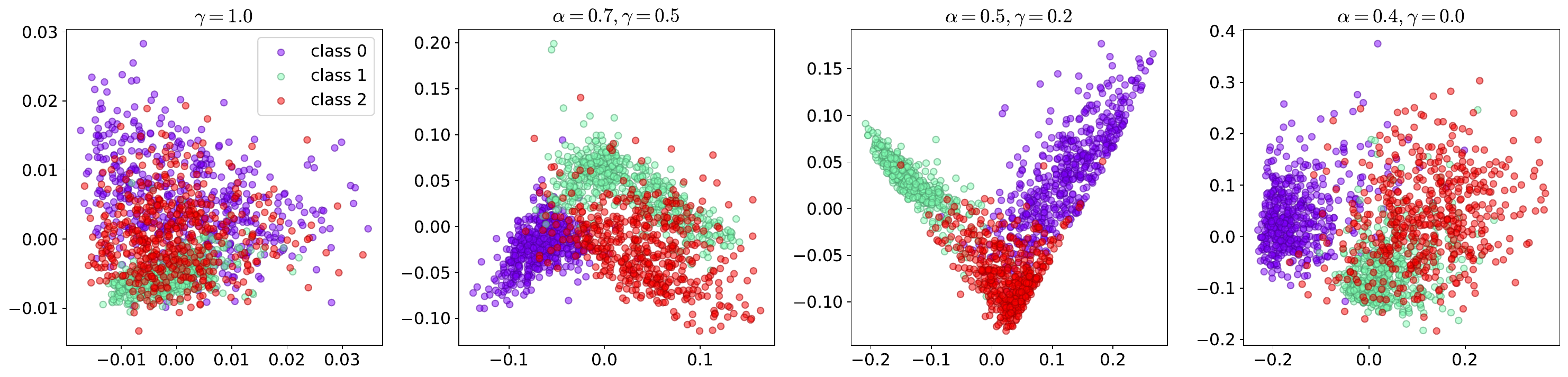}
    \caption{Visualisation of features for MNIST data. We use the top two PCA components to plot the points on a 2D space.}
    \label{fig:mnist_features}
\end{figure}
\begin{figure}
    \centering
    \includegraphics[width=0.9\linewidth]{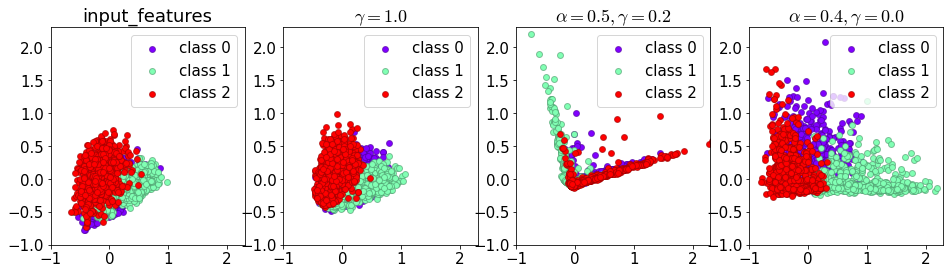}
    \caption{Visualisation of features learnt for the Cifar10 experiment. The models are trained by taking as input the hidden representation of a ResNet18 trained on ImageNet (first figure on the left). We use the top two PCA components to plot the points on a 2D space.}
    \label{fig:cifar_features}
\end{figure}

\section{Hyper-parameter transfer.}
When scaling-up neural networks, hyper-parameters tuning becomes prohibitively expensive. In practice, one performs hyper-parameter optimization on a smaller version of the model, and uses (transfers) the found values for training the larger model. However, this requires stability of the optimal parameters. As identified in \cite{yang2022tensor}, the standard pytorch implementation is not stable as the width increases, which can be a major challenge to scale-up models. In this section, we empirically show that the asymmetrical parameterization enjoys stability of the optimal learning rate. We train FFNN with a single hidden layer on Cifar10 for different width $P=1024, 2048, 4096$. We compare the standard Pytorch parameterization with the asymmetrical one ($\gamma=0.2$, $\alpha=0.5$). The results are reported in \ref{fig:lr_stability}. As expected, in the standard parameterization, the optimal learning rate shifts; as the width increases, the optimal learning rate becomes smaller. On the other hand, with the asymmetrical scaling, the optimal learning rate remains stable as the width increases

\begin{figure}
    \centering
    \includegraphics[width=0.35\linewidth]{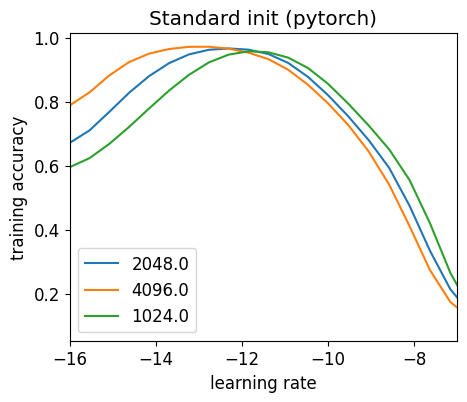}
    \includegraphics[width=0.35\linewidth]{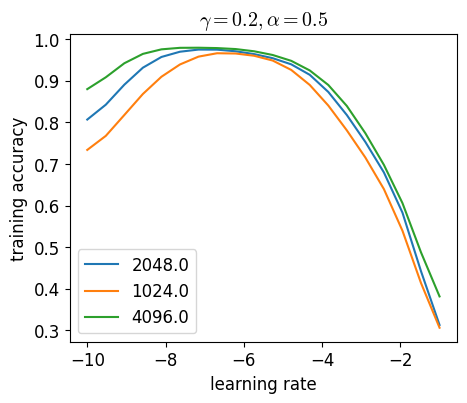}
    \includegraphics[width=0.35\linewidth]{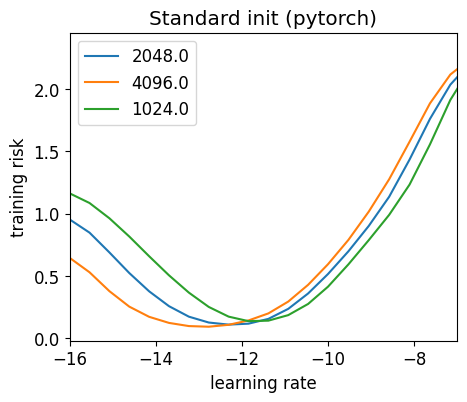}
    \includegraphics[width=0.35\linewidth]{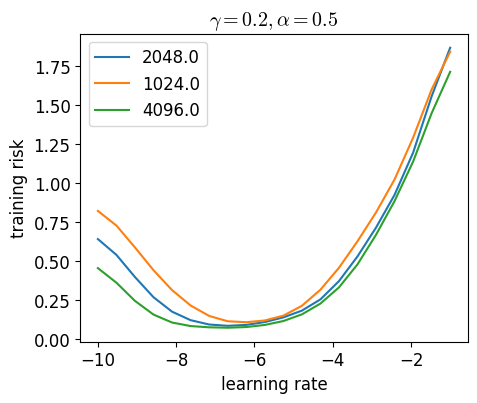}
    \caption{Stability of the optimal learning rate as the width increases. Training error in terms of (top) accuracy (bottom) cross-entropy for (left) standard parameterisation and (right) asymmetrical parameterisation.}
    \label{fig:lr_stability}
\end{figure}

\end{document}